\documentclass{article}

\usepackage{microtype}
\usepackage{graphicx}
\usepackage{subfigure}
\usepackage{booktabs} 

\usepackage{hyperref}

\let\oldaddcontentsline\addcontentsline
\usepackage[accepted]{icml2024}
\let\addcontentsline\oldaddcontentsline

\usepackage{amsmath}
\usepackage{amssymb}
\usepackage{mathtools}
\usepackage{amsthm}

\usepackage[capitalize,noabbrev]{cleveref}

\usepackage[textsize=tiny]{todonotes}

\usepackage[utf8]{inputenc} 
\usepackage[T1]{fontenc}    
\usepackage{hyperref}       
\usepackage{url}            
\usepackage{booktabs}       
\usepackage{amsfonts}       
\usepackage{nicefrac}       
\usepackage{microtype}      
\usepackage{xcolor}         
\usepackage{dsfont}         

\usepackage{graphicx}
\usepackage{xcolor}
\definecolor{Highlight}{HTML}{39b54a}  

\usepackage{bbding}
\usepackage{multirow}
\usepackage{algorithm}
\usepackage{algorithmic}
\usepackage{amsmath}
\usepackage{amsthm}
\usepackage{amssymb}
\usepackage{bm}
\usepackage{threeparttable}

\usepackage{xspace}
\usepackage{enumitem}
\usepackage{bbm}
\usepackage{cleveref}
\usepackage{dsfont}
\usepackage{subfigure}
\usepackage{caption}
\usepackage{graphicx}
\usepackage{titletoc}
\Crefname{problem}{Problem}{Problems}
\usepackage{amssymb}
\usepackage{stfloats}

\usepackage{aisecure-math}
\usepackage{wrapfig}
\usepackage{footmisc} 
\usepackage{etoolbox}

\newcommand{\squishlist}{
\begin{list}{{{\small{$\bullet$}}}}
{\setlength{\itemsep}{3pt}      \setlength{\parsep}{1pt}
\setlength{\topsep}{1pt}       \setlength{\partopsep}{0pt}
\setlength{\leftmargin}{1em} \setlength{\labelwidth}{1em}
\setlength{\labelsep}{0.5em} } }
\newcommand{\squishend}{  \end{list}  }

\newcommand\DoToC{%
  \startcontents
  \printcontents{}{1}{\textbf{Contents}\vskip3pt\hrule\vskip5pt}
  \vskip3pt\hrule\vskip5pt
}

\definecolor{color1}{rgb}{0.251,0.255,0.236}
\definecolor{color2}{rgb}{0.233,0.243,0.255}
\definecolor{color3}{rgb}{0.241,0.255,0.231}
\definecolor{color4}{rgb}{0.255,0.246,0.232}

\newbool{IsTwoColumn}
\setbool{IsTwoColumn}{true}

\newcommand{\name}{\texttt{C-RAG}\xspace}
\newcommand{\eqsmall}{\small} 
\newcommand{\eqnsmall}[1]{\scalebox{0.9}{\ensuremath{#1}}} 

\newcommand{\pspace}{-.0em} 
\newcommand{\stspace}{-.0em} 
\newcommand{\tspaceup}{-0em} 
\newcommand{\minspace}{-.0em}

\newcommand{\mcolor}[2]{\textcolor{#1}{#2}}
\newcommand{\bo}[1]{\mcolor{blue}{Bo: #1}}

\newcommand{\mintongedit}[1]{\mcolor{cyan}{#1}}
\newcommand{\merve}[1]{\mcolor{violet}{#1}}
\newcommand{\merveedit}[1]{\mcolor{black}{#1}}

\Crefname{figure}{Fig.}{Figs.}
\Crefname{equation}{Eq.}{Eqs.}
\Crefname{algorithm}{Alg.}{Algs.}
\Crefname{appendix}{App.}{Apps.}
\Crefname{section}{Sec.}{Secs.}
\Crefname{theorem}{Thm.}{Thms.}
\Crefname{lemma}{Lem.}{Lems.}
\Crefname{proposition}{Prop.}{Props.}
\Crefname{definition}{Def.}{Defs.}
\Crefname{corollary}{Cor.}{Cors.}

\begin{document}

\twocolumn[
\icmltitle{\name: Certified Generation Risks for Retrieval-Augmented Language Models}




\begin{icmlauthorlist}
\icmlauthor{Mintong Kang}{1}
\icmlauthor{Nezihe Merve Gürel}{2}
\icmlauthor{Ning Yu}{3}
\icmlauthor{Dawn Song}{4}
\icmlauthor{Bo Li}{1,5}
\end{icmlauthorlist}

\icmlaffiliation{1}{University of Illinois at Urbana-Champaign, USA}
\icmlaffiliation{2}{Delft University of Technology, Netherlands}
\icmlaffiliation{3}{Netflix Eyeline Studios, USA}
\icmlaffiliation{4}{University of California, Berkeley, USA}
\icmlaffiliation{5}{University of Chicago, USA}

\icmlcorrespondingauthor{Mintong Kang}{mintong2@illinois.edu}
\icmlcorrespondingauthor{Bo Li}{lbo@illinois.edu}

\icmlkeywords{Machine Learning, ICML}

\vskip 0.3in
]



\printAffiliationsAndNotice{}  

\begin{abstract}
    Despite the impressive capabilities of large language models (LLMs) across diverse applications, they still suffer from trustworthiness issues, such as hallucinations and misalignments. Retrieval-augmented language models (RAG) have been proposed to enhance the credibility of generations by grounding external knowledge, but the theoretical understandings of their generation risks remains unexplored. In this paper, we answer:
    \textit{1) whether RAG can indeed lead to low generation risks, 2) how to provide provable guarantees on the generation risks of RAG and vanilla LLMs, and 3) what sufficient conditions enable RAG models to reduce generation risks.}
    We propose \name, a novel framework to certify generation risks for RAG models.
    Specifically, we provide conformal risk analysis for RAG models and certify an upper confidence bound of generation risks, which we refer to as \textit{conformal generation risk}.
    We also provide theoretical guarantees on conformal generation risks for general bounded risk functions under test distribution shifts.
  We prove that RAG achieves a lower conformal generation risk than that of a single LLM when the quality of the retrieval model and transformer is non-trivial.
    Our intensive empirical results demonstrate the soundness and tightness of our conformal generation risk guarantees across four widely-used NLP datasets on four state-of-the-art retrieval models.
\end{abstract}

\vspace{\tspaceup}
\section{Introduction}
\vspace{\pspace}
Large language models (LLMs) \cite{touvron2023llama,openai2023gpt4} recently exhibit emergent abilities across different NLP tasks, such as text summarization, question answering, and machine translation. 
However, existing works \cite{decowang,liang2022holistic,liu2023trustworthy} show that the generations of LLMs can be unreliable, untrustworthy, and risky in many cases.
Therefore, \textit{certifiably controlling the generation risks of LLMs} becomes particularly important before the deployment of LLMs, especially in safety-critical domains.
\vspace{\pspace}

Retrieval-augmented language models (RAG) \cite{lewis2020retrieval,karpukhin2020dense,xiong2020approximate} have been proposed to enhance the credibility of LLMs by retrieving relevant documents from an external knowledge base and generating contents conditioned on the retrieved knowledge.
RAG models are shown effective in mitigating generation risks via in-context learning from the retrieved documents \cite{brown2020language}. 
However, theoretical understandings of their generation risks still remain unexplored. 
In this work, we focus on this problem and ask:
\vspace{\pspace}

\textit{Can RAG indeed lead to low generation risks?  How can we provide provable guarantees on the generation risks of RAG and vanilla LLMs? What are the sufficient conditions that enable RAG models to reduce generation risks? Can we provably control the generation risks below a desired level?}

\vspace{\pspace}
To theoretically analyze the generation risks of RAG and answer the above questions, we propose \name, a novel framework of certified generation risks for RAG models. 
We first propose a constrained generation protocol for RAG models to produce a controlled set of generations. The protocol operates based on specific parameter configurations, including the number of retrieved examples, the size of the generation set, and a similarity threshold for generation diversity.
We then provide conformal analysis \cite{bates2021distribution,angelopoulos2021learn,angelopoulos2022conformal} for RAG models under the constrained generation protocol, aiming to provably control the generation risks based on test statistics from in-distribution calibration samples.
To achieve this goal, we derive a high-probability upper bound of generation risks during inference time, which we call \textit{conformal generation risk}.  
We show that (a) the conformal generation risk serves as a sound upper bound to the empirical generation risks given a RAG configuration in \Cref{risk_guarantee_1};
(b) the generation risk can be certifiably controlled below a desired level by computing a valid set of RAG configurations via \name in \Cref{risk_guarantee_2}; 
(c) the conformal analysis can be extended to more complex scenarios under test-time distribution shifts in \Cref{pro:conf_shf}, which presents the \textit{first} generation risk guarantee under test-time distribution shifts for general bounded risk functions. 
Based on our conformal analysis for the generation risks of RAG and vanilla LLMs, we prove that 
(a) the conformal generation risk of RAG is lower than that of the corresponding vanilla LLM in \Cref{thm:gene_rag}; 
(b) under bounded test-time distribution shifts, RAG also lowers the conformal generation risks compared to the vanilla LLM in \Cref{thm:comp_shft}.
\vspace{\pspace}

We evaluate the conformal generation risk guarantees of \name with different retrieval models on four widely-used datasets.
For all retrieval methods and datasets, we empirically validate that our conformal generation risk guarantees are sound and tight even with distribution shifts, as they upper bound the empirical generation risks observed on random test sets while maintaining only a minimal gap, narrowing down to the scale of $\eqnsmall{1e-3}$.
We empirically show that RAG consistently achieves a lower conformal generation risk than a single LLM without retrieval, which is consistent with our theoretical findings in \Cref{sec:rag_conf,sec:distribution_shift}.
We also evaluate the conformal generation risk for different SOTA retrieval models, such as sparse encoding metrics BM25 \cite{robertson2009probabilistic},  text-embedding-ada-002 model from OpenAI, bge model from BAAI \cite{zhang2023retrieve}, and supervised fine-tuned embedding model  \cite{wang2023learning} to validate our analysis on retrieval quality.
We show that among these models, text-embedding-ada-002 and supervised fine-tuned embedding models outperform other baselines in achieving low conformal generation risks.

\vspace{\stspace}
\section{Related work}
\vspace{\pspace}
\textbf{Retrieval augmented generation} (RAG) is a framework for improving the generation quality of LLMs via retrieving relevant information from an external knowledge base and grounding the model on the retrieved knowledge for conditional generations. 
SOTA retrieval methods \cite{lewis2020retrieval,karpukhin2020dense,xiong2020approximate} employ dual encoders to project both query and candidate texts into the embedding space and retrieve candidate texts that exhibit high similarity to the embedded query text.
Although RAG is demonstrated to be effective in enhancing the generation credibility, their theoretical understanding is limited.
\citeauthor{basu2023statistical} conduct retrieval analysis for a constrained function and data class from a statistical perspective, but the results cannot be applicable to self-attention transformers and to arbitrary data distribution. 
In this work, we provide the first theoretical analysis of how RAG leads to low generation risks in self-attention transformers for arbitrary data distribution.
\vspace{\pspace}

\textbf{Conformal prediction} is a statistical technique used to create prediction sets with assured coverage.  \citep{vovk1999machine,vovk2005algorithmic,lei2013distribution,yang2021finite}.
Broadly, conformal risk control methods  \cite{bates2021distribution,angelopoulos2021learn,angelopoulos2022conformal,quach2023conformal} provide a high-confidence risk guarantee for black-box risk functions, assuming data exchangeability.
However, the risk guarantee can be broken under test-time distribution shifts.
While \citeauthor{angelopoulos2022conformal} and \citeauthor{farinhas2023nonexchangeable} offer a valid conformal risk guarantee for monotonic risk functions under distribution shifts, the monotonicity assumption is not always practical. 
In this work, we introduce a generally applicable conformal risk guarantee for general bounded risk functions, under distribution shifts at test time.

\vspace{\stspace}
\section{Preliminaries}
\vspace{\pspace}
\label{sec:prelim}

Before introducing \name, we first review the preliminaries of conformal controlling methods \cite{bates2021distribution,angelopoulos2021learn,angelopoulos2022conformal}, which calibrate machine learning models to ensure their predictions meet explicit finite-sample statistical guarantees.
Let $\eqnsmall{R(\lambda)}$ and $\eqnsmall{\hat{R}(\lambda)}$ denote the population risk and empirical risk, respectively, where $\eqnsmall{\lambda \in \Lambda}$ is a parameter that induces the risk. We consider a finite parameter space $\eqnsmall{\Lambda}$ with $\eqnsmall{N}$ parameters, $\eqnsmall{\lambda_1, \ldots, \lambda_N}$.
For a desired risk level $\eqnsmall{\alpha~(0<\alpha<1)}$, we define $\eqnsmall{N}$ null hypotheses: $\eqnsmall{\mathcal{H}_j: R(\lambda_j) > \alpha~(j \in {1, \ldots, N})}$. Let $\eqnsmall{p_j}$ be the p-value of the null hypothesis $\mathcal{H}_j$. Given $\eqnsmall{n}$ calibration samples, \citeauthor{bates2021distribution,angelopoulos2021learn} provide valid p-values as follows:
\vspace{-.5em}
\begin{equation}
\small
    p_j \hspace{-0.15em} = \hspace{-0.15em} \min \hspace{-0.15em} \left\{ \hspace{-0.15em} \exp \hspace{-0.15em} \left\{ \hspace{-0.15em} -n h(\hat{R}(\lambda_j),\alpha) \hspace{-0.15em} \right\} \hspace{-0.15em}, \hspace{-0.15em} e\sP\left[ \text{Bin}(n,\alpha)\le \lceil n \hat{R}(\lambda_j) \rceil \right]  \hspace{-0.15em}\right\} \hspace{-0.15em},
\end{equation}
where $\eqnsmall{h(a,b)=a\log(\nicefrac{a}{b})+ (1-a)\log(\nicefrac{(1-a)}{(1-b)})}$.
Let $\eqnsmall{J}$ be the index set of false null hypothesis. 
An algorithm $\eqnsmall{\mathcal{A}:[0,1]^N \mapsto 2^{{1, \ldots, N}}}$ is a family-wise error rate (FWER)-controlling algorithm at level $\delta$ if $\eqnsmall{\sP\left[ \mathcal{A}(p_1, \ldots, p_N) \subseteq J \right] \ge 1 - \delta}$. They finally show that:
\vspace{-.5em}
\begin{equation}
    \small
    \sP\left[ \sup_{\lambda \in \hat{\Lambda}} \left\{ R(\lambda) \right\} \le \alpha \right] \ge 1- \delta,
\end{equation}
where {\eqsmall $\hat{\Lambda}$} denotes the output of the FWER algorithm $\mathcal{A}$.

\ifbool{IsTwoColumn}{
\begin{figure*}[th]
    \centering
    \includegraphics[width=0.7\linewidth]{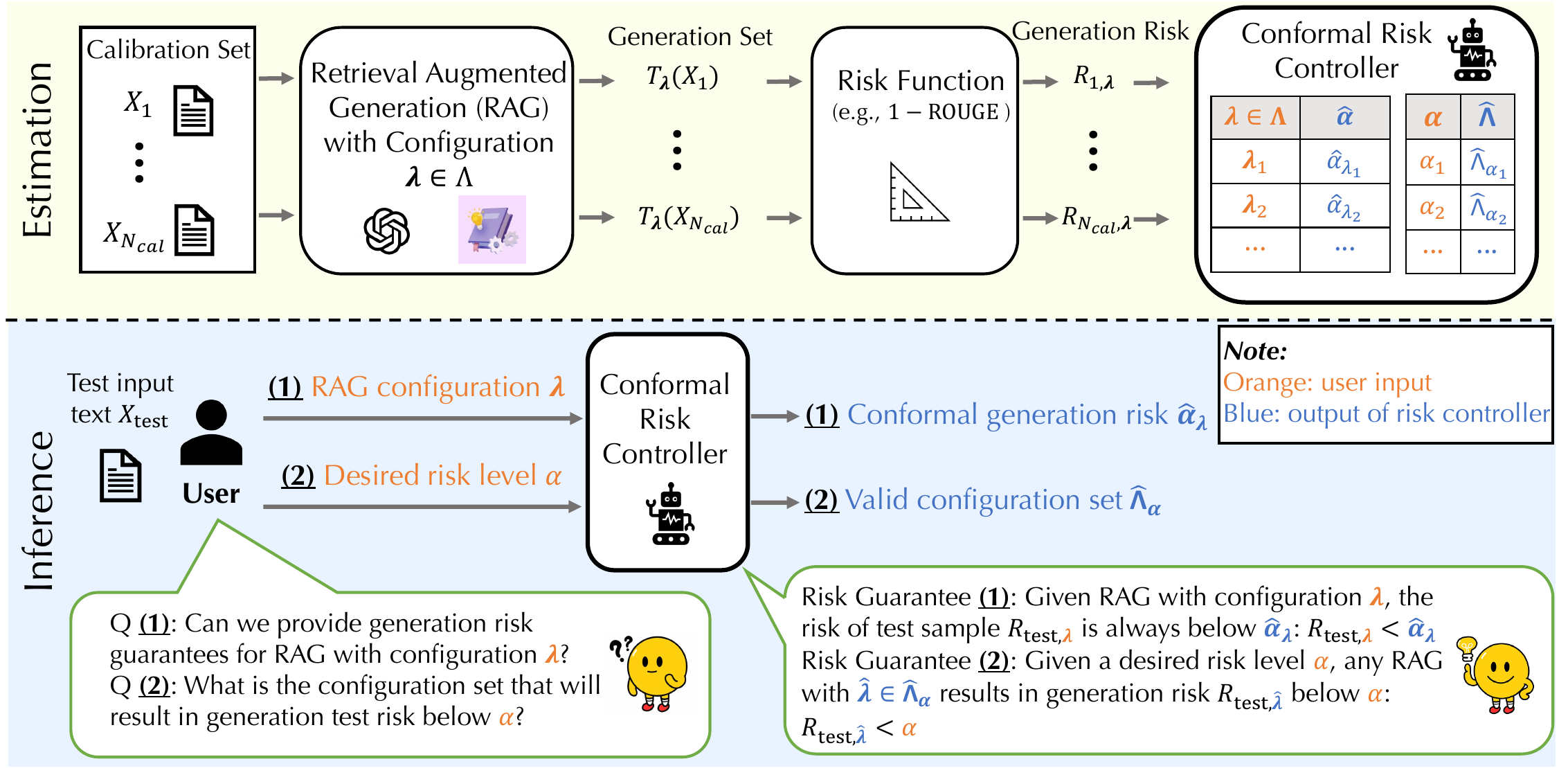}
    \vspace{-0.2em}
    \caption{\small 
    Overview of \name. 
    In the estimation stage (upper row), the conformal risk controller computes conformal generation risks for different RAG configurations (\Cref{risk_guarantee_1}), and valid configuration sets for different risk levels (\Cref{risk_guarantee_2}), both based on risk statistics on the calibration set. 
    In the inference stage (lower row), for any configuration $\eqnsmall{\vlambda}$ and any desired risk level $\eqnsmall{\alpha}$ provided by users, the conformal risk controller outputs the conformal generation risk $\eqnsmall{\hat{\alpha}_\vlambda}$ with Risk Guarantee \textbf{\underline{(1)}} and the configuration set $\eqnsmall{\hat{\Lambda}_\alpha}$ with Risk Guarantee \textbf{\underline{(2)}}.
    }
    \vspace{-0.5em}
    \label{fig:framework}
\end{figure*}
}
{
\begin{figure*}[th]
    \centering
    \includegraphics[width=1.0\linewidth]{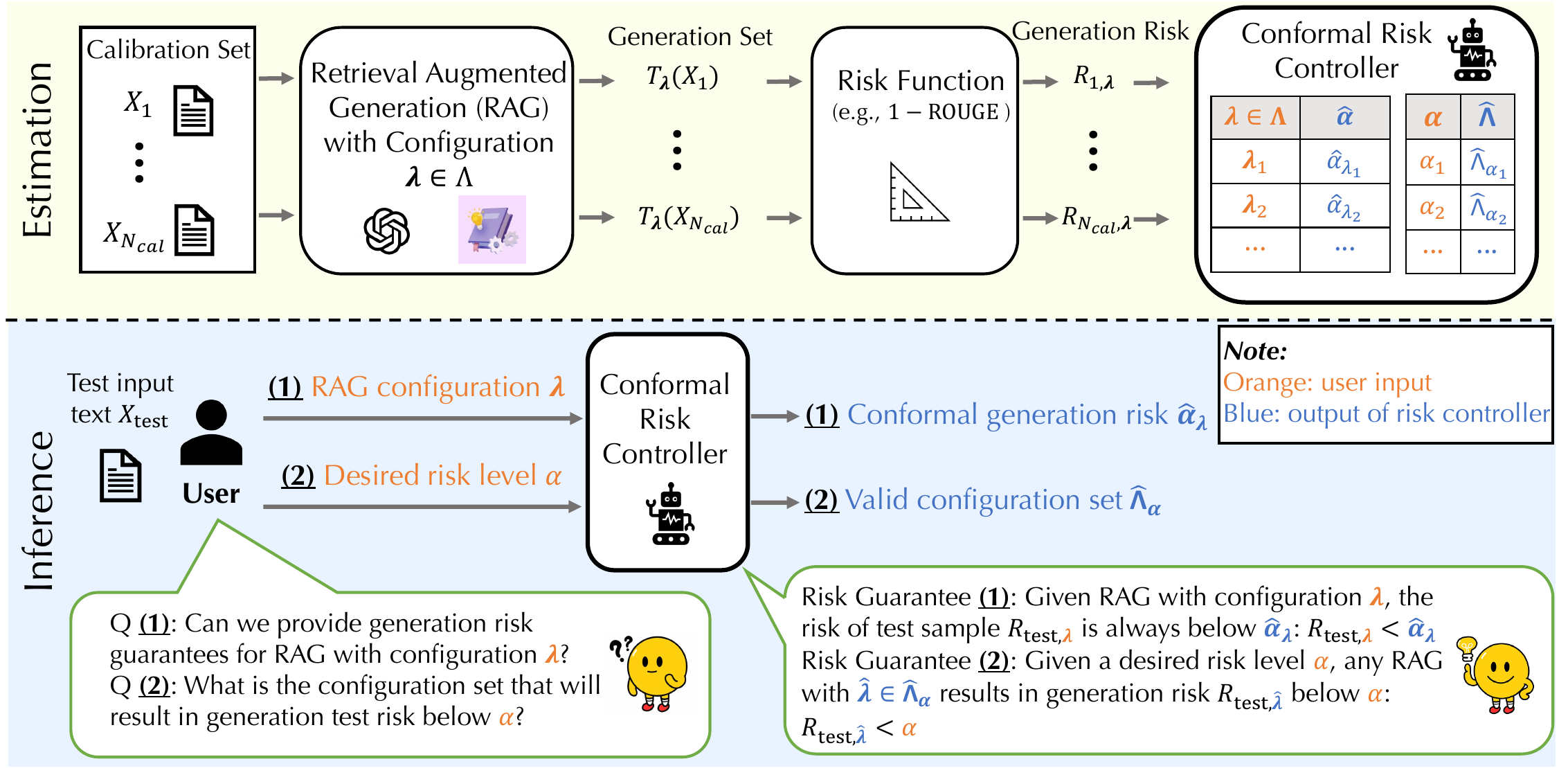}
    \caption{
    Overview of \name. 
    In the estimation stage (upper row), the conformal risk controller computes conformal generation risks for different RAG configurations (\Cref{risk_guarantee_1}), and valid configuration sets for different risk levels (\Cref{risk_guarantee_2}), both based on risk statistics on the calibration set. 
    In the inference stage (lower row), for any configuration $\eqnsmall{\vlambda}$ and any desired risk level $\eqnsmall{\alpha}$ provided by users, the conformal risk controller outputs the conformal generation risk $\eqnsmall{\hat{\alpha}_\vlambda}$ with Risk Guarantee \textbf{\underline{(1)}} and the configuration set $\eqnsmall{\hat{\Lambda}_\alpha}$ with Risk Guarantee \textbf{\underline{(2)}}.
    }
    \label{fig:framework}
\end{figure*}
}

\vspace{\stspace}
\section{Conformal generation risks of RAG models}
\label{sec:cgen}
\vspace{\pspace}
We introduce the problem setup in \Cref{sec:setup_not}, our constrained generation protocol for RAG models in \Cref{sec:rag_proto}, and the conformal generation risks in \Cref{sec:crc1}.
We prove that (1) Given a RAG configuration, \name provides a high-probability generation risk upper bound in \Cref{risk_guarantee_1}, and (2) Given a desired risk level $\eqnsmall{\alpha}$, \name offers a set of configurations that can provably maintain the risk below $\eqnsmall{\alpha}$ in \Cref{risk_guarantee_2}.
\vspace{\stspace}
\subsection{Problem setup}
\label{sec:setup_not}
\vspace{\pspace}
For a pretrained language model (LM) and any user input text, we aim to provide rigorous guarantees for the generation risks (e.g., {\eqsmall$1-\text{ROUGE}$}). 
To achieve this, we develop a constrained generation protocol for RAG models.
The generation protocol is governed by adjustable parameter configurations (e.g., number of retrieved examples, size of generations), which allow for more controlled RAG generations to achieve a desired risk level.

\vspace{\pspace}
Formally, we let $\eqnsmall{\gV}$ be the vocabulary set, $\eqnsmall{n_I}$ be the maximal length of input text and $\eqnsmall{n_O}$ be the maximal length of output text.
Let $\eqnsmall{\gX:=\gV^{n_I}}$ be the input text space, and $\eqnsmall{\gY:=\gV^{n_O}}$ be the output text space.
We notate $\eqnsmall{p_{\theta_l}(y | x)~(x \in \gX, y \in \gY)}$ as the probability distribution of output text $y$ given input text $x$, estimated by a pretrained LM parameterized by $\eqnsmall{\theta_l}$. 
Consider a RAG generation protocol $\eqnsmall{T_{\vlambda,p_{\theta_l}}: \gX \mapsto 2^\gY}$ with LM $\eqnsmall{\theta_l}$ and parameter configuration $\eqnsmall{\vlambda \in \Lambda=\sR^{B}}$, where $\eqnsmall{B}$ is the maximal number of parameters to control the generation procedure.
To evaluate the quality of the generation under $\eqnsmall{T_{\vlambda,p_{\theta_l}}(x)}$ given input $x$, we define a risk function $\eqnsmall{R(T_{\vlambda,p_{\theta_l}}(x), y): 2^\gY \times \gY \mapsto [0,1]}$, where $y$ is the reference text of $x$. For text generation tasks, a typical selection of the risk function could be $\eqnsmall{1- \max_{y' \in T_{\vlambda,p_{\theta_l}}(x)}\text{ROUGE}(y',y)}$, 
where $\eqnsmall{\text{ROUGE}}$ measures the matching score between the generation $\eqnsmall{y'}$ and reference text $\eqnsmall{y}$.
Notably, our generation protocol outputs a set of generations instead of just one, allowing us to explore better generations and adjust the generation set size through a parameter for risk control.
\vspace{\stspace}
\subsection{Constrained generation protocol for RAG models}
\label{sec:rag_proto}
\vspace{\pspace}
RAG models \cite{wang2023learning,rubin2021learning,huang2023raven} combine a retrieval model and a generation LM. The retrieval model retrieves $\eqnsmall{N_{\text{rag}}}$ relevant examples to the query from an external knowledge base, and the LM learns in-context from these examples. The knowledge base contains $\eqnsmall{N_{\text{ext}}}$ samples in $\eqnsmall{\hat{\gD}_{\text{ext}}=\{(X_i,Y_i)\}_{i=1}^{N_{\text{ext}}}}$. 
The retrieval model uses an encoder to map instances into an embedding space, and then identifies the relevant examples to the query $\eqnsmall{X_{\text{test}}}$ based on similarity. This similarity, defined by $\eqnsmall{s_{\theta_r}(\cdot,\cdot): \gX \times \gX \mapsto \sR}$ and parameterized by $\eqnsmall{\theta_r}$, is used to find the nearest examples using KNN search in the embedding space.

\vspace{\pspace}
We arrange the retrieved $\eqnsmall{N_{\text{rag}}}$ in-context examples and the test example $\eqnsmall{X_{\text{test}}}$ into augmented input text $\eqnsmall{X^{\text{(rag)}}}$ using a template.
We then sample the generation from $\eqnsmall{p_{\theta_l}(\cdot|X^{\text{(rag)}})}$ repeatedly until $\eqnsmall{\lambda_g}$ generations are collected. To control the diversity of generations, we reject those with a similarity higher than a threshold $\eqnsmall{\lambda_s}$ to the previous generations. 
In essence, the constrained generation protocol is controlled by configuration $\eqnsmall{\vlambda=[N_{\text{rag}}, \lambda_g, \lambda_s]}$ and output a generation set $\eqnsmall{T_{\vlambda,p_{\theta_l}}(x)}$ 
based on the configuration $\eqnsmall{\vlambda}$ and input $\eqnsmall{x}$. We refer to \Cref{alg:gen_pro} in \Cref{app:cons_gen_rag} for the pseudocode of the protocol.

\vspace{\stspace}
\subsection{Conformal generation risks for RAG models}
\label{sec:crc1}
\vspace{\pspace}

We certify generation risks of the RAG models with the constrained generation protocol {\eqsmall$\eqsmall T_{\vlambda,p_{\theta_l}}$} via conformal risk analysis \cite{bates2021distribution,angelopoulos2022conformal,angelopoulos2021learn}.
Conformal analysis provably controls the generation risks based on test statistics from in-distribution calibration samples.
In this work, we consider a calibration set {\eqsmall$\hat{\gD}_{\text{cal}}= \{(X_i,Y_i)\}_{i=1}^{N_{\text{cal}}}$} with size {\eqsmall$\eqsmall N_{\text{cal}}$}, and 
compute the empirical generation risk $\eqnsmall{\hat{R}(\hat{\gD}_{\text{cal}}) = \nicefrac{1}{N_{\text{cal}}} \sum_{(x,y) \in \hat{\gD}_{\text{cal}}} R(T_{\vlambda,p_{\theta_l}}(x),y)}$.
\vspace{\pspace}
\vspace{\pspace}
\vspace{-1.em}
\paragraph{Risk Guarantees for RAG Models} For an LM $\eqnsmall{\theta_l}$, calibration set $\eqnsmall{\hat{\gD}_{\text{cal}}}$, test sample $\eqnsmall{(X_{\text{test}}, Y_{\text{test}})}$, generation protocol $\eqnsmall{T_{\vlambda,p_{\theta_l}}}$ with configuration $\vlambda$ and confidence level $\eqnsmall{1-\delta~(\delta \in [0,1])}$, 
\name provides two types of generation risk guarantees for RAG models:

\begin{proposition}[Risk Guarantee \textbf{\underline{(1)}}, adaptation of \cite{bates2021distribution} to constrained RAG generation]
\label{risk_guarantee_1}
    Given a configuration $\eqnsmall{\vlambda}$ in generation protocol, \name guarantees that: 
\vspace{\pspace}
\begin{equation}
    \eqsmall
    \label{eq:conformal_guarantee}
    \sP \left[R(T_{{\vlambda},p_{\theta_l}}(x), y) \le \hat{\alpha}_\vlambda \right] \ge 1-\delta,
    \vspace{\pspace}
\end{equation}
\noindent where the high-probability risk upper bound $\eqnsmall{\hat{\alpha}_\vlambda}$, the so-called \textbf{conformal generation risk}, is given by:
\vspace{\pspace}
\begin{equation*} \label{eq:conformal_risk_}
    \eqnsmall{\hat{\alpha} = \min \left\{ h^{-1}\left(\dfrac{\ln(1/\delta)}{N_{\text{cal}}};\hat{R}(\hat{\gD}_{\text{cal}})\right),\Phi^{-1}_{\text{bin}}\left(\dfrac{\delta}{e};N_{\text{cal}},\hat{R}(\hat{\gD}_{\text{cal}})\right) \right\}}
    \vspace{\pspace}
\end{equation*}
with $\eqnsmall{h^{-1}(\cdot ;\cdot)}$ as the partial inverse $\eqnsmall{h^{-1}(h(a,b);a)=b}$ of $\eqnsmall{h(a,b)=a\log(\nicefrac{a}{b})+ (1-a)\log(\nicefrac{(1-a)}{(1-b)})}$, and $\eqnsmall{\Phi^{-1}_{\text{bin}}}$ as the inverse of binomial cumulative distribution function (CDF). 

\end{proposition}

\vspace{\pspace}
\begin{proposition}[Risk Guarantee \textbf{\underline{(2)}}, adaptation of \cite{angelopoulos2021learn} to constrained RAG generation]
\label{risk_guarantee_2}
Given a desired risk level $\eqnsmall{\alpha}$, \name computes a configuration set $\eqnsmall{\hat{\Lambda}_\alpha}$ such that each configuration in $\eqnsmall{\hat{\Lambda}_\alpha}$ is guaranteed to keep the generation risk below $\eqnsmall{\alpha}$. Namely, 
\vspace{\pspace}
\begin{equation}\label{eq:conformal_guarantee_2}
\vspace{\pspace} 
    \eqnsmall{\sP \left[ \sup_{ \hat{\vlambda} \in \hat{\Lambda}_\alpha} \left\{R\left(T_{\hat{\vlambda},p_{\theta_l}}(x),y \right)  \right\} \le \alpha \right] \ge 1 - \delta},
    \vspace{\minspace}
\end{equation}
where the valid configuration set $\eqnsmall{\hat{\Lambda}_\alpha}$ is given by family-wise error rate controlling algorithms such as Bonferroni correction: $\eqnsmall{\hat{\Lambda}_\alpha = \{ \hat{\vlambda}_j: p_j \le \delta / |\Lambda| \}}$
where $\eqnsmall{p_j}$ is the $\eqsmall p$-value of the null hypothesis: $\eqnsmall{\gH_j: R(T_{\vlambda,p_{\theta_l}}(x),y)>\alpha~(j \in \{1,...,|\Lambda|\})}$ and can be computed by finite-sample valid bounds as shown in \Cref{app:fwer}.
\end{proposition}
\vspace{\pspace}
\vspace{\pspace}
\vspace{\pspace}
\vspace{\pspace}
\vspace{-0.7em}
\paragraph{Connection between Risk Guarantees \underline{(1)} and \underline{(2)}} Risk Guarantee \textbf{\underline{(1)}} computes the conformal generation risk (risk upper bound) $\eqnsmall{\hat{\alpha}_\vlambda}$ given a configuration $\eqnsmall{\vlambda}$, while Risk Guarantee \textbf{\underline{(2)}} computes a configuration set $\eqnsmall{\hat{\Lambda}_\alpha}$ such that any configuration in the set results in a risk below the desired level $\eqnsmall{\alpha}$. 
Risk Guarantee \textbf{\underline{(2)}} can be conceptualized as accepting configurations with generation risks statistically below $\eqnsmall{\alpha}$ with a certain error rate (p-value), such that the union of error rates over parameter space is within the uncertainty budget $\eqnsmall{\delta}$.
The Bonferroni correction in \Cref{risk_guarantee_2}
adopts an even partition of the uncertainty budget, while we can have a dynamic partition algorithm based on graph search (see \Cref{app:fwer}).
Therefore, Risk Guarantee \textbf{\underline{(1)}} and \textbf{\underline{(2)}} are connected by the duality between p-values and confidence intervals \cite{bates2021distribution}. We mainly focus on the conformal analysis of Risk Guarantee \textbf{\underline{(1)}} in the following, and the results can be extrapolated to Risk Guarantee \textbf{\underline{(2)}} directly.
We defer the proofs of \Cref{risk_guarantee_1,risk_guarantee_2} to \Cref{app:pre}.
\vspace{\stspace}
\vspace{-0.7em}
\paragraph{Advance of \name compared to conformal controlling methods \cite{bates2021distribution,angelopoulos2021learn,angelopoulos2022conformal}} Existing conformal risk analysis assumes that test and calibration samples come from the same distribution, which allows statistical risk predictions for test samples based on the calibration data. While the conformal generation risk bounds are previously studied \cite{angelopoulos2022conformal, farinhas2023nonexchangeable}, the scope is limited to the monotonic risk functions.
In this work, we extend the scope to provide the first conformal generation risk analysis under test-time distribution shifts for general bounded risk functions in \Cref{sec:distribution_shift}.
In addition, we propose a constrained RAG generation protocol for enhanced effectiveness and efficiency of risk controlling for LLM generations, as illustrated in \Cref{sec:cgen}.
We also prove that RAG achieves a lower conformal generation risk than vanilla LLMs under scenarios with or without distribution shifts in \Cref{thm:gene_rag,thm:comp_shft}.

\vspace{\stspace}
\section{Theoretical analysis of \name}
\vspace{\pspace}
\label{sec:rag_conf}
In this section, we prove that RAG model achieves a lower conformal generation risk compared to LLMs without retrievals and its benefits are correlated with the quality of the retrieval model and transformer. 
We provide the structure of our theoretical analysis and conclusions in \Cref{fig:certi}. 
\vspace{\stspace}
\vspace{\minspace}
\ifbool{IsTwoColumn}{
\vspace{-1em}
    \begin{figure}[ht]
        \centering
        \includegraphics[width=\linewidth]{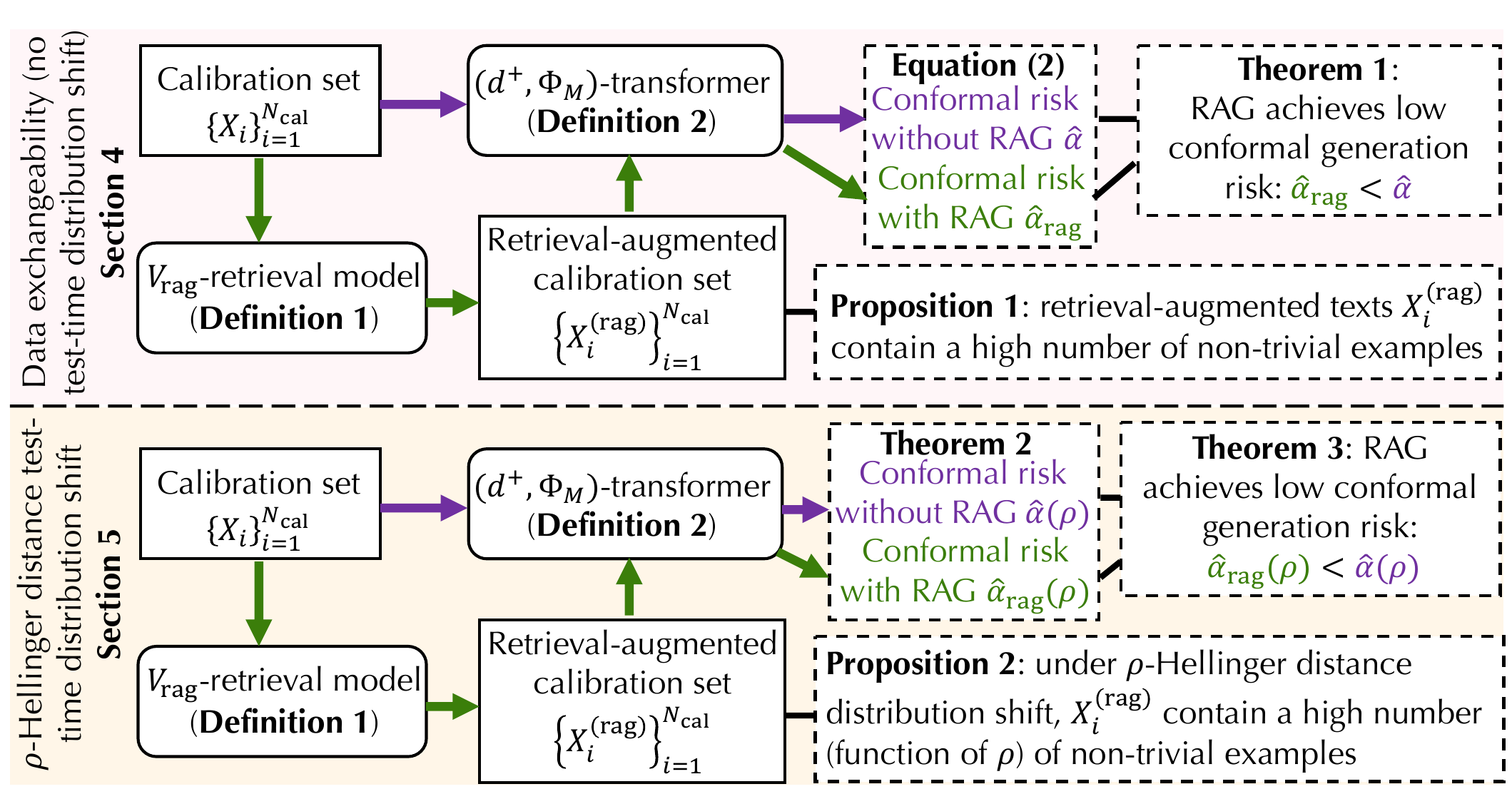}
        \vspace{\tspaceup}
        \vspace{-1.5em}
        \caption{\small Certification framework of \name. We provide theoretical results with the data exchangeability assumption in \Cref{sec:rag_conf} (upper row) and extend the results to more complex scenarios under test-time distribution shifts in \Cref{sec:distribution_shift} (lower row).}
        \label{fig:certi}
    \end{figure}
    \vspace{-1em}
}{
      \begin{figure}[ht]
        \centering
        \includegraphics[width=0.65\linewidth]{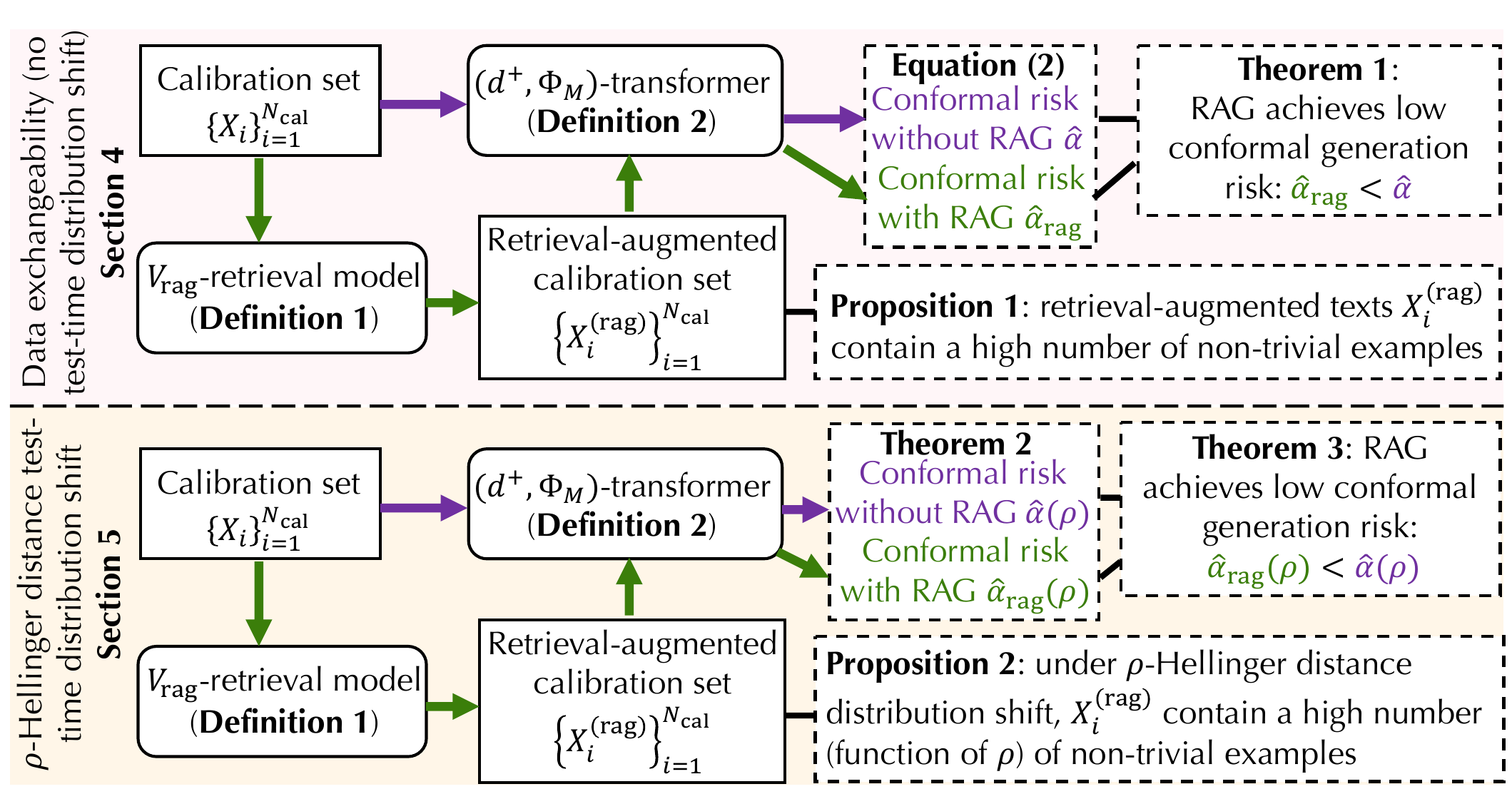}
        \vspace{-2.3em}
        \caption{Certification framework of \name. We provide theoretical results with the data exchangeability assumption in \Cref{sec:rag_conf} (upper row) and extend the results to more complex scenarios under test distribution shifts in \Cref{sec:distribution_shift} (lower row).}
        \label{fig:certi}
    \end{figure}
}
\vspace{\stspace}
\vspace{\pspace}
\vspace{\pspace}
\vspace{\minspace}
\subsection{Analysis setup} 
\label{sec:rag_setup}
\vspace{\pspace}

For our analysis, paralleling the previous transformer studies by \cite{von2023transformers,zhang2023trained,han2023context}, we consider a one-layer self-attention transformer parameterized with the embedding matrix $\eqnsmall{W_E: \gV \mapsto \sR^{d_1}}$, query matrix $\eqnsmall{W_Q: \sR^{d_1} \mapsto \sR^{d_2}}$, key matrix $\eqnsmall{W_K: \sR^{d_1} \mapsto \sR^{d_2}}$, value matrix $\eqnsmall{W_V: \sR^{d_1} \mapsto \sR^{d_2}}$, projection matrix $\eqnsmall{W_P: \sR^{d_2} \mapsto \Delta^{|\gV|}}$, and treat each instance approximately as a single token. 
The retrieval-augmented input text then consists of $\eqnsmall{N_{\text{rag}}}$ retrieved examples and $\eqnsmall{1}$ query example. We denote the augmented input text by $\eqnsmall{\vq \in \gV^{N_{\text{rag}}+1}}$.
We categorize pairs of queries $\eqnsmall{\vq_i,\vq_j~(i,j \in [N_{\text{rag}}+1])}$ as positive if they convey identical semantic meanings, indicated by $\eqnsmall{g(\vq_i)=g(\vq_j)}$. In this context, $\eqnsmall{\vq_i}$ can be referred to as a \textbf{positive example} of $\eqnsmall{\vq_j}$. Conversely, pairs $\eqnsmall{\vq_i,\vq_j~(i,j \in [N_{\text{rag}}+1])}$ are considered negative if they are semantically different.
We use such a definition for clear interpretation of our findings, but our analysis can be extended to include broader definitions of positive pairs, as addressed in the remarks of \Cref{thm:gene_rag}.

\vspace{\pspace}
Following the single-layer transformer formulation in \cite{von2023transformers}, given an input text $\eqnsmall{\vq}$, we consider the single-token output $\eqnsmall{O^{\text{(rag)}}(\vq) \in \Delta^{|\gV|}}$ corresponding to the query example $\eqnsmall{\vq_{N_{\text{rag}+1}}}$ at the last position, formulated as:
\ifbool{IsTwoColumn}{
\vspace{\stspace}
 \begin{equation}
    \eqsmall
    \begin{aligned}
        O^{\text{(rag)}}(\vq) &=  \sigmoid \big( W_P W_V \big\{ \underbrace{W_E \vq_{N_{\text{rag}+1}}  }_{\text{residual} } + \\ & (W_E \vq) \underbrace{\sigma ( (W_K W_E \vq)^T (W_Q W_E \vq_{N_{\text{rag}+1}}) ) }_{\text{attention scores to } \vq_{N_{\text{rag}+1}}} \big\} \big),
    \end{aligned}
    \vspace{\stspace}
    \vspace{\pspace}
    \vspace{+0.15em}
    \end{equation} 
}{
    \begin{equation}
    \eqsmall
    \begin{aligned}
        O^{\text{(rag)}}(\vq) =  \sigmoid \left( W_P W_V \left\{ W_E \vq_{N_{\text{rag}+1}}  + (W_E \vq) \sigma \left( (W_K W_E \vq)^T W_Q W_E \vq_{N_{\text{rag}+1}} \right) \right\} \right),
    \end{aligned}
    \end{equation} 
}
where $\eqnsmall{\sigma(\cdot)}$ is the Softmax function. Note that without RAG, the output probability vector $\eqnsmall{O(\vq)}$ is formulated as $\eqnsmall{O(\vq) = \sigmoid \left( W_P W_V W_E \vq_{N_{\text{rag}+1}} \right)}$.

\vspace{\stspace}
\subsection{Retrieval quality analysis}
\label{sec:retr_analysis}
\vspace{\pspace}
To quantify the quality of retrieval models, we introduce the concept of $\eqnsmall{V_{\text{rag}}}$-retrieval model, where $\eqnsmall{V_{\text{rag}}}$ measures the variance of the contrastive loss of the retrieval model. A small $\eqnsmall{V_{\text{rag}}}$ implies a well-trained low-variance retrieval model and can be theoretically linked to the retrieval quality, 
which is measured by the number of retrieved positive examples with respect to the query text.

\begin{definition}[$\eqnsmall{V_{\text{rag}}}$-retrieval model]
\label{def:ret_mod}
    Consider a retrieval model with similarity measurement $\eqnsmall{s_{\theta_r}(\cdot,\cdot)}$ parameterized with $\eqnsmall{\theta_r}$ and trained with contrastive loss $\eqnsmall{\gL_{\text{cont}}}$.
    Let $\eqnsmall{x^+},\eqnsmall{x^-}$ be positive and negative samples to sample $\eqnsmall{x}$. Consider common contrastive loss 
    $\eqnsmall{\gL_{\text{cont}}=-\log\left(\sigmoid_{\text{sig}}(\exp\{s_\theta(x,x^-)-\exp\{s_\theta(x,x^+))\right)}$, where $\eqnsmall{\sigmoid_{\text{sig}}(\cdot)}$ is the sigmoid function. We define a $\eqnsmall{V_{\text{rag}}}$-retrieval model as the retrieval model with (a) a non-trivial utility such that the expected contrastive loss $\eqnsmall{L_\tau}$ is better than random: $\eqnsmall{L_\tau:=\mathbb{E}[\gL_{\text{cont}}]<\ln2}$ (i.e., $\eqnsmall{\mathbb{E}[s_\theta(x,x^+)-s_\theta(x,x^-)]>0}$); and (b) bounded variance such that the training is stable and converges well: 
    $\eqnsmall{V_{\text{rag}} := {\mathbb{V}[s_\theta(x,x^+)-s_\theta(x,x^-)]^{1/2}} {\log(\exp\{L_\tau\}-1)}<1}$
\end{definition}

\begin{remark}
    {\underline{(R1)}} Note that a retrieval model with random initialization can achieve $\eqnsmall{\mathbb{E}[s_\theta(x,x^+)]=\mathbb{E}[s_\theta(x,x^-)]}$ asymptotically. We merely assume a $\eqnsmall{V_{\text{rag}}}$-retrieval model that can non-trivially differentiate the positive from negative examples.
    {\underline{(R2)}} 
    We also assume a moderate stability with bounded variance, which implicitly assumes a moderate generalization of the retrieval model based on the variance-generalization link \cite{lam2016robust,gotoh2018robust,namkoong2017variance}.
    This is essential for the analysis as the knowledge base distribution is non-identical to the calibration/test distribution.
    {\underline{(R3)}} We define $\eqnsmall{V_{\text{rag}}}$-retrieval model using a standard contrastive loss in \cite{wang2023learning,rubin2021learning}, 
    but it can be adapted for other contrastive loss such as triplet loss \cite{hermans2017defense}, by altering the logarithmic factor in the $\eqnsmall{V_{\text{rag}}}$ formula.
    We defer the proof sketches as well as the detailed proofs and remarks to \Cref{app:proof_first}. 
\end{remark}

\vspace{\pspace}
With a $\eqnsmall{V_{\text{rag}}}$-retrieval model, we show that \name can retrieve a high number of positive examples as in-context demonstrations as follows. 
\vspace{\pspace}
\begin{proposition}[lower bound of the number of retrieved positive examples]
\label{lem:rag}
    Consider the $\eqnsmall{V_{\text{rag}}}$-retrieval model in \Cref{def:ret_mod} and RAG generation protocol in \Cref{sec:rag_proto}.
    Let {\eqsmall $r^{(c)}_{\text{cal}}$} and {\eqsmall $r^{(c)}_{\text{ext}}$} be the portion of data with groundtruth output $\eqnsmall{c\in \gY}$ in the calibration data distribution and external knowledge data distribution, respectively.
    We have:
    \ifbool{IsTwoColumn}{
    \vspace{\stspace}
    \vspace{\minspace}
      \begin{equation}
        \label{eq:lem_rag}
        \eqsmall
        \begin{aligned}
            \mathbb{E}\left[ N_{\text{pos}} \right] \ge & \dfrac{9}{10}N_{\text{rag}}\left(1-  \sum_{c \in \gY} r^{(c)}_{\text{cal}} \left(N_{\text{ext}} -  r^{(c)}_{\text{ext}}N_{\text{ext}} \right. \right. \\ & \left. \left. + {\sqrt{2\ln{10}}} \right)   V_{\text{rag}}^{0.5 \left(r^{(c)}_{\text{ext}}N_{\text{ext}} - {\sqrt{2\ln{10}}} \right)}  \right)
                \vspace{\stspace}
                \vspace{\minspace}
                \vspace{\minspace}
        \end{aligned}
        \end{equation} 
    }{
      \begin{equation}
        \label{eq:lem_rag}
        \eqsmall
        \begin{aligned}
            \mathbb{E}\left[ N_{\text{pos}} \right] \ge & \dfrac{9}{10}N_{\text{rag}}\big(1-  \sum_{c \in \gY} r^{(c)}_{\text{cal}} \left(N_{\text{ext}} -  r^{(c)}_{\text{ext}}N_{\text{ext}}  + {\sqrt{2\ln{10}}} \right)   V_{\text{rag}}^{0.5 \left(r^{(c)}_{\text{ext}}N_{\text{ext}} - {\sqrt{2\ln{10}}} \right)} \big)
        \end{aligned}
        \end{equation} 
    }
    where $\eqnsmall{N_{\text{pos}}}$ is the number of retrieved positive examples, $\eqnsmall{N_{\text{rag}}}$ is the total number of retrieved examples, and $\eqnsmall{N_{\text{ext}}}$ is the number of examples in the external knowledge base.
\end{proposition}
\vspace{\pspace}
\begin{remark}
\Cref{lem:rag} offers a guarantee on the minimum number of positive examples retrieved by the 
$\eqnsmall{V_{\text{rag}}}$-retrieval model. 
{\underline{(R1)}} The ratio of the retrieved examples that are positive increases at an exponential rate with respect to $\eqnsmall{N_{\text{ext}}}$, 
which suggests that expanding the external knowledge base could enhance the retrieval quality and therefore benefit in-context learning of LLMs, as shown in \cite{min2022rethinking,wang2022towards}.
For a sufficiently large $\eqnsmall{N_{\text{ext}}}$ (a common scenario in practice), the lower bound approximately scales with $\eqnsmall{0.9N_{\text{rag}}}$.
{\underline{(R2)}} 
If the knowledge base is highly long-tailed such that samples of certain reference texts are rare (i.e., small $\eqnsmall{r_{\text{ext}}^{(c)}}$), we require a larger sample size of knowledge base $\eqnsmall{N_{\text{ext}}}$ to compensate for the long-tail distribution and achieve comparable retrieval quality.
{\underline{(R3)}} A low-variance retrieval model is expected to generalize well to test distribution and increase retrieval quality. The above guarantee we provide for $\eqnsmall{\mathbb{E}[N_{\text{pos}}]}$ in relation to $\eqnsmall{V_{\text{rag}}}$ is a rigorous demonstration of this.
\end{remark}

\vspace{\stspace}
\subsection{RAG achieves provably lower conformal generation risk than a single LLM without retrieval}\label{sec:benefit_rag}
\vspace{\pspace}
Besides retrieval quality, the generation risk in RAG models is also affected by LLM quality, and to measure transformer quality, we define a $\eqnsmall{(d^+,\Phi_M)}$-transformer as follows.
\begin{definition}[$\eqnsmall{(d^+,\Phi_M)}$-transformer]
\label{def:trans}
    We assume that each in-context example $\eqnsmall{(X_i,Y_i)~(i \in [N_{\text{rag}}+1])}$ is encoded with a single token $\eqnsmall{\vq_i}$. 
    Let $\eqnsmall{\vq}$ be the retrieval-augmented input, consisting of $\eqnsmall{N_{\text{rag}}}$ retrieved in-context examples and $\eqnsmall{1}$ query example.
    We define a random variable to represent the negative prediction margin: $\eqnsmall{M = \max_{c \neq g(\vq_{N_{\text{rag}}+1})}O_c(\vq) -  O_{g(\vq_{N_{\text{rag}}+1})}(\vq)}$, where $\eqnsmall{O(\vq)}$ is the output probability vector without RAG. Let $\eqnsmall{\Phi_M(\cdot)}$ be the CDF of random variable $\eqnsmall{M}$.
     We define a $\eqnsmall{(d^+,\Phi_M)}$-transformer as a single-layer self-attention transformer with (a) non-trivial self-attention layer with $\eqnsmall{\sigma\left((W_K W_E \vq_i)^T (W_Q W_E \vq_j )\right) \ge d^+ > 0}$ for semantically identical examples with $\eqnsmall{g(\vq_i)=g(\vq_j)}$; and (b) the prediction utility that is better than random: $\eqnsmall{\int_{-1}^1 \Phi_M(v)dv > 1}$.
\end{definition}
\vspace{\pspace}
\begin{remark}
    \underline{(R1)} 
    $\eqnsmall{d^+}$ measures the minimal attention scores for positive pairs and reflects the effectiveness of the transformer's embedding, key, and query matrices. Since we always have $\eqnsmall{d^+\ge 0}$ due to the Softmax activation, the condition $\eqnsmall{d^+>0}$ only assumes a non-trivial self-attention layer.
    \underline{(R2)} The integral $\eqnsmall{\int_{-1}^1 \Phi_M(v)dv}$ measures the quality of the embedding, value, and projection matrices.
    Note that a random prediction margin $\eqnsmall{M_{\text{rand}}}$ over a uniform distribution $\eqnsmall{[-1,1]}$ results in $\eqnsmall{\int_{-1}^1 \Phi_{M_{\text{rand}}}(v)dv=1}$. Thus, $\eqnsmall{\int_{-1}^1 \Phi_M(v)dv > 1=\int_{-1}^1 \Phi_{M_{\text{rand}}}(v)dv}$ only indicates better-than-random prediction utility.
\end{remark}
\vspace{\pspace}

\vspace{\pspace}
Next, we prove that RAG in \name achieves a lower conformal generation risk than a single LLM without retrieval with high probability.
\vspace{\minspace}
\begin{theorem}[RAG reduces the conformal generation risk]
\label{thm:gene_rag}
    Consider the setup in \Cref{sec:rag_setup} as well as the \underline{$\eqnsmall{V_{\text{rag}}}$-retrieval model} in \Cref{def:ret_mod} and \underline{$\eqnsmall{(d^+,\Phi_M)}$-transformer} in \Cref{def:trans}.
    Let $\eqnsmall{r^{(c)}_{\text{cal}}}$ and $\eqnsmall{r^{(c)}_{\text{ext}}}$ be as defined in \Cref{lem:rag}.
    We show that the conformal generation risk of RAG $\eqnsmall{\hat{\alpha}_{\text{rag}}}$ is smaller than that of a single LLM $\eqnsmall{\hat{\alpha}}$ with high probability:
    \begin{equation}
    \label{eq:conclu}
    \eqsmall
        \begin{aligned}
             & \sP\left[ \hat{\alpha}_{\text{rag}} < \hat{\alpha} \right] \ge 1 - p_{\text{t}} - p_{\text{r}}, \quad \text{where} \\ & p_{\text{t}} \hspace{-0.1em} = \hspace{-0.1em} \exp \{ \hspace{-0.1em} -2N_{\text{cal}} 
             [ \underbrace{\Phi_M ( \dfrac{1}{2}{\overbrace{d^+( \hspace{-0.2em} \int_{-1}^1 \hspace{-0.5em} \Phi_M(v) dv \hspace{-0.1em} - \hspace{-0.1em} 1 )}^{\text{quality of transformers}}}N_{\text{rag}} ) \hspace{-0.1em} - \hspace{-0.1em} \Phi_M(0)}_{\text{improvement of generation quality with RAG}} ]^2 \hspace{-0.1em} \} \\ 
             & p_{\text{r}} \hspace{-0.1em} =  \hspace{-0.1em} \dfrac{25}{N_{\text{rag}}} ( 4 - 9 \underbrace{\sum_{c=1}^C r^{(c)}_{\text{cal}} (1.5 N_{\text{ext}} - r^{(c)}_{\text{ext}}N_{\text{ext}} )   V_{\text{rag}}^{0.25 r^{(c)}_{\text{ext}}N_{\text{ext}}  }}_{\text{number of retrieved negative examples}} )^{-2}
        \end{aligned}
    \end{equation}
provided that {\footnotesize $\footnotesize{N_{\text{ext}}>{2\sqrt{2 \ln10}}/{\min_c r_{\text{ext}}^{(c)}}}$}, {\footnotesize$\eqnsmall{N_{\text{rag}}>{2}/{d^+}}$} and {\footnotesize$\footnotesize{N_{\text{ext}} V_{\text{rag}}^{0.25\min_c r_{\text{ext}}^{(c)} N_{\text{ext}}}<{4}/{9}}$}.
$\eqnsmall{p_{t},p_{r}}$ are the uncertainty induced by the quality of transformer and retrieval model.
\end{theorem}
\vspace{\pspace}
\begin{remark}
{\underline{(R1)}} The probability of reduced risk with RAG ($\eqnsmall{\sP\left[ \hat{\alpha}_{\text{rag}} < \hat{\alpha} \right]}$) increases with both $\eqnsmall{N_{\text{cal}}}$, which improves risk approximation via enhanced calibration, and $\eqnsmall{N_{\text{rag}}}$ and $\eqnsmall{N_{\text{ext}}}$, which expand the scope of retrieved knowledge. {\underline{(R2)}} The reduced risk probability also increases with the transformer's quality induced by attention scores and prediction capability. \underline{{(R3)}} A low-variance retrieval model (small $\eqnsmall{V_{\text{rag}}}$) enhances generalization and reduces retrieval model uncertainty $\eqnsmall{p_r}$. \underline{{(R4)}} For certification simplicity, we define positive pairs as semantically identical examples, but this can be expanded to pairs similar in the embedding space for boosting attention scores, in-context learning, and generation quality. \underline{{(R5)}} These result can readily extend to various conformal risks such as \cite{angelopoulos2022conformal}. \underline{{(R6)}} For sufficiently large sample size {$\eqnsmall{N_{\text{ext}}}$} in the external knowledge base, we further have $\eqnsmall{\sP\left[ \hat{\alpha}_{\text{rag}} < \hat{\alpha} \right] \ge 1 - p_{t} - {25}/{16 N_{\text{rag}}}}$ 
(\Cref{thm1} in \Cref{app:asym}). In contrast to \Cref{thm:gene_rag}, the bound has no dependency on the external knowledge distribution $\eqnsmall{r_{\text{ext}}^{(c)}}$ and calibration distribution $\eqnsmall{r_{\text{cal}}^{(c)}}$.
\end{remark}

\vspace{-.5em}

\vspace{\minspace}
\section{\name under distribution shifts}
\label{sec:distribution_shift}
\vspace{\pspace}
Here, we present a valid distribution-drift conformal generation risk and prove the benefit of RAG compared to vanilla LLM under test-time distribution shifts.
\vspace{\pspace}
\vspace{-0.2em}
\subsection{Analysis setup}
\vspace{-0.2em}
\label{sec:setup_shft}
\vspace{\pspace}
Conformal risk guarantees often assume that calibration and testing samples come from the same distribution~\cite{bates2021distribution,angelopoulos2022conformal,angelopoulos2021learn}. Next, building on \Cref{sec:rag_setup}, we extend these guarantees to distribution shifts between calibration and testing.
\vspace{\pspace}
\vspace{-0.2em}
\subsection{Conformal generation risk under distribution shifts}
\vspace{-0.2em}
\label{sec:conf_shf}
\vspace{\pspace}
Under test-time distribution shifts, the certification guarantees of prior work \cite{angelopoulos2022conformal,farinhas2023non} are limited to the monotonic risk functions and to distribution shifts caused by changes in sample weights. Here, we provide generation risk certification for any bounded risk function and any distribution shift, which is as follows.
\begin{theorem}[Conformal generation risk under distribution shifts]
\label{pro:conf_shf}
    Suppose that the test instance $\eqnsmall{(X_{\text{test}},Y_{\text{test}})}$ is sampled from a shifted distribution $\eqnsmall{\gQ}$ with bounded Hellinger distance from the calibration distribution $\eqnsmall{\gD}$: $\eqnsmall{H(\gD,\gQ) \le \rho}$. 
    Let $\eqnsmall{\hat{R}=\sum_{i=1}^{N_{\text{cal}}} R(Z_i) / N_{\text{cal}}}$ be the empirical risk on calibration samples $\eqnsmall{Z_i~(i \in \{1,...,N_{\text{cal}}\})}$ and $\eqnsmall{\hat{V}={1}/{N_{\text{cal}}(N_{\text{cal}}-1)} \sum_{1\le i < j \le N_{\text{cal}}} (R(Z_i)-R(Z_j))^2}$ be the unbiased estimator of the risk variance on the calibration set.
    Then we have the following guarantee of conformal generation risk on the shifted distribution $\eqnsmall{\gQ}$: 
    \ifbool{IsTwoColumn}{
      \begin{equation}
    \label{pro:eq1}
    \eqsmall
    \begin{aligned}
        & \sP_{(X_{\text{test}},Y_{\text{test}})\sim \gQ} \left[ R(T_{{\vlambda}}(X_{\text{test}}),Y_{\text{test}})  \le \hat{\alpha}(\rho) \right] \ge 1 - \delta, \quad \text{where}\\
        & \hat{\alpha}(\rho) := \min \left\{ h^{-1}\left(\dfrac{\ln(8/\delta)}{N_{\text{cal}}};\overline{\hat{R}_\rho}\right), \Phi^{-1}_{\text{bin}}\left(\dfrac{\delta}{8e};N_{\text{cal}}, \overline{\hat{R}_\rho} \right) \right\}.
    \end{aligned}
     \end{equation}
    }{
      \begin{equation}
    \label{pro:eq1}
    \eqsmall
    \begin{aligned}
        & \sP \left[ R(T_{\hat{\vlambda}};X_{\text{test}},Y_{\text{test}})  \le \hat{\alpha}(\rho) \right] \ge 1 - \delta, \quad \text{where}
        \quad \hat{\alpha}(\rho) := \min \left\{ h^{-1}\left(\dfrac{\ln(8/\delta)}{N_{\text{cal}}};\overline{\hat{R}_\rho}\right), \Phi^{-1}_{\text{bin}}\left(\dfrac{\delta}{8e};N_{\text{cal}}, \overline{\hat{R}_\rho} \right) \right\}
    \end{aligned}
     \end{equation}
    }
where $\eqnsmall{h^{-1}(\cdot;\cdot)}$ is the partial inverse function as defined in \Cref{risk_guarantee_1} and $\eqnsmall{\overline{\hat{R}_\rho}}$ is formulated as:
     \ifbool{IsTwoColumn}{
       \begin{equation*}
         \eqsmall
         \label{pro:eq2}
         \begin{aligned}
             & \overline{\hat{R}_\rho} = \underbrace{\hat{R} + \rho^2(2-\rho^2)( 1 - \hat{R} )}_{\text{empirical mean scaled by }\rho} + \underbrace{2\rho(1-\rho^2)\sqrt{2-\rho^2} \sqrt{\hat{V}}}_{\text{estimated variance scaled by }\rho}
             + \\
             & \underbrace{(1 \hspace{-0.1em} - \hspace{-0.1em}\rho^2) \hspace{-0.3em} \left(\dfrac{1 -\rho^2}{\sqrt{2N_{\text{cal}}}} \hspace{-0.2em} + \hspace{-0.2em} \dfrac{2\sqrt{2}\rho\sqrt{2-\rho^2}}{\sqrt{N_{\text{cal}}-1}}\right) \hspace{-0.4em} \sqrt{\ln(4/\delta)} + \hspace{-0.3em} \sqrt{\dfrac{\ln{(8/\delta)}}{2N_{\text{cal}}}}}_{\text{finite-sample error}}
         \end{aligned}
         \end{equation*}
    }{
       \begin{equation*}
     \small
     \label{pro:eq2}
     \begin{aligned}
          \overline{\hat{R}_\rho} = \underbrace{\hat{R} + \rho^2(2-\rho^2)( 1 - \hat{R} )}_{\text{empirical mean scaled by }\rho} + \underbrace{2\rho(1-\rho^2)\sqrt{2-\rho^2} \sqrt{\hat{V}}}_{\text{estimated variance scaled by }\rho}
         + 
          \underbrace{(1 \hspace{-0.1em} - \hspace{-0.1em}\rho^2) \hspace{-0.3em} \left(\dfrac{1 -\rho^2}{\sqrt{2N_{\text{cal}}}} \hspace{-0.2em} + \hspace{-0.2em} \dfrac{2\sqrt{2}\rho\sqrt{2-\rho^2}}{\sqrt{N_{\text{cal}}-1}}\right) \hspace{-0.4em} \sqrt{\ln(4/\delta)} + \hspace{-0.3em} \sqrt{\dfrac{\ln{(8/\delta)}}{2N_{\text{cal}}}}}_{\text{finite-sample error}}
     \end{aligned}
     \end{equation*}
    }
    where the radius $\eqnsmall{\rho}$ satisfies the following: $\rho^2 \le 1 - \big[ 1 + \big( \nicefrac{\hat{R} - 1 + \sqrt{\nicefrac{\ln(4/\delta)}{2N_{\text{cal}}}}\big)^2}{\big(\sqrt{\hat{V}} + \sqrt{\nicefrac{2 \ln(2/\delta)}{(N_{\text{cal}}-1)}}\big)^2 } \big]^{-2}$.
\end{theorem}

\begin{remark}
\Cref{pro:conf_shf} offers a distribution-drift conformal generation risk $\eqnsmall{\hat{\alpha}(\rho)}$, under the distribution shift with radius $\eqnsmall{\rho}$.
   \underline{{(R1)}} 
   We adopt the Hellinger distance for measuring distribution distances due to its f-divergence properties and direct applicability to total variation distance between $\eqnsmall{\gD}$ and $\eqnsmall{\gQ}$ (\cite{steerneman1983total}, Equation 1).
    \underline{{(R2)}} For conformal guarantees under distribution shifts, we derive an empirical risk upper bound considering worst-case shifts in \Cref{pro:eq2}, which is efficiently calculable with empirical statistics on calibration distribution $\eqnsmall{\gD}$.
 \underline{{(R3)}} We recover the empirical mean {\small $\hat{R}$} from {\small $\overline{\hat{R}_\rho}$} when {\small $\rho \rightarrow 0$} and {\small $N_{\text{cal}} \rightarrow \infty$} in \Cref{pro:eq2}.
\end{remark}

\vspace{\pspace}
\subsection{RAG achieves provably lower conformal generation risk than a single LLM under distribution shifts}
\label{sec:conf_risk_dist_shft}
\vspace{\pspace}
Next, we prove that RAG mitigates conformal generation risk better than a single LLM under distribution shifts.\footnote{Additionally, we examine retrieval quality and prove a lower bound of retrieved positive examples under test-time distribution shifts. 
We leave the analysis to \Cref{app:rag_dist_shft} for interested readers.
} 
\begin{theorem}[RAG reduces conformal generation risk even under distribution shifts]
\label{thm:comp_shft}
    Suppose that the shifted test distribution $\eqnsmall{\gQ}$ is within bounded Hellinger distance $\eqnsmall{\rho>0}$ to the calibration distribution $\eqnsmall{\gD}$. Consider the same setup and assumptions as \Cref{thm:gene_rag}. Consider also a \underline{$\eqnsmall{V_{\text{rag}}}$-retrieval model} in \Cref{def:ret_mod} and \underline{$\eqnsmall{(d^+,\Phi_M)}$-transformer} in \Cref{def:trans}.
    Under the condition that $\eqnsmall{N_{\text{ext}}>{2\sqrt{2 \ln10}} / {r_{\text{ext}}^m}}$, $\eqnsmall{N_{\text{ext}} V_{\text{rag}}(\rho)^{0.25r_{\text{ext}}^m N_{\text{ext}}}<{8} / {17}}$, and $\eqnsmall{N_{\text{rag}}>{2}/{d^+}}$ , we have:
    \ifbool{IsTwoColumn}{
       \begin{equation}
    \eqsmall
    \label{eq:conclu_shft}
        \begin{aligned}
             & \sP\left[ \hat{\alpha}_{\text{rag}}(\rho) < \alpha(\rho) \right]\ge 1 - p_{t} - p_r(\rho), \quad \text{where} \\ & p_r(\rho) =  \dfrac{100}{N_{\text{rag}}} \left( 8 - 17  N_{\text{ext}}  V_{\text{rag}}(\rho)^{0.25 \left(\min_c r^{(c)}_{\text{ext}}N_{\text{ext}} \right)} \right)^{-2},
        \end{aligned}
    \vspace{-0.7em}
    \end{equation}
    }{
       \begin{equation}
    \eqsmall
    \label{eq:conclu_shft}
        \begin{aligned}
             & \sP\left[ \hat{\alpha}_{\text{rag}}(\rho) < \alpha(\rho) \right]\ge 1 - p_{t} - p_r(\rho), \quad \text{where}\quad p_r(\rho) =  \dfrac{100}{N_{\text{rag}}} \left( 8 - 17  N_{\text{ext}}  V_{\text{rag}}(\rho)^{0.25 \left(\min_c r^{(c)}_{\text{ext}}N_{\text{ext}} \right)} \right)^{-2},
        \end{aligned}
    \end{equation}
    }
    where $\eqnsmall{p_{\text{t}}}$ is the uncertainty induced by the transformer quality as \Cref{eq:conclu} and $\eqnsmall{p_r(\rho)}$ is the uncertainty induced by the retrieval model. Moreover, $\eqnsmall{V_{\text{rag}}(\rho)=m(\rho) V_{\text{rag}}}$ quantifies the quality of retrieval models under distance $\eqnsmall{\rho}$, where
    \ifbool{IsTwoColumn}{
      \begin{equation*}
    \eqsmall
    \begin{aligned}
    \vspace{-1em}
         m(\rho) = \underbrace{\left( \dfrac{\sqrt{-6\rho^4+12\rho^2+1}-4\rho(1-\rho^2)\sqrt{2-\rho^2}}{1-16\rho^2+8\rho^4} \right)^{-2}}_{\text{retrieval model quality decay factor by distribution shifts}}.
    \end{aligned}
    \end{equation*}
    }{
      \begin{equation*}
    \eqsmall
    \label{eq:v_rag_rho}
    \begin{aligned}
    \vspace{-1em}
         V_{\text{rag}}(\rho):= m(\rho) V_{\text{rag}}, \quad \text{where} \quad
          m(\rho) = \underbrace{\left( \dfrac{\sqrt{-6\rho^4+12\rho^2+1}-4\rho(1-\rho^2)\sqrt{2-\rho^2}}{1-16\rho^2+8\rho^4} \right)^{-2}}_{\text{retrieval model quality decay factor under distribution shifts}}.
    \end{aligned}
    \end{equation*}
    }
\end{theorem}
\begin{remark}
    Our result rigorously characterizes the effect of distribution shift on the reduced risk guarantee of RAG. 
    \underline{{(R1)}} Compared to \Cref{thm:gene_rag}, only the uncertainty of retrieval model $\eqnsmall{p_r(\rho)}$ is affected by the distribution shift $\eqnsmall{\rho}$. This affect is reflected on the the retrieval quality $\eqnsmall{V_{\text{rag}}(\rho)}$. In particular, a large distance radius $\eqnsmall{\rho}$ will downgrade the retrieval quality $\eqnsmall{V_{\text{rag}}(\rho)}$ and thus lead to a higher uncertainty $\eqnsmall{p_r(\rho)}$. However, the influence of $\eqnsmall{\rho}$ on $\eqnsmall{p_r(\rho)}$ can be reduced by $\eqnsmall{N_{\text{rag}}}$ inverse proportionally and by $\eqnsmall{N_{\text{ext}}}$ exponentially, demonstrating the robustness of RAG with more retrieval knowledge.
    \underline{{(R2)}} Since $\eqnsmall{V_{\text{rag}}(\rho)}$ is proportional to model variance $\eqnsmall{V_{\text{rag}}}$, a low-variance retrieval model demonstrates better robustness against distribution drifts, aligning with existing empirical observations \cite{lam2016robust,gotoh2018robust,namkoong2017variance}, which evaluate the generalization ability of low-variance retrieval models under distribution shifts.
    \underline{(R3)} Compared to \Cref{thm:gene_rag}, \Cref{thm:comp_shft} has no dependence on varying label portions $\eqnsmall{r^{(c)}_{\text{cal}}}$ during distribution shifts, as long as the size of external knowledge base $\eqnsmall{N_{\text{ext}}}$ is moderately large,
    a condition often met in practice with large knowledge bases.
\end{remark}

\vspace{\stspace}
\section{Evaluation}
\vspace{\pspace}
We evaluate \name on four datasets using different retrieval models.
We find that \textbf{{(1)}} our conformal generation risks in \Cref{risk_guarantee_1} is empirically sound and tight in \Cref{sec:exp_no_dist}, \textbf{{(2)}} RAG reduces the conformal generation risks for different retrieval models, which empirically validates \Cref{thm:gene_rag} in \Cref{sec:exp_no_dist}, \textbf{{(3)}} the conformal generation risk under distribution shifts in \Cref{pro:conf_shf} is empirically sound and tight for varying distances in \Cref{sec:exp_dist}, \textbf{{(4)}} multi-dimensional RAG configurations maintain sound and tight conformal generation risks in \Cref{sec:exp_multi}, and \textbf{{(5)}} \name computes valid configurations with empirical risks always below the desired risk level in \Cref{sec:exp_multi_2}.

The codes are publicly available at \url{https://github.com/kangmintong/C-RAG}.

\subsection{Evaluation setup}
\label{sec:exp_setup}
\vspace{\pspace}
\paragraph{Datasets \& knowledge base} We evaluate \name on four widely used NLP datasets, including AESLC \cite{zhang2019email}, CommonGen \cite{lin2019commongen}, DART \cite{nan2020dart}, and E2E \cite{novikova2017e2e}.
Following \cite{wang2023learning,cheng2023uprise}, we construct the knowledge base as a collection of 30 public datasets from 9 distinct categories with over 6 million documents.

\vspace{\pspace}
\vspace{-0.7em}
\paragraph{Retrieval models} We consider four retrieval models: (1) BM25 \cite{robertson2009probabilistic} with token-level matching scores, (2) BAAI/bge \cite{zhang2023retrieve} as SOTA embedding model in MTEB benchmark \cite{muennighoff2022mteb}, (3) OpenAI/ada as SOTA close source text embedding model, and (4) Biencoder-SFT \cite{wang2023learning} as a biencoder retriever trained with in-domain data samples.
\vspace{\pspace}

\vspace{-0.7em}
\paragraph{RAG Generation protocol}  
We use our generation protocol (\Cref{alg:gen_pro} in \Cref{app:cons_gen_rag}) controlled by the number of retrieved examples $\eqnsmall{N_{\text{rag}}}$, generation set size $\eqnsmall{\lambda_g}$, and diversity threshold $\eqnsmall{\lambda_s}$.
We use Llama-2-7b for inference and perform conformal calibration on validation sets with uncertainty $\eqnsmall{\delta=0.1}$.
We use $\eqnsmall{1-\text{ROUGE-L}}$ as the risk function.
See \Cref{app:setup} for more details of evaluation setup.

\ifbool{IsTwoColumn}
{
\begin{figure*}[thb]
\subfigure{
    \rotatebox{90}{\hspace{-3.5em} Evaluation Risk}
    \begin{minipage}{0.2\linewidth}
    \centerline{\footnotesize{\quad AESLC}}
 	\vspace{1pt}
\centerline{\includegraphics[width=1.0\textwidth]{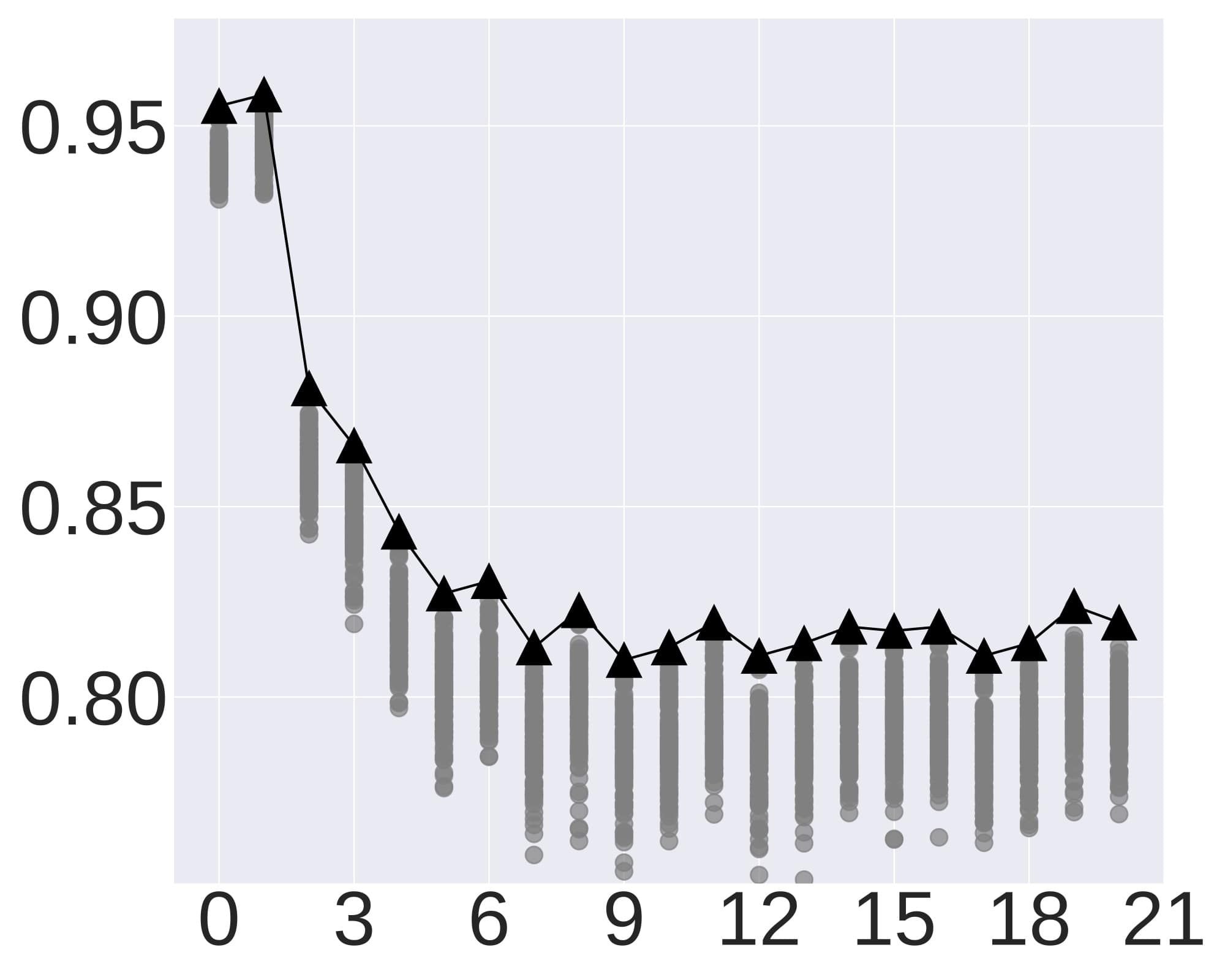}}

 	\centerline{\footnotesize{~~~\# Retrieved examples $N_{\text{rag}}$}}
  \vspace{-0.5em}
 \end{minipage}
 \hspace{+2em}
 \begin{minipage}{0.2\linewidth}
    \centerline{\footnotesize{\quad CommonGen}}
 	\vspace{1pt}
 	\centerline{\includegraphics[width=1.0\textwidth]{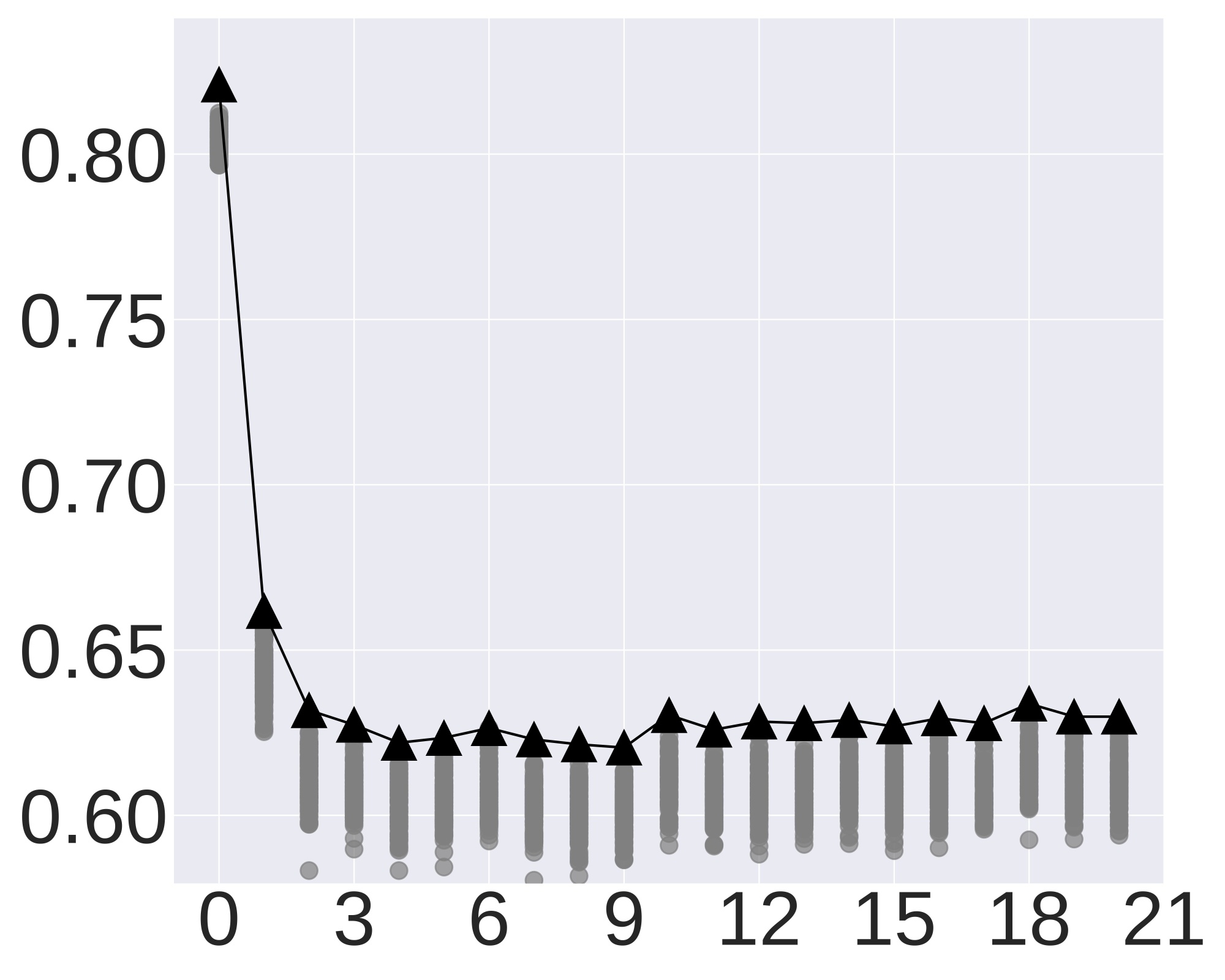}}
 	\centerline{\footnotesize{~~~\# Retrieved examples $N_{\text{rag}}$}}
  \vspace{-0.5em}
 \end{minipage}
 \hspace{+2em}
 \begin{minipage}{0.2\linewidth}
    \centerline{\footnotesize{\quad DART}}
 	\vspace{1pt}
 	\centerline{\includegraphics[width=1.0\textwidth]{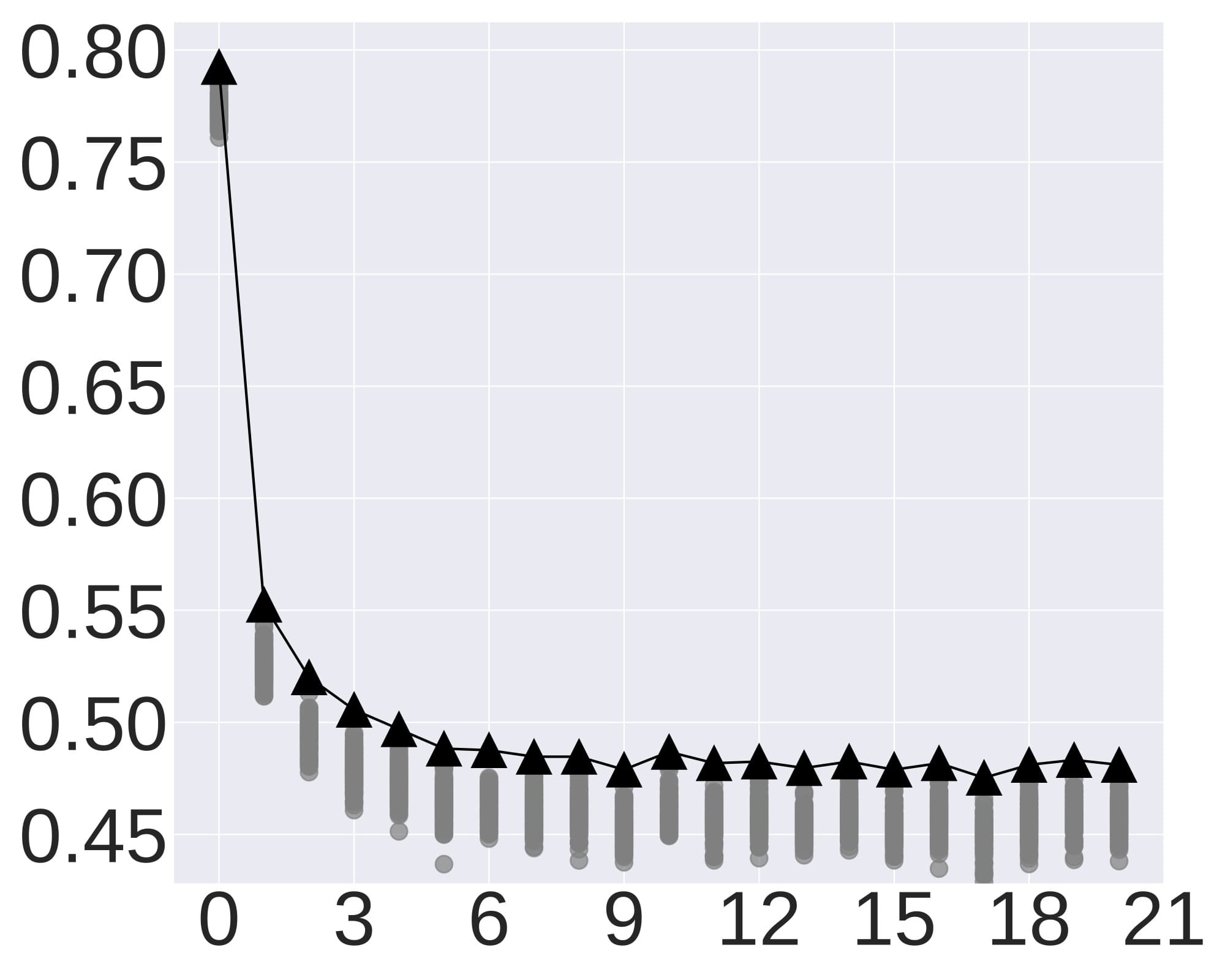}}
 	\centerline{\footnotesize{~~~\# Retrieved examples $N_{\text{rag}}$}}
  \vspace{-0.5em}
 \end{minipage}
 \hspace{+2em}
 \begin{minipage}{0.2\linewidth}
    \centerline{\footnotesize{\quad E2E}}
 	\vspace{1pt}
 	\centerline{\includegraphics[width=1.0\textwidth]{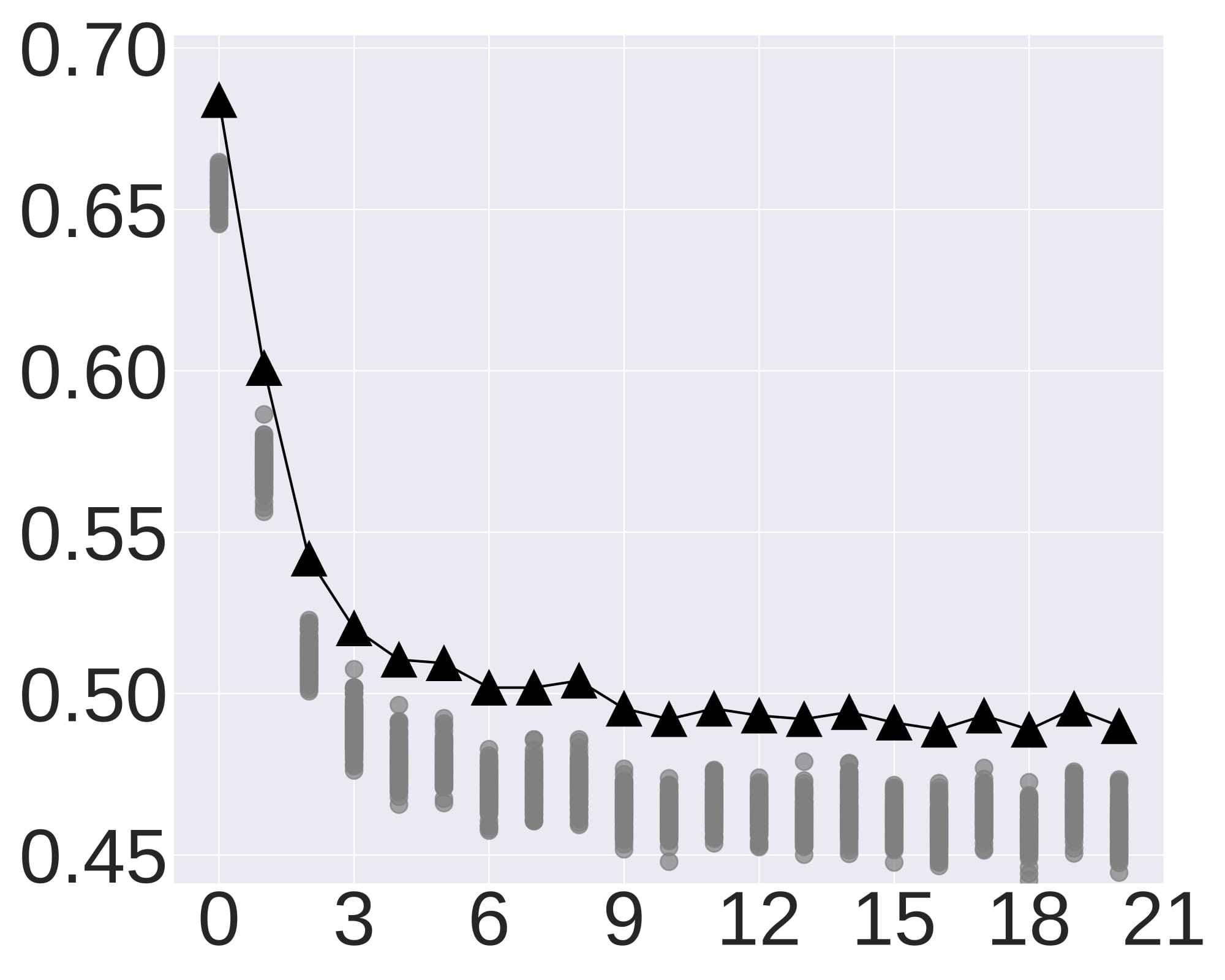}}
 	\centerline{\footnotesize{~~~\# Retrieved examples $N_{\text{rag}}$}}
  \vspace{-0.5em}
 \end{minipage}
}
\subfigure{
\centerline{\includegraphics[width=0.4\textwidth]{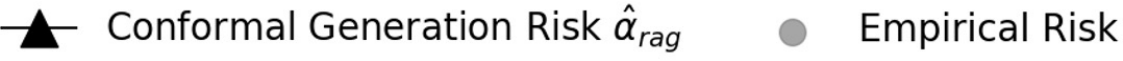}}}
\vspace{-2em}
\caption{\small Conformal generation risk $\eqnsmall{\hat{\alpha}_{\text{rag}}}$ and empirical risk based on retrieval model OpenAI/ada taking different $\eqnsmall{N_{\text{rag}}}$ ($\eqnsmall{\lambda_g=1,\lambda_s=1.0}$). 
{We observe that our conformal generation risk (\Cref{risk_guarantee_1}) is valid and tight; 
larger $\eqnsmall{N_{\text{rag}}}$ reduces risk $\eqnsmall{\hat{\alpha}_{\text{rag}}}$ (empirically validating \Cref{thm:gene_rag})}.
}
\label{fig:bound_and_simulation_ada}
\end{figure*}
\begin{figure*}[thb]
\subfigure{
    \rotatebox{90}{\hspace{-5.0em} Conformal Risk $\hat{\alpha}_{\text{rag}}$}
    \begin{minipage}{0.2\linewidth}
    \centerline{\footnotesize{\quad AESLC}}
 	\vspace{1pt}
\centerline{\includegraphics[width=1.0\textwidth]{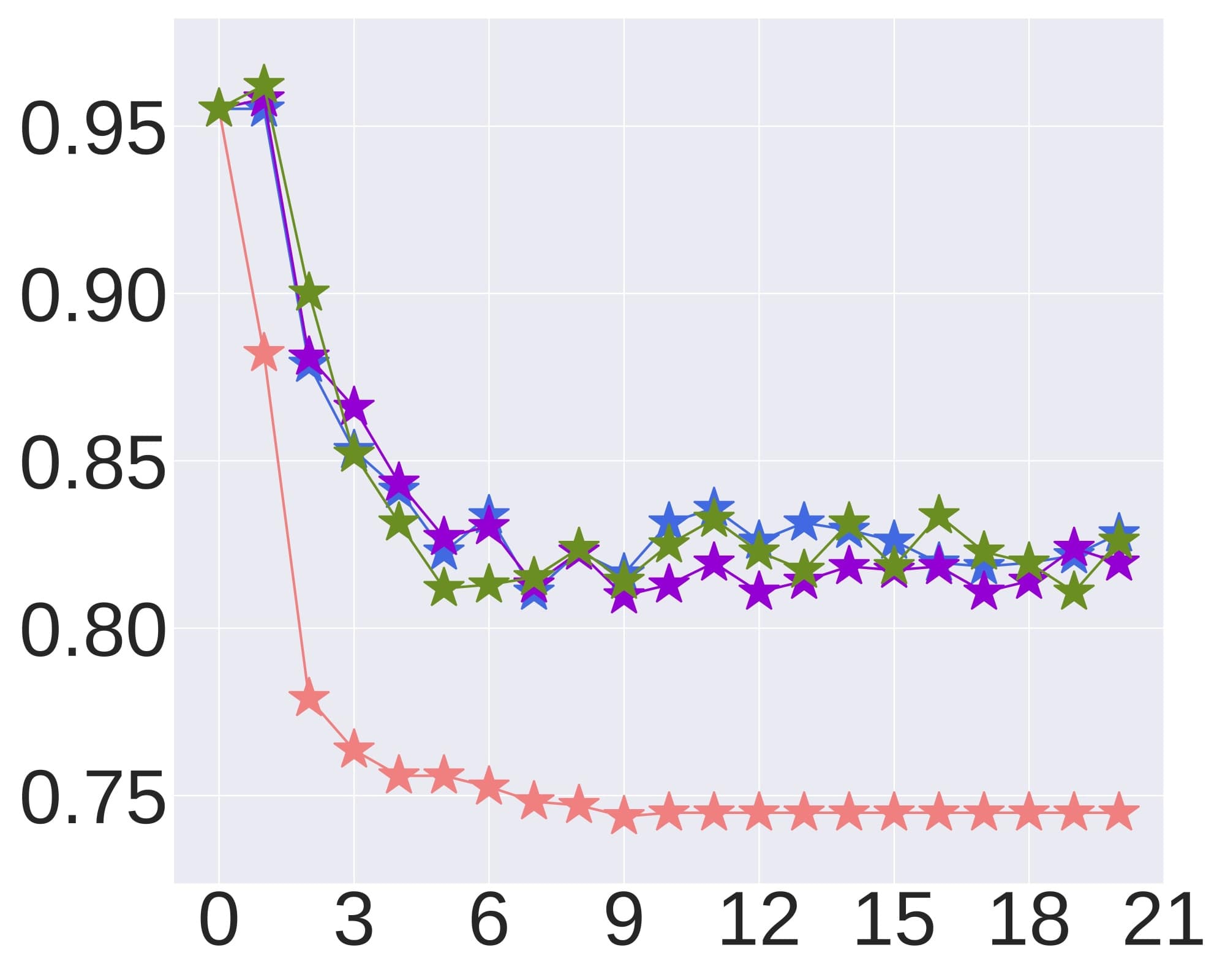}}
 	\centerline{\footnotesize{~~~\# Retrieved examples $N_{\text{rag}}$}}
  \vspace{-0.5em}
 \end{minipage}
 \hspace{+2em}
 \begin{minipage}{0.2\linewidth}
    \centerline{\footnotesize{\quad CommonGen}}
 	\vspace{1pt}
 	\centerline{\includegraphics[width=1.0\textwidth]{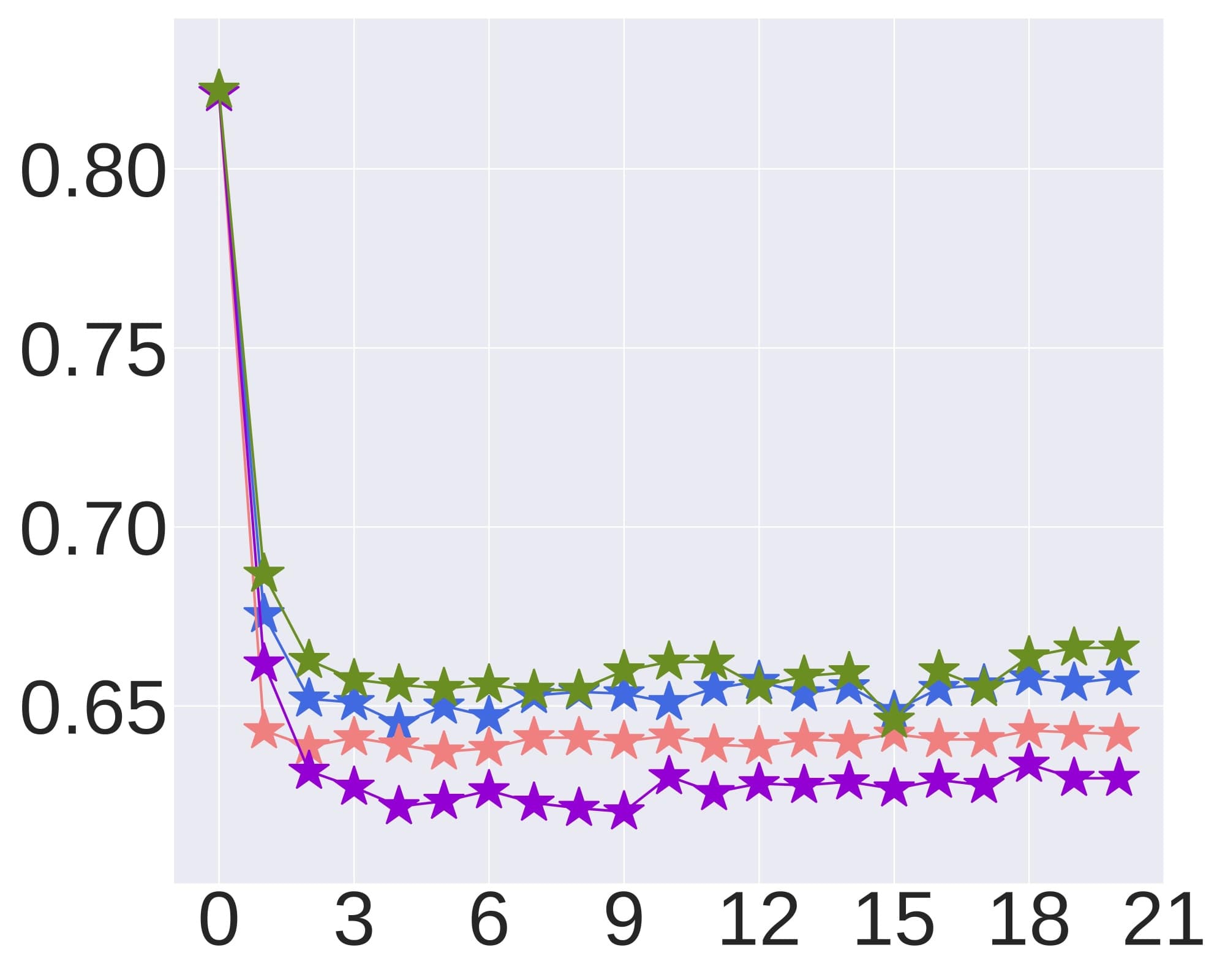}}
 	\centerline{\footnotesize{~~~\# Retrieved examples $N_{\text{rag}}$}}
  \vspace{-0.5em}
 \end{minipage}
 \hspace{+2em}
 \begin{minipage}{0.2\linewidth}
    \centerline{\footnotesize{\quad DART}}
 	\vspace{1pt}
 	\centerline{\includegraphics[width=1.0\textwidth]{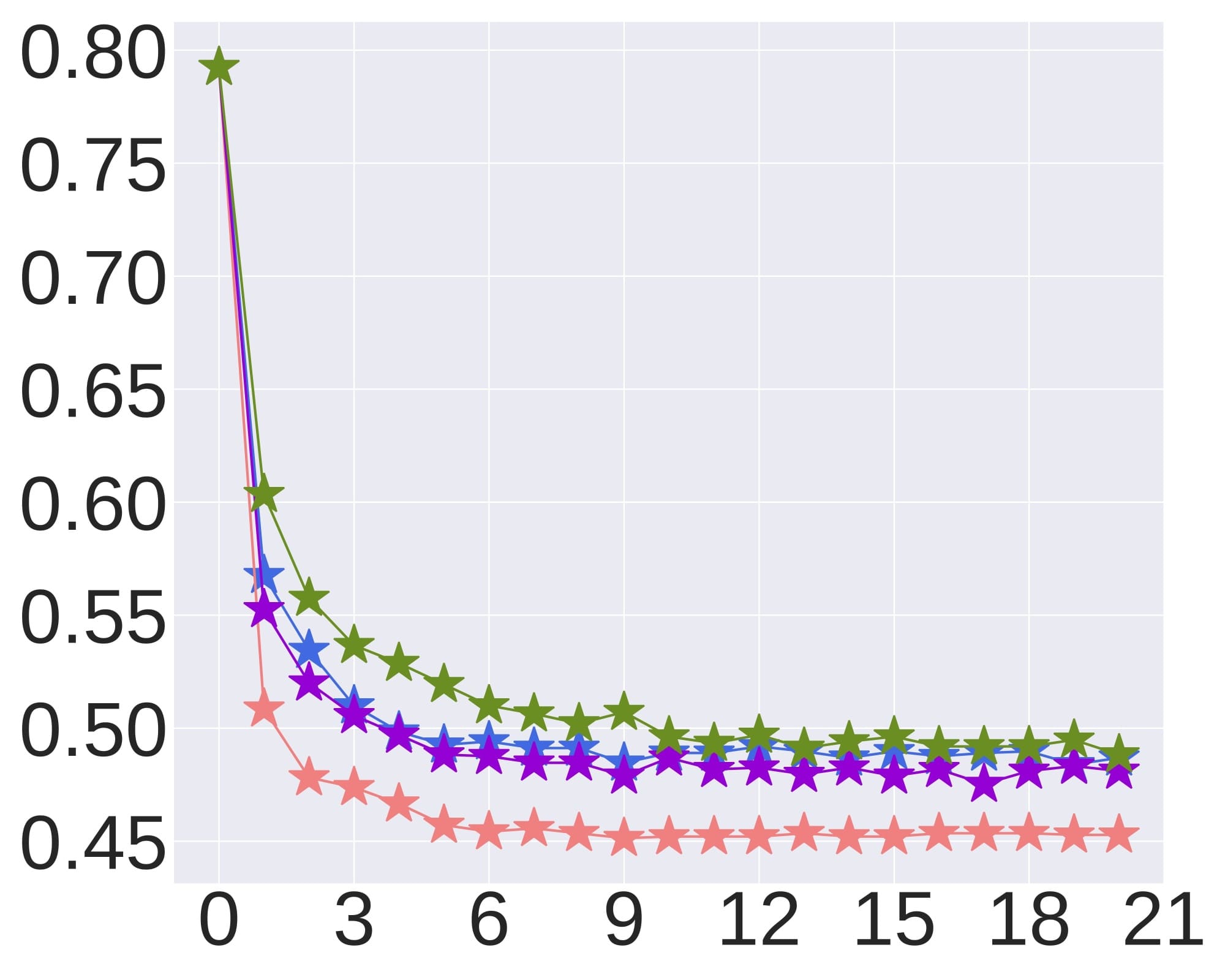}}
 	\centerline{\footnotesize{~~~\# Retrieved examples $N_{\text{rag}}$}}
  \vspace{-0.5em}
 \end{minipage}
 \hspace{+2em}
 \begin{minipage}{0.2\linewidth}
    \centerline{\footnotesize{\quad E2E}}
 	\vspace{1pt}
 	\centerline{\includegraphics[width=1.0\textwidth]{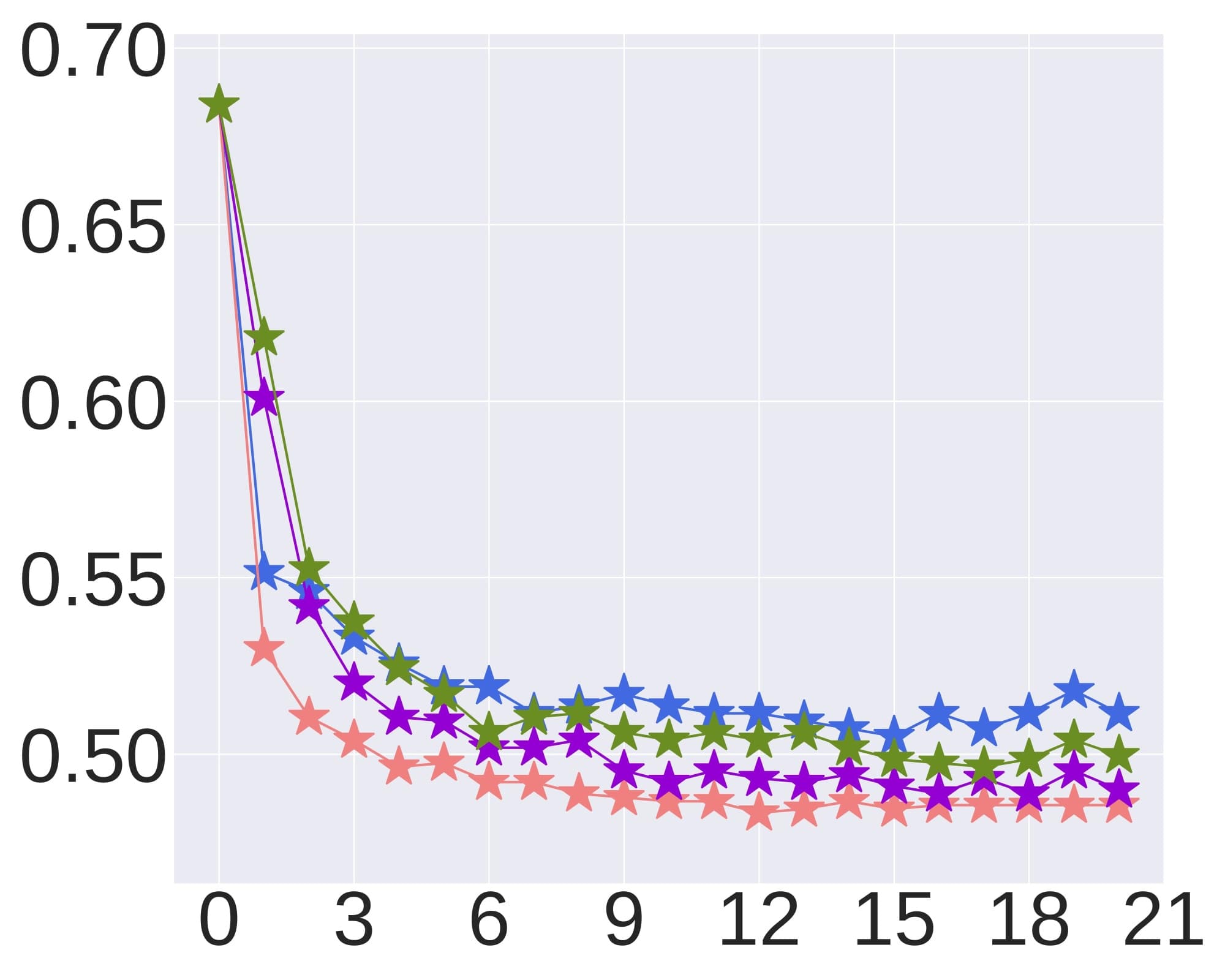}}
 	\centerline{\footnotesize{~~~\# Retrieved examples $N_{\text{rag}}$}}
  \vspace{-0.5em}
 \end{minipage}
}
\subfigure{
\centerline{\includegraphics[width=0.5\textwidth]{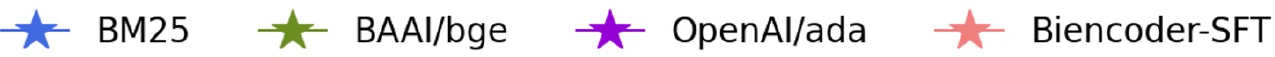}}}
\vspace{-2em}
\caption{\small Conformal generation risk $\eqnsmall{\hat{\alpha}_{\text{rag}}}$ with different $\eqnsmall{N_{\text{rag}}}$ using different retrieval models ($\eqnsmall{\lambda_g=1,\lambda_s=1.0}$). 
{We observe that large $\eqnsmall{N_{\text{rag}}}$ effectively reduces $\eqnsmall{\hat{\alpha}_{\text{rag}}}$ for different models; the trained Biencoder-SFT usually leads to the lowest conformal generation risk.}}
\label{fig:comparison_retrieval}
\vspace{-1em}
\end{figure*}

\begin{figure*}[t]
\subfigure{
    \rotatebox{90}{\hspace{-3.5em} Evaluation Risk}
    \begin{minipage}{0.2\linewidth}
    \centerline{\footnotesize{\quad AESLC}}
 	\vspace{1pt}
\centerline{\includegraphics[width=1.0\textwidth]{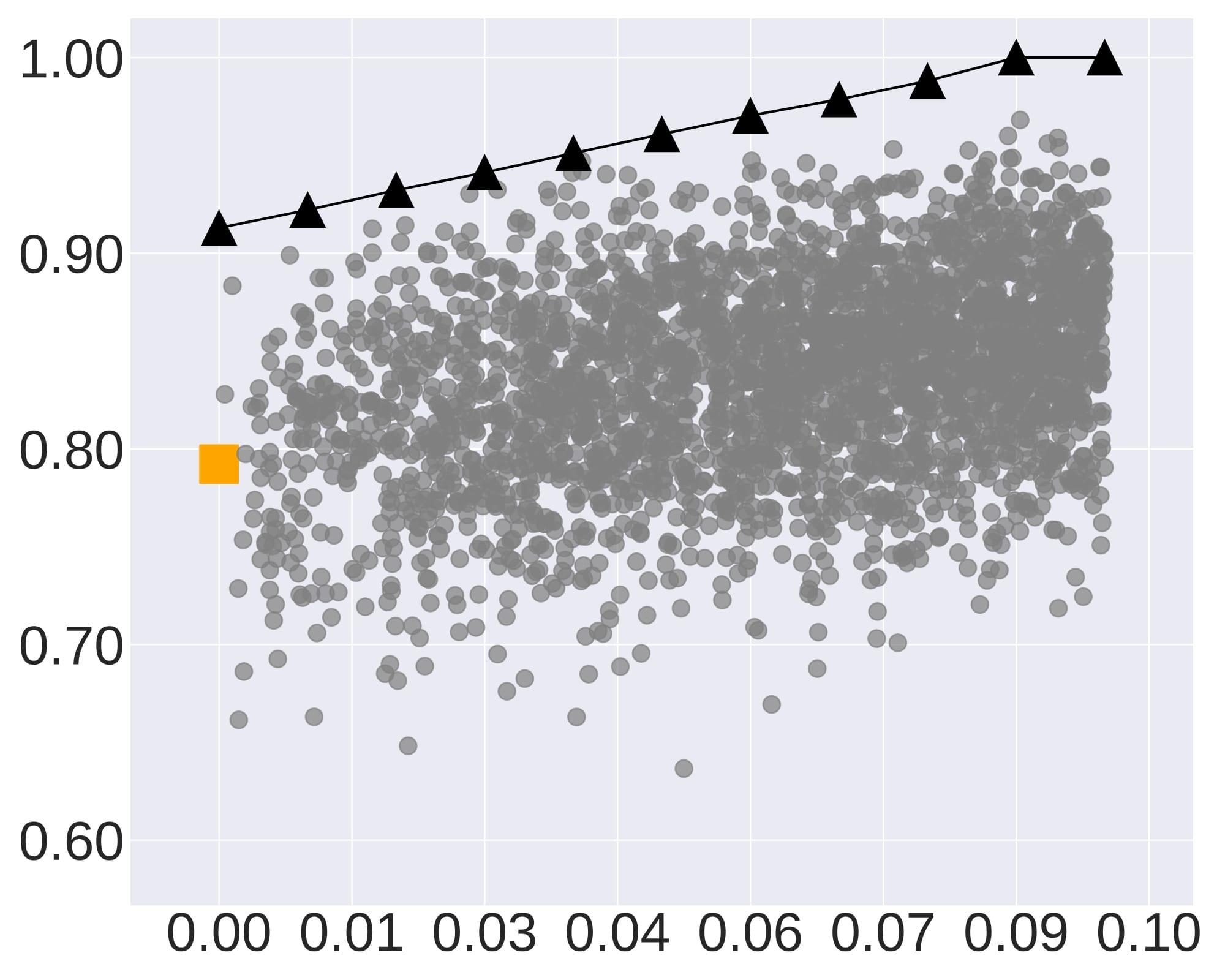}}
 	\centerline{\footnotesize{~~~Hellinger Distance $\rho$}}
\vspace{-0.5em}
 \end{minipage}
 \hspace{+2em}
 \begin{minipage}{0.2\linewidth}
    \centerline{\footnotesize{\quad CommonGen}}
 	\vspace{1pt}
 	\centerline{\includegraphics[width=1.0\textwidth]{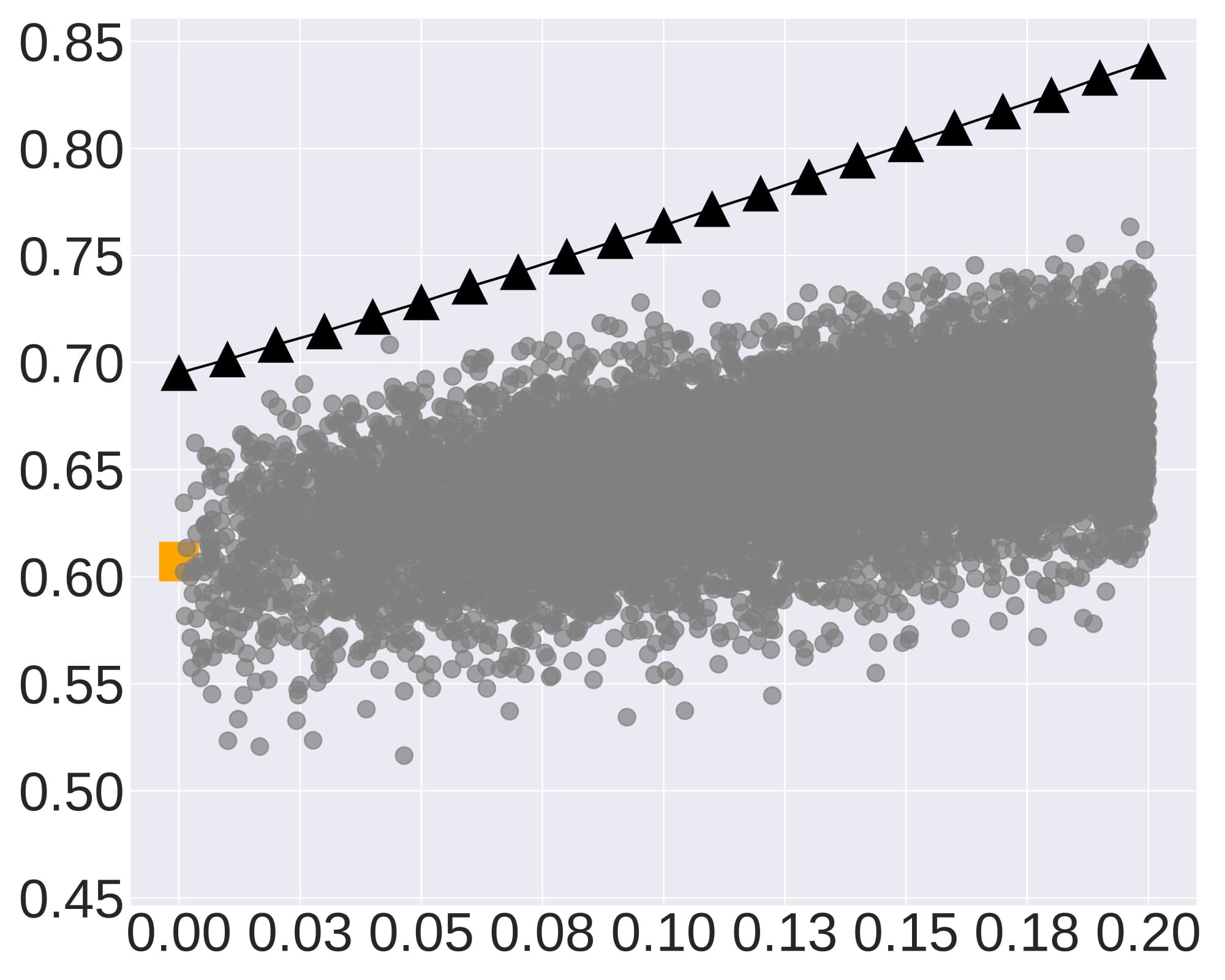}}

 	\centerline{\footnotesize{~~~Hellinger Distance $\rho$}}
\vspace{-0.5em}
 \end{minipage}
 \hspace{+2em}
 \begin{minipage}{0.2\linewidth}
    \centerline{\footnotesize{\quad DART}}
 	\vspace{1pt}
 	\centerline{\includegraphics[width=1.0\textwidth]{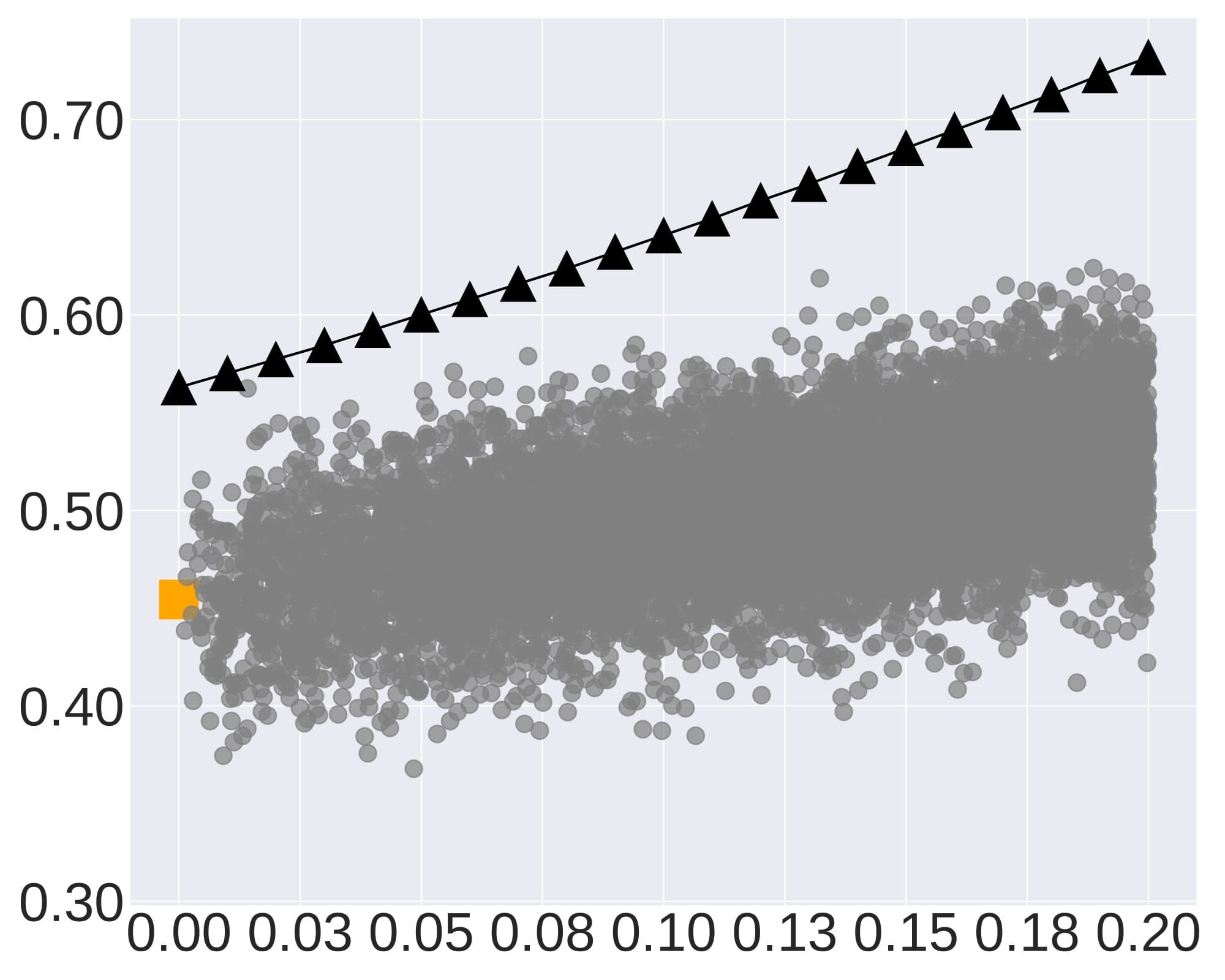}}

 	\centerline{\footnotesize{~~~Hellinger Distance $\rho$}}
\vspace{-0.5em}
 \end{minipage}
 \hspace{+2em}
 \begin{minipage}{0.2\linewidth}
    \centerline{\footnotesize{\quad E2E}}
 	\vspace{1pt}
 	\centerline{\includegraphics[width=1.0\textwidth]{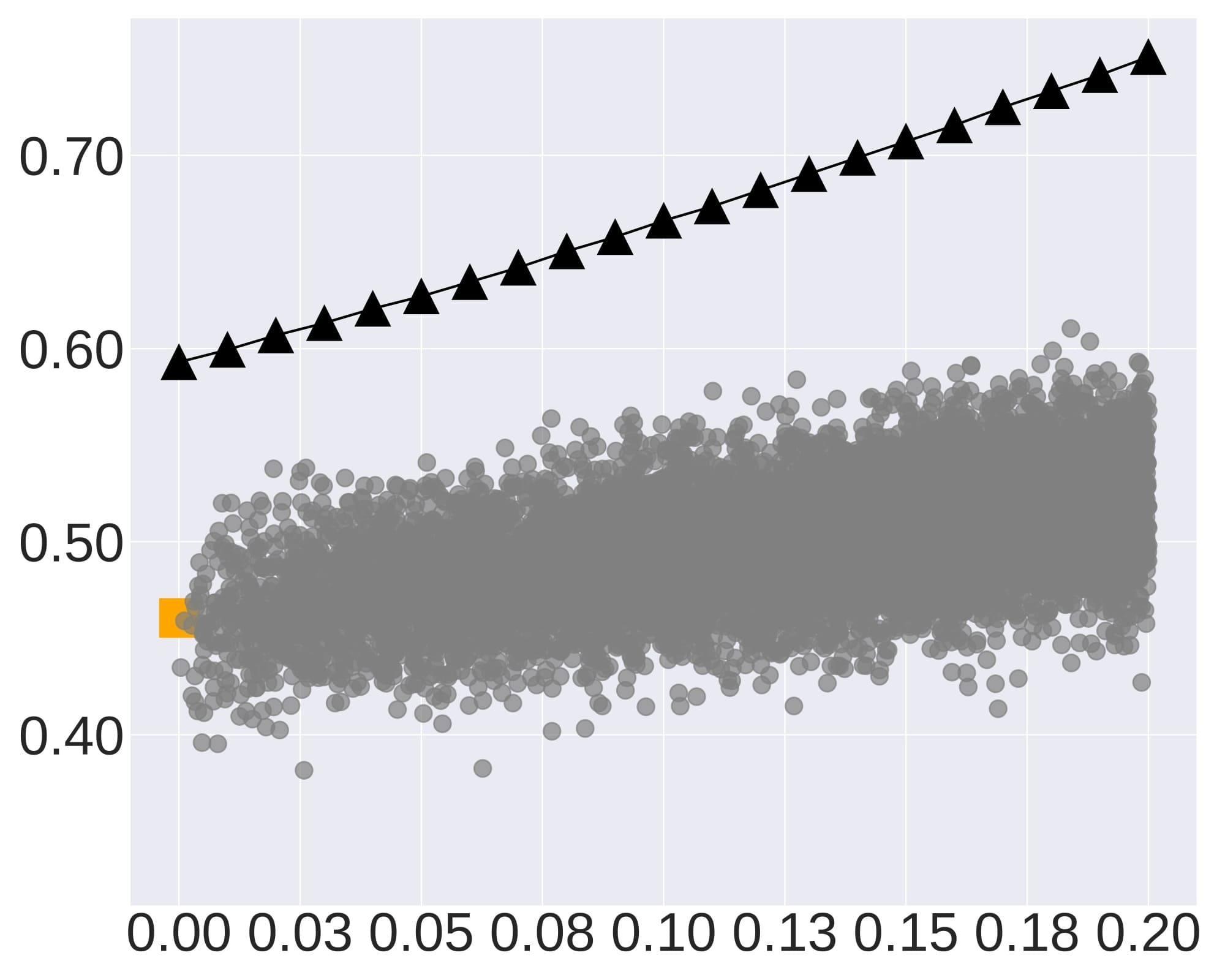}}
\centerline{\footnotesize{~~~Hellinger Distance $\rho$}}
\vspace{-0.5em}
 \end{minipage}
}
\subfigure{
\centerline{\includegraphics[width=0.75\textwidth]{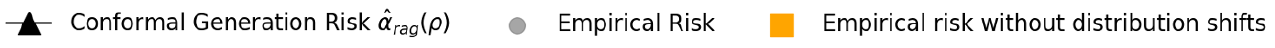}}}
\vspace{-2em}
\caption{\small 
{Conformal generation risk $\eqnsmall{\hat{\alpha}_{\text{rag}}(\rho)}$ and empirical risks based on retrieval model OpenAI/ada under distribution shifts ($\eqnsmall{N_{\text{rag}}=15, \lambda_g=1, \lambda_s=1.0}$).
We observe that our distribution-drift conformal generation risk (\Cref{pro:conf_shf}) is empirically valid and tight.
}}
\label{fig:bound_and_simulation_ada_dist_shft}
\end{figure*}

\begin{figure}[thb]
\subfigure{
    \hspace{+1em}
    \begin{minipage}{0.42\linewidth}
    \centerline{\footnotesize{\quad AESLC}}
 	\vspace{1pt}
\centerline{\includegraphics[width=1.0\textwidth]{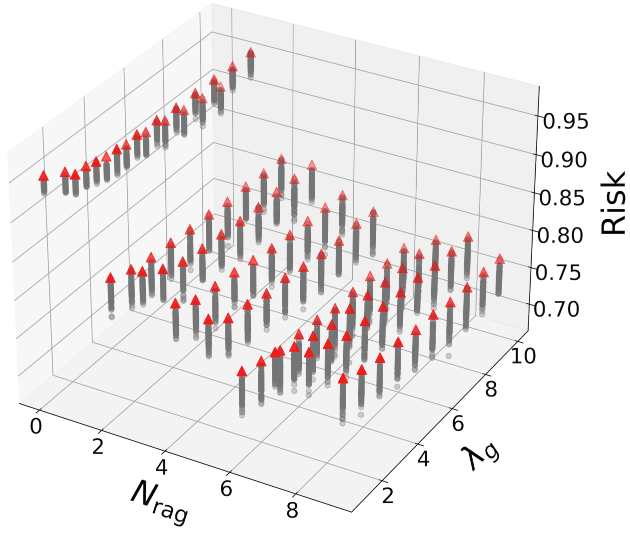}}
 \vspace{-0.5em}
 \end{minipage}
 \hspace{+1em}
 \begin{minipage}{0.42\linewidth}
    \centerline{\footnotesize{\quad CommonGen}}
 	\vspace{1pt}
\centerline{\includegraphics[width=1.0\textwidth]{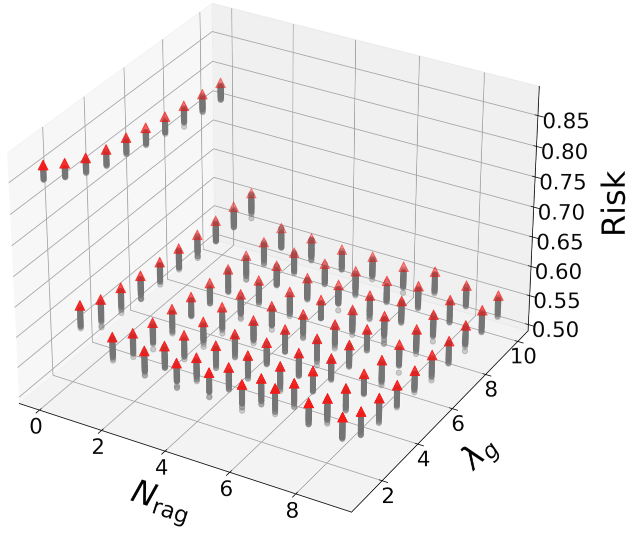}}
 \vspace{-0.5em}
 \end{minipage}
}
\subfigure{
\centerline{\includegraphics[width=0.4\textwidth]{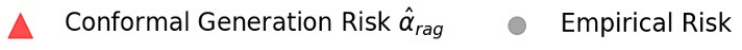}}}
\vspace{-2em}
\caption{\small Conformal generation risk $\eqnsmall{\hat{\alpha}_{\text{rag}}}$ and empirical risks with different $\eqnsmall{\lambda_g}$ and $\eqnsmall{N_{\text{rag}}}$ for OpenAI/ada.}
\label{fig:bound_and_simulation_openai_multi_gen}
\end{figure}
\vspace{-2em}
}
{

\begin{figure*}[t]
\subfigure{
    \rotatebox{90}{\hspace{-3.5em} Evaluation Risk}
    \begin{minipage}{0.24\linewidth}
    \centerline{\footnotesize{\quad AESLC}}
 	\vspace{1pt}
\centerline{\includegraphics[width=1.0\textwidth]{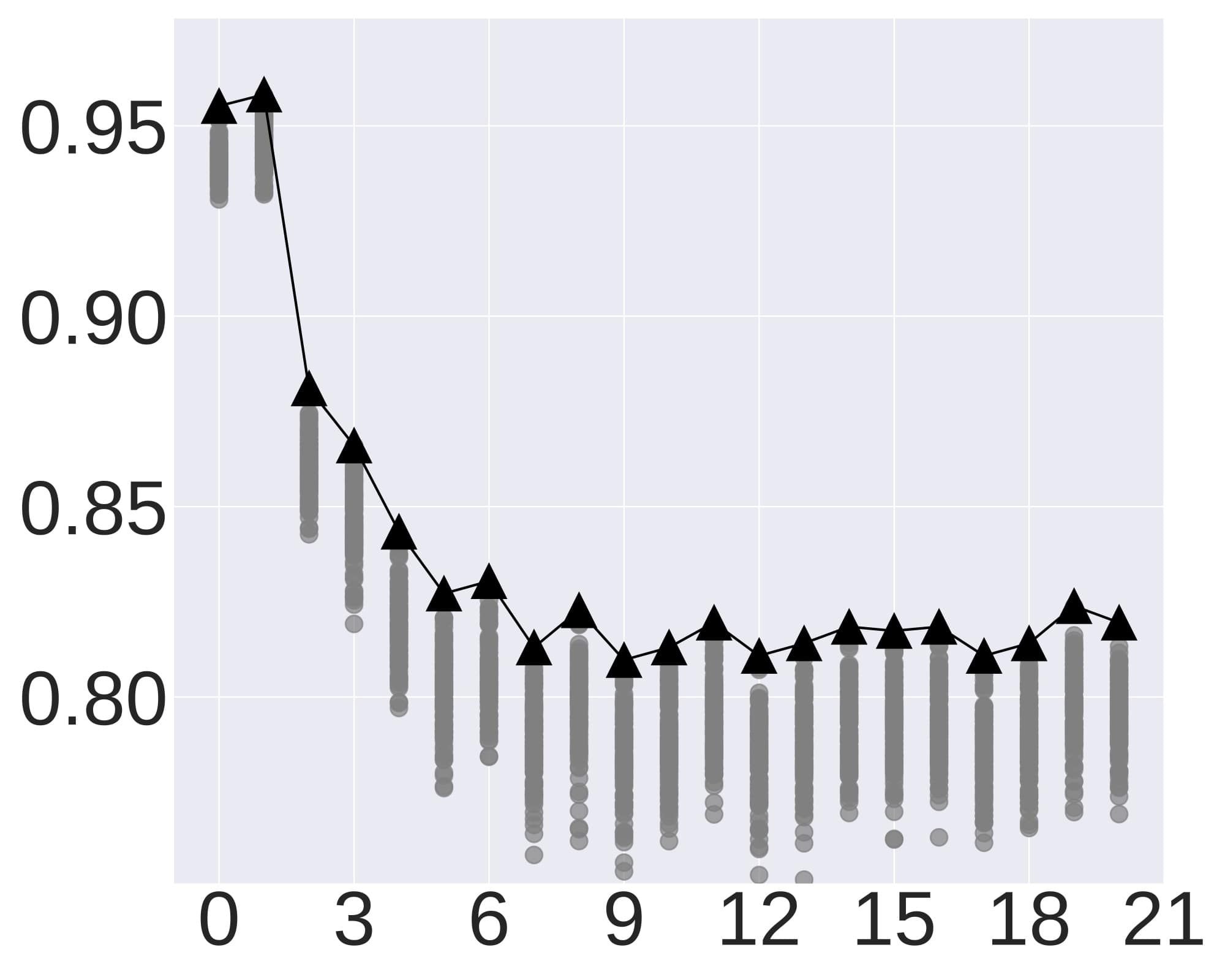}}
 	\centerline{\footnotesize{~~~\# Retrieved examples $N_{\text{rag}}$}}
 \end{minipage}
 \begin{minipage}{0.24\linewidth}
    \centerline{\footnotesize{\quad CommonGen}}
 	\vspace{1pt}
 	\centerline{\includegraphics[width=1.0\textwidth]{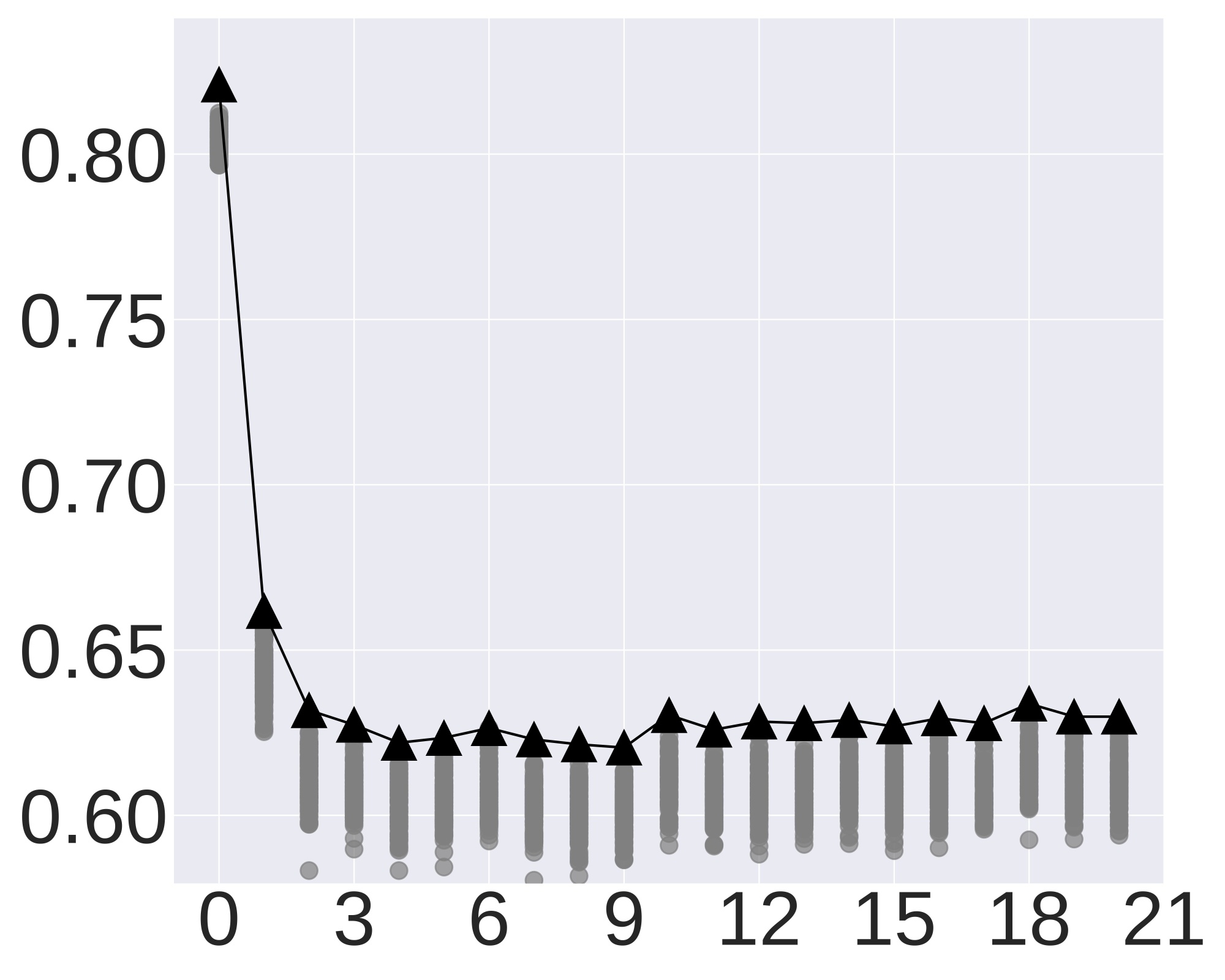}}
 	\centerline{\footnotesize{~~~\# Retrieved examples $N_{\text{rag}}$}}
 \end{minipage}
 \begin{minipage}{0.24\linewidth}
    \centerline{\footnotesize{\quad DART}}
 	\vspace{1pt}
 	\centerline{\includegraphics[width=1.0\textwidth]{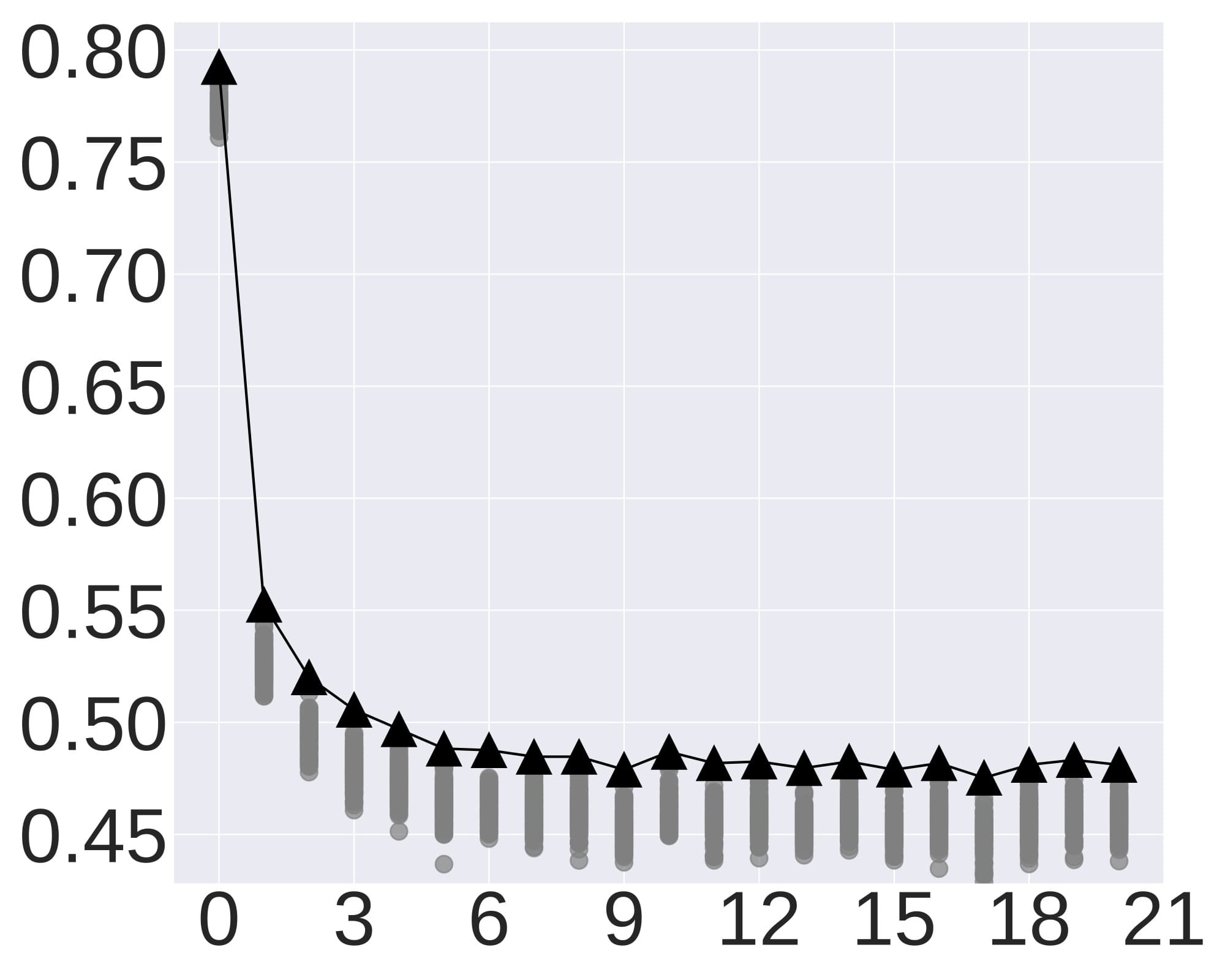}}
 	\centerline{\footnotesize{~~~\# Retrieved examples $N_{\text{rag}}$}}
 \end{minipage}
 \begin{minipage}{0.24\linewidth}
    \centerline{\footnotesize{\quad E2E}}
 	\vspace{1pt}
 	\centerline{\includegraphics[width=1.0\textwidth]{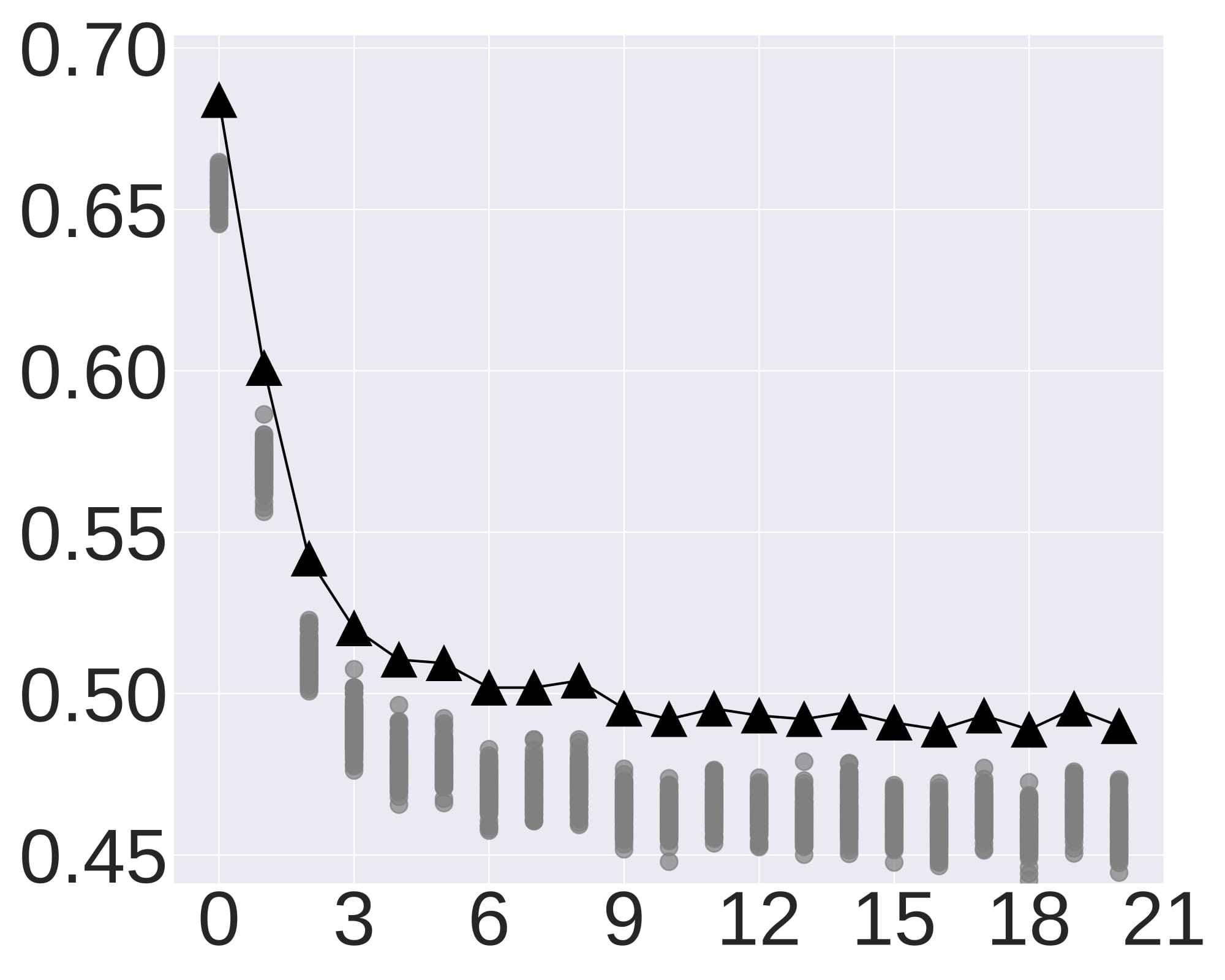}}
 	\centerline{\footnotesize{~~~\# Retrieved examples $N_{\text{rag}}$}}
 \end{minipage}
}
\subfigure{
\centerline{\includegraphics[width=0.5\textwidth]{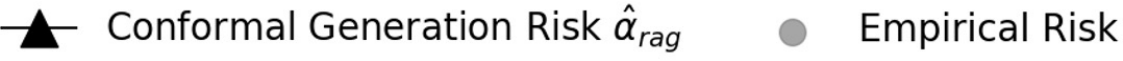}}}
\caption{Conformal generation risk $\eqnsmall{\hat{\alpha}_{\text{rag}}}$ and empirical risk based on retrieval model OpenAI/ada taking different $\eqnsmall{N_{\text{rag}}}$ ($\eqnsmall{\lambda_g=1,\lambda_s=1.0}$). 
{We observe that our conformal generation risk (\Cref{risk_guarantee_1}) is valid and tight; 
larger $\eqnsmall{N_{\text{rag}}}$ reduces risk $\eqnsmall{\hat{\alpha}_{\text{rag}}}$ (empirically validating \Cref{thm:gene_rag})}.
}
\label{fig:bound_and_simulation_ada}
\end{figure*}

\begin{figure*}[t]
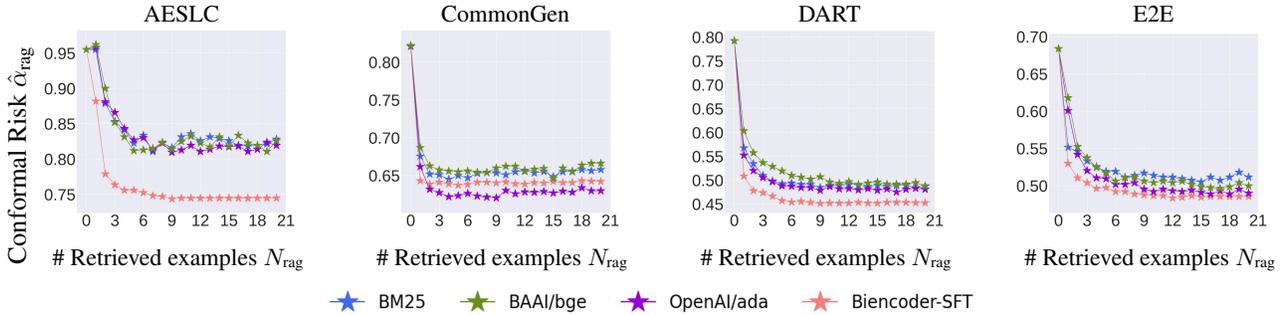

\subfigure{
    \rotatebox{90}{\hspace{-5.0em} Conformal Risk $\hat{\alpha}_{\text{rag}}$}
    \begin{minipage}{0.24\linewidth}
    \centerline{\footnotesize{\quad AESLC}}
 	\vspace{1pt}
\centerline{\includegraphics[width=1.0\textwidth]{figures/aeslc_conformal_risk.jpg}}
 	\centerline{\footnotesize{~~~\# Retrieved examples $N_{\text{rag}}$}}
 \end{minipage}
 \begin{minipage}{0.24\linewidth}
    \centerline{\footnotesize{\quad CommonGen}}
 	\vspace{1pt}
 	\centerline{\includegraphics[width=1.0\textwidth]{figures/common_gen_conformal_risk.jpg}}
 	\centerline{\footnotesize{~~~\# Retrieved examples $N_{\text{rag}}$}}
 \end{minipage}
 \begin{minipage}{0.24\linewidth}
    \centerline{\footnotesize{\quad DART}}
 	\vspace{1pt}
 	\centerline{\includegraphics[width=1.0\textwidth]{figures/dart_conformal_risk.jpg}}
 	\centerline{\footnotesize{~~~\# Retrieved examples $N_{\text{rag}}$}}
 \end{minipage}
 \begin{minipage}{0.24\linewidth}
    \centerline{\footnotesize{\quad E2E}}
 	\vspace{1pt}
 	\centerline{\includegraphics[width=1.0\textwidth]{figures/e2e_nlg_conformal_risk.jpg}}
 	\centerline{\footnotesize{~~~\# Retrieved examples $N_{\text{rag}}$}}
 \end{minipage}
}
\subfigure{
\centerline{\includegraphics[width=0.5\textwidth]{figures/legend2.pdf}}}
\caption{Conformal generation risk $\eqnsmall{\hat{\alpha}_{\text{rag}}}$ with different $\eqnsmall{N_{\text{rag}}}$ using different retrieval models ($\eqnsmall{\lambda_g=1,\lambda_s=1.0}$). 
{We observe that large $\eqnsmall{N_{\text{rag}}}$ effectively reduces $\eqnsmall{\hat{\alpha}_{\text{rag}}}$ for different models; the trained Biencoder-SFT usually leads to the lowest conformal generation risk.}}
\label{fig:comparison_retrieval}
\end{figure*}

\begin{figure*}[t]
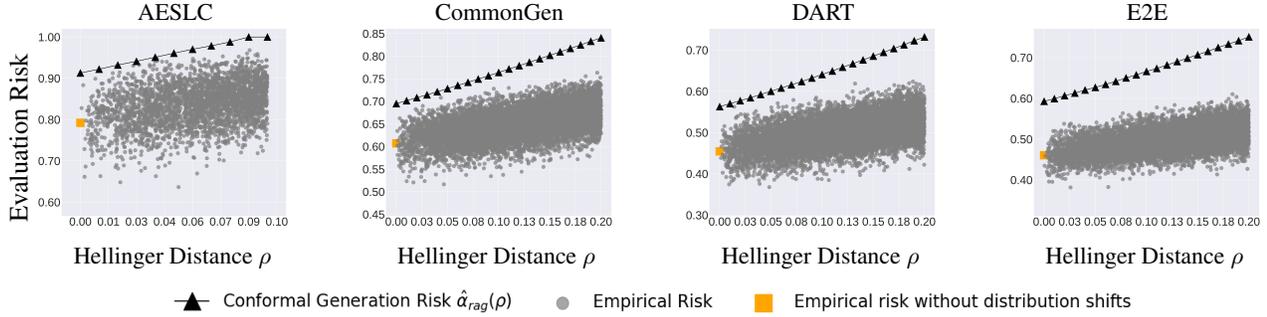

\subfigure{
    \rotatebox{90}{\hspace{-3.5em} Evaluation Risk}
    \begin{minipage}{0.24\linewidth}
    \centerline{\footnotesize{\quad AESLC}}
 	\vspace{1pt}
\centerline{\includegraphics[width=1.0\textwidth]{figures/aeslc_openai_dist_shift_conformal_bound_simulation.jpg}}
 	\centerline{\footnotesize{~~~Hellinger Distance $\rho$}}
 \end{minipage}
 \begin{minipage}{0.24\linewidth}
    \centerline{\footnotesize{\quad CommonGen}}
 	\vspace{1pt}
 	\centerline{\includegraphics[width=1.0\textwidth]{figures/common_gen_openai_dist_shift_conformal_bound_simulation.jpg}}
 	\centerline{\footnotesize{~~~Hellinger Distance $\rho$}}
 \end{minipage}
 \begin{minipage}{0.24\linewidth}
    \centerline{\footnotesize{\quad DART}}
 	\vspace{1pt}
 	\centerline{\includegraphics[width=1.0\textwidth]{figures/dart_openai_dist_shift_conformal_bound_simulation.jpg}}
 	\centerline{\footnotesize{~~~Hellinger Distance $\rho$}}
 \end{minipage}
 \begin{minipage}{0.24\linewidth}
    \centerline{\footnotesize{\quad E2E}}
 	\vspace{1pt}
 	\centerline{\includegraphics[width=1.0\textwidth]{figures/e2e_nlg_openai_dist_shift_conformal_bound_simulation.jpg}}
 	\centerline{\footnotesize{~~~Hellinger Distance $\rho$}}
 \end{minipage}
}
\subfigure{
\centerline{\includegraphics[width=0.85\linewidth]{figures/legend3.pdf}}}
\caption{
{Conformal generation risk $\eqnsmall{\hat{\alpha}_{\text{rag}}(\rho)}$ and empirical risks based on retrieval model OpenAI/ada under distribution shifts ($\eqnsmall{N_{\text{rag}}=15, \lambda_g=1, \lambda_s=1.0}$).
We observe that our distribution-drift conformal generation risk (\Cref{pro:conf_shf}) is empirically valid and tight.
}}
\label{fig:bound_and_simulation_ada_dist_shft}

\end{figure*}

\begin{figure}[t]
\subfigure{
    \hspace{+1em}
    \begin{minipage}{0.4\linewidth}
    \centerline{\footnotesize{\quad AESLC}}
 	\vspace{1pt}
\centerline{\includegraphics[width=1.0\linewidth]{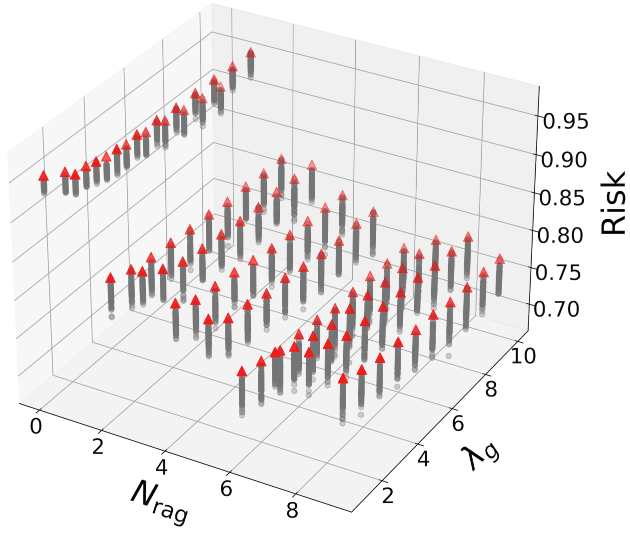}}
 \end{minipage}
 \hspace{+0.1\linewidth}
 \begin{minipage}{0.4\linewidth}
    \centerline{\footnotesize{\quad CommonGen}}
 	\vspace{1pt}
\centerline{\includegraphics[width=1.0\linewidth]{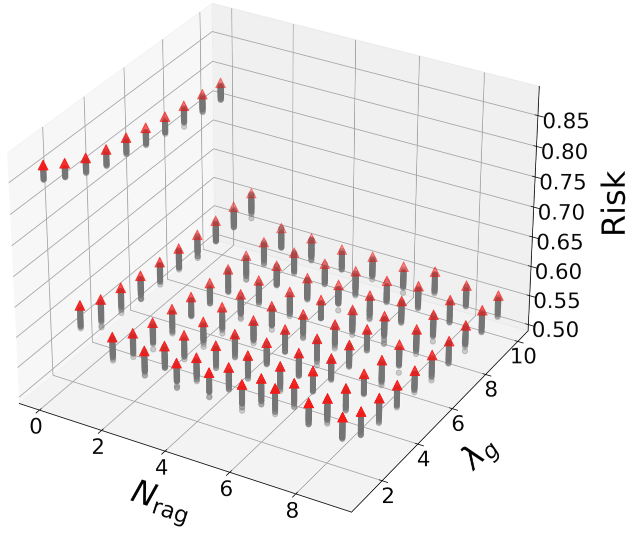}}
 \end{minipage}
}
\subfigure{
\centerline{\includegraphics[width=0.5\textwidth]{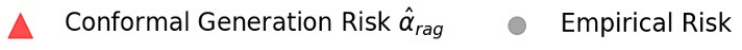}}}
\vspace{-2em}
\caption{Conformal generation risk $\eqnsmall{\hat{\alpha}_{\text{rag}}}$ and empirical risks with different $\eqnsmall{\lambda_g}$ and $\eqnsmall{N_{\text{rag}}}$ for OpenAI/ada.}
\label{fig:bound_and_simulation_openai_multi_gen}
\end{figure}
\vspace{-1em}
}

\vspace{\stspace}
\vspace{+1.5em}
\subsection{{Evaluation of conformal generation risks}}
\label{sec:exp_no_dist}
\vspace{\pspace}
\paragraph{Soundness and tightness of conformal generation risks} 
To achieve generation risk guarantee in \Cref{eq:conformal_guarantee}, \name computes the conformal generation risk using \Cref{eq:conformal_risk_}.
We evaluate the conformal generation risks of RAG models $\eqnsmall{\hat{\alpha}_{\text{rag}}}$ under different numbers of retrieved examples $\eqnsmall{N_{\text{rag}}}$ by calibration statistics on the validation set. 
To validate the soundness and tightness of the conformal generation risk guarantee, we evaluate the empirical risks on randomly sampled test instances. The sampling protocol is detailed in \Cref{alg:simulation_no_shft} in \Cref{app:exp_no_dist}.
We provide the results using OpenAI/ada retrieval model in \Cref{fig:bound_and_simulation_ada} and results for BM25, BAAI/bge, Biencoder-SFT in \Cref{fig:bound_and_simulation_bm25,fig:bound_and_simulation_baai,fig:bound_and_simulation_llmr} in \Cref{app:exp_no_dist}.
The results show that (1) the conformal generation risks $\eqnsmall{\hat{\alpha}_{\text{rag}}}$ upper bound the empirical risks of the sampled test instances, (2) for some test instances, the empirical risks nearly reach the conformal generation risk, demonstrating the soundness and tightness of our generation risk guarantees, (3)
the conformal generation risk decreases as the number of retrieved examples $\eqnsmall{N_{\text{rag}}}$ increases, which shows the effectiveness of RAG models and aligns with our theoretical analysis in \Cref{thm:gene_rag}.
\vspace{\stspace}
\vspace{-0.7em}
\paragraph{Comparisons of different SOTA retrieval models} We compare the conformal generation risks for token-level BM25 scores, and SOTA embedding models BAAI/bge, OpenAI/ada, and Biencoder-SFT. The results in \Cref{fig:comparison_retrieval} show that RAG achieves lower conformal generation risks {than LLM without retrieval (i.e., $\eqnsmall{N_{\text{rag}}=0}$)} for different retrieval models. 
Biencoder-SFT, trained with in-domain data samples, leads to lower conformal generation risk in general compared with other retrieval models. OpenAI/ada, which is known of high quality and trained on large open corpus, also demonstrates low conformal generation risks.

\vspace{\stspace}
\subsection{Conformal generation risk under distribution shifts}
\label{sec:exp_dist}
\vspace{\pspace}
\paragraph{Soundness and tightness of conformal generation risk under distribution shifts}
In practice, user input text may deviate from the calibration distribution.
In \Cref{pro:conf_shf}, we provide the first conformal generation risk under distribution shifts for general bounded risk functions. 
We evaluate the conformal generation risk $\eqnsmall{\hat{\alpha}(\rho)}$ in \Cref{pro:eq1}.
To empirically verify the soundness, we create test sets with covariate shifts by varying sample weights. 
The Hellinger distance is computed using original and shifted sample weights, with details in \Cref{alg:simulation} in \Cref{app:exp_dist}.
We compare the conformal generation risk and empirical risks with $\eqnsmall{N_{\text{rag}}=15}$ using OpenAI/ada in \Cref{fig:bound_and_simulation_ada_dist_shft}, and using BM25, BAAI/bge and Biencoder-SFT in \Cref{fig:bound_and_simulation_bm25_dist_shft,fig:bound_and_simulation_baai_dist_shft,fig:bound_and_simulation_llmr_dist_shft} in \Cref{app:exp_dist}.
The results show that (1) our conformal generation risks under distribution shifts are sound and tight across various models, and (2) conformal generation risks increase linearly with Hellinger distance $\eqnsmall{\rho}$, remaining non-trivial up to $\eqnsmall{\rho = 0.2}$.
\vspace{\pspace}
\vspace{-0.7em}
\paragraph{Comparison of SOTA retrieval models under distribution shifts}
We compare conformal generation risks for different retrieval models under distribution shifts in \Cref{fig:comparison_retrieval_dist_shft} in \Cref{app:exp_dist}. 
All models show a linear rise in risk with increasing Hellinger distance, with BiEncoder-SFT and OpenAI/ada showing lower risks at varying distances.

\ifbool{IsTwoColumn}
{
\begin{figure}[t]
\subfigure{
    \hspace{+1em}
    \begin{minipage}{0.42\linewidth}
    \centerline{\footnotesize{\quad AESLC}}
 	\vspace{1pt}
\centerline{\includegraphics[width=1.0\textwidth]{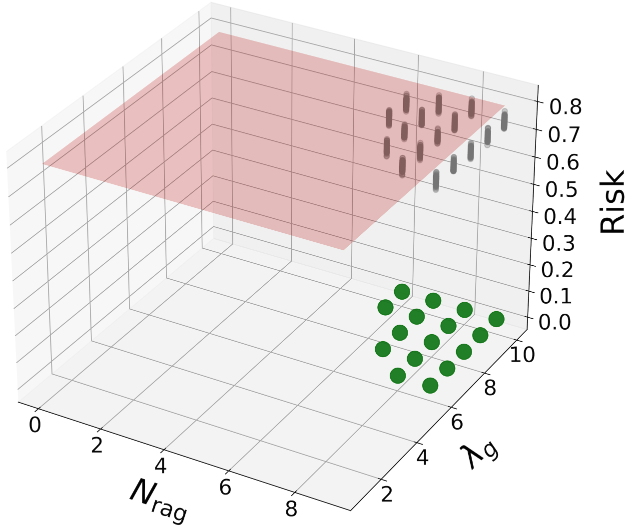}}
  
 \vspace{-0.5em}
 \end{minipage}
 \hspace{+1em}
 \begin{minipage}{0.42\linewidth}
    \centerline{\footnotesize{\quad CommonGen}}
 	\vspace{1pt}
\centerline{\includegraphics[width=1.0\textwidth]{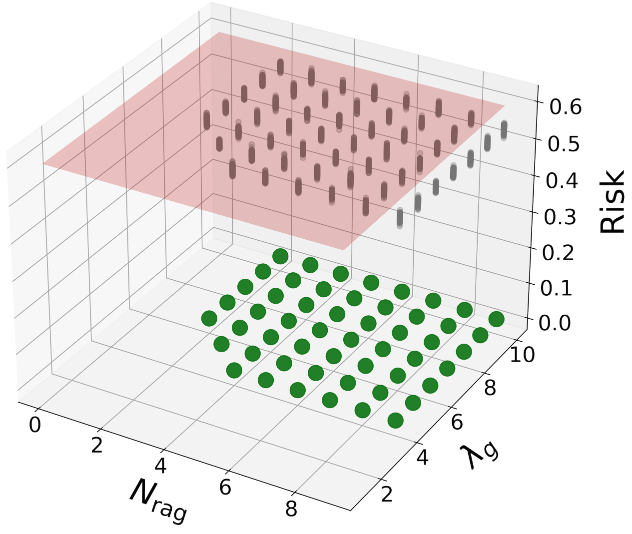}}

 \vspace{-0.5em}
 \end{minipage}
}

\subfigure{
\centerline{\hspace{-0.6em}\includegraphics[width=0.48\textwidth]{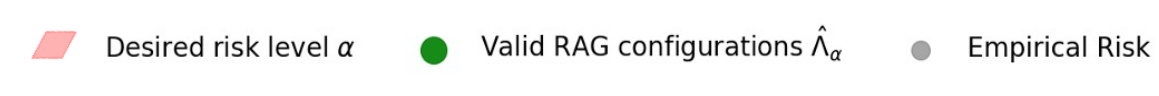}}}
\vspace{-2em}
\caption{\small Valid configurations $\eqnsmall{\hat{\Lambda}_\alpha}$ given a desired risk level $\eqnsmall{\alpha}$ and the empirical risks with different $\eqnsmall{\lambda_g}$ and $\eqnsmall{N_{\text{rag}}}$ for OpenAI/ada.
}

\label{fig:bound_and_simulation_openai_multi_gen_2}
\vspace{-1em}
\end{figure}
}
{
\begin{figure}[t]
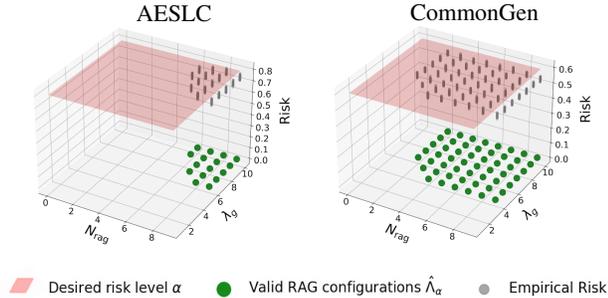

\subfigure{
    \hspace{+1em}
    \begin{minipage}{0.42\linewidth}
    \centerline{\footnotesize{\quad AESLC}}
 	\vspace{1pt}
\centerline{\includegraphics[width=1.0\textwidth]{figures/aeslc_multi.png}}

 \vspace{-0.5em}
 \end{minipage}
 \hspace{+1em}
 \begin{minipage}{0.42\linewidth}
    \centerline{\footnotesize{\quad CommonGen}}
 	\vspace{1pt}
\centerline{\includegraphics[width=1.0\textwidth]{figures/common_multi.png}}

 \vspace{-0.5em}
 \end{minipage}
}
\subfigure{
\centerline{\hspace{-0.6em}\includegraphics[width=0.7\textwidth]{figures/legend6.pdf}}}
\vspace{-2em}
\caption{Valid configurations $\eqnsmall{\hat{\Lambda}_\alpha}$ given a desired risk level $\eqnsmall{\alpha}$ and the empirical risks with different $\eqnsmall{\lambda_g}$ and $\eqnsmall{N_{\text{rag}}}$ for OpenAI/ada.
}
\label{fig:bound_and_simulation_openai_multi_gen_2}
\end{figure}
}

\vspace{\pspace}

\subsection{\name with multi-dimensional RAG configurations}
\label{sec:exp_multi}
\vspace{\pspace}

So far, we demonstrate the effectiveness of retrieved in-context examples \merveedit{quantified} by $\eqnsmall{N_{\text{rag}}}$. 
To further improve the conformal generation risk, we can adjust the RAG configurations, such as the number of generations $\eqnsmall{\lambda_g}$ and the diversity-controlling similarity threshold $\eqnsmall{\lambda_s}$. 
We follow {RAG generation protocol} in \Cref{alg:gen_pro} and define the risk function as the minimal risk among $\lambda_g$ candidate generations. 
Our tests on AESLC, CommonGen, DART, and E2E datasets (see \Cref{fig:bound_and_simulation_openai_multi_gen} and \Cref{fig:bound_and_simulation_openai_multi_gen_2_dart} in \Cref{app:exp_multi}) show that the multi-dimensional RAG configurations maintain sound and tight conformal generation risks.
Notably, a higher $\eqnsmall{N_{\text{rag}}}$ reduces generation risks more effectively than adjusting $\eqnsmall{\lambda_g}$.


\vspace{\pspace}

\subsection{Valid configurations given desired risk levels}
\label{sec:exp_multi_2}
\vspace{\pspace}

In risk guarantee \textbf{\underline{(2)}} outlined in \Cref{sec:crc1}, given a desired risk level $\eqnsmall{\alpha}$, \name computes a valid RAG configuration set $\eqnsmall{\hat{\Lambda}_\alpha}$, such that configurations within this set will lead to generation risks below $\eqnsmall{\alpha}$. 
We apply the {Bonferroni correction in \Cref{risk_guarantee_2}} for rigorous family-wise error rate control and assess empirical risks on random test sets with the identified valid configurations. 
We provide the results on AESLC and CommenGen in \Cref{fig:bound_and_simulation_openai_multi_gen_2} and results on DART and ECE in \Cref{fig:bound_and_simulation_openai_multi_gen_2_dart} in \Cref{app:exp_multi_2}.
These results validate our certification, as the empirical risks of generated configurations $\eqnsmall{\hat{\Lambda}_\alpha}$ are consistently below the given conformal generation risk $\eqnsmall{\alpha}$.
The results also show that a high number of retrieved examples $\eqnsmall{N_{\text{rag}}}$ and a larger generation set size $\eqnsmall{\lambda_g}$ contribute to reducing conformal generation risk.
The impact of the diversity threshold $\eqnsmall{\lambda_s}$ is explored in \Cref{app:exp_multi_2}.

In \Cref{subsec:llm}, we provide comparisons of the conformal generation risks for different LLMs as inference models.
In \Cref{subsec:qualitative_example}, we also provide a quantitative example to show how the constrained generation protocol benefits in reducing LLM hallucination risks.

\vspace{\pspace}
\vspace{-0.2em}
\section{Conclusion}
\label{sec:conclude}
\vspace{\pspace}
\vspace{-0.4em}
In this paper, we propose \name to provide conformal generation risk guarantees for RAG models.
\name certifies \textbf{\underline{(1)}} a conformal generation risk for a given RAG configuration, and \textbf{\underline{(2)}} a valid configuration set for a given desired risk level.
We theoretically show that RAG reduces conformal generation risks of a single LLM.
We empirically validate the soundness and tightness of our risk guarantees.

\vspace{-0.2em}
\section*{Acknowledgements}
\vspace{-0.2em}
This work is partially supported by the National Science Foundation under grant No. 1910100, No. 2046726, No. 2229876, DARPA GARD, the National Aeronautics and Space Administration (NASA) under grant No. 80NSSC20M0229, the Alfred P. Sloan Fellowship, the Amazon research award, and the eBay research award.

\vspace{-0.2em}
\section*{Impact Statement}
\vspace{-0.2em}
The \name framework is able to safeguard LLMs' practical deployment and applications against ethical and societal concerns. 
Existing research shows that the responses of LLMs can be biased towards some demographic groups and not be aligned with human ethics.
With \name, we can define a bias/ethics risk function and control the generation risk below a desired level.
The risk guarantee provided by \name enhances the use of LLMs, addressing societal issues and regulatory infringements. We do not expect any negative societal consequences for our work.


\bibliography{references}
\bibliographystyle{icml2024}

\newpage
\appendix
\onecolumn

\section*{Appendices} 
\DoToC

\clearpage

\section{Discussions of limitations and future work}
\label{app:discussion}

\paragraph{Limitations} One potential challenge of applying \name practically could be the collection of calibration data. In practice, the user input texts are sampled from a time-series data distribution. Therefore, accessing in-distribution calibration samples requires collecting real-time query samples, which could pose the challenge of computational resources and system latency.
Another potential limitation may lie in the probability of the guarantee. Since \name can only provide a high-confidence risk bound via conformal risk analysis, generations with excessive risks can exist.
Therefore, we may need more calibration samples to counter for a higher confidence level, and thus mitigate the appearance of outliers to a large extent.
\merveedit{Also, although the analysis in \name shows the benefits of a large external knowledge base to a low conformal generation risk, the large knowledge base may induce a larger time complexity of KNN searching and space complexity of storing the examples, leading to a trade-off between the generalization/utility and inference efficiency. }

\paragraph{Future work} One interesting future work is to provide conformal risk analysis for time-series data. Conformal prediction for time series \cite{zaffran2022adaptive,xu2021conformal,stankeviciute2021conformal} adaptively adjusts the prediction coverage for sequential data for the regression and classification task. However, the adaptive risk calibration for time series is unexplored but important to practical deployments. Therefore, conformal risk analysis for time series can further motivate the application of conformal risk analysis for LLMs.

\paragraph{Further discussions on calibration data collection.} In principle, if test and calibration instances are from the same distribution, randomly sampling from this distribution with a sufficient sample size $N_{\text{cal}}$ already provides competitive generation risk guarantees (Proposition 1 and Theorem 1). Otherwise, if sampling from the test distribution is impractical, we should sample instances from a proposal distribution with a small distribution distance ($\rho$) to the test distribution and sample variance ($\hat{V}$) so the distribution of the calibration set mimics that of test data (Theorem 2).
We can further use the following techniques for calibration data selection: (1) rejection sampling: drawing samples from a proposal distribution and then rejecting some of these samples based on the known criterion of the test distribution, (2) importance sampling: adjusting the sample weights to closely match the target distribution, and (3) variance reduction such as input normalization. To exemplify, consider a composite domain with medical support, wiki question answering, and service assistance fields, where only a broad proposal distribution is available. We can leverage the strategies mentioned above as follows. We can (1) reject out-of-scope samples, (2) perform importance sampling by adjusting the sample weights based on the proposal distribution and the test distribution, and (3) normalize samples to minimize the distribution gap and variance, for instance, through a unified prompt reformulation.
To improve the probability of the risk guarantee given fixed sample sizes, one can seek advanced concentration analysis with additional constraints on data distribution, which may lead to tighter risk bounds in practice.

\section{Additional related work}
\label{app:related_work}

\textbf{Retrieval augmented generation} (RAG) is a framework for improving the generation quality of LLMs via retrieving relevant information from the external knowledge base and grounding the model on the information for conditional generations. Biencoder retrieval methods \cite{lewis2020retrieval,karpukhin2020dense,xiong2020approximate} leverage two encoders to map the query text and candidate texts into the embedding space and retrieve candidate texts with high embedding similarity to the query text embedding. End-to-end retrieval methods \cite{tay2022transformer,wang2022neural,kishore2023incdsi} train a model to map the query text to the id of relevant candidate documents directly.
Another line of work \cite{luo2023augmented,gou2023critic} leverages external tools such as LLMs to retrieve relevant documents via prompting design.
Although RAG demonstrates impressive capacities, the theoretical analysis of retrieval models for LLM generations is limited.
\citeauthor{basu2023statistical} analyze the retrieval model of constrained function class from a statistical perspective, but the results cannot be generalized to the self-attention transformers. 
In this work, we provide the first analysis of how RAG enhances the generation quality and mitigate generation risks of self-attention transformers.

\textbf{Conformal prediction} is a statistical tool to construct the prediction set with guaranteed prediction coverage \citep{vovk1999machine,vovk2005algorithmic,lei2013distribution,yang2021finite,kang2023certifiably,kang2024colep}, assuming that the data is exchangeable. 
However, conformal prediction can only provide guarantees for the regression and classification tasks and is not directly applicable to the generation tasks, which are more relevant for LLMs.
Conformal risk controlling methods  \cite{bates2021distribution,angelopoulos2021learn,angelopoulos2022conformal,quach2023conformal} provide a high-confidence risk guarantee with the data exchangeability assumption for any black-box risk functions.
We can define a specific risk function for a RAG model and certify a risk upper bound of generations based on statistics on in-distribution calibration set.
However, the risk guarantee is violated under distribution shifts at test time.
\citeauthor{angelopoulos2022conformal, farinhas2023nonexchangeable} offer a valid conformal risk for monotonic risk functions under distribution shifts, but the monotonicity assumption may not always hold in practice. In this work, we introduce the first conformal risk bound for general bounded risk functions under test-time distribution shifts.

\section{Conformal generation risks for RAG models}
\label{app:conf_gen}

\subsection{Constrained generation protocol for RAG models}
\label{app:cons_gen_rag}

To safeguard diverse foundation model-based applications \cite{chen2024unified,chen2024decix,jiang2024hummer,li2023compressing,chen2023vlp,zhang2024mm}, we typically leverage RAG to enhance the trustworthiness of generations \cite{wang2023decodingtrust,kang2024diffattack}.
RAG models \cite{wang2023learning,rubin2021learning,huang2023raven} combine a retrieval model and a generation LM. The retrieval model retrieves $\eqnsmall{N_{\text{rag}}}$ relevant examples to the query from an external knowledge base, and the LM learns in-context from these examples. The knowledge base contains $\eqnsmall{N_{\text{ext}}}$ samples in $\eqnsmall{\hat{\gD}_{\text{ext}}=\{(X_i,Y_i)\}_{i\in[N_{\text{ext}}]}}$. The retrieval model uses an encoder to map instances into an embedding space, and then identifies the relevant examples to the query $\eqnsmall{X_{\text{test}}}$ based on similarity. This similarity, defined by $\eqnsmall{s_{\theta_r}(\cdot,\cdot): \gX \times \gX \mapsto \sR}$ and parameterized by $\eqnsmall{\theta_r}$, is used to find the nearest examples using KNN search in the embedding space.

We arrange the retrieved $\eqnsmall{N_{\text{rag}}}$ in-context examples and the test example $\eqnsmall{X_{\text{test}}}$ into augmented input text $\eqnsmall{X^{\text{(rag)}}}$ using a template.
We then sample the generation from $\eqnsmall{p_{\theta_l}(\cdot|X^{\text{(rag)}})}$ repeatedly until $\eqnsmall{\lambda_g}$ generations are collected. To control the diversity of generations, we reject those with a similarity higher than a threshold $\eqnsmall{\lambda_s}$ to the previous generations. 
In essence, the constrained generation protocol is controlled by configuration $\eqnsmall{\vlambda=[N_{\text{rag}}, \lambda_g, \lambda_s]}$ and output a generation set $\eqnsmall{T_{\vlambda,p_{\theta_l}}(x)}$ 
based on the configuration $\eqnsmall{\vlambda}$ and input $\eqnsmall{x}$. We refer to \Cref{alg:gen_pro} for the pseudocode of the protocol.

\begin{algorithm}[t]
    \caption{Constrained generation protocol for RAG}
    \label{alg:gen_pro}
    \begin{algorithmic}[1]
        \STATE {\bfseries Input:} input prompt $X_{\text{test}}$, LM $p_{\theta_l}(y|x)$, generation set size $\lambda_s$, retrieved example size $N_{\text{rag}}$, 
        external knowledge base $\hat{\gD}_{\text{ext}}$,
        similarity measurement function $s_{\theta_{r}}(\cdot,\cdot)$ with embedding model parameterized by $\theta_{r}$, 
        generation similarity threshold $\lambda_g$, 
        parameter configuration $\vlambda=[N_{\text{rag}}, \lambda_g, \lambda_s]$
        \STATE {\bfseries Output:} Generation set $\gG_\vlambda(X_{\text{test}})$
        \STATE $\gG_\vlambda(X_{\text{test}}) \gets \Phi$
        \STATE $\gZ \gets \text{KNN}(X_{\text{test}},N_{\text{rag}};\hat{\gD}_{\text{ext}},s_{\theta_r})$ \COMMENT{Retrieve $N_{\text{rag}}$ examples from $\hat{\gD}_{\text{ext}}$ via KNN search with similarity measurement $s_{\theta_r}(\cdot,\cdot)$}
        \STATE $X^{\text{(rag)}} \gets \text{Template}(X_{\text{test}}, \gZ )$ \COMMENT{Augmented prompt with $X_{\text{test}}$ and retrieved examples $\gZ$ with a template}
        \WHILE{$ |\gG_\vlambda(X_{\text{test}})| < \lambda_s$}
        \STATE $y \sim p_{\theta_l}(\cdot | X^{\text{(rag)}})$
        \WHILE{$\exists g \in \gG_\vlambda, s_{\theta_r}(y,g)>\lambda_g$}
            \STATE $y \sim p_{\theta_l}(\cdot | X^{\text{(rag)}})$ \COMMENT{Reject sampling}
        \ENDWHILE
        \STATE $\gG_\vlambda(X_{\text{test}}) = \gG_\vlambda(X_{\text{test}}) \cup \left\{  y \right\}$
        \ENDWHILE
        \STATE \textbf{Return} $\gG_\vlambda(X_{\text{test}})$
    \end{algorithmic}
\end{algorithm}

\section{Risk Guarantees}
\label{app:pre}

\subsection{Risk guarantee \textbf{\underline{(1)}} (\Cref{risk_guarantee_1})}
\label{app:risk_guarantee_1}

\begin{proof}[Proof of \Cref{risk_guarantee_1}]
    The proof sketch follows \cite{angelopoulos2021learn}. 
    Since the risk function $R(\cdot, \cdot)$ is upper bounded by $1$, we can apply a tighter version of Hoeffding's inequality \cite{hoeffding1994probability} for $\hat{\alpha}>\mathbb{E}[R(T_{{\vlambda},p_{\theta_l}}(x), y)]$:
    \begin{equation}
    \label{eq;hoeff_1}
        \sP\left[ R(T_{{\vlambda},p_{\theta_l}}(x), y) \ge \hat{\alpha} \right] \le \exp\left\{ -N_{\text{cal}}h(\hat{R}(\hat{\gD}_{\text{cal}}),\hat{\alpha}) \right\}
    \end{equation}
    Also, applying Bentkus inequality \cite{bentkus2004hoeffding}, we have:
    \begin{equation}
    \label{eq:bentk_2}
        \sP\left[ R(T_{{\vlambda},p_{\theta_l}}(x), y) \ge \hat{\alpha} \right] \le e\sP\left[\text{Bin}(N_{\text{cal}}, \hat{\alpha}) \le \left\lceil N_{\text{cal}} \hat{R}(\hat{\gD}_{\text{cal}}) \right\rceil \right]
    \end{equation}
    Combining \Cref{eq;hoeff_1,eq:bentk_2}, we have:
    \begin{equation}
        \sP\left[ R(T_{{\vlambda},p_{\theta_l}}(x), y) \ge \hat{\alpha} \right] \le  \min\left( \exp\left\{-N_{\text{cal}} h\left(\hat{R}(\hat{\gD}_{\text{cal}}), \hat{\alpha}) \right) \right\}, e\sP\left[\text{Bin}(N_{\text{cal}}, \hat{\alpha}) \le \left\lceil N_{\text{cal}} \hat{R}(\hat{\gD}_{\text{cal}}) \right\rceil \right] \right)
    \end{equation}
    Or equivalently, given uncertainty $1-\delta$, we have:
    \begin{equation}
        \delta = \min\left( \exp\left\{-N_{\text{cal}} h\left(\hat{R}(\hat{\gD}_{\text{cal}}), \hat{\alpha}) \right) \right\}, e\sP\left[\text{Bin}(N_{\text{cal}}, \hat{\alpha}) \le \left\lceil N_{\text{cal}} \hat{R}(\hat{\gD}_{\text{cal}}) \right\rceil \right] \right),
    \end{equation}
    which leads to the following by formulating an inverse function:
    \begin{equation}
        \hat{\alpha} = \min \left\{ h^{-1}\left(\dfrac{\ln(1/\delta)}{N_{\text{cal}}};\hat{R}(\hat{\gD}_{\text{cal}})\right),\Phi^{-1}_{\text{bin}}\left(\dfrac{\delta}{e};N_{\text{cal}},\hat{R}(\hat{\gD}_{\text{cal}})\right) \right\}
    \end{equation}
\end{proof}

\begin{remark}
    Given the constrained generation protocol $T_{\vlambda,p_{\theta_l}}$, RAG generation parameter $\vlambda$, a calibration set $\hat{\gD}_{\text{cal}}$, and a risk function $R(\cdot,\cdot): 2^\gY \times \gY \mapsto \sR$, we aim to provide a risk guarantee of the test sample $(X_{\text{test}},Y_{\text{test}})$:
\begin{equation}
    \sP\left[ R(T_{{\vlambda}, p_{\theta_l}}(X_{\text{test}}),Y_{\text{test}})  \le \hat{\alpha} \right] \ge 1 - \delta,
\end{equation}
where $\hat{\alpha}$ is the conformal risk upper bound, and the confidence level $1-\delta$ can be computed by Hoeffding-Bentkus inequalities \cite{bates2021distribution}:
\begin{equation}
\label{eq:hb_ineq}
\begin{aligned}
    \delta = \min\left( \exp\left\{-N_{\text{cal}} h\left(\hat{R}(\hat{\gD}_{\text{cal}}), \hat{\alpha}) \right) \right\}, e\sP\left[\text{Bin}(N_{\text{cal}}, \hat{\alpha}) \le \left\lceil N_{\text{cal}} \hat{R}(\hat{\gD}_{\text{cal}}) \right\rceil \right] \right),
\end{aligned}
\end{equation}
where $h(a,b)=a\log(a/b) + (1-a)\log((1-a)/(1-b))$, $\text{Bin}(\cdot,\cdot)$ denotes the binomial distribution, $N_{\text{cal}}$ is the number of samples in the calibration set, and $\hat{R}(\cdot)$ computes the empirical risk on the calibration set.

Given the confidence level $1-\delta$, we can also inversely compute the conformal risk upper bound $\hat{\alpha}$ as the following:
\begin{equation}
\label{eq:conformal_bnd_1}
    \hat{\alpha} = \min \left\{ h^{-1}\left(\dfrac{\ln(1/\delta)}{N_{\text{cal}}};\hat{R}(\hat{\gD}_{\text{cal}})\right),\Phi^{-1}_{\text{bin}}\left(\dfrac{\delta}{e};N_{\text{cal}},\hat{R}(\hat{\gD}_{\text{cal}})\right) \right\}
\end{equation}
where $h^{-1}(\cdot;\cdot)$ is the partial inverse function such that $h^{-1}(h(a,b);a)=b$ with $h(a,b)=a\log(a/b) + (1-a)\log((1-a)/(1-b))$, and $\Phi^{-1}_{\text{bin}}$ denotes the inverse of CDF of binomial distribution. The HB bound uses the empirical risk on the calibration set as test statistics and provides finite-sample statistical results with surprising empirical effectiveness.
\end{remark}

\paragraph{Alternative approach for Risk Guarantee \Cref{risk_guarantee_1}} We can obtain a tighter guarantee of the conformal risk if we assume that the given configuration vector $\vlambda$ has dimension $1$ (i.e., $B=1$) and the risk function $R(\cdot,\cdot)$ monotonically increases in parameter $\vlambda$ and is upper bounded by $C$ $(R: 2^\gY \times \gY \mapsto (-\infty,C])$. Then, we can have the following conformal risk guarantee according to \cite{angelopoulos2022conformal}:
\begin{equation}
\label{eq:conform2}
\begin{aligned}
    \dfrac{N_{\text{cal}}}{N_{\text{cal}}+1} \hat{R}(\hat{\gD}_{\text{cal}}) - \dfrac{C}{N_{\text{cal}}+1} \le \mathbb{E}\left[ R(T_{\hat{\vlambda},p_{\theta_l}}(X_{\text{test}}),Y_{\text{test}}) \right]  \le \dfrac{N_{\text{cal}}}{N_{\text{cal}}+1} \hat{R}(\hat{\gD}_{\text{cal}}) + \dfrac{C}{N_{\text{cal}}+1}.
\end{aligned}
\end{equation}

Although the guarantee in \Cref{eq:conform2} is tighter with the guarantee of the upper bound and the lower bound, the additional assumption of single dimensionality and monotonicity does not hold for many practical generation protocols. Therefore, we mainly consider the conformal risk bound in \Cref{eq:conformal_bnd_1} across the analysis.
We also add discussions that \name is flexible in considering different types of conformal risk bounds, which basically presents an explicit function of the controlled risk with respect to the empirical risk.
Since we build the connection between the empirical risk $\hat{R}$ to the retrieval model in the risk in the analysis, we only need to directly connect the empirical risk to the controlled risk via the explicit function to achieve end-to-end certification.

\subsection{Risk guarantee \textbf{\underline{(2)}} (\Cref{risk_guarantee_2})}
\label{app:fwer}

\begin{proof}[Proof of \Cref{risk_guarantee_2}]
    The proof follows \cite{holm1979simple}. 
    We consider $|\Lambda|$ independent hypothesis test corresponding to the $|\Lambda|$ Null hypothesis. By the Bonferroni method, each test is performed at a significance level of $\dfrac{\delta}{|\Lambda|}$.
    Therefore, The probability of not making a Type I error in a single test is $1-\dfrac{\delta}{|\Lambda|}$.
    The probability of making no Type I error in all $|\Lambda|$ tests is $(1-\dfrac{\delta}{|\Lambda|})^{|\Lambda|}$.
     The probability of making at least one Type I error (i.e., FWER) is the complement of making no Type I errors, which is $1-(1-\dfrac{\delta}{|\Lambda|})^{|\Lambda|}\le \delta$.
     Therefore, we prove that the familywise error rate is $\delta$ for Bonferroni correction. 
     Thus, going back to the risk guarantee, we have:
     \begin{equation}
         \sP \left[ \sup_{ \hat{\vlambda} \in \hat{\Lambda}_\alpha} \left\{R\left(T_{\hat{\vlambda},p_{\theta_l}}(x),y \right)  \right\} \le \alpha \right] \ge 1 - \delta
     \end{equation}
\end{proof}

\begin{remark}
    To achieve conformal analysis \textbf{\underline{(2)}}, we follow the procedure in \cite{angelopoulos2021learn}: (a) for each parameter configuration $\vlambda$ in the feasible region $\Lambda$, associate the null hypothesis: $\gH_j: R(T_{\vlambda,p_{\theta_l}})>\alpha$ (rejecting the null hypothesis implies controlling the risk below $\alpha$ with hyperparameter $\vlambda$), (b) for each null hypothesis, compute a finite-sample valid p-value $p_j$ using Hoeffding-Bentkus inequality, and (c) return $\hat{\Lambda}_\alpha=\gA(p_1,...,p_{|\Lambda|})$, where $\gA$ is an algorithm that controls the family-wise error rate (FWER). 
Essentially, FWER controls the error by union bounds over the hyperparameter space.
The Bonferroni correction yields $\hat{\Lambda}_\alpha = \{ \hat{\vlambda}_j: \delta_j \le \delta / |\Lambda| \}$. The graph-based search in \Cref{alg:multi_search} dynamically assigns the error levels and yields a tighter certification. Specifically, we maintain a directed graph with nodes denoting the error rate of the parameter configuration and edges denoting the correlations between two parameters. The correlations can be instantiated randomly. We first randomly assign error rates to all feasible parameter configurations, and then once we search for one valid parameter with a smaller p-value than the assigned error rate, we will add the parameter to the valid set and propagate the excessive error rate to other nodes. The procedure repeats until no valid parameter can be found.\\
\textbf{Computation of p-values}. Due to the duality between p-values and confidence intervals \cite{bates2021distribution}, we compute the p-value by applying the Hoeffding-Bentkus inequality as the following:
\begin{equation}
    p_j = \min\left( \exp\left\{-N_{\text{cal}} h\left( \hat{\mathbb{E}}[R(T_{\vlambda_j,p_{\theta_l}}(x),y)], \hat{\alpha}) \right) \right\}, e\sP\left[\text{Bin}(N_{\text{cal}}, \hat{\alpha}) \le \left\lceil N_{\text{cal}} \hat{\mathbb{E}}[R(T_{\vlambda_j,p_{\theta_l}}(x),y)] \right\rceil \right] \right),
\end{equation}
where $\hat{\mathbb{E}}[R(T_{\vlambda_j,p_{\theta_l}}(x),y)]$ denotes the empirical mean risk with configuration $\vlambda_j~(j \in \{1,2,..,|\Lambda|\})$.
\end{remark}

\begin{algorithm*}[t]
    \caption{Graph-based valid configurations search}
    \label{alg:multi_search}
    \begin{algorithmic}[1]
        \STATE {\bfseries Input}: confidence error level $\delta$, parameter configurations $\Lambda=\{\vlambda_1,...,\vlambda_N \}$, $p-$values $(p_1,...,p_N)$, graph $\gG$, initial error budget $\delta_i$ such that $\sum_i \delta_i = \delta$
        \STATE {\bfseries Output}: valid configurations set with certified conformal risk $\hat{\Lambda}$
        \STATE $\hat{\Lambda} \gets \Phi$
        \WHILE{$\exists i: p_i \le \delta_i$}
            \STATE Select any $i$ such that $p_i \le \delta_i$
            \STATE $\hat{\Lambda} \gets \hat{\Lambda} \cup \{\vlambda_i\}$
            \STATE Update the error level and the graph:\\
            \begin{equation}
                \delta_j \gets \left\{ \begin{aligned}
                    &\delta_j + \delta_i g_{i,j},~~\lambda_j \in \Lambda \backslash \hat{\Lambda} \\
                    &0,~~\text{otherwise}
                \end{aligned} \right. \quad \text{and} \quad
                g_{k,j} \gets \left\{ \begin{aligned}
                    &\dfrac{g_{k,j}+g_{k,i}g_{i,j}}{1-g_{k,i}g_{i,k}}, ~~ \lambda_k, \lambda_j \in \Lambda \backslash \hat{\Lambda}, k \neq j \\ &0, ~~\text{otherwise}
                \end{aligned} \right.
            \end{equation}
        \ENDWHILE
        \STATE \textbf{Return} $\hat{\Lambda}$
    \end{algorithmic}
\end{algorithm*}

\section{Grammian generalization bound}

\begin{lemma}[\cite{weber2022certifying}]
Let $\gD$ and $\gQ$ denote two distributions supported on $\gX\times\gY$.
Let $h_\theta: \gX \mapsto \gY$ be any black-box pretrained model. Consider any risk/loss function $\ell: \gY \times \gY \mapsto \sR$ such that $0\le \ell(h_\theta(X),Y) \le T$,
    then
    \begin{equation}
        \begin{aligned}
            & \max_{\gQ, \theta} \E_{(X,Y)\sim\gQ} [\ell(h_\theta(X),Y)] \quad \mathrm{s.t.} \quad H(\gD,\gQ) \le \rho \\
            \le & \E_{(X,Y)\sim\gD} [\ell(h_\theta(X),Y)] + 2C_\rho \sqrt{\sV_{(X,Y)\sim\gD}[\ell(h_\theta(X),Y)]} + \\
            & \rho^2 (2-\rho^2) 
            \left(
            T - \E_{(X,Y)\sim\gD}[\ell(h_\theta(X),Y)] - \dfrac{\sV_{(X,Y)\sim\gD}[\ell(h_\theta(X),Y)]}{T - \E_{(X,Y)\sim\gD}[\ell(h_\theta(X),Y)]}
            \right),
        \end{aligned}
    \end{equation}
    where $C_\rho = \sqrt{\rho^2 (1 - \rho^2)^2 (2 - \rho^2)}$, for any given distance bound $\rho > 0$ that satisfies
    \begin{equation}
        \rho^2 \le 1 - \left( 1 + \dfrac{(T - \E_{(X,Y)\sim\gD}[\ell(h_\theta(X),Y)])^2}{\sV_{(X,Y)\sim\gD}[\ell(h_\theta(X),Y)]} \right)^{-1/2}.
    \end{equation}
    \label{lem:gramian-bound}
\end{lemma}

This theorem provides a closed-form expression that upper bounds the risk of $h_\theta(\cdot)$ on shifted distribution~(namely $\mathbb{E}_{\gQ} [\ell(h_\theta(X),Y)]$), given bounded Hellinger distance $H(\gD,\gQ)$ and the mean $E$ and variance $V$ of loss on $\gD$ under two mild conditions: 
(1)~the function is positive and bounded~(denote the upper bound by $T$); and (2)~the distance $H(\gD,\gQ)$ is not too large~(specifically, $H(\gD,\gQ)^2 \le \bar\gamma^2 := 1 - (1+(T-E)^2/V)^{-\frac12}$).
Since \Cref{lem:gramian-bound} holds for arbitrary models and risk functions $\ell(h_\theta(\cdot),\cdot)$ as long as the function value is bounded by $[0, T]$, using \Cref{lem:gramian-bound} allows us to provide a generic and succinct retrieval analysis and conformal risk certificate in \Cref{pro:conf_shf} and \Cref{thm:comp_shft} that holds for generic models including DNNs without engaging complex model architectures.
Indeed, we only need to query the mean and variance under $\gD$ for the retrieval model to compute the certificate in \Cref{lem:gramian-bound}.

\section{Proofs and detailed remarks in \Cref{sec:rag_conf}}
\label{app:proof_first}

\subsection{Detailed remark of \Cref{def:ret_mod}}

\begin{remark}
    {\underline{(R1)}} Note that a retrieval model with random initialization can achieve $\mathbb{E}[s_\theta(x,x^+)]=\mathbb{E}[s_\theta(x,x^-)]$ asymptotically. We only assume a $V_{\text{rag}}$-retrieval model that differentiates positive examples from negative examples slightly better than random with the condition $\mathbb{E}[s_\theta(x,x^+)-s_\theta(x,x^-)]>0$.
    {\underline{(R2)}} We also only assume a moderate stability characterized with bounded variance $\mathbb{V}[s_\theta(x,x^+)-s_\theta(x,x^-)]^{1/2}<\ln(\exp\{-L_\tau\}/(1-\exp\{-L_\tau\}))$, which implicitly assumes a moderate generalization of the retrieval model by the variance-generalization connection \cite{lam2016robust,gotoh2018robust,namkoong2017variance}.
    The moderate generalization ability of the retrieval model is essential since we do not assume that the knowledge base distribution is identical to the calibration/test distribution.
    Since the frequently adopted cosine similarity is bounded in $[-1,1]$, the variance of difference in similarities $\mathbb{V}[s_\theta(x,x^+)-s_\theta(x,x^-)]$ is upper bounded by $1$ (derived by the variance bound $\mathbb{V}[X] \le (b-a)^2/4$ for random variable $X$ bounded in $[a,b]$). Therefore, as long as $L_\tau$ is small such that $\ln(\exp\{-L_\tau\}/(1-\exp\{-L_\tau\}))>1$, the variance requirement can be automatically satisfied.
    {\underline{(R3)}} We define the $V_{\text{rag}}$-retrieval model with a commonly used contrastive loss for retrieval model training \cite{wang2023learning,rubin2021learning}, but we also allow for flexibility in considering other types of contrastive loss such as the triplet loss \cite{hermans2017defense}. Towards that, we only need to connect a different loss formulation to the denominator of the formulation of $V_{\text{rag}}$ (i.e., $\ln(\exp\{-L_\tau\}/(1-\exp\{-L_\tau\}))$).
\end{remark}

\subsection{Detailed remark of \Cref{def:trans}}

\begin{remark}
    \underline{(R1)} $d^+$ represents the attention score of positive pairs and quantifies the utility of embedding matrix $W_E$, key matrix $W_K$, and query matrix $W_Q$ of the transformer. Note that the attention scores are always non-negative after the Softmax activation, so $d^+>0$ usually holds.
    \underline{(R2)} $\int_{-1}^1 \Phi_M(v)dv$ characterizes the quality of the embedding matrix $W_E$, value matrix $W_V$, and projection matrix $W_P$. We have a Softmax normalization for output probability vectors and $M$ is bounded in $[-1,1]$, so the integral over $[-1,1]$ traverses the support of $M$. Predictions aligning well with the reference text induce a generally small value of $M$ and thus a large integral of CDF $\int_{-1}^1 \Phi_M(v)dv$. Also, note that a random prediction margin $M_{\text{rand}}$ with a uniform distribution over $[-1,1]$ satisfies $\int_{-1}^1 \Phi_{M_{\text{rand}}}(v)dv=1$. Therefore, $\int_{-1}^1 \Phi_M(v)dv > 1=\int_{-1}^1 \Phi_{M_{\text{rand}}}(v)dv$ only assumes a better-than-random prediction margin.
\end{remark}

\subsection{Proof and detailed remark of \Cref{lem:rag}}

\begin{remark}
    \Cref{eq:lem_rag} shows the lower bound of the expectation of the retrieved positive examples by the retrieval model. 
    \merveedit{{\underline{(R1)}} For a sufficiently large number of instances in the external knowledge base $N_{\text{ext}}$ (a typical scenario in practice), the lower bound approximately scales with $0.9N_{\text{rag}}$ and this scaling occurs at an exponential rate with respect to $N_{\text{ext}}$. These findings imply that with a large external knowledge base, a large ratio of the retrieved examples is positive (i.e., with the same groundtruth output), which is valuable as
    the retrieved positive examples which share similar semantic meanings as the query examples can improve in-context learning of LLMs \cite{min2022rethinking,wang2022towards}.
 To formulate this observation in a rigorous way, we theoretically show the benefits of retrieved positive in-context examples in achieving low conformal generation risks in \Cref{thm:gene_rag}.}
    {\underline{(R2)}} The lower bound of the expected retrieved positive examples also correlates with the balance in the external knowledge base (i.e., $r^{(c)}_{\text{ext}}$). The correlation implies that if the knowledge base is highly long-tailed such that samples of certain reference texts are rare (i.e., $r_{\text{ext}}^{(c)}$ is small), we require a larger sample size of knowledge base $N_{\text{ext}}$ to compensate for the long-tail distribution and achieve comparable retrieval quality.
    {\underline{(R3)}} The bound also shows that a low-variance retrieval model (i.e., a small $V_{\text{rag}}$) can generalize well to test distribution and induce a better retrieval quality.  
\end{remark}

\textit{Proof sketch}. We first formulate the expectation of similarity difference between positive pairs and negative pairs $\mathbb{E}[s_{\theta_r}(x,x^+)-s_{\theta_r}(x,x^-)]$ as a function of the contrastive loss of the retrieval model $L_\tau$. Then, we apply Chebyshev's inequality to upper bound the failure probability $\sP[ s_\theta(x,x^+) < s_\theta(x,x^-)]$ as a function of $V_{\text{rag}}$. We then derive a lower bound of the number of retrieved positive examples, which follows a binomial distribution with $N_{\text{rag}}$ trials and the failure rate as a function of $V_{\text{rag}}$. We finally correct the bound with the finite-sample errors of the knowledge base by the tail bound of categorical distribution.

\begin{proof}[Proof of \Cref{lem:rag}]
    Let $\gD$ be the data distribution, which is also the training distribution of the retrieval model, conformal calibration distribution, and test distribution. For a sample $(x,y) \sim \gD$, we denote $\gD_{\text{ext}}^+(x)$, $\gD_{\text{ext}}^-(x)$ be the distribution of positive examples (with the same groundtruth output $y$) and negative examples (with different groundtruth output to $y$) of sample $x$ in the external knowledge base. Then we can formulate the expectation of contrastive loss of the retrieval model as:
    \begin{equation}
    \begin{aligned}
        L_\tau &= \mathbb{E}_{x\sim \gD, x^+ \sim \gD^+_{\text{ext}}(x), x^- \sim \gD^-_{\text{ext}}(x)}\left[ \gL_{\text{contrastive}}(x,x^+,x^-) \right]\\
        &= \mathbb{E}_{x\sim \gD, x^+ \sim \gD^+_{\text{ext}}(x), x^- \sim \gD^-_{\text{ext}}(x)}\left[ - \log \dfrac{\exp\left\{s_\theta(x,x^+)\right\}}{\exp\left\{s_\theta(x,x^+)\right\} + \exp\left\{s_\theta(x,x^-)\right\} } \right],
    \end{aligned}
    \end{equation}
    which is equivalent to
    \begin{equation}
        \mathbb{E}_{x\sim \gD, x^+ \sim \gD^+_{\text{ext}}(x), x^- \sim \gD^-_{\text{ext}}(x)}\left[ s_\theta(x,x^+) - s_\theta(x,x^-) \right] = \ln \dfrac{\exp\left\{ -L_\tau \right\}}{1 - \exp\left\{ -L_\tau \right\}} > 0.
    \end{equation}
    Then we can apply Chebyshev's inequality \cite{saw1984chebyshev} to the random variable $s_\theta(x,x^+) - s_\theta(x,x^-)$ and get the following:
    \begin{align}
        \sP\left[ s_\theta(x,x^+) < s_\theta(x,x^-) \right] &=  \sP_{x\sim \gD, x^+ \sim \gD^+_{\text{ext}}(x), x^- \sim \gD^-_{\text{ext}}(x)}\left[ s_\theta(x,x^+) - s_\theta(x,x^-) < 0 \right]\\ 
        &\le \dfrac{\mathbb{V}\left[s_\theta(x,x^+) - s_\theta(x,x^-)\right]}{\mathbb{E}\left[s_\theta(x,x^+) - s_\theta(x,x^-)\right]^2} \\
        &\le \left( \dfrac{\sqrt{\mathbb{V}\left[s_\theta(x,x^+) - s_\theta(x,x^-)\right]}}{\mathbb{E}\left[s_\theta(x,x^+) - s_\theta(x,x^-)\right]} \right)^2 = V_{\text{rag}}^2 < 1. \label{ineq:bnd_sim}
    \end{align}

    We first focus on one demonstration example $Z_r$ retrieved by $s_\theta(\cdot,\cdot)$. According to the retrieval mechanism, the example $Z_r$ has the highest similarity (measured by $s_\theta(\cdot,\cdot)$) to the query sample $Z_q$. 
    Recall that $r^{(c)}_{\text{cal}}$ and $r^{(c)}_{\text{ext}}$ be the event probability of the $c$-th category in the categorical calibration distribution and categorical external knowledge distribution.
    Let $\vr_{\text{ext}}=\left[ r^{(1)}_{\text{ext}}, r^{(2)}_{\text{ext}},...,r^{(C)}_{\text{ext}}\right]$. Since we only have $N_{\text{ext}}$ finite sample drawn from $\gD_{\text{ext}}$ in the external knowledge base in practice, we notate the empirical categorical portions as $\hat{\vq}_{\text{ext}} = \left[ \hat{r}^{(1)}_{\text{ext}}, \hat{r}^{(2)}_{\text{ext}},...,\hat{r}^{(C)}_{\text{ext}}\right]$, where $\hat{r}^{(c)}_{\text{ext}}~(c \in \{1,..,C\})$ represents the portion of samples with grountruth text $c \in \gY$ in the external knowledge base.
    Then we can apply the concentration bound of categorical distribution as \cite{agrawal2017optimistic}:
    \begin{align}
        \sP\left[ \|\hat{\vr}_{\text{ext}} - \vr_{\text{ext}} \|_1 \ge \dfrac{\sqrt{2\ln{(1/\delta_{\text{ext}})}}}{N_{\text{ext}}} \right] \le \delta_{\text{ext}}, \label{ineq:fin_cat}
    \end{align}
    where $\delta_{\text{ext}}>0$ represents the confidence level of the tail bound.
    Then we can upper bound the probability that the groundtruth output $g(Z_r)$ of the retrieved sample $Z_r$ is not equal to that of the query sample $g(Z_q)$ as the following:
    \begin{align}
        \sP\left[ g(Z_r) \notin  g(Z_q) \right] &= \sP_{Z_q \sim \gD}\left[ g(Z_r) \notin  g(Z_q)  \left| s_\theta(Z_q, Z_r) \ge \max_{z \in \gD_{\text{ext}}} s_\theta (z, Z_q) \right. \right] \\
        &= \sP_{Z_q \sim \gD}\left[ \max_{Z^- \in \gD^-_{\text{ext}}(Z_q)}  s_\theta(Z^-, Z_q) \ge  \max_{Z^+ \in \gD^+_{\text{ext}}(Z_q)} s_\theta(Z^+, Z_q) \right] \\
        &= \sP_{Z_q \sim \gD}\left[ s_\theta(Z^-, Z_q) \ge  s_\theta(Z^+, Z_q), ~~\forall Z^+ \in \gD^+_{\text{ext}}(Z_q),~ \exists Z^- \in \gD^-_{\text{ext}}(Z_q)  \right] \\
        &\le \sum_{c=1}^C r^{(c)}_{\text{cal}} \left(1 - r^{(c)}_{\text{ext}}  \right) N_{\text{ext}} \sP\left[ s_\theta(x,x^+) < s_\theta(x,x^-) \right]^{N_{\text{ext}} r^{(c)}_{\text{ext}}}, \label{ineq:rag_ext}
    \end{align}
    where \Cref{ineq:rag_ext} is derived by applying the union bound. Considering finite-sample error of categorical distribution in \Cref{ineq:fin_cat} and combining \Cref{ineq:bnd_sim}, we finally have:
    \begin{align}
        \sP\left[ g(Z_r) \notin  g(Z_q) \right] &\le \sum_{c=1}^C r^{(c)}_{\text{cal}} \left(1 - r^{(c)}_{\text{ext}}  \right) N_{\text{ext}} \sP\left[ s_\theta(x,x^+) < s_\theta(x,x^-) \right]^{N_{\text{ext}} r^{(c)}_{\text{ext}}} \\
        &\le \sum_{c=1}^C r^{(c)}_{\text{cal}} \left(1 - r^{(c)}_{\text{ext}} + \dfrac{\sqrt{2\ln{(1/\delta_{\text{ext}})}}}{N_{\text{ext}}} \right) N_{\text{ext}} V_{\text{rag}}^{0.5N_{\text{ext}} \left(r^{(c)}_{\text{ext}} - {\sqrt{2\ln{(1/\delta_{\text{ext}})}}}/{N_{\text{ext}}} \right)}. \label{ineq:failrate}
    \end{align}
    Since we assume that the retrieval model retrieves samples identically from the external knowledge base, the number of retrieved positive examples $N_{\text{pos}}$ follows a Binomial distribution with $N_{\text{rag}}$ trials and failure rate in \Cref{ineq:failrate}. Therefore, we can lower bound the expectation of $N_{\text{pos}}$ as the following:
    \begin{equation}
        \mathbb{E}\left[ N_{\text{pos}} \right] \ge N_{\text{rag}} \left(1-\delta_{\text{ext}}\right)\left(1- \sum_{c=1}^C r^{(c)}_{\text{cal}} \left(1 - r^{(c)}_{\text{ext}} + \dfrac{\sqrt{2\ln{(1/\delta_{\text{ext}})}}}{N_{\text{ext}}} \right) N_{\text{ext}} V_{\text{rag}}^{0.5N_{\text{ext}} \left(r^{(c)}_{\text{ext}} - {\sqrt{2\ln{(1/\delta_{\text{ext}})}}}/{N_{\text{ext}}} \right)}  \right),
    \end{equation}
    which holds for any $\delta_{\text{ext}}>0$. Therefore, letting $\delta_{\text{ext}}=0.1$, we can finally conclude that:
     \begin{equation}
        \mathbb{E}\left[ N_{\text{pos}} \right] \ge \dfrac{9}{10}N_{\text{rag}}\left(1- \sum_{c=1}^C r^{(c)}_{\text{cal}} \left(N_{\text{ext}} - r^{(c)}_{\text{ext}}N_{\text{ext}} + {\sqrt{2\ln{10}}} \right)  V_{\text{rag}}^{0.5 \left(r^{(c)}_{\text{ext}}N_{\text{ext}} - {\sqrt{2\ln{10}}} \right)}  \right).
    \end{equation}
\end{proof}

\subsection{Proof and detailed remark of \Cref{thm:gene_rag}}
\label{app:thm1}

\begin{remark}
In \Cref{thm:gene_rag}, we theoretically show that the conformal generation risk with RAG $\hat{\alpha}_{\text{rag}}$ is smaller than the risk without RAG $\hat{\alpha}$ with a high probability.
{\underline{(R1)}} We can observe that the probability monotonically increases in the sample size of the calibration set $N_{\text{cal}}$, the size of retrieved examples $N_{\text{rag}}$, and the number of instances in the external knowledge base $N_{\text{ext}}$. In particular, a large $N_{\text{cal}}$ reduces the finite-sample error during the calibration and induces a better approximation of the true generation risk with the empirical risk on the calibration set. A large $N_{\text{rag}}$ and $N_{\text{ext}}$ brings in related background information from a more knowledge-intensive knowledge base, which enhances the quality of generations augmented by retrieval.
{\underline{(R2)}} Furthermore, the probability $\sP\left[ \hat{\alpha}_{\text{rag}} < \hat{\alpha} \right]$ increases with the increase in transformer's quality, which is quantified by the attention scores for a positive pair (i.e., $d^+$ 
) and the prediction capability (without RAG) (i.e., $\eqsmall \int_{-1}^1 \Phi_M(v) dv - 1 $ 
).
Since $1-\Phi_M(0)$ represents the population risk without RAG, the difference of the prediction margin CDF $\Phi_M(\cdot)$ (monotonically increasing) directly characterizes the benefit of generation quality with RAG. 
 The quality improvement provided by RAG also exponentially induces a larger probability of reducing the conformal generation risk of a single LLM.
The transformer uncertainty $p_t$ also decreases exponentially with a large number of retrieved examples, indicating that more examples retrieved by a good retrieval model benefit a lower conformal generation risk.
\underline{{(R3)}} 
The retrieval model uncertainty $p_r$ decreases with a low-variance retrieval model (small $V_{\text{rag}}$), which can generalize well to test distribution.
\underline{{(R4)}} We focus on the conformal generation risk formulated in \Cref{eq:conformal_risk_}, 
but we can easily adapt the results to any other forms of conformal risks. 
Since we build the connection between the empirical risk $\hat{R}$ to the retrieval model quality, we only need to directly connect the empirical risk to the certified generation risk via the explicit function to achieve end-to-end certification.
\underline{{(R5)}} We define the positive pairs as examples sharing the same semantic meanings of reference texts for simplicity of the certification results. In the certification framework, we can also consider a relaxed definition of positive pairs by the similarity of reference texts in the embedding space. Similarly, the examples with high similarity of reference texts to the query example will induce high attention scores and benefit the generation with the attention mechanism.
\end{remark}

\textit{Proof sketch}. We decompose the one-layer self-attention mapping as the combinations of attention with positive examples and attention with negative examples. Based on the explicit formulation, we then derive a lower bound of the logit difference of the ground truth token by taking a lower bound of the number of positive examples (derived from \Cref{lem:rag}) and the attention scores with positive examples. Next, we get a lower bound of the risk difference between the transformer without RAG and the RAG model. Finally, we apply Hoeffding's inequality to derive a lower bound of the difference in empirical risks, and accordingly, conformal risk bounds. Applying union bounds over all uncertainty levels concludes the proof.

\begin{proof}[Proof of \Cref{thm:gene_rag}]

From \Cref{lem:rag}, we prove that:
\begin{equation}
    \mathbb{E}\left[ N_{\text{pos}} \right] \ge \underline{N}_{\text{pos}} := \dfrac{9}{10}N_{\text{rag}}\left(1- \sum_{c=1}^C r^{(c)}_{\text{cal}} \left(N_{\text{ext}} - r^{(c)}_{\text{ext}}N_{\text{ext}} + {\sqrt{2\ln{10}}} \right)  V_{\text{rag}}^{0.5 \left(r^{(c)}_{\text{ext}}N_{\text{ext}} - {\sqrt{2\ln{10}}} \right)}  \right).
\end{equation}
Since $N_{\text{pos}}$ is a binomial random variable with $N_{\text{rag}}$ trials, we have the upper bound of the variance $\mathbb{V}[N_{\text{pos}}] \le \dfrac{N_{\text{rag}}}{4}$.
Applying Chebyshev's inequality to the random variable $N_{\text{pos}}$, the following holds $\forall n_{\text{pos}}< \underline{N}_{\text{pos}}$:
\begin{equation}
    \sP\left[ N_{\text{pos}} \ge n_{\text{pos}} \right] \ge 1 - \dfrac{\mathbb{V}[N_{\text{rag}}]}{(\underline{N}_{\text{pos}}-n_{\text{pos}})^2} 
    \ge 1 - \dfrac{N_{\text{rag}}}{4(\underline{N}_{\text{pos}}-n_{\text{pos}})^2}, \label{eq:npos_cheby}
\end{equation}
which implicates that we can do the analysis with $N_{\text{pos}} \ge n_{\text{pos}}$ with probability $1 - \dfrac{\mathbb{V}[N_{\text{rag}}]}{(\underline{N}_{\text{pos}}-n_{\text{pos}})^2} $.

Since the query example is encoded at the last position of the sequence (i.e., $N_{\text{rag}}+1$-th position), we let $N:=N_{\text{rag}}+1$ for ease of notation. We denote the probability vector at the position as $O^{\text{(rag)}}(\vq)$ with RAG and $O(\vq_N)$ without RAG (without RAG, the input text is only the query sample $\vq_N$).
Recall the single-layer self-attention transformer:
\begin{equation}
\label{eq:lin_trans}
    O^{\text{(rag)}}(\vq) =  W_P\left\{ W_V W_E \vq_N + (W_V W_E \vq) \sigma \left( (W_K W_E \vq)^T W_Q W_E \vq_N \right) \right\}.
\end{equation}

Note that each raw vector of the linear projection matrix (fully connected layer) represents the prototype embedding denoted as $\vp_c$ of the corresponding groundtruth output $c$. Formally, we have $W_P:=[\vp_1,\vp_2,...,\vp_C]^T$.
Recall that we denote $g(\vq_N)$ as the groundtruth output of example $\vq_N$.
Then we can reformulate \Cref{eq:lin_trans} as the following:
\begin{equation}
    O^{\text{(rag)}}(\vq) =  [\vp_1,\vp_2,...,\vp_C]^T \left\{ W_V W_E \vq_N + (W_V W_E \vq) \sigma \left( (W_K W_E \vq)^T W_Q W_E \vq_N \right) \right\},
\end{equation}
which indicates the formulation of $O_c^{\text{(rag)}}(\vq)$ denoting the probability of query sample $\vq_N$ being with groundtruth output $c$:
\begin{equation}
\label{eq:output_c}
    O_c^{\text{(rag)}}(\vq) = \vp_c^T \left\{  W_V W_E \vq_N + (W_V W_E \vq) \sigma \left( (W_K W_E \vq)^T W_Q W_E \vq_N \right) \right\}.
\end{equation}
We can also similarly formulate the prediction without RAG:
\begin{equation}
\label{eq:out_wo_rag_rep}
    O_c(\vq) = \vp_c^T W_V W_E \vq_N.
\end{equation}

Then we will focus on analyzing $O_c^{\text{(rag)}}(\vq)$ and connect it with the quantities of characterizing the quality of the transformer (i.e., $d^+,d^-,t^+,t^-$) and the quality of retrieved examples.
Towards that, we let $\gI^+(\vq_N)$ be the index set of retrieved examples with the same groundtruth output as $\vq_N$ (i.e., positive examples), and $\gI^-(\vq_N)$ be the index set of retrieved examples with the different groundtruth output to $\vq_N$ (i.e., negative examples). Then we can reformulate \Cref{eq:output_c} as the following:
\begin{equation}
\label{eq:out_decomp}
    \begin{aligned}
        O_c^{\text{(rag)}}(\vq) =& \vp_c^T  W_V W_E \vq_N + \vp_c^T (W_V W_E \vq) \sigma \left( (W_K W_E \vq)^T W_Q W_E \vq_N \right) \\
        =& \vp_c^T  W_V W_E \vq_N + \sum_{i^+ \in \gI^+(\vq_N)} \sigma \left( (W_K W_E \vq_{i^+}) W_Q W_E \vq_N \right) \vp_c^T (W_V W_E \vq_{i^+}) \\
        &+ \sum_{i^- \in \gI^-(\vq_N)} \sigma \left( (W_K W_E \vq_{i^-}) W_Q W_E \vq_N \right) \vp_c^T (W_V W_E \vq_{i^-})
    \end{aligned}
\end{equation}

Recall that we have the following assumption:
\begin{equation}
    \begin{aligned}
        & \sigma \left((W_K W_E \vq_i)^T (W_Q W_E \vq_j )\right) \ge d^+ > 0, \quad \text{for}~ g(\vq_i)=g(\vq_j), \\
    \end{aligned}
\end{equation}

We denote $N_{\text{pos}}$ as the number of positive retrieved examples to the query sample $\vq_N$ and the lower bound of it $n_\text{pos}$ with probability $1 - \dfrac{\mathbb{V}[N_{\text{rag}}]}{(\underline{N}_{\text{pos}}-n_{\text{pos}})^2} $ according to \Cref{eq:npos_cheby}. 
Note that the attention scores are normalized by Softmax with the summation of them being $1$.
By \Cref{eq:out_decomp}, $\forall c \neq g(\vq_N)$, we have:
\begin{equation}
\label{eq:der}
    \begin{aligned}
       & \mathbb{E} \left[ O_{g(\vq_N)}^{\text{(rag)}}(\vq) - O_c^{\text{(rag)}}(\vq) - (\vp_{g(\vq_N)}^T  W_V W_E \vq_N - \vp_c^T  W_V W_E \vq_N) \right] \\ =& \mathbb{E} \left[ \sum_{i^+ \in \gI^+(\vq_N)} \sigma \left( (W_K W_E \vq_{i^+}) W_Q W_E \vq_N \right) (\vp_{g(\vq_N)}-\vp_c)^T (W_V W_E \vq_{i^+}) \right. \\
        & \left. + \sum_{i^- \in \gI^-(\vq_N)} \sigma \left( (W_K W_E \vq_{i^-}) W_Q W_E \vq_N \right) (\vp_{g(\vq_N)}-\vp_c)^T (W_V W_E \vq_{i^-}) \right] \\
        \ge& \mathbb{E} \left[ d^+ \sum_{i^+ \in \gI^+(\vq_N)} (\vp_{g(\vq_N)}-\vp_c)^T (W_V W_E \vq_{i^+}) + (1- n_{\text{pos}} d^+) \left( -\sum_{i^+ \in \gI^+(\vq_N)} (\vp_{g(\vq_N)}-\vp_c)^T (W_V W_E \vq_{i^+})  \right) \right] \\
        \ge& \left( (n_{\text{pos}}+1)d^+ - 1 \right) \mathbb{E} \left[ \sum_{i^+ \in \gI^+(\vq_N)} (\vp_{g(\vq_N)}-\vp_c)^T (W_V W_E \vq_{i^+}) \right] \\
        \ge& \left( (n_{\text{pos}}+1)d^+ - 1 \right) n_{\text{pos}} \mathbb{E} \left[ \vp_{g(\vq_N)}^T W_V W_E \vq_{N} - \vp_c^T W_V W_E \vq_{N} \right] \\
        \ge& \left( (n_{\text{pos}}+1)d^+ - 1 \right) n_{\text{pos}} \mathbb{E} \left[ \vp_{g(\vq_N)}^T W_V W_E \vq_{N} - \max_{c \neq g(\vq_N)} \vp_c^T W_V W_E \vq_{N} \right]
    \end{aligned}
\end{equation}

Recall that $\Phi_M(\cdot)$ is the CDF function of the random variable of prediction marigin $\max_{c \neq g(\vq_N)} O_c(\vq) - O_{g(\vq_N)}(\vq)$ such that $\Phi_M(v)=\sP[\max_{c \neq g(\vq_N)} O_c(\vq) - O_{g(\vq_N)}(\vq) < v]$. Since the output probability of the transformer is bounded in $[0,1]$, we define a new random variable $X=\max_{c \neq g(\vq_N)} O_c(\vq) - O_{g(\vq_N)}(\vq)+1$ with $P[0 \le X \le 2]=1$. Then we have the following:
\begin{align}
    \mathbb{E}\left[O_{g(\vq_N)}(\vq) - \max_{c \neq g(\vq_N)} O_c(\vq) \right] &= 1 - \mathbb{E}\left[ X \right] \\
    &= 1 - \int_0^2 \left( 1 - \Phi_X(x) \right)dx \\
    &= 1 - \int_{-1}^1 \left( 1 - \Phi_M(v) \right)dv \\
    &= \int_{-1}^1 \Phi_M(v) dv - 1 \label{eq:o_phi}
\end{align}

Note that from \Cref{eq:out_wo_rag_rep}, we have $O_{g(\vq_N)}(\vq) - O_c(\vq) = \vp_{g(\vq_N)}^T  W_V W_E \vq_N - \vp_c^T  W_V W_E \vq_N$. Combining \Cref{eq:der} and \Cref{eq:o_phi}, the following holds $\forall c \neq g(\vq_N)$:
\begin{equation}
    \mathbb{E}\left[O_{g(\vq_N)}^{\text{(rag)}}(\vq) - O_c^{\text{(rag)}}(\vq) \right] \ge \mathbb{E} \left[ O_{g(\vq_N)}(\vq) - O_c(\vq) \right] + \left( (n_{\text{pos}}+1)d^+ - 1 \right) n_{\text{pos}} \left(\int_{-1}^1 \Phi_M(v) dv - 1 \right).
\end{equation}

Letting $c^* = \argmax_{c \neq g(\vq_N)} O_c^{\text{(rag)}}(\vq)$, we have:
\begin{equation}
\begin{aligned}
    \mathbb{E}\left[ O_{g(\vq_N)}^{\text{(rag)}}(\vq) - \max_{c \neq g(\vq_N)} O_{c*}^{\text{(rag)}}(\vq) \right] &= \mathbb{E}\left[ O_{g(\vq_N)}^{\text{(rag)}}(\vq) - O_{c*}^{\text{(rag)}}(\vq) \right] \\
    & \ge \mathbb{E}\left[ O_{g(\vq_N)}(\vq) - O_{c^*}(\vq) \right] + \left( (n_{\text{pos}}+1)d^+ - 1 \right) n_{\text{pos}} \left(\int_{-1}^1 \Phi_M(v) dv - 1 \right) \\
    & \ge \mathbb{E}\left[ O_{g(\vq_N)}(\vq) - \max_{c \neq g(\vq_N)} O_{c}(\vq) \right] + \left( (n_{\text{pos}}+1)d^+ - 1 \right) n_{\text{pos}} \left(\int_{-1}^1 \Phi_M(v) dv - 1 \right),
\end{aligned}
\end{equation}
which implies the following:
\begin{equation}
    \mathbb{E}\left[O_{g(\vq_N)}^{\text{(rag)}}(\vq) - \max_{c \neq g(\vq_N)} O_c^{\text{(rag)}}(\vq) \right] - \mathbb{E}\left[O_{g(\vq_N)}(\vq) - \max_{c \neq g(\vq_N)} O_c(\vq) \right] \ge \left( (n_{\text{pos}}+1)d^+ - 1 \right) n_{\text{pos}} \left(\int_{-1}^1 \Phi_M(v) dv - 1 \right),
\end{equation}
which is equivalent to the following:
\begin{equation}
\label{eq:rea_inter}
    \mathbb{E}\left[\max_{c \neq g(\vq_N)} O_c^{\text{(rag)}}(\vq) - O_{g(\vq_N)}^{\text{(rag)}}(\vq) \right] \le \mathbb{E}\left[\max_{c \neq g(\vq_N)} O_c(\vq) - O_{g(\vq_N)}(\vq) \right] - \left( (n_{\text{pos}}+1)d^+ - 1 \right) n_{\text{pos}} \left(\int_{-1}^1 \Phi_M(v) dv - 1 \right),
\end{equation}

Recall that we define the risk as $1-\text{Accuracy}$ and notate $R$ and $R_{\text{rag}}$ as the risk without RAG and with RAG, respectively.
Then we have:
\begin{align}
    \mathbb{E}\left[ R \right] &= 1 - \mathbb{E}\left[ \sI\left[ \max_{c \neq g(\vq_N)} O_c(\vq) - O_{g(\vq_N)}(\vq) < 0 \right] \right] \\
    &= 1 - \sP \left[\max_{c \neq g(\vq_N)} O_c(\vq) - O_{g(\vq_N)}(\vq) < 0 \right] \\
    &= 1 - \Phi_M(0). \label{eq:r}
\end{align}

Similarly for the risk with RAG $R_{\text{rag}}$, we have:
\begin{align}
    \mathbb{E}\left[ R_{\text{rag}} \right] &= 1 - \mathbb{E}\left[ \sI\left[ \max_{c \neq g(\vq_N)} O^{\text{(rag)}}_c(\vq) - O^{\text{(rag)}}_{g(\vq_N)}(\vq) < 0 \right] \right] \\
    &= 1 - \sP \left[\max_{c \neq g(\vq_N)} O^{\text{(rag)}}_c(\vq) - O^{\text{(rag)}}_{g(\vq_N)}(\vq) < 0 \right] \\
    &\le 1 - \sP \left[\max_{c \neq g(\vq_N)} O_c(\vq) - O_{g(\vq_N)}(\vq) < \left( (n_{\text{pos}}+1)d^+ - 1 \right) n_{\text{pos}} \left(\int_{-1}^1 \Phi_M(v) dv - 1 \right) \right] \label{eq:interm} \\
    &\le 1- \Phi_M\left( \left( (n_{\text{pos}}+1)d^+ - 1 \right) n_{\text{pos}} \left(\int_{-1}^1 \Phi_M(v) dv - 1 \right) \right), \label{eq:r_rag}
\end{align}
where \Cref{eq:interm} holds by applying \Cref{eq:rea_inter}.
Therefore, combining \Cref{eq:r} and \Cref{eq:r_rag}, the following holds:
\begin{equation}
    \mathbb{E}\left[ R - R_{\text{rag}} \right] \ge \Phi_M\left( \left( (n_{\text{pos}}+1)d^+ - 1 \right) n_{\text{pos}} \left(\int_{-1}^1 \Phi_M(v) dv - 1 \right) \right) - \Phi_M(0). \label{eq:dif_r}
\end{equation}

Let $n_{\text{rag}}=N_{\text{rag}}/2$. Combining \Cref{eq:dif_r} and \Cref{eq:npos_cheby}, we get that if $\underline{N}_{\text{pos}} > {N_{\text{rag}}}/{2} > 1 / d^+$, with probability $1 - \dfrac{N_{\text{rag}}}{4(\underline{N}_{\text{pos}}-{N_{\text{rag}}}/{2})^2}$, we have:
\begin{equation}
    \mathbb{E}\left[ R - R_{\text{rag}} \right] \ge \Phi_M\left( \dfrac{d^+\left(\int_{-1}^1 \Phi_M(v) dv - 1 \right)N_{\text{rag}}}{2} \right) - \Phi_M(0). \label{eq:r_dif}
\end{equation}

Let $R(Z)$ be the risk of $Z$ sampled from the distribution $\gD$.
Define the empirical risk $\hat{R}$ and $\hat{R}_{\text{rag}}$ as the following:
\begin{align}
    \hat{R} = \dfrac{1}{N_{\text{cal}}} \sum_{i=1}^{N_{\text{cal}}} R(Z_i), \quad \hat{R}_{\text{rag}} = \dfrac{1}{N_{\text{cal}}} \sum_{i=1}^{N_{\text{cal}}} R_{\text{rag}}(Z_i)
\end{align}

According to \Cref{eq:hb_ineq}, the statistical guarantee of conformal risk $\hat{\alpha}$ and $\hat{\alpha}_{\text{rag}}$ with confidence $1-\delta$ can be formulated as:
\begin{align}
    \hat{\alpha} &= \min \left\{ h^{-1}\left(\dfrac{\ln(1/\delta)}{N_{\text{cal}}};\hat{R}\right),\Phi^{-1}_{\text{bin}}\left(\dfrac{\delta}{e};N_{\text{cal}},\hat{R}\right) \right\}, \\
    \hat{\alpha}_{\text{rag}} &= \min \left\{ h^{-1}\left(\dfrac{\ln(1/\delta)}{N_{\text{cal}}};\hat{R}_{\text{rag}}\right),\Phi^{-1}_{\text{bin}}\left(\dfrac{\delta}{e};N_{\text{cal}},\hat{R}_{\text{rag}}\right) \right\},
\end{align}
where $\Phi^{-1}_{\text{bin}}(\cdot)$ is the inverse function of CDF of binomial distribution.
Noting that $\hat{\alpha}$ is monotonically increasing in $\hat{R}$, the following holds by Hoeffding's inequality: 
\begin{equation}
     \sP\left[ \hat{\alpha}_{\text{rag}} < \hat{\alpha} \right] \ge \sP\left[ \hat{R}_{\text{rag}} < \hat{R} \right] \ge 1 - \exp\left\{ -2N_{\text{cal}} \mathbb{E}\left[ R - R_{\text{rag}} \right]^2 \right\}. \label{eq:comp_alp}
\end{equation}

Combining \Cref{eq:r_dif} and \Cref{eq:comp_alp} and using the union bound, under the condition that $\underline{N}_{\text{pos}} > {N_{\text{rag}}}/{2} > 1 / d^+$, we have:
\begin{equation}
    \sP\left[ \hat{\alpha}_{\text{rag}} < \hat{\alpha} \right] \ge 1 - \exp\left\{ -2N_{\text{cal}} \left[ \Phi_M\left( \dfrac{d^+\left(\int_{-1}^1 \Phi_M(v) dv - 1 \right)N_{\text{rag}}}{2} \right) - \Phi_M(0) \right]^2 \right\} - \dfrac{N_{\text{rag}}}{4(\underline{N}_{\text{pos}}-{N_{\text{rag}}}/{2})^2}. \label{eq:final-1}
\end{equation}

Let $r_{\text{ext}}^m:=\min_{c \in \{1,..,C\}} r_{\text{ext}}^{(c)}$.
Then we can show that one sufficient condition of $\underline{N}_{\text{pos}} > {N_{\text{rag}}}/{2} > 1 / d^+$ is that $N_{\text{ext}}>\dfrac{2\sqrt{2 \ln10}}{r_{\text{ext}}^m}$, $N_{\text{ext}} V_{\text{rag}}^{0.25r_{\text{ext}}^m N_{\text{ext}}}<\dfrac{4}{9}$, and $N_{\text{rag}}>\dfrac{2}{d^+}$.
Rearranging the terms and considering the sufficient condition in \Cref{eq:final-1}, we can finally conclude that, under the condition that $N_{\text{ext}}>\dfrac{2\sqrt{2 \ln10}}{r_{\text{ext}}^m}$, $N_{\text{ext}} V_{\text{rag}}^{0.25r_{\text{ext}}^m N_{\text{ext}}}<\dfrac{4}{9}$, and $N_{\text{rag}}>\dfrac{2}{d^+}$, the following holds:
\begin{equation}
    \begin{aligned}
         \sP\left[ \hat{\alpha}_{\text{rag}} < \hat{\alpha} \right] \ge& 1 - \exp\left\{ -2N_{\text{cal}} \left[ \Phi_M\left( \dfrac{d^+\left(\int_{-1}^1 \Phi_M(v) dv - 1 \right)N_{\text{rag}}}{2} \right) - \Phi_M(0) \right]^2 \right\} \\ &- \dfrac{25}{N_{\text{rag}}} \left( 4 - 9 \sum_{c=1}^C r^{(c)}_{\text{cal}} \left(N_{\text{ext}} - r^{(c)}_{\text{ext}}N_{\text{ext}} + {\sqrt{2\ln{10}}} \right)  V_{\text{rag}}^{0.5 \left(r^{(c)}_{\text{ext}}N_{\text{ext}} - {\sqrt{2\ln{10}}} \right)} \right)^{-2}.
    \end{aligned}
\end{equation}
\end{proof}

\section{\Cref{thm1}: Asymtotic result of \Cref{thm:gene_rag}}
\label{app:asym}

\begin{corollary}[\Cref{thm:gene_rag} with a sufficiently large external knwoledge base]
\label{thm1}
Under the same conditions as \Cref{thm:gene_rag}, suppose that we have a \textbf{sufficiently large sample size} {$N_{\text{ext}}$} in the external knowledge base. We then have the following guarantee:
    \begin{equation}
    \label{eq:rag_benefit}
    \eqsmall
    \begin{aligned}
        & \sP\left[ \hat{\alpha}_{\text{rag}} < \hat{\alpha} \right] \ge 1 - p_{t} - \dfrac{25}{16 N_{\text{rag}}},
    \end{aligned}
    \end{equation}
where $p_{\text{t}}$ is the uncertainty induced by the quality of the transformer as formulated in \Cref{eq:conclu}, and $\hat{\alpha}_{\text{rag}}$ and $\hat{\alpha}$ are the conformal generation risks with and without RAG, respectively.
\end{corollary}

\begin{proof}[Proof of \Cref{thm1}]
    The proof directly follows that of \Cref{thm:gene_rag}. Considering \Cref{eq:conclu} in the asymptoptic limit (i.e., $N_{\text{ext}} \rightarrow +\infty$), we obtain the formulation in \Cref{eq:rag_benefit}.
\end{proof}

\begin{remark}
    \underline{{(R1)}} \Cref{thm1} shows that the conformal generation risk of transformer with RAG $\hat{\alpha}_{\text{rag}}$ is smaller than that of without RAG $\hat{\alpha}$ with high probability (RHS of \Cref{eq:rag_benefit}), which asymptotically approaches $1$ with a sufficiently large sizes of the calibration set $N_{\text{cal}}$ and retrieved augmented examples $N_{\text{rag}}$. 
    \underline{{(R2)}} \merveedit{In contrast to \Cref{thm:gene_rag}, the bound has no dependency on distributions in the external knowledge base \merveedit{$r_{\text{ext}}^{(c)}$} since a sufficiently large knowledge base can cover also rare examples.}
    \underline{{(R3)}} Compared to \Cref{thm:gene_rag}, the bound also has no dependency on the distribution of the calibration/test distribution \merveedit{$r_{\text{cal}}^{(c)}$}, showing that a sufficiently large external knowledge base can better generalize to unknown test distributions. 
    Additionally, the lower bound in \Cref{eq:rag_benefit} is tighter than that in \Cref{eq:conclu}, \merveedit{which demonstrates the benefit of the large external knowledge base}.
\end{remark}

\section{\Cref{pro:rag_shft}: Retrieval quality analysis under distribution shifts}
\label{app:rag_dist_shft}

Under test-time distribution shifts, retrieval model quality declines. To derive conformal generation risk, we first examine the lower bound of retrieved positive examples.
\vspace{-.2em}
\begin{proposition}[Lower bound to the retrieved positive examples under test-time distribution shifts]
\label{pro:rag_shft}
    Suppose that the potential test distribution $\eqnsmall{\gQ}$ is shifted from the original test distribution $\eqnsmall{\gD}$ with bounded Hellinger distance $\eqnsmall{\rho>0}$.
    Consider the same setup as \Cref{lem:rag} and a large external knowledge base where $\eqnsmall{N_{\text{ext}} > 2\sqrt{2\ln10}/\min_c r^{(c)}_{\text{ext}}}$.
    We have:
    \vspace{-.3em}
    \begin{equation}
    \label{eq:rag_shift}
    \eqsmall
        \mathbb{E}\left[ N_{\text{pos}} \right] \ge \dfrac{9}{10}N_{\text{rag}}\left(1- 1.5N_{\text{ext}}  V_{\text{rag}}(\rho)^{0.25 \left(\min_c r^{(c)}_{\text{ext}}N_{\text{ext}} \right)}  \right),
            \vspace{-.5em}
    \end{equation}
    where $\eqnsmall{V_{\text{rag}}(\rho):=m(\rho) V_{\text{rag}}}$ and
    \vspace{-.3em}
    \ifbool{IsTwoColumn}{
      \begin{equation*}
    \eqsmall
    \label{eq:v_rag_rho}
    \begin{aligned}
    \vspace{-1em}
         m(\rho) = \underbrace{\left( \dfrac{\sqrt{-6\rho^4+12\rho^2+1}-4\rho(1-\rho^2)\sqrt{2-\rho^2}}{1-16\rho^2+8\rho^4} \right)^{-2}}_{\text{retrieval model quality decay factor by distribution shifts}}.
    \end{aligned}
    \end{equation*}
    }{
      \begin{equation*}
    \eqsmall
    \label{eq:v_rag_rho}
    \begin{aligned}
    \vspace{-1em}
         V_{\text{rag}}(\rho):= m(\rho) V_{\text{rag}}, \quad \text{where} \quad
          m(\rho) = \underbrace{\left( \dfrac{\sqrt{-6\rho^4+12\rho^2+1}-4\rho(1-\rho^2)\sqrt{2-\rho^2}}{1-16\rho^2+8\rho^4} \right)^{-2}}_{\text{retrieval model quality decay factor under distribution shifts}}.
    \end{aligned}
    \end{equation*}
    }
\end{proposition}

\textit{Proof sketch}. We first formulate the expectation of similarity difference between positive pairs and negative pairs $\mathbb{E}[s_{\theta_r}(x,x^+)-s_{\theta_r}(x,x^-)]$ as a function of the contrastive loss of the retrieval model $L_\tau$. Then, we apply Chebyshev's inequality to upper bound the failure probability $\sP[ s_\theta(x,x^+) < s_\theta(x,x^-)]$ as a function of the variance and expectation of the similarity difference. Considering the distribution shifts, the variance and expectation bound can be derived via Grammian bound as \cite{weber2022certifying}. We then plug in the failure rate corrected by distribution shifts and follow the proof structure of \Cref{lem:rag}.

\begin{remark}
    \underline{{(R1)}} 
    Different from \Cref{lem:rag}, the quality of retrieval models under distribution shifts is decreased from $\eqnsmall{V_{\text{rag}}}$ to $\eqnsmall{V_{\text{rag}}(\rho)}$ with a linear decay factor $\eqnsmall{m(\rho)}$.
    As we require $\eqnsmall{V_{\text{rag}}(\rho)<1}$ to ensure high retrieval quality, large distribution shift radius $\eqnsmall{\rho}$ must be compensated by small $\eqnsmall{V_{\text{rag}}}$. This is consistent with the existing observations that low-variance models can generalize better under distribution shifts \cite{lam2016robust,gotoh2018robust,namkoong2017variance}. 
    \underline{(R2)} Compared to \Cref{lem:rag}, \Cref{pro:rag_shft} removes the dependency on varying label portions $\eqnsmall{r^{(c)}_{\text{cal}}}$ during distribution shifts, as long as the size of external knowledge base is sufficiently large $\eqnsmall{N_{\text{ext}}}$
    to offset the worst-case long-tail distributions, a condition often met in practice with large knowledge bases.
\end{remark}

\begin{proof}[Proof of \Cref{pro:rag_shft}]
    Let $\gD$ be the data distribution, which is also the training distribution of the retrieval model and conformal calibration distribution. 
    Let $\gQ$ be the test distribution where the test samples are drawn from.
    $\gQ$ is within Hellinger distance $\rho$ from the distribution $\gD$: $H(\gD,\gQ) \le \rho$.
    
    For a sample $(x,y) \sim \gD$, we denote $\gD_{\text{ext}}^+(x)$, $\gD_{\text{ext}}^-(x)$ be the distribution of positive examples (with the same groundtruth output $y$) and negative examples (with different groundtruth output to $y$) of sample $x$ in the external knowledge base. Then we can formulate the expectation of contrastive loss of the retrieval model as:
    \begin{equation}
    \begin{aligned}
        L_\tau &= \mathbb{E}_{x\sim \gD, x^+ \sim \gD^+_{\text{ext}}(x), x^- \sim \gD^-_{\text{ext}}(x)}\left[ \gL_{\text{contrastive}}(x,x^+,x^-) \right]\\
        &= \mathbb{E}_{x\sim \gD, x^+ \sim \gD^+_{\text{ext}}(x), x^- \sim \gD^-_{\text{ext}}(x)}\left[ - \log \dfrac{\exp\left\{s_\theta(x,x^+)\right\}}{\exp\left\{s_\theta(x,x^+)\right\} + \exp\left\{s_\theta(x,x^-)\right\} } \right].
    \end{aligned}
    \end{equation}
    
    
    We can apply Chebyshev's inequality \cite{saw1984chebyshev} to the random variable $s_\theta(x,x^+) - s_\theta(x,x^-)$ and get the following:
    \begin{align}
        \sP\left[ s_\theta(x,x^+) < s_\theta(x,x^-) \right] &=  \sP_{x\sim \gQ, x^+ \sim \gD^+_{\text{ext}}(x), x^- \sim \gD^-_{\text{ext}}(x)}\left[ s_\theta(x,x^+) - s_\theta(x,x^-) < 0 \right]\\ 
        &\le \dfrac{\mathbb{V}\left[s_\theta(x,x^+) - s_\theta(x,x^-)\right]}{\mathbb{E}\left[s_\theta(x,x^+) - s_\theta(x,x^-)\right]^2}. \label{ineq:bnd_sim_shft}
    \end{align}
    For ease of notation, we notate random variable $S=s_\theta(x,x^+) - s_\theta(x,x^-)$. Then we have $\sP[-2 \le S \le 2]=1$, and $\mathbb{E}[S] \ge \underline{\mathbb{E}[S]} := \dfrac{\exp\left\{ -L_\tau(\rho) \right\}}{1 - \exp\left\{ -L_\tau(\rho) \right\}}$.
    Note that $\sP[0\le S^2 \le 4]=1$, we have $\mathbb{V}[S^2] \le (4-0)^2 / 4 = 4$.
    Then by \Cref{lem:gramian-bound}, we have the following:
    \begin{align}
        \mathbb{E}_\gQ \left[ S^2 \right] &\le \mathbb{E}_\gD \left[ S^2 \right] + \rho^2(2-\rho^2)\left(1 - \mathbb{E}_\gD \left[ S^2 \right] \right) + 2 \rho(1-\rho^2)\sqrt{2-\rho^2} \sqrt{\mathbb{V}_\gD \left[ S^2 \right]} \\
        &\le (1-\rho^2)^2 \mathbb{E}_\gD \left[ S^2 \right] + \rho^2(2-\rho^2) + 4 \rho(1-\rho^2)\sqrt{2-\rho^2}. \label{eq:up_bnd_eq_s}
    \end{align}
    By applying the lower bound of expectation values in Theorem A.2 in \cite{weber2022certifying},which is a straightforward variation of \Cref{lem:gramian-bound}, we have the following:
    \begin{equation}
        \mathbb{E}_\gQ\left[S\right] \ge (1-\rho^2)^2 \mathbb{E}_\gD\left[S\right] - 2 \rho(1-\rho^2)\sqrt{2-\rho^2} \sqrt{\mathbb{V}_\gD\left[S\right]} + \rho^2(2-\rho^2) \dfrac{\mathbb{V}_\gD\left[S\right]}{\mathbb{E}_\gD\left[S\right]}.
    \end{equation}
    Since we assume that $\mathbb{E}_\gD\left[S\right] = \mathbb{E}_\gD\left[s_\theta(x,x^+) - s_\theta(x,x^-)\right]>0$, we have the following:
    \begin{equation}
    \label{eq:low_bnd_eq_s}
         \mathbb{E}_\gQ\left[S\right] \ge (1-\rho^2)^2 \mathbb{E}_\gD\left[S\right] - 2 \rho(1-\rho^2)\sqrt{2-\rho^2} \sqrt{\mathbb{V}_\gD\left[S\right]}.
    \end{equation}
    Combining \Cref{eq:up_bnd_eq_s,eq:low_bnd_eq_s}, we have the following:
    \begin{align}
        \dfrac{\mathbb{V}_\gQ\left[S\right]}{\mathbb{E}_\gQ\left[S\right]^2} &= \dfrac{\mathbb{E}_\gQ \left[ S^2 \right] - \mathbb{E}_\gQ \left[ S \right]^2}{\mathbb{E}_\gQ\left[S\right]^2} \\
        &\le \dfrac{(1-\rho^2)^2 \left(\mathbb{V}_\gD \left[ S \right] + \mathbb{E}_\gD\left[S\right]^2 \right) + \rho^2(2-\rho^2) + 4 \rho(1-\rho^2)\sqrt{2-\rho^2}}{\mathbb{E}_\gQ\left[S\right]^2} - 1 \\
        &\le \dfrac{(1-\rho^2)^2 \left(\mathbb{V}_\gD \left[ S \right] + \mathbb{E}_\gD\left[S\right]^2 \right) + \rho^2(2-\rho^2) + 4 \rho(1-\rho^2)\sqrt{2-\rho^2}}{\left( (1-\rho^2)^2 \mathbb{E}_\gD\left[S\right] - 2 \rho(1-\rho^2)\sqrt{2-\rho^2} \sqrt{\mathbb{V}_\gD\left[S\right]} \right)^2} - 1.
    \end{align}
    Through some algebraic rearrangement, we can show that:
    \begin{equation}
    \label{eq:bnd_v_e2}
        \dfrac{\mathbb{V}_\gQ\left[S\right]}{\mathbb{E}_\gQ\left[S\right]^2} \le V_{\text{rag}}(\rho):= \dfrac{\mathbb{V}_\gD\left[S\right]}{\mathbb{E}_\gD\left[S\right]^2} \left( \dfrac{\sqrt{-6\rho^4+12\rho^2+1}-4\rho(1-\rho^2)\sqrt{2-\rho^2}}{1-16\rho^2+8\rho^4} \right)^{-2}.
    \end{equation}
    Therefore, one sufficient condition of $V_{\text{rag}}(\rho)<1$ is that:
    \begin{align}
        \sqrt{\mathbb{V}_\gD\left[S\right]} &\le \left( \dfrac{\sqrt{-6\rho^4+12\rho^2+1}-4\rho(1-\rho^2)\sqrt{2-\rho^2}}{1-16\rho^2+8\rho^4} \right) \mathbb{E}_\gD\left[S\right] \\
        &= \left( \dfrac{\sqrt{-6\rho^4+12\rho^2+1}-4\rho(1-\rho^2)\sqrt{2-\rho^2}}{1-16\rho^2+8\rho^4} \right) \ln \dfrac{\exp\left\{ -L_\tau \right\}}{1 - \exp\left\{ -L_\tau \right\}}
    \end{align}

    Then, we follow a similar procedure as the proof of \Cref{lem:rag} to analyze the expected retrieved positive examples. 
    We first focus on one demonstration example $Z_r$ retrieved by $s_\theta(\cdot,\cdot)$. According to the retrieval mechanism, the example $Z_r$ has the highest similarity (measured by $s_\theta(\cdot,\cdot)$) to the query sample $Z_q$. 
    Recall that $r^{(c)}_{\text{test}}(\gQ)$ and $r^{(c)}_{\text{ext}}$ be the event probability of the $c$-th category in the categorical test distribution (i.e., $\gQ$) and categorical external knowledge distribution.
    Let $\vr_{\text{ext}}=\left[ r^{(1)}_{\text{ext}}, r^{(2)}_{\text{ext}},...,r^{(C)}_{\text{ext}}\right]$. Since we only have $N_{\text{ext}}$ finite sample drawn from $\gD_{\text{ext}}$ in the external knowledge base in practice, we notate the empirical categorical portions as $\hat{\vq}_{\text{ext}} = \left[ \hat{r}^{(1)}_{\text{ext}}, \hat{r}^{(2)}_{\text{ext}},...,\hat{r}^{(C)}_{\text{ext}}\right]$, where $\hat{r}^{(c)}_{\text{ext}}~(c \in \{1,..,C\})$ represents the portion of samples with grountruth output $c$ in the external knowledge base.
    Then we can apply the concentration bound of categorical distribution as \cite{agrawal2017optimistic}:
    \begin{align}
        \sP\left[ \|\hat{\vr}_{\text{ext}} - \vr_{\text{ext}} \|_1 \ge \dfrac{\sqrt{2\ln{(1/\delta_{\text{ext}})}}}{N_{\text{ext}}} \right] \le \delta_{\text{ext}}, \label{ineq:fin_cat_shft}
    \end{align}
    where $\delta_{\text{ext}}>0$ represents the confidence level of the tail bound.
    Then we can upper bound the probability that the groundtruth output $g(Z_r)$ of the retrieved sample $Z_r$ is not equal to the groundtruth output of the query sample $g(Z_q)$ as the following:
    \begin{align}
        \sP\left[ g(Z_r) \notin  g(Z_q) \right] &= \sP_{Z_q \sim \gD}\left[ g(Z_r) \notin  g(Z_q)  \left| s_\theta(Z_q, Z_r) \ge \max_{z \in \gD_{\text{ext}}} s_\theta (z, Z_q) \right. \right] \\
        &= \sP_{Z_q \sim \gD}\left[ \max_{Z^- \in \gD^-_{\text{ext}}(Z_q)}  s_\theta(Z^-, Z_q) \ge  \max_{Z^+ \in \gD^+_{\text{ext}}(Z_q)} s_\theta(Z^+, Z_q) \right] \\
        &= \sP_{Z_q \sim \gD}\left[ s_\theta(Z^-, Z_q) \ge  s_\theta(Z^+, Z_q), ~~\forall Z^+ \in \gD^+_{\text{ext}}(Z_q),~ \exists Z^- \in \gD^-_{\text{ext}}(Z_q)  \right] \\
        &\le \max_\gQ \sum_{c=1}^C r^{(c)}_{\text{test}}(\gQ) \left(1 - r^{(c)}_{\text{ext}}  \right) N_{\text{ext}} \sP\left[ s_\theta(x,x^+) < s_\theta(x,x^-) \right]^{N_{\text{ext}} r^{(c)}_{\text{ext}}}, \label{ineq:rag_ext_shft}
    \end{align}
    where \Cref{ineq:rag_ext_shft} is derived by applying the union bound. Considering finite-sample error of categorical distribution in \Cref{ineq:fin_cat_shft} and combining \Cref{ineq:bnd_sim_shft,eq:bnd_v_e2}, we finally have:
    \begin{align}
        \sP\left[ g(Z_r) \notin  g(Z_q) \right] &\le \max_\gQ \sum_{c=1}^C r^{(c)}_{\text{test}}(\gQ) \left(1 - r^{(c)}_{\text{ext}}  \right) N_{\text{ext}} \sP\left[ s_\theta(x,x^+) < s_\theta(x,x^-) \right]^{N_{\text{ext}} r^{(c)}_{\text{ext}}} \\
        &\le \max_\gQ \sum_{c=1}^C r^{(c)}_{\text{test}}(\gQ) \left(1 - r^{(c)}_{\text{ext}} + \dfrac{\sqrt{2\ln{(1/\delta_{\text{ext}})}}}{N_{\text{ext}}} \right) N_{\text{ext}} V_{\text{rag}}(\rho)^{0.5N_{\text{ext}} \left(r^{(c)}_{\text{ext}} - {\sqrt{2\ln{(1/\delta_{\text{ext}})}}}/{N_{\text{ext}}} \right)} \label{ineq:failrate_shft}
    \end{align}
    Since we assume that the retrieval model retrieves samples identically from the external knowledge base, the number of retrieved positive examples $N_{\text{pos}}$ follows a Binomial distribution with $N_{\text{rag}}$ trials and failure rate in \Cref{ineq:failrate_shft}. Therefore, we can lower bound the expectation of $N_{\text{pos}}$ as the following:
    \ifbool{IsTwoColumn}
    {
    \begin{align}
        \mathbb{E}\left[ N_{\text{pos}} \right] &\ge \min_\gQ N_{\text{rag}} \left(1-\delta_{\text{ext}}\right)\left(1- \sum_{c=1}^C r^{(c)}_{\text{test}}(\gQ) \left(1 - r^{(c)}_{\text{ext}} + \dfrac{\sqrt{2\ln{(1/\delta_{\text{ext}})}}}{N_{\text{ext}}} \right) N_{\text{ext}} V_{\text{rag}}(\rho)^{0.5N_{\text{ext}} \left(r^{(c)}_{\text{ext}} - {\sqrt{2\ln{(1/\delta_{\text{ext}})}}}/{N_{\text{ext}}} \right)}  \right)
    \end{align}
    }
    {
    \small{
    \begin{align}
        \mathbb{E}\left[ N_{\text{pos}} \right] &\ge \min_\gQ N_{\text{rag}} \left(1-\delta_{\text{ext}}\right)\left(1- \sum_{c=1}^C r^{(c)}_{\text{test}}(\gQ) \left(1 - r^{(c)}_{\text{ext}} + \dfrac{\sqrt{2\ln{(1/\delta_{\text{ext}})}}}{N_{\text{ext}}} \right) N_{\text{ext}} V_{\text{rag}}(\rho)^{0.5N_{\text{ext}} \left(r^{(c)}_{\text{ext}} - {\sqrt{2\ln{(1/\delta_{\text{ext}})}}}/{N_{\text{ext}}} \right)}  \right)
    \end{align}
    }}
    which holds for any $\delta_{\text{ext}}>0$. Therefore, letting $\delta_{\text{ext}}=0.1$, we can finally derive the following:
     \begin{align}
        \mathbb{E}\left[ N_{\text{pos}} \right] &\ge \min_\gQ \dfrac{9}{10}N_{\text{rag}}\left(1- \sum_{c=1}^C r^{(c)}_{\text{test}}(\gQ) \left(N_{\text{ext}} - r^{(c)}_{\text{ext}}N_{\text{ext}} + {\sqrt{2\ln{10}}} \right)  V_{\text{rag}}(\rho)^{0.5 \left(r^{(c)}_{\text{ext}}N_{\text{ext}} - {\sqrt{2\ln{10}}} \right)}  \right) \\
        &\ge \dfrac{9}{10}N_{\text{rag}}\left(1- 1.5N_{\text{ext}}  V_{\text{rag}}(\rho)^{0.25 \left(\min_c r^{(c)}_{\text{ext}}N_{\text{ext}} \right)}  \right),
    \end{align}
    with a sample size in the external knowledge base such that $N_{\text{ext}} > 2\sqrt{2\ln10}/\min_c r^{(c)}_{\text{ext}}$.
    
\end{proof}

\section{Proofs and detailed remarks in \Cref{sec:distribution_shift}}
\label{app:proof_second}

\subsection{Proof and detailed remark of \Cref{pro:conf_shf}}

\begin{remark}
    \underline{{(R1)}} 
    Different from \Cref{lem:rag}, the quality of retrieval models under distribution shifts is decreased from $\eqnsmall{V_{\text{rag}}}$ to $\eqnsmall{V_{\text{rag}}(\rho)}$ with a linear decay factor $\eqnsmall{m(\rho)}$.
    As we require $\eqnsmall{V_{\text{rag}}(\rho)<1}$ to ensure high retrieval quality, large distribution shift radius $\eqnsmall{\rho}$ must be compensated by small $\eqnsmall{V_{\text{rag}}}$. This is consistent with the existing observations that low-variance models can generalize better under distribution shifts \cite{lam2016robust,gotoh2018robust,namkoong2017variance}. 
    \underline{(R2)} Compared to \Cref{lem:rag}, \Cref{pro:rag_shft} removes the dependency on varying label portions $\eqnsmall{r^{(c)}_{\text{cal}}}$ during distribution shifts, as long as the size of external knowledge base is sufficiently large $\eqnsmall{N_{\text{ext}}}$
    to offset the worst-case long-tail distributions, a condition often met in practice with large knowledge bases.
\end{remark}

\textit{Proof sketch}. We first formulate the expectation of similarity difference between positive pairs and negative pairs $\mathbb{E}[s_{\theta_r}(x,x^+)-s_{\theta_r}(x,x^-)]$ as a function of the contrastive loss of the retrieval model $L_\tau$. Then, we apply Chebyshev's inequality to upper bound the failure probability $\sP[ s_\theta(x,x^+) < s_\theta(x,x^-)]$ as a function of the variance and expectation of the similarity difference. Considering the distribution shifts, the variance and expectation bound can be derived via Grammian bound as \cite{weber2022certifying}. We then plug in the failure rate corrected by distribution shifts and follow the proof structure of \Cref{lem:rag}.

\begin{proof}[Proof of \Cref{pro:conf_shf}]
    Recall that the calibration samples $(Z_1,Z_2,...,Z_{N_{\text{cal}}})$ are drawn from the distribution $\gD$ and the test sample $Z_{\text{test}}=(X_{\text{test}},Y_{\text{test}})$ is sampled from the distribution $\gQ$.
    The distribution $\gQ$ and $\gD$ have bounded Hellinger distance $H(\gQ,\gD)\le \rho$ and the risk is upper bounded by $1$.
    Let $R_\gQ$, $R_\gD$ be the population risk on distribution $\gQ$, $\gD$ and $V_\gD$ be the variance of risk on distribution $\gD$.
    We can apply \Cref{lem:gramian-bound} and get the following:
    \begin{equation}
    \label{eq:gram_1}
        R_\gQ \le R_\gD + 2 \rho(1-\rho^2)\sqrt{2-\rho^2} \sqrt{V_\gD} + \rho^2(2-\rho^2)\left(1 - R_\gD - \dfrac{V_\gD}{1-R_\gD} \right),
    \end{equation}
    with any given distance bound $\rho>0$ that satisfies:
    \begin{equation}
        \rho^2 \le 1 - \left( 1 + \dfrac{(1 - R_\gD)^2}{V_\gD} \right)^{-1/2}.
    \end{equation}
    Since the variance is non-negative, \Cref{eq:gram_1} further implies that:
    \begin{equation}
    \label{eq:gram_2}
        R_\gQ \le R_\gD + 2 \rho(1-\rho^2)\sqrt{2-\rho^2} \sqrt{V_\gD} + \rho^2(2-\rho^2)\left(1 - R_\gD \right).
    \end{equation}
    Then we will consider the finite-sample error of the calibration set. From Hoeffding's inequality \cite{hoeffding1994probability}, with probability $1-\delta/4$, we have:
    \begin{equation}
    \label{eq:fin_mean}
        R_\gD \le \hat{R} + \sqrt{\dfrac{\ln(4/\delta)}{2N_{\text{cal}}}}.
    \end{equation}
    From sample variance bound in \cite{maurer2009empirical}, with probability $1-\delta/4$, we have:
    \begin{equation}
    \label{eq:fin_var}
        \sqrt{V_\gD}  \le  \sqrt{\hat{V}} + \sqrt{\dfrac{2 \ln(4/\delta)}{N_{\text{cal}}-1}}.
    \end{equation}
    Note that the RHS of \Cref{eq:gram_2} monotonically increases in $R_\gD$ and $V_\gD$, combining \Cref{eq:fin_mean,eq:fin_var,eq:gram_2} and applying the union bound, with probability $1-\delta/2$, we have:
    \begin{equation}
    \label{eq:fin_gram_1}
    \begin{aligned}
        R_\gQ \le \hat{R} + \rho^2(2-\rho^2)\left( 1 - \hat{R} \right) + 2\rho(1-\rho^2)\sqrt{2-\rho^2} \sqrt{\hat{V}} 
         + (1-\rho^2)(\dfrac{1-\rho^2}{\sqrt{2N_{\text{cal}}}}+\dfrac{2\sqrt{2}\rho\sqrt{2-\rho^2}}{\sqrt{N_{\text{cal}}-1}}) \sqrt{\ln(4/\delta)}.
    \end{aligned}
    \end{equation}
    By two-sided Hoeffding's inequality, with probability $1-\delta/4$, we have:
    \begin{equation}
    \label{eq:fin_gram_2}
        \hat{R}_\rho \le R_\gQ + \sqrt{\dfrac{\ln{(8/\delta)}}{2N_{\text{cal}}}}.
    \end{equation}
    Combining \Cref{eq:fin_gram_1,eq:fin_gram_2} and applying the union bound, with probability $1-3\delta/4$, we have:
    \begin{equation}
    \label{eq:r_hat_ubnd}
        \hat{R}_\rho \le \hat{R} + \rho^2(2-\rho^2)\left( 1 - \hat{R} \right) + 2\rho(1-\rho^2)\sqrt{2-\rho^2} \sqrt{\hat{V}} 
         + (1-\rho^2)(\dfrac{1-\rho^2}{\sqrt{2N_{\text{cal}}}}+\dfrac{2\sqrt{2}\rho\sqrt{2-\rho^2}}{\sqrt{N_{\text{cal}}-1}}) \sqrt{\ln(4/\delta)} + \sqrt{\dfrac{\ln{(8/\delta)}}{2N_{\text{cal}}}}.
    \end{equation}
    Recall that the controlled conformal risk $\hat{\alpha}$ can be formulated as a function of the empirical risk $\hat{R}$ as \Cref{eq:hb_ineq}. When $(X_{\text{test}},Y_{\text{test}})$ is sampled from $\gD$, the guarantee of conformal risk is as follows:
    \begin{equation}
    \label{eq:shf_final-2}
        \sP \left[ R(T_{\hat{\vlambda}};X_{\text{test}},Y_{\text{test}})  \le \hat{\alpha} := \min \left\{ h^{-1}\left(\dfrac{\ln(8/\delta)}{N_{\text{cal}}};\hat{R} \right),\Phi^{-1}_{\text{bin}}\left(\dfrac{\delta}{8e};N_{\text{cal}},\hat{R} \right) \right\}\right] \ge 1 - \delta/4,
    \end{equation}
     where $h^{-1}(\cdot;\cdot)$ is the partial inverse function such that $h^{-1}(h(a,b);a)=b$ with $h_1(a,b)=a\log(a/b) + (1-a)\log((1-a)/(1-b))$. 
     Note that $\hat{\alpha}$ monotonically increases in $\hat{R}$.
     By \Cref{eq:r_hat_ubnd}, with probability $1-3\delta/4$, the following holds about the conformal risk on distribution $\gQ$ (denoted by $\hat{\alpha}_\rho$), which is within bounded Hellinger distance $\rho$ to the original distribution $\gD$:
     \begin{equation}
     \label{eq:shf_final-1}
         \hat{\alpha}_\rho \le \min \left\{ h^{-1}\left(\dfrac{\ln(8/\delta)}{N_{\text{cal}}};\overline{\hat{R}_\rho} \right),\Phi^{-1}_{\text{bin}}\left(\dfrac{\delta}{8e};N_{\text{cal}},\overline{\hat{R}_\rho} \right) \right\},
     \end{equation}
     where $\overline{\hat{R}_\rho}$ is formulated as:
     \begin{equation}
         \overline{\hat{R}_\rho} = \hat{R} + \rho^2(2-\rho^2)\left( 1 - \hat{R} \right) + 2\rho(1-\rho^2)\sqrt{2-\rho^2} \sqrt{\hat{V}} 
         + (1-\rho^2)(\dfrac{1-\rho^2}{\sqrt{2N_{\text{cal}}}}+\dfrac{2\sqrt{2}\rho\sqrt{2-\rho^2}}{\sqrt{N_{\text{cal}}-1}}) \sqrt{\ln(4/\delta)} + \sqrt{\dfrac{\ln{(8/\delta)}}{2N_{\text{cal}}}}.
     \end{equation}
     Combining \Cref{eq:shf_final-1,eq:shf_final-2} and applying the union bound, we can finally conclude that:
     \begin{equation}
         \sP \left[ R(T_{\hat{\vlambda}};X_{\text{test}},Y_{\text{test}})  \le \hat{\alpha}_\rho := \min \left\{ h^{-1}\left(\dfrac{\ln(8/\delta)}{N_{\text{cal}}};\overline{\hat{R}_\rho}\right),\Phi^{-1}_{\text{bin}}\left(\dfrac{\delta}{8e};N_{\text{cal}}, \overline{\hat{R}_\rho} \right) \right\}\right] \ge 1 - \delta,
     \end{equation}
     where $\overline{\hat{R}_\rho}$ is formulated as:
     \begin{equation}
         \overline{\hat{R}_\rho} = \hat{R} + \rho^2(2-\rho^2)\left( 1 - \hat{R} \right) + 2\rho(1-\rho^2)\sqrt{2-\rho^2} \sqrt{\hat{V}} 
         + (1-\rho^2)(\dfrac{1-\rho^2}{\sqrt{2N_{\text{cal}}}}+\dfrac{2\sqrt{2}\rho\sqrt{2-\rho^2}}{\sqrt{N_{\text{cal}}-1}}) \sqrt{\ln(4/\delta)} + \sqrt{\dfrac{\ln{(8/\delta)}}{2N_{\text{cal}}}}.
     \end{equation}
\end{proof}

\subsection{Proof and detailed remark of \Cref{thm:comp_shft}}

\begin{remark}
    Our result rigorously characterizes the effect of distribution shift on the reduced risk guarantee of RAG. 
    \underline{{(R1)}} Compared to \Cref{thm:gene_rag}, only the uncertainty of retrieval model $\eqnsmall{p_r(\rho)}$ is affected by the distribution shift $\eqnsmall{\rho}$. This affect is reflected on the the retrieval quality $\eqnsmall{V_{\text{rag}}(\rho)}$. In particular, a large distance radius $\eqnsmall{\rho}$ will downgrade the retrieval quality $\eqnsmall{V_{\text{rag}}(\rho)}$ and thus lead to a higher uncertainty $\eqnsmall{p_r(\rho)}$. However, the influence of $\eqnsmall{\rho}$ on $\eqnsmall{p_r(\rho)}$ can be reduced by $\eqnsmall{N_{\text{rag}}}$ inverse proportionally and by $\eqnsmall{N_{\text{ext}}}$ exponentially, demonstrating the robustness of RAG with more retrieval knowledge.
    \underline{{(R2)}} Since $\eqnsmall{V_{\text{rag}}(\rho)}$ is proportional to model variance $\eqnsmall{V_{\text{rag}}}$, a low-variance retrieval model demonstrates better robustness against distribution drifts, aligning with existing empirical observations \cite{lam2016robust,gotoh2018robust,namkoong2017variance}, which evaluate the generalization ability of low-variance retrieval models under distribution shifts.
     Different from \Cref{lem:rag}, the quality of retrieval models under distribution shifts is decreased from $\eqnsmall{V_{\text{rag}}}$ to $\eqnsmall{V_{\text{rag}}(\rho)}$ with a linear decay factor $\eqnsmall{m(\rho)}$.
    As we require $\eqnsmall{V_{\text{rag}}(\rho)<1}$ to ensure high retrieval quality, large distribution shift radius $\eqnsmall{\rho}$ must be compensated by small $\eqnsmall{V_{\text{rag}}}$. This is consistent with the existing observations that low-variance models can generalize better under distribution shifts \cite{lam2016robust,gotoh2018robust,namkoong2017variance}. 
    \underline{(R3)} Compared to \Cref{thm:gene_rag}, \Cref{thm:comp_shft} has no dependence on varying label portions $\eqnsmall{r^{(c)}_{\text{cal}}}$ during distribution shifts, as long as the size of external knowledge base $\eqnsmall{N_{\text{ext}}}$ is moderately large,
    ($\eqnsmall{N_{\text{ext}}>2\sqrt{2\ln10}/\min_c r^{(c)}_{\text{ext}}}$)
    to offset the worst-case long-tail distributions, 
    a condition often met in practice with large knowledge bases.
\end{remark}

\textit{Proof sketch}. We apply \Cref{pro:rag_shft} to provide the lower bound of the retrieved positive examples under distribution shifts. We plug in the term and analyze the functionality (logit difference statistics) of the self-attention transformer as proof of \Cref{thm:gene_rag}. Connecting the logit difference statistics to the empirical risks and the distribution-shifted conformal risk bound in \Cref{pro:conf_shf} finally concludes the proof.

\begin{proof}[Proof of \Cref{thm:comp_shft}]
    From \Cref{pro:rag_shft}, we prove that:
    \begin{equation}
    \label{eq:n_pos_lbnd_shft}
        \mathbb{E}\left[ N_{\text{pos}} \right] \ge \underline{N}_{\text{pos}} := \dfrac{9}{10}N_{\text{rag}}\left(1- 1.5N_{\text{ext}}  V_{\text{rag}}(\rho)^{0.25 \left(\min_c r^{(c)}_{\text{ext}}N_{\text{ext}} \right)}  \right).
    \end{equation}
    Since $N_{\text{pos}}$ is a binomial random variable with $N_{\text{rag}}$ trials, we have the upper bound of the variance $\mathbb{V}[N_{\text{pos}}] \le \dfrac{N_{\text{rag}}}{4}$.
    Applying Chebyshev's inequality to the random variable $N_{\text{pos}}$, the following holds $\forall n_{\text{pos}}< \underline{N}_{\text{pos}}$:
    \begin{equation}
        \sP\left[ N_{\text{pos}} \ge n_{\text{pos}} \right] \ge 1 - \dfrac{\mathbb{V}[N_{\text{rag}}]}{(\underline{N}_{\text{pos}}-n_{\text{pos}})^2} 
        \ge 1 - \dfrac{N_{\text{rag}}}{4(\underline{N}_{\text{pos}}-n_{\text{pos}})^2}, \label{eq:npos_cheby_shft}
    \end{equation}
    which implicates that we can do the analysis with $N_{\text{pos}} \ge n_{\text{pos}}$ with probability $1 - \dfrac{\mathbb{V}[N_{\text{rag}}]}{(\underline{N}_{\text{pos}}-n_{\text{pos}})^2} $.

    According to the proof of \Cref{thm:gene_rag}, by \Cref{eq:dif_r}, we have the following:
    \begin{equation}
        \mathbb{E}\left[ R - R_{\text{rag}} \right] \ge \Phi_M\left( \left( (n_{\text{pos}}+1)d^+ - 1 \right) n_{\text{pos}} \left(\int_{-1}^1 \Phi_M(v) dv - 1 \right) \right) - \Phi_M(0).
        \label{eq:dif_r_shft}
    \end{equation}.

    Let $n_{\text{rag}}=N_{\text{rag}}/2$. Combining \Cref{eq:dif_r_shft} and \Cref{eq:npos_cheby_shft}, we get that if $\underline{N}_{\text{pos}} > {N_{\text{rag}}}/{2} > 1 / d^+$, with probability $1 - \dfrac{N_{\text{rag}}}{4(\underline{N}_{\text{pos}}-{N_{\text{rag}}}/{2})^2}$, we have:
    \begin{equation}
        \mathbb{E}\left[ R - R_{\text{rag}} \right] \ge \Phi_M\left( \dfrac{d^+\left(\int_{-1}^1 \Phi_M(v) dv - 1 \right)N_{\text{rag}}}{2} \right) - \Phi_M(0). \label{eq:r_dif_shft}
    \end{equation}

    Let $R(Z)$ be the risk of $Z$ sampled from the distribution $\gQ$.
    Define the empirical risk $\hat{R}$ and $\hat{R}_{\text{rag}}$ as the following:
    \begin{align}
        \hat{R} = \dfrac{1}{N_{\text{cal}}} \sum_{i=1}^{N_{\text{cal}}} R(Z_i), \quad \hat{R}_{\text{rag}} = \dfrac{1}{N_{\text{cal}}} \sum_{i=1}^{N_{\text{cal}}} R_{\text{rag}}(Z_i)
    \end{align}

    Note that since the risk $R(\cdot)$ is bounded in $[0,1]$, the variance estimator $\hat{V}=\dfrac{1}{N_{\text{cal}}(N_{\text{cal}}-1)} \sum_{1\le i < j \le N_{\text{cal}}} (R(Z_i)-R(Z_j))^2$ is bounded in $[0,1]$.
    Leveraging the fact and according to \Cref{pro:conf_shf}, the statistical guarantee of conformal risk $\hat{\alpha}_\rho$ and $\hat{\alpha}_\rho^{\text{rag}}$ with confidence $1-\delta$ can be formulated as:
    \begin{equation}
        \begin{aligned}
            & \hat{\alpha}_\rho := \min \left\{ h^{-1}\left(\dfrac{\ln(8/\delta)}{N_{\text{cal}}};\overline{\hat{R}_\rho}\right),\Phi^{-1}_{\text{bin}}\left(\dfrac{\delta}{8e};N_{\text{cal}}, \overline{\hat{R}_\rho} \right) \right\} \\
            \text{where} ~~ & \overline{\hat{R}_\rho} = {\hat{R} + \rho^2(2-\rho^2)( 1 - \hat{R} )} + 2\rho(1-\rho^2)\sqrt{2-\rho^2} 
         + C(\rho,N_{\text{cal}}),
        \end{aligned}
    \end{equation}
    \begin{equation}
        \begin{aligned}
            & \hat{\alpha}_\rho^{\text{rag}} := \min \left\{ h^{-1}\left(\dfrac{\ln(8/\delta)}{N_{\text{cal}}};\overline{\hat{R}_\rho^{\text{rag}}}\right),\Phi^{-1}_{\text{bin}}\left(\dfrac{\delta}{8e};N_{\text{cal}}, \overline{\hat{R}_\rho^{\text{rag}}} \right) \right\} \\
            \text{where} ~~ & \overline{\hat{R}_\rho^{\text{rag}}} = {\hat{R}_{\text{rag}} + \rho^2(2-\rho^2)( 1 - \hat{R}_{\text{rag}} )} + 2\rho(1-\rho^2)\sqrt{2-\rho^2} 
         + C(\rho,N_{\text{cal}}),
        \end{aligned}
    \end{equation}
    where $C(\rho,N_{\text{cal}})={(1-\rho^2)\left(\dfrac{1-\rho^2}{\sqrt{2N_{\text{cal}}}}+\dfrac{2\sqrt{2}\rho\sqrt{2-\rho^2}}{\sqrt{N_{\text{cal}}-1}}\right) \sqrt{\ln(4/\delta)} + \sqrt{\dfrac{\ln{(8/\delta)}}{2N_{\text{cal}}}}}$.

    Noting that $\hat{\alpha}_\rho$ is monotonically increasing in $\overline{\hat{R}_\rho}$ and $\overline{\hat{R}_\rho}$ is monotonically increasing in $\hat{R}$, $\hat{\alpha}_\rho$ is monotonically increasing in $\hat{R}$.
    Then, the following holds by Hoeffding's inequality: 
    \begin{equation}
         \sP\left[ \hat{\alpha}_\rho^{\text{rag}} < \hat{\alpha}_\rho \right] \ge \sP\left[ \hat{R}_{\text{rag}} < \hat{R} \right] \ge 1 - \exp\left\{ -2N_{\text{cal}} \mathbb{E}\left[ R - R_{\text{rag}} \right]^2 \right\}. \label{eq:comp_alp_shft}
    \end{equation}

    Combining \Cref{eq:r_dif_shft} and \Cref{eq:comp_alp_shft} and using the union bound, under the condition that $\underline{N}_{\text{pos}} > {N_{\text{rag}}}/{2} > 1 / d^+$, we have:
    \begin{equation}
        \sP\left[ \hat{\alpha}_{\text{rag}} < \hat{\alpha} \right] \ge 1 - \exp\left\{ -2N_{\text{cal}} \left[ \Phi_M\left( \dfrac{d^+\left(\int_{-1}^1 \Phi_M(v) dv - 1 \right)N_{\text{rag}}}{2} \right) - \Phi_M(0) \right]^2 \right\} - \dfrac{N_{\text{rag}}}{4(\underline{N}_{\text{pos}}-{N_{\text{rag}}}/{2})^2}. \label{eq:final-1_shft}
    \end{equation}

    Combining \Cref{eq:n_pos_lbnd_shft,eq:final-1_shft}, under the condition that $N_{\text{ext}}>\dfrac{2\sqrt{2 \ln10}}{r_{\text{ext}}^m}$, $N_{\text{ext}} V_{\text{rag}}(\rho)^{0.25r_{\text{ext}}^m N_{\text{ext}}}<\dfrac{8}{17}$, and $N_{\text{rag}}>\dfrac{2}{d^+}$, we can finally conclude that:
    \begin{equation}
    \begin{aligned}
         \sP\left[ \hat{\alpha}_{\text{rag}} < \hat{\alpha} \right] \ge& 1 - \exp\left\{ -2N_{\text{cal}} \left[ \Phi_M\left( \dfrac{d^+\left(\int_{-1}^1 \Phi_M(v) dv - 1 \right)N_{\text{rag}}}{2} \right) - \Phi_M(0) \right]^2 \right\} \\ &- \dfrac{100}{N_{\text{rag}}} \left( 8 - 17  N_{\text{ext}}  V_{\text{rag}}(\rho)^{0.25 \left(\min_c r^{(c)}_{\text{ext}}N_{\text{ext}} \right)} \right)^{-2}.
    \end{aligned}
\end{equation}
    
\end{proof}

\section{Additional evaluation results}

\begin{algorithm}[t]
    \caption{Test distribution sampling protocol.}
    \label{alg:simulation_no_shft}
    \begin{algorithmic}[1]
        \STATE {\bfseries Input}: original test set $\gD$, test sample pool $\gD_{\text{pool}}$, Risk function $R(\cdot,\cdot): \gX \times \gY \mapsto \sR$
        \STATE {\bfseries Output}: empirical risk on the sampled test set $\gQ$ $\hat{R}_\gQ$
        \STATE Randomly sample $\gQ$ from $\gD_{\text{pool}}$ with equalized set size as $\gD$: $|\gQ|=|\gD|$
        \STATE Evaluate empirical risks for all samples in $\gQ$: $\hat{{R}}_\gQ = \dfrac{1}{|\gQ|} \sum_{(x,y)\in \gQ} R(x,y)$
        \STATE \textbf{Return} $\hat{{R}}_\gQ$
    \end{algorithmic}
\end{algorithm}

\begin{figure*}[t]
\subfigure{
    \rotatebox{90}{\hspace{-3.5em} Evaluation Risk}
    \begin{minipage}{0.24\linewidth}
    \centerline{\footnotesize{\quad AESLC}}
 	\vspace{1pt}
\centerline{\includegraphics[width=1.0\textwidth]{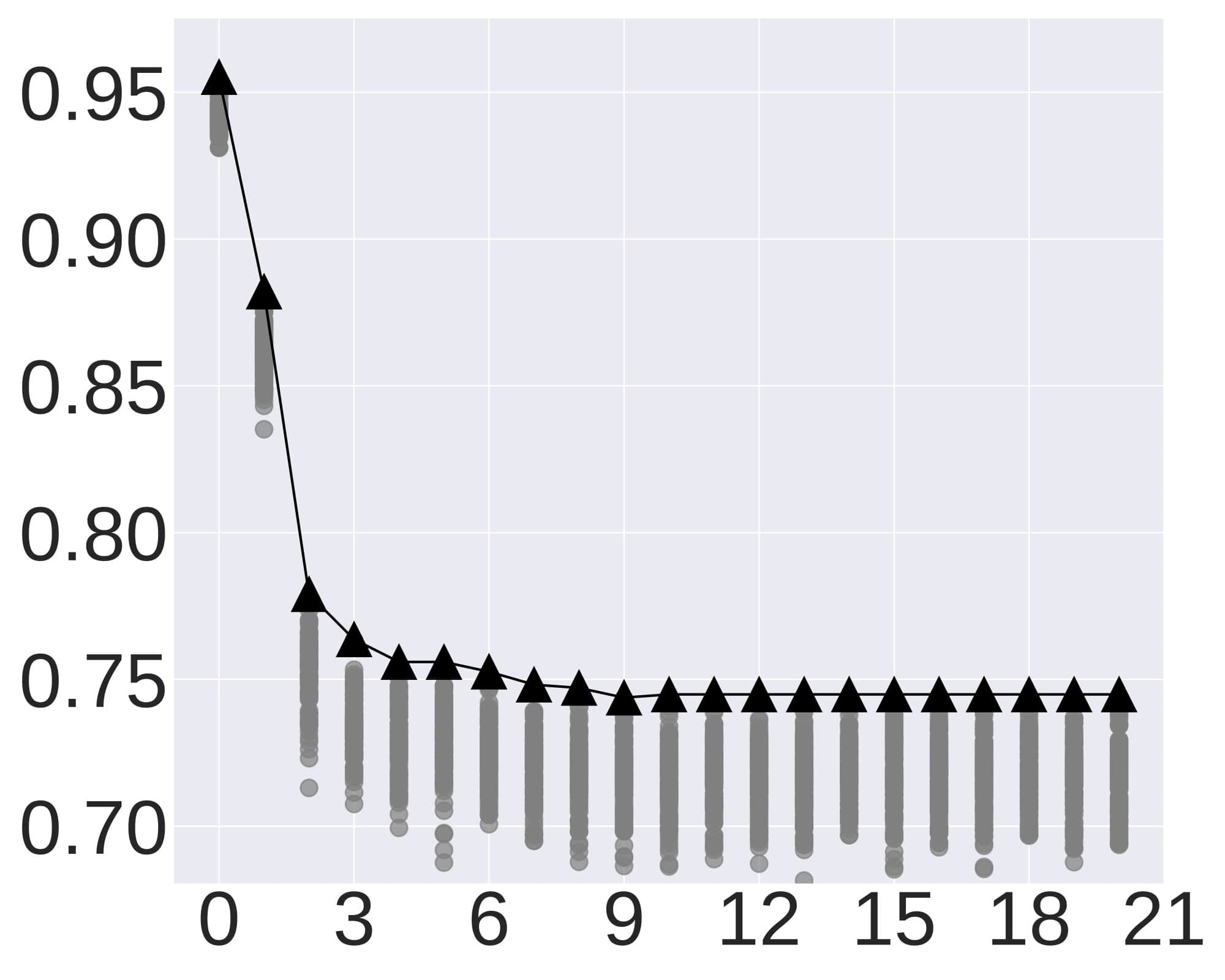}}
  \vspace{-0.5em}
 	\centerline{\footnotesize{~~~\# Retrieved examples $N_{\text{rag}}$}}
 \vspace{-0.5em}
 \end{minipage}
 \begin{minipage}{0.24\linewidth}
    \centerline{\footnotesize{\quad CommonGen}}
 	\vspace{1pt}
 	\centerline{\includegraphics[width=1.0\textwidth]{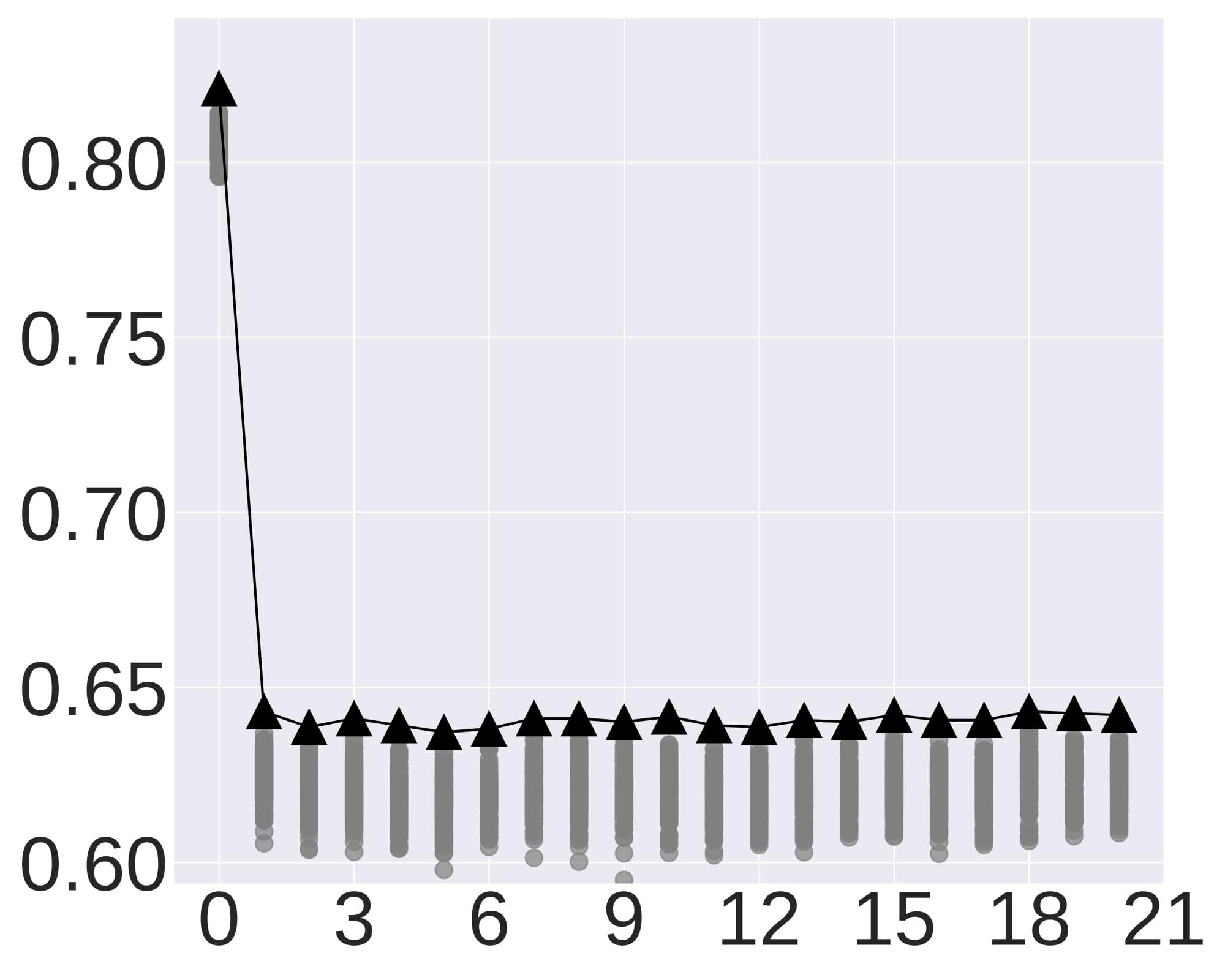}}
  \vspace{-0.5em}
 	\centerline{\footnotesize{~~~\# Retrieved examples $N_{\text{rag}}$}}
 \vspace{-0.5em}
 \end{minipage}
 \begin{minipage}{0.24\linewidth}
    \centerline{\footnotesize{\quad DART}}
 	\vspace{1pt}
 	\centerline{\includegraphics[width=1.0\textwidth]{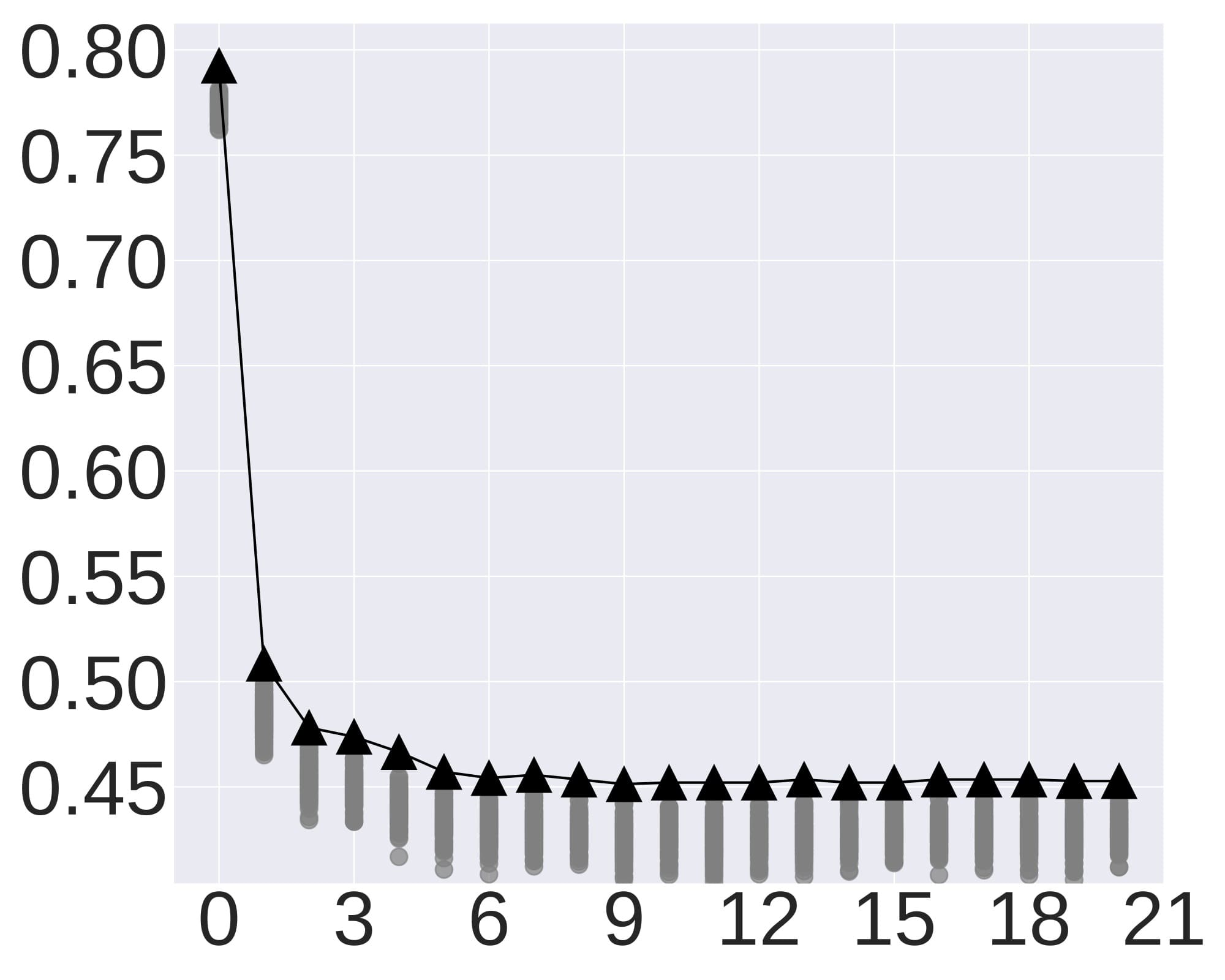}}
  \vspace{-0.5em}
 	\centerline{\footnotesize{~~~\# Retrieved examples $N_{\text{rag}}$}}
 \vspace{-0.5em}
 \end{minipage}
 \begin{minipage}{0.24\linewidth}
    \centerline{\footnotesize{\quad E2E}}
 	\vspace{1pt}
 	\centerline{\includegraphics[width=1.0\textwidth]{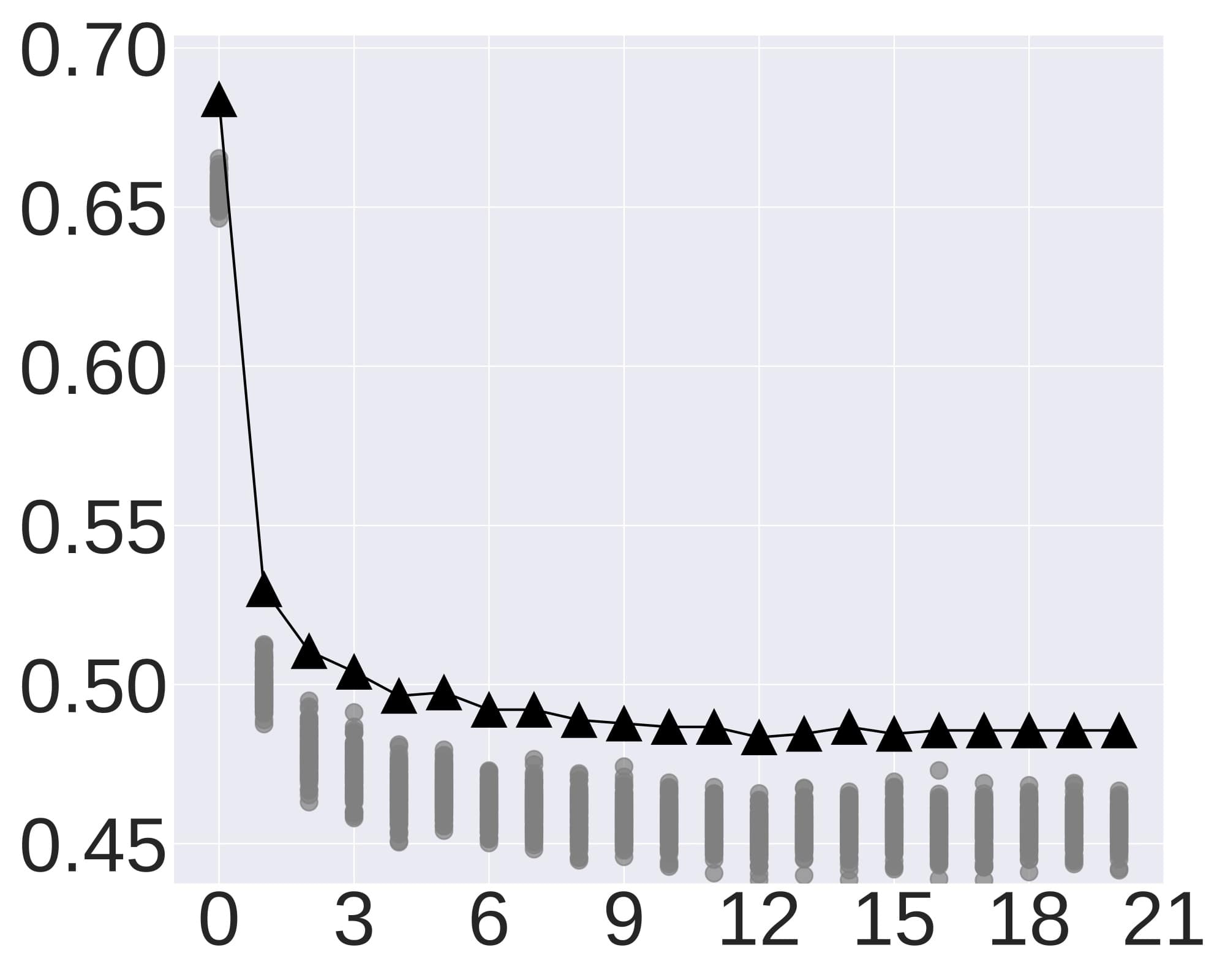}}
  \vspace{-0.5em}
 	\centerline{\footnotesize{~~~\# Retrieved examples $N_{\text{rag}}$}}
 \vspace{-0.5em}
 \end{minipage}
}
\subfigure{
\centerline{\includegraphics[width=0.4\textwidth]{figures/legend.pdf}}}
\caption{Conformal generation risk $\hat{\alpha}_{\text{rag}}$ and simulations of empirical risks with Biencoder-SFT for different $N_{\text{rag}}$ and fixed $\lambda_g=1,\lambda_s=1.0$.}
\label{fig:bound_and_simulation_llmr}
\end{figure*}

\begin{figure*}[t]
\subfigure{
    \rotatebox{90}{\hspace{-3.5em} Evaluation Risk}
    \begin{minipage}{0.24\linewidth}
    \centerline{\footnotesize{\quad AESLC}}
 	\vspace{1pt}
\centerline{\includegraphics[width=1.0\textwidth]{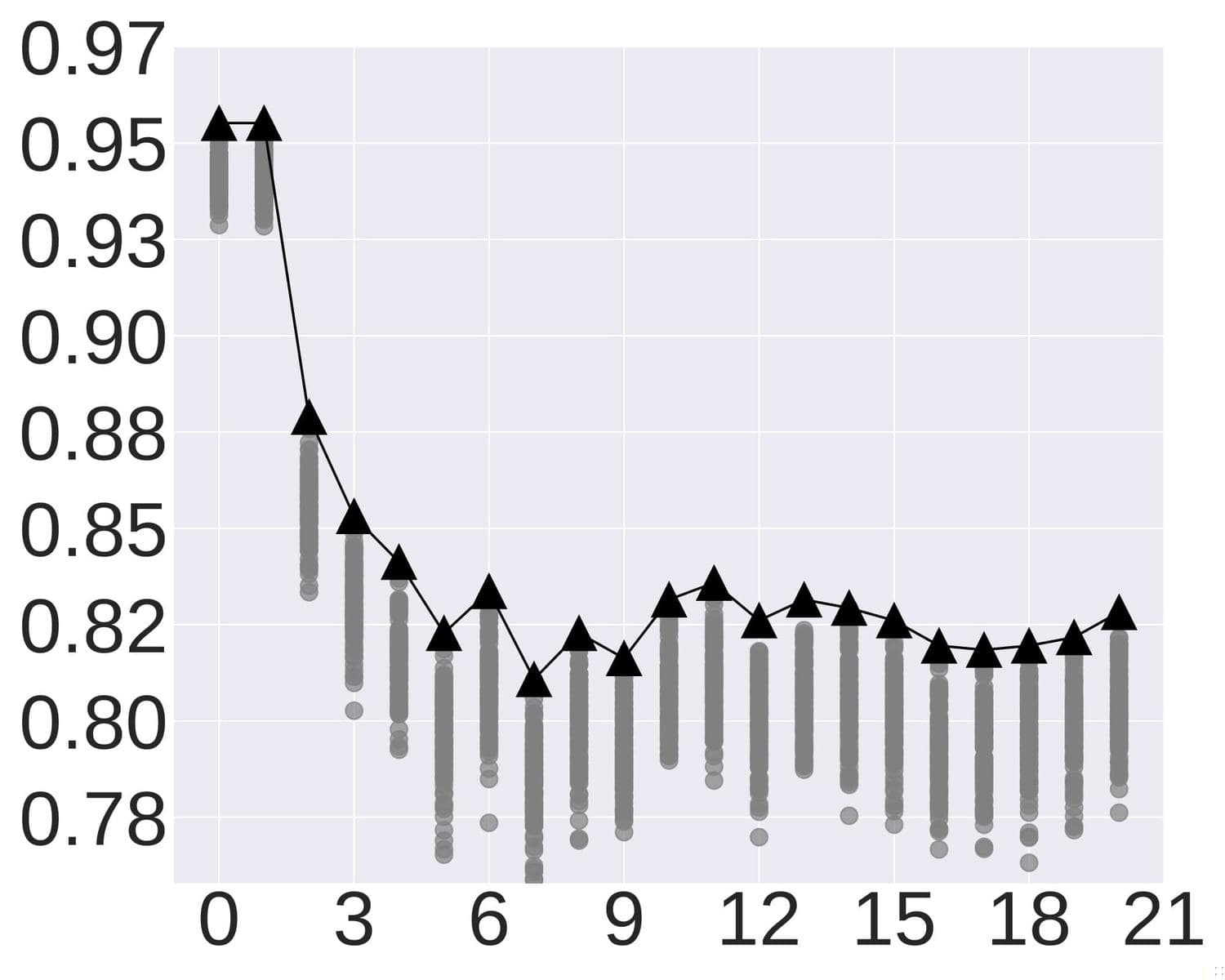}}
  \vspace{-0.5em}
 	\centerline{\footnotesize{~~~\# Retrieved examples $N_{\text{rag}}$}}
 \vspace{-0.5em}
 \end{minipage}
 \begin{minipage}{0.24\linewidth}
    \centerline{\footnotesize{\quad CommonGen}}
 	\vspace{1pt}
 	\centerline{\includegraphics[width=1.0\textwidth]{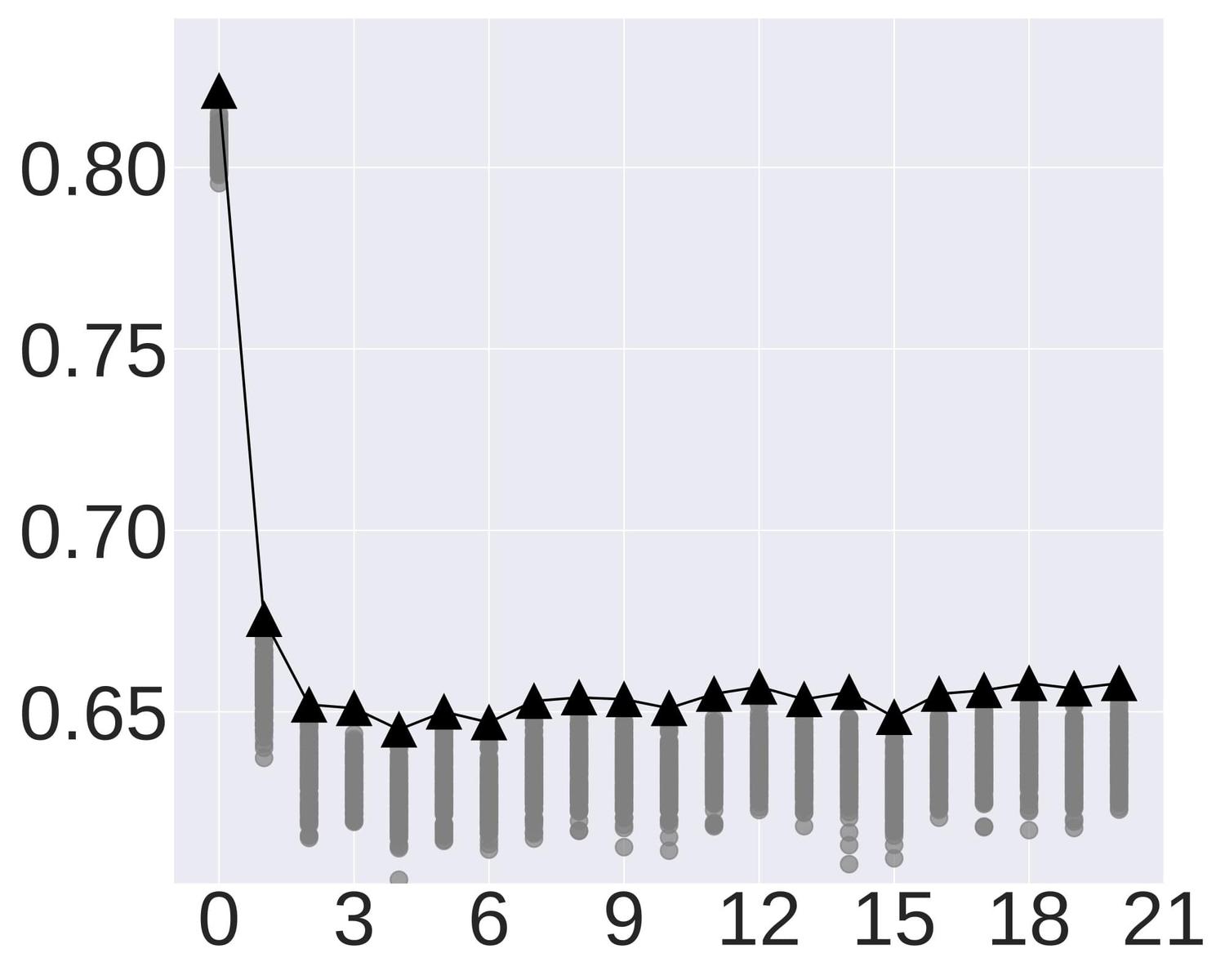}}
  \vspace{-0.5em}
 	\centerline{\footnotesize{~~~\# Retrieved examples $N_{\text{rag}}$}}
 \vspace{-0.5em}
 \end{minipage}
 \begin{minipage}{0.24\linewidth}
    \centerline{\footnotesize{\quad DART}}
 	\vspace{1pt}
 	\centerline{\includegraphics[width=1.0\textwidth]{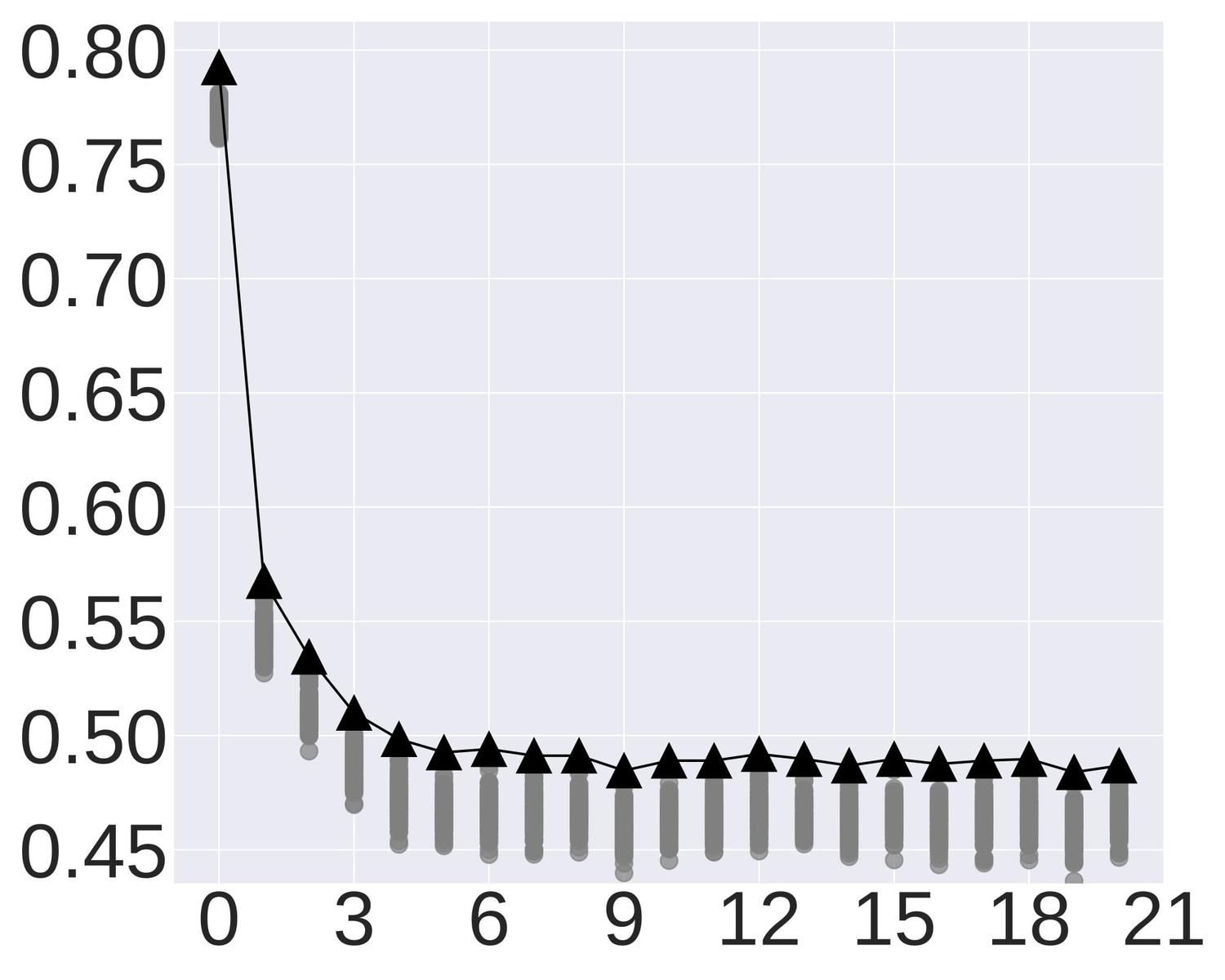}}
  \vspace{-0.5em}
 	\centerline{\footnotesize{~~~\# Retrieved examples $N_{\text{rag}}$}}
 \vspace{-0.5em}
 \end{minipage}
 \begin{minipage}{0.24\linewidth}
    \centerline{\footnotesize{\quad E2E}}
 	\vspace{1pt}
 	\centerline{\includegraphics[width=1.0\textwidth]{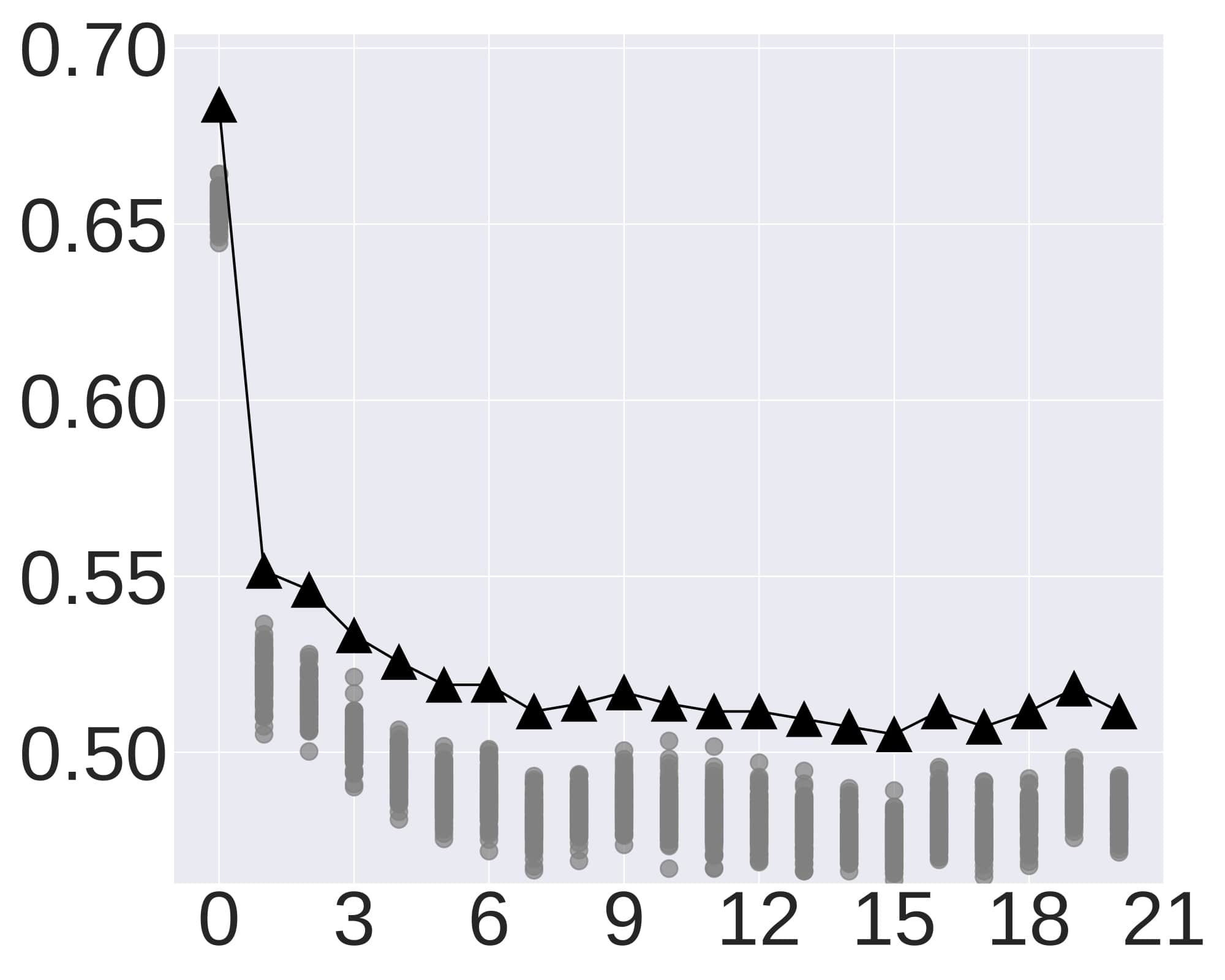}}
  \vspace{-0.5em}
 	\centerline{\footnotesize{~~~\# Retrieved examples $N_{\text{rag}}$}}
 \vspace{-0.5em}
 \end{minipage}
}
\subfigure{
\centerline{\includegraphics[width=0.4\textwidth]{figures/legend.pdf}}}
\caption{ Conformal generation risk $\hat{\alpha}_{\text{rag}}$ and simulations of empirical risks with BM25 for different $N_{\text{rag}}$ and fixed $\lambda_g=1,\lambda_s=1.0$.}
\label{fig:bound_and_simulation_bm25}
\end{figure*}

\begin{figure*}[t]
\subfigure{
    \rotatebox{90}{\hspace{-3.5em} Evaluation Risk}
    \begin{minipage}{0.24\linewidth}
    \centerline{\footnotesize{\quad AESLC}}
 	\vspace{1pt}
\centerline{\includegraphics[width=1.0\textwidth]{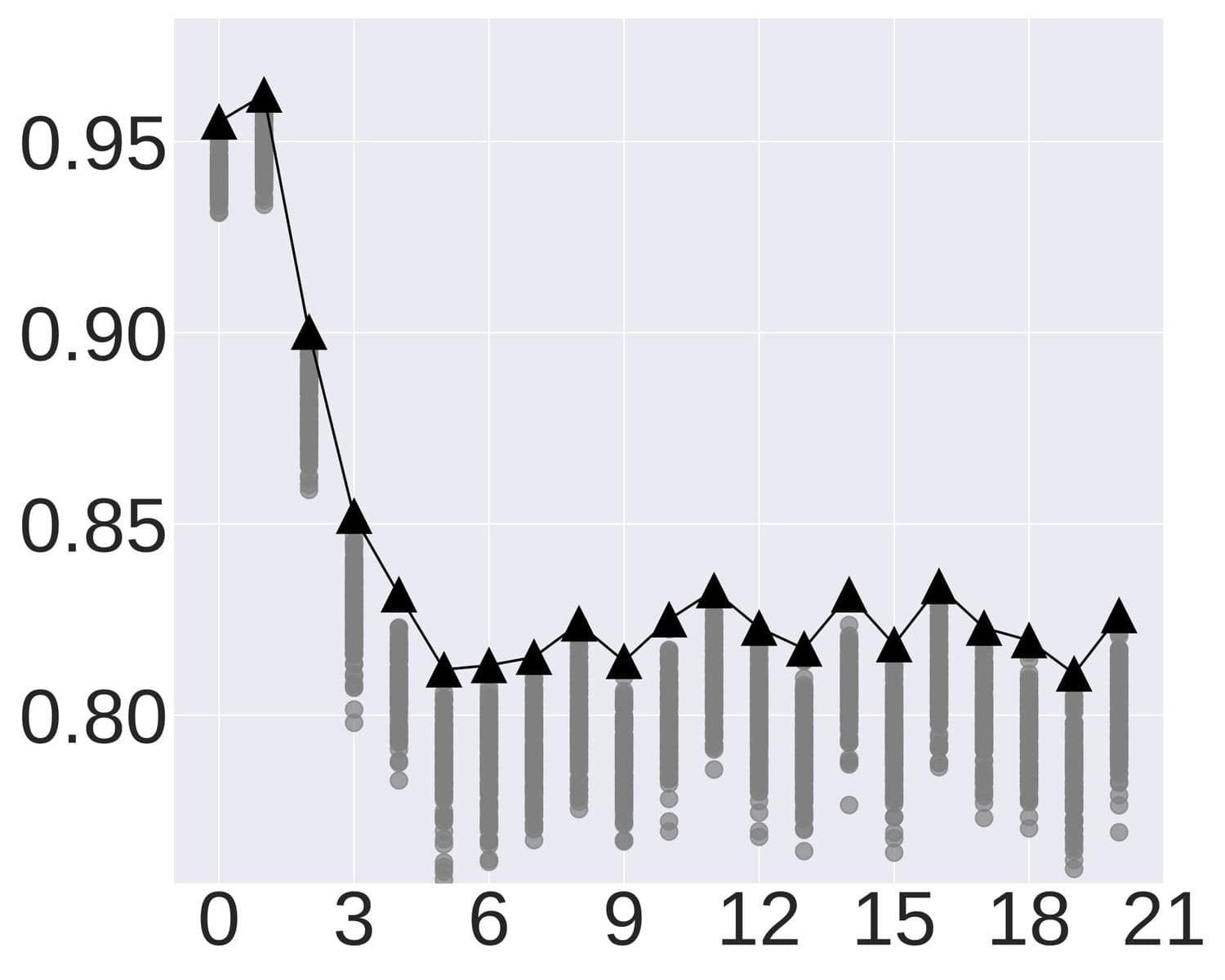}}
  \vspace{-0.5em}
 	\centerline{\footnotesize{~~~\# Retrieved examples $N_{\text{rag}}$}}
 \vspace{-0.5em}
 \end{minipage}
 \begin{minipage}{0.24\linewidth}
    \centerline{\footnotesize{\quad CommonGen}}
 	\vspace{1pt}
 	\centerline{\includegraphics[width=1.0\textwidth]{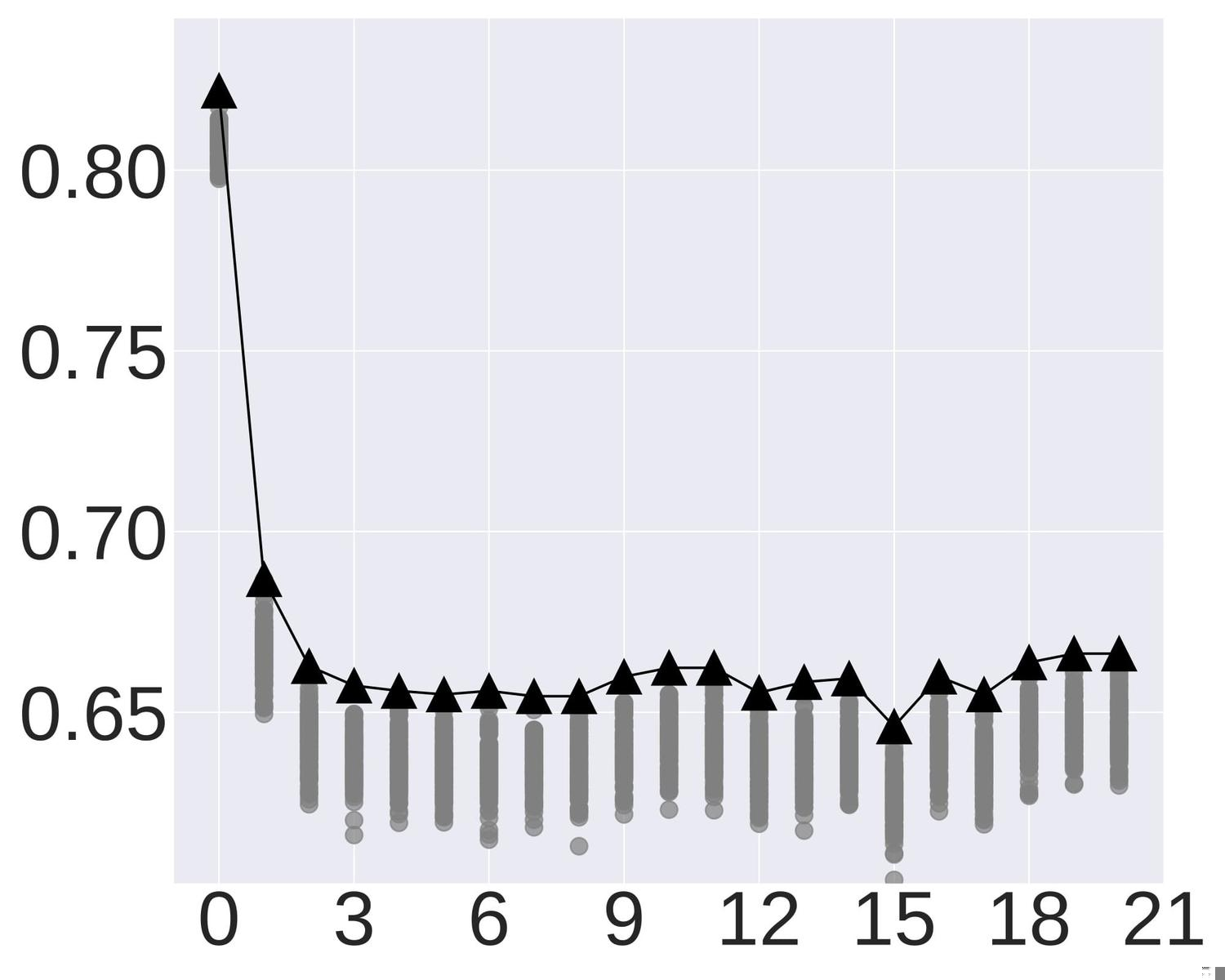}}
  \vspace{-0.5em}
 	\centerline{\footnotesize{~~~\# Retrieved examples $N_{\text{rag}}$}}
 \vspace{-0.5em}
 \end{minipage}
 \begin{minipage}{0.24\linewidth}
    \centerline{\footnotesize{\quad DART}}
 	\vspace{1pt}
 	\centerline{\includegraphics[width=1.0\textwidth]{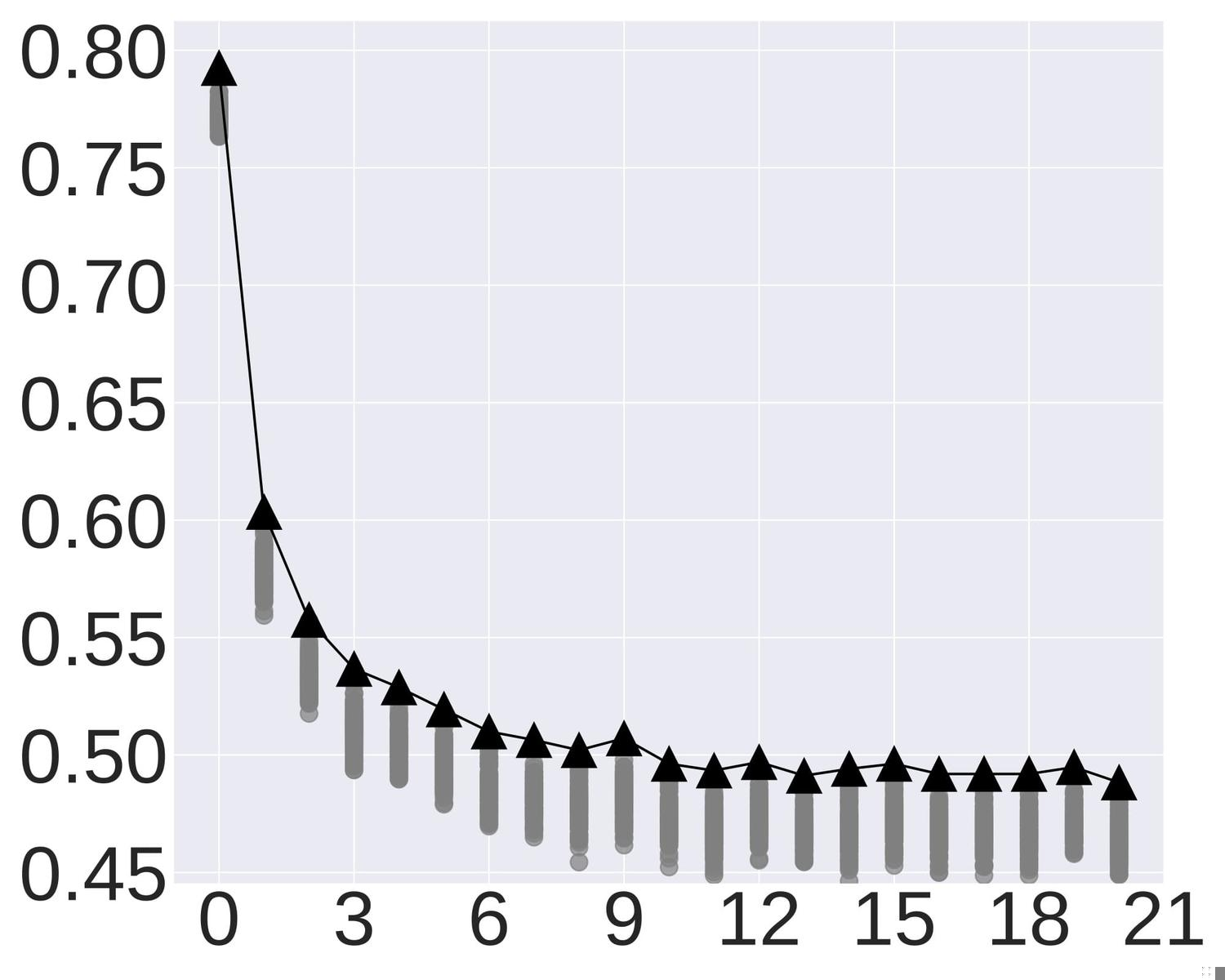}}
  \vspace{-0.5em}
 	\centerline{\footnotesize{~~~\# Retrieved examples $N_{\text{rag}}$}}
 \vspace{-0.5em}
 \end{minipage}
 \begin{minipage}{0.24\linewidth}
    \centerline{\footnotesize{\quad E2E}}
 	\vspace{1pt}
 	\centerline{\includegraphics[width=1.0\textwidth]{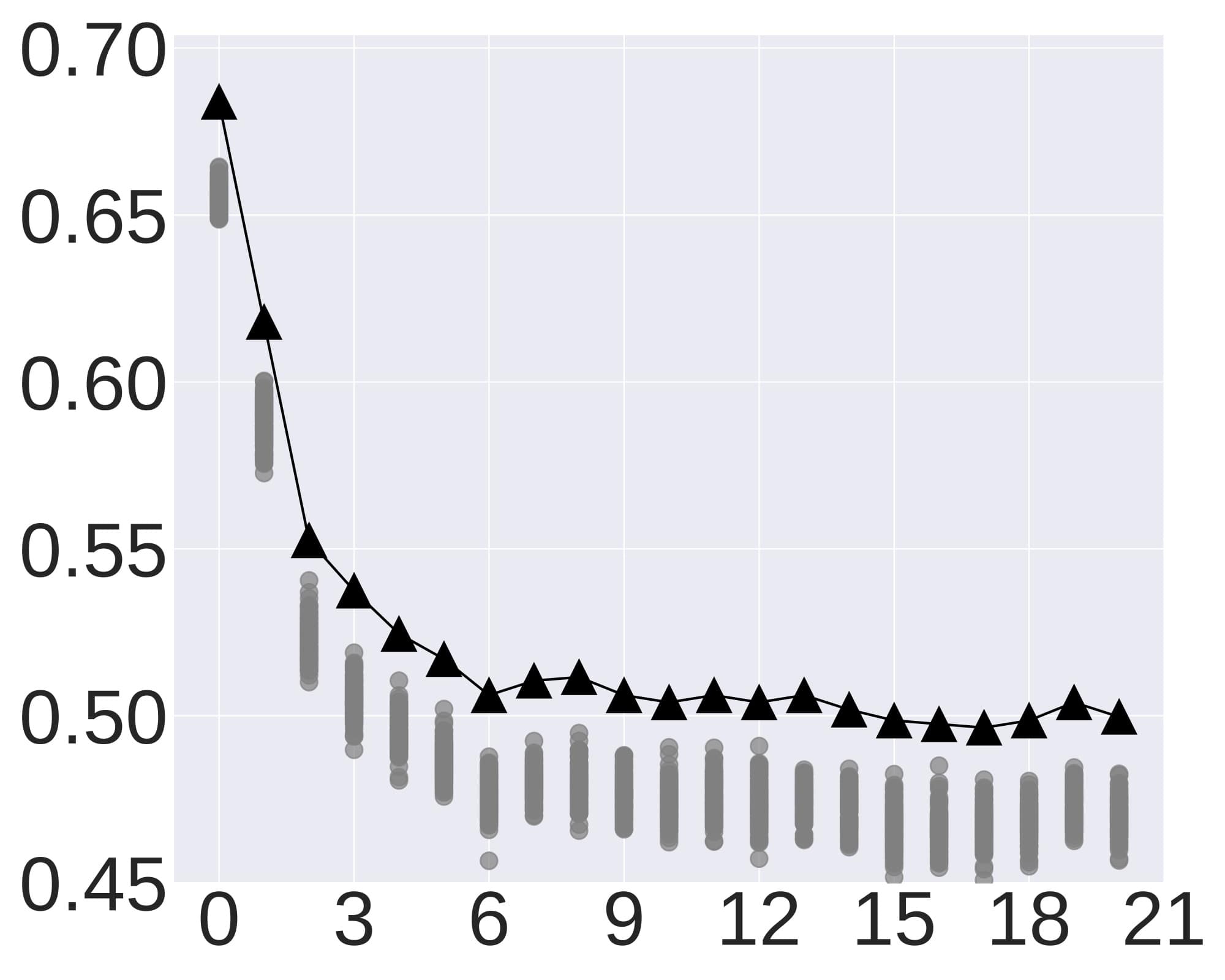}}
  \vspace{-0.5em}
 	\centerline{\footnotesize{~~~\# Retrieved examples $N_{\text{rag}}$}}
 \vspace{-0.5em}
 \end{minipage}
}
\vspace{-0.9em}
\subfigure{
\centerline{\includegraphics[width=0.4\textwidth]{figures/legend.pdf}}}
\vspace{-1.0em}
\caption{Conformal generation risk $\hat{\alpha}_{\text{rag}}$ and simulations of empirical risks with BAAI/bge for different $N_{\text{rag}}$ and fixed $\lambda_g=1,\lambda_s=1.0$.}
\label{fig:bound_and_simulation_baai}
\end{figure*}

\begin{algorithm}[t]
    \caption{Shifted distribution sampling protocol.}
    \label{alg:simulation}
    \begin{algorithmic}[1]
        \STATE {\bfseries Input}: original test set $\gD$, test sample pool $\gD_{\text{pool}}$, Risk function $R(\cdot,\cdot): \gX \times \gY \mapsto \sR$
        \STATE {\bfseries Output}: empirical risk on the shifted test set $\gQ$ $\hat{R}_\gQ$, Hellinger distance between $\gD$ and $\gQ$ $H_{\gP\gQ}$
        \STATE Randomly sample $\gQ$ from $\gD_{\text{pool}}$ with equalized set size as $\gD$: $|\gQ|=|\gD|$
        \STATE Randomly sample the sample weight vector of $\gQ$ $\vw \in \Delta^{|\gD|}$
        \STATE Compute Hellinger distance as $H_{\gP\gQ}=\sqrt{1-\sum_{i=1}^{|\gD|}\sqrt{{\vw_i}/{|\gD|}}}$
        \STATE Evaluate risks for all samples in $\gQ$ with risk function $R(\cdot,\cdot)$: $\hat{\bm{R}}_\gQ \in \sR^{|\gD|}$
        \STATE Compute the empirical risk on $\gQ$ with weight vector $\vw$: $\hat{R}_\gQ = \vw^T \hat{\bm{R}}_\gQ$
        \STATE \textbf{Return} $\hat{R}_\gQ$, $H_{\gP\gQ}$
    \end{algorithmic}
\end{algorithm}

\begin{figure*}[t]
\subfigure{
    \rotatebox{90}{\hspace{-3.5em} Evaluation Risk}
    \begin{minipage}{0.24\linewidth}
    \centerline{\footnotesize{\quad AESLC}}
 	\vspace{1pt}
\centerline{\includegraphics[width=1.0\textwidth]{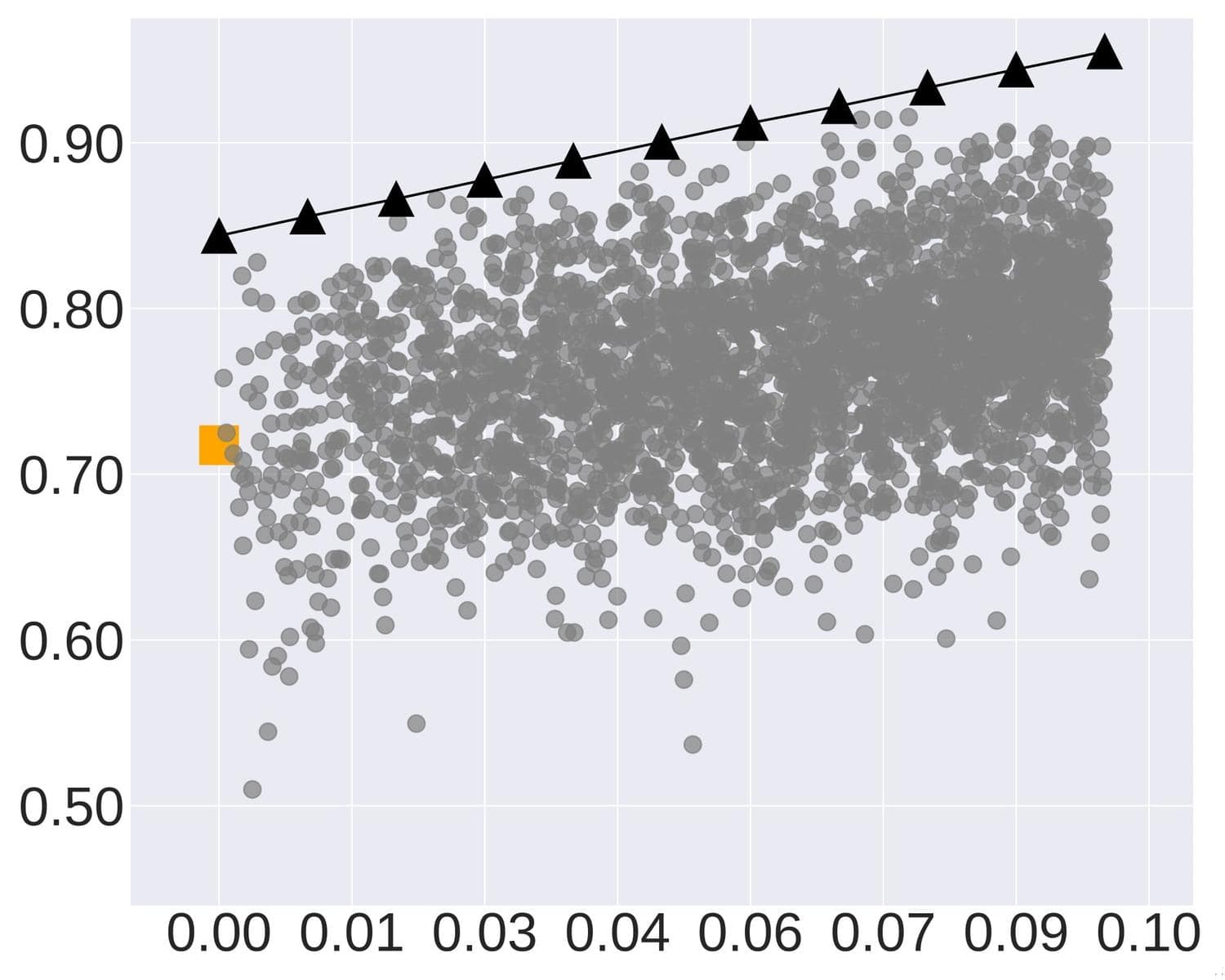}}
  \vspace{-0.5em}
 	\centerline{\footnotesize{~~~Hellinger Distance}}
 \vspace{-0.5em}
 \end{minipage}
 \begin{minipage}{0.24\linewidth}
    \centerline{\footnotesize{\quad CommonGen}}
 	\vspace{1pt}
 	\centerline{\includegraphics[width=1.0\textwidth]{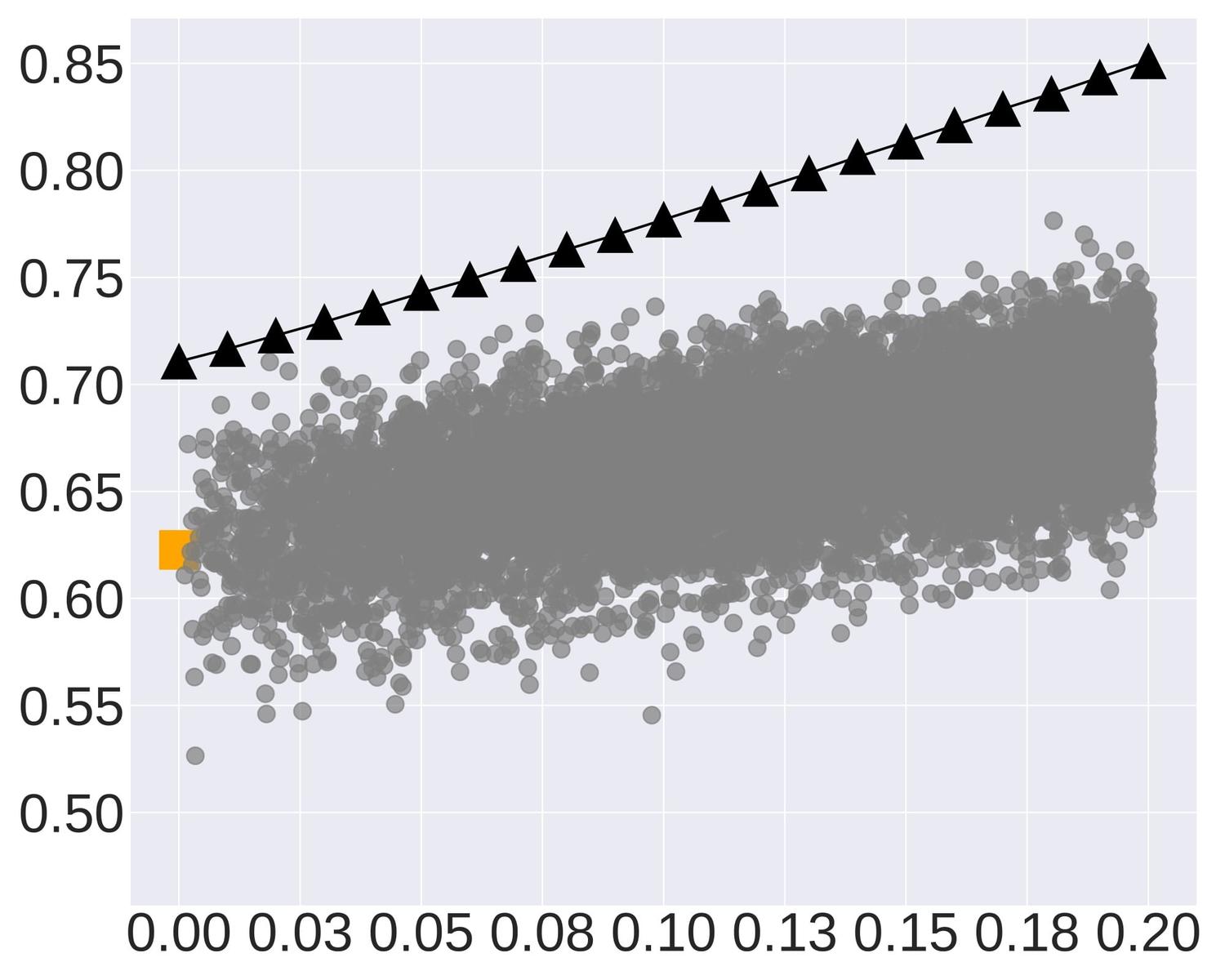}}
  \vspace{-0.5em}
 	\centerline{\footnotesize{~~~Hellinger Distance}}
 \vspace{-0.5em}
 \end{minipage}
 \begin{minipage}{0.24\linewidth}
    \centerline{\footnotesize{\quad DART}}
 	\vspace{1pt}
 	\centerline{\includegraphics[width=1.0\textwidth]{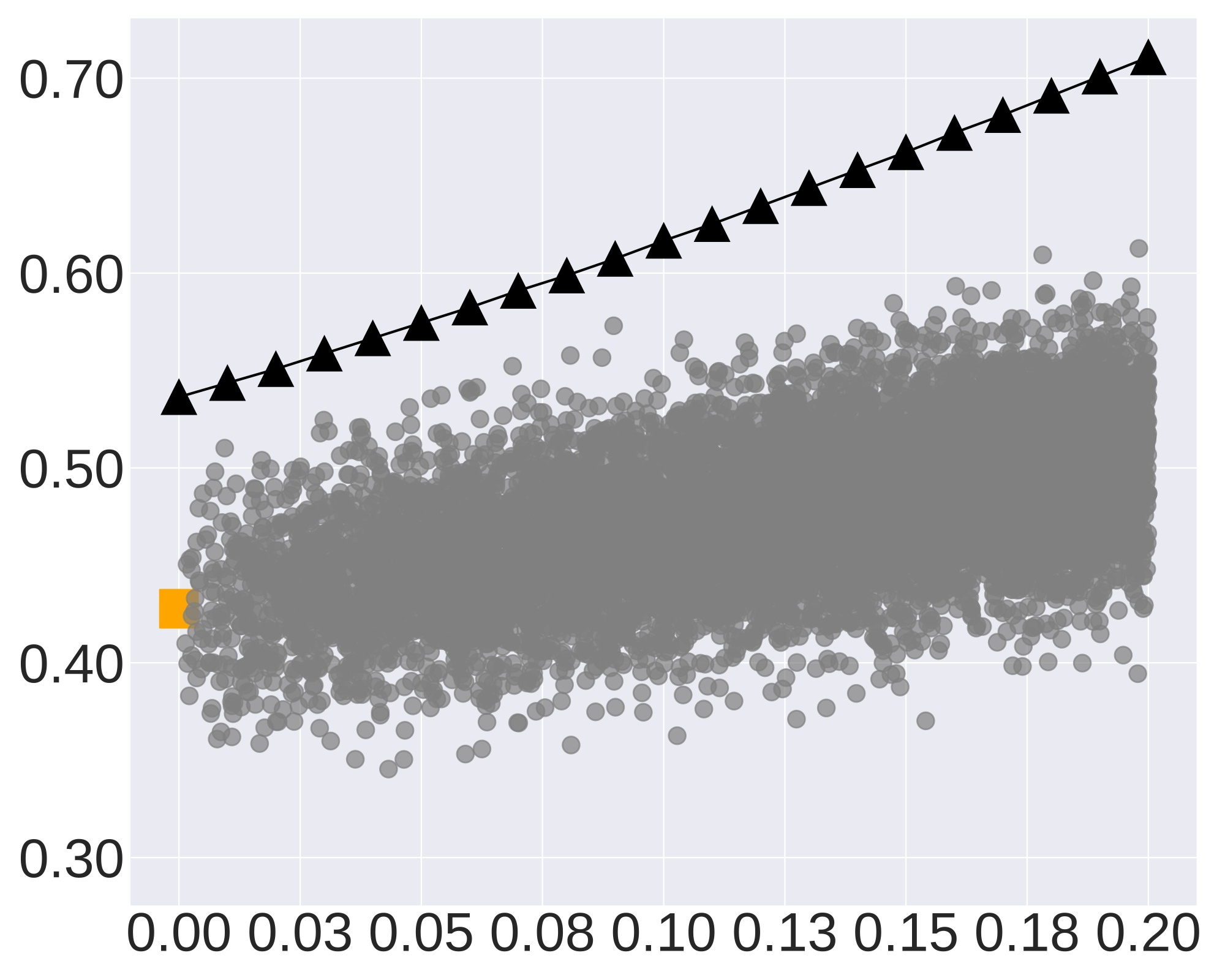}}
  \vspace{-0.5em}
 	\centerline{\footnotesize{~~~Hellinger Distance}}
 \vspace{-0.5em}
 \end{minipage}
 \begin{minipage}{0.24\linewidth}
    \centerline{\footnotesize{\quad E2E}}
 	\vspace{1pt}
 	\centerline{\includegraphics[width=1.0\textwidth]{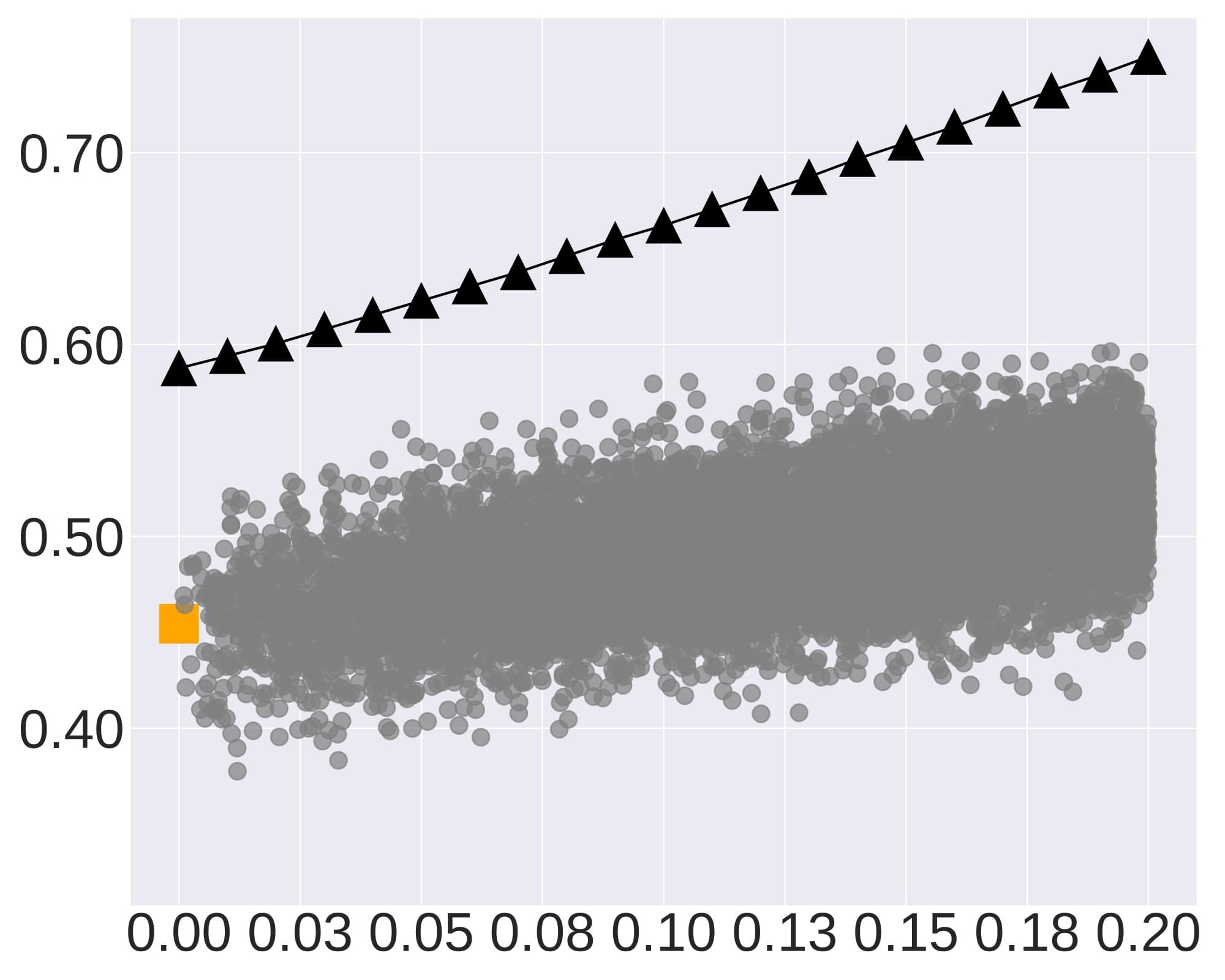}}
  \vspace{-0.5em}
 	\centerline{\footnotesize{~~~Hellinger Distance}}
 \vspace{-0.5em}
 \end{minipage}
}
\subfigure{
\centerline{\includegraphics[width=0.75\textwidth]{figures/legend3.pdf}}}
\caption{Conformal generation risk $\hat{\alpha}_{\text{rag}}$ and empirical risks with Biencoder-SFT under distribution shifts with $N_{\text{rag}}=15, \lambda_g=1, \lambda_s=1.0$.}
\label{fig:bound_and_simulation_llmr_dist_shft}
\end{figure*}

\begin{figure*}[t]
\subfigure{
    \rotatebox{90}{\hspace{-3.5em} Evaluation Risk}
    \begin{minipage}{0.24\linewidth}
    \centerline{\footnotesize{\quad AESLC}}
 	\vspace{1pt}
\centerline{\includegraphics[width=1.0\textwidth]{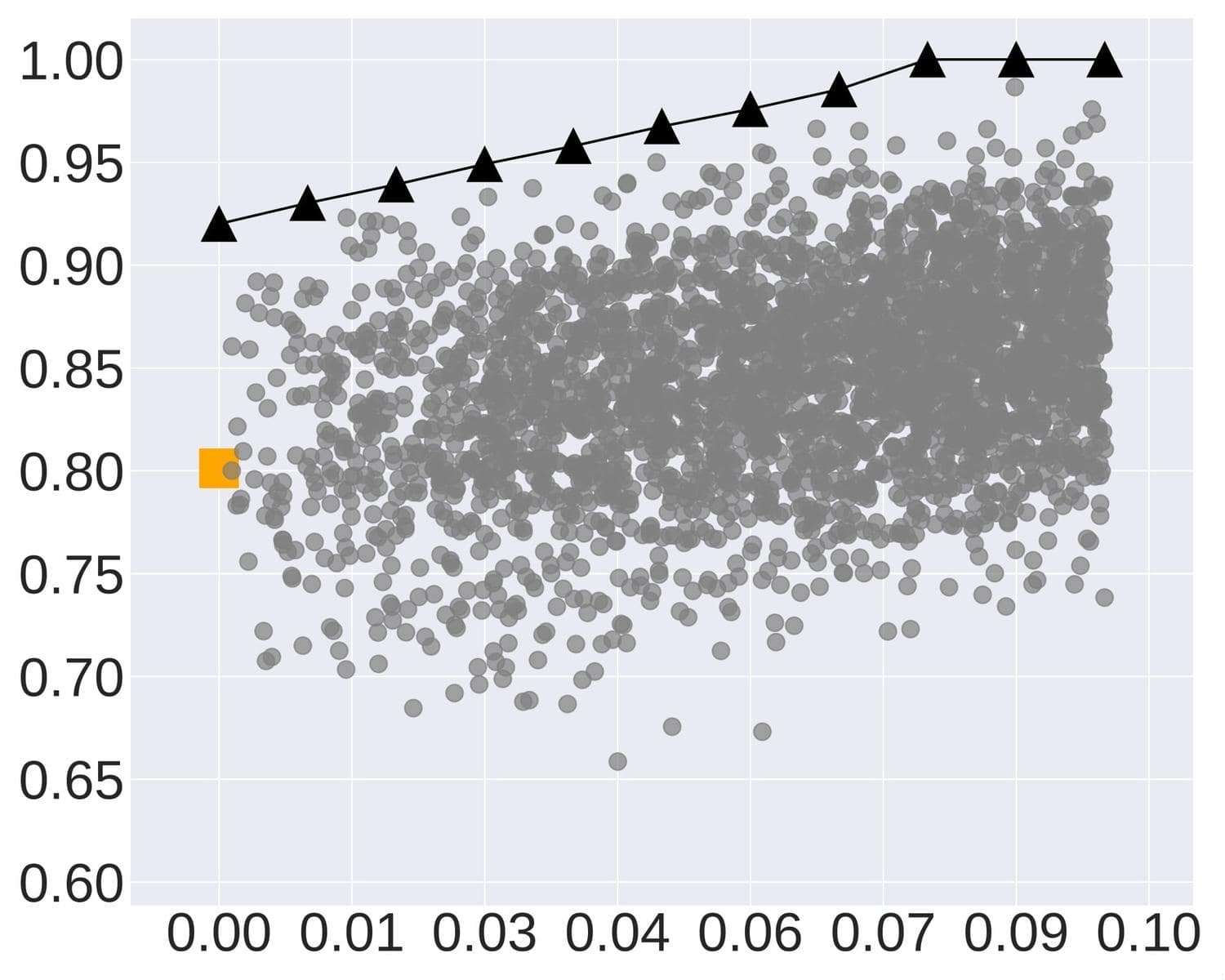}}
  \vspace{-0.5em}
 	\centerline{\footnotesize{~~~Hellinger Distance $\rho$}}
 \vspace{-0.5em}
 \end{minipage}
 \begin{minipage}{0.24\linewidth}
    \centerline{\footnotesize{\quad CommonGen}}
 	\vspace{1pt}
 	\centerline{\includegraphics[width=1.0\textwidth]{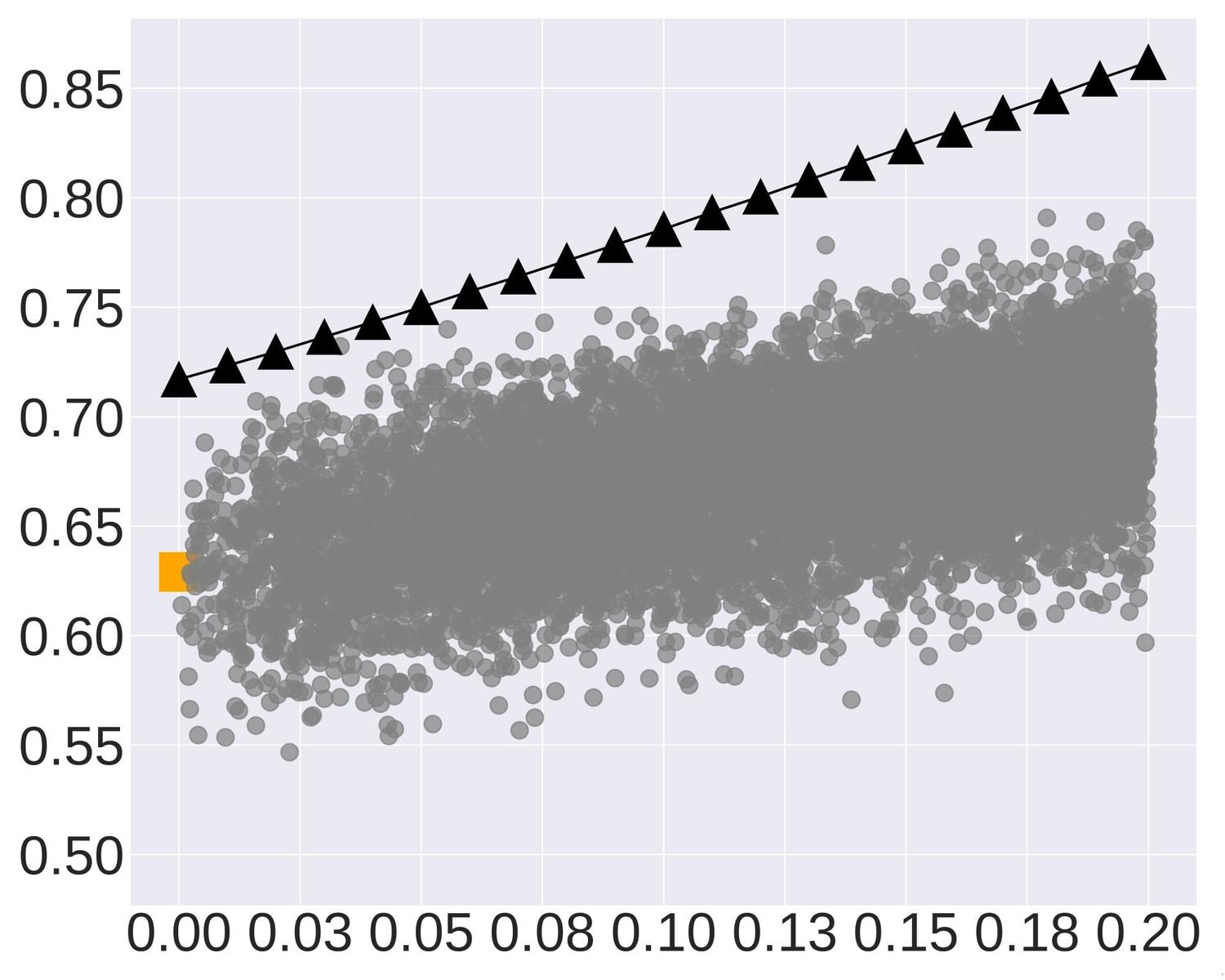}}
  \vspace{-0.5em}
 	\centerline{\footnotesize{~~~Hellinger Distance $\rho$}}
 \vspace{-0.5em}
 \end{minipage}
 \begin{minipage}{0.24\linewidth}
    \centerline{\footnotesize{\quad DART}}
 	\vspace{1pt}
 	\centerline{\includegraphics[width=1.0\textwidth]{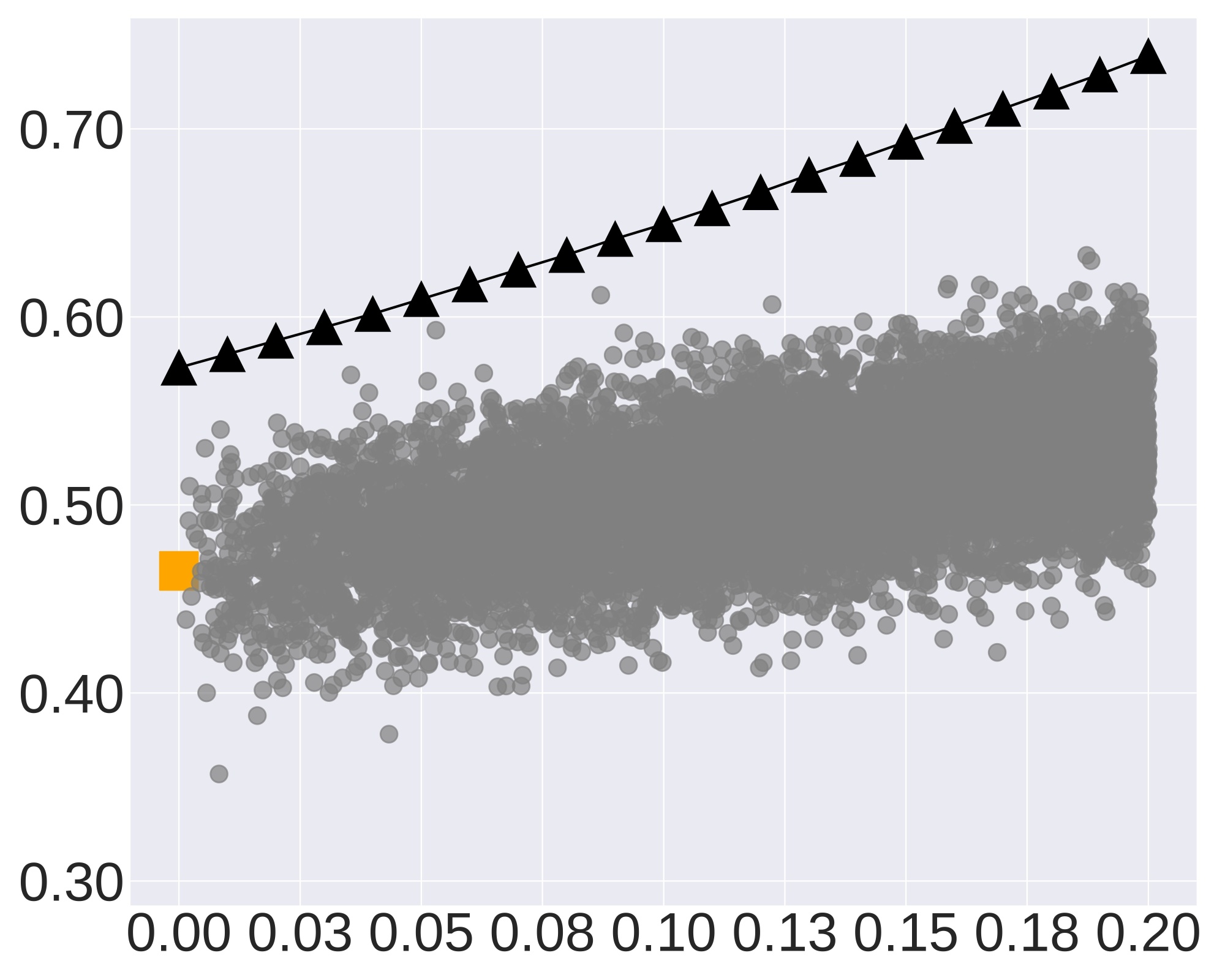}}
  \vspace{-0.5em}
 	\centerline{\footnotesize{~~~Hellinger Distance $\rho$}}
 \vspace{-0.5em}
 \end{minipage}
 \begin{minipage}{0.24\linewidth}
    \centerline{\footnotesize{\quad E2E}}
 	\vspace{1pt}
 	\centerline{\includegraphics[width=1.0\textwidth]{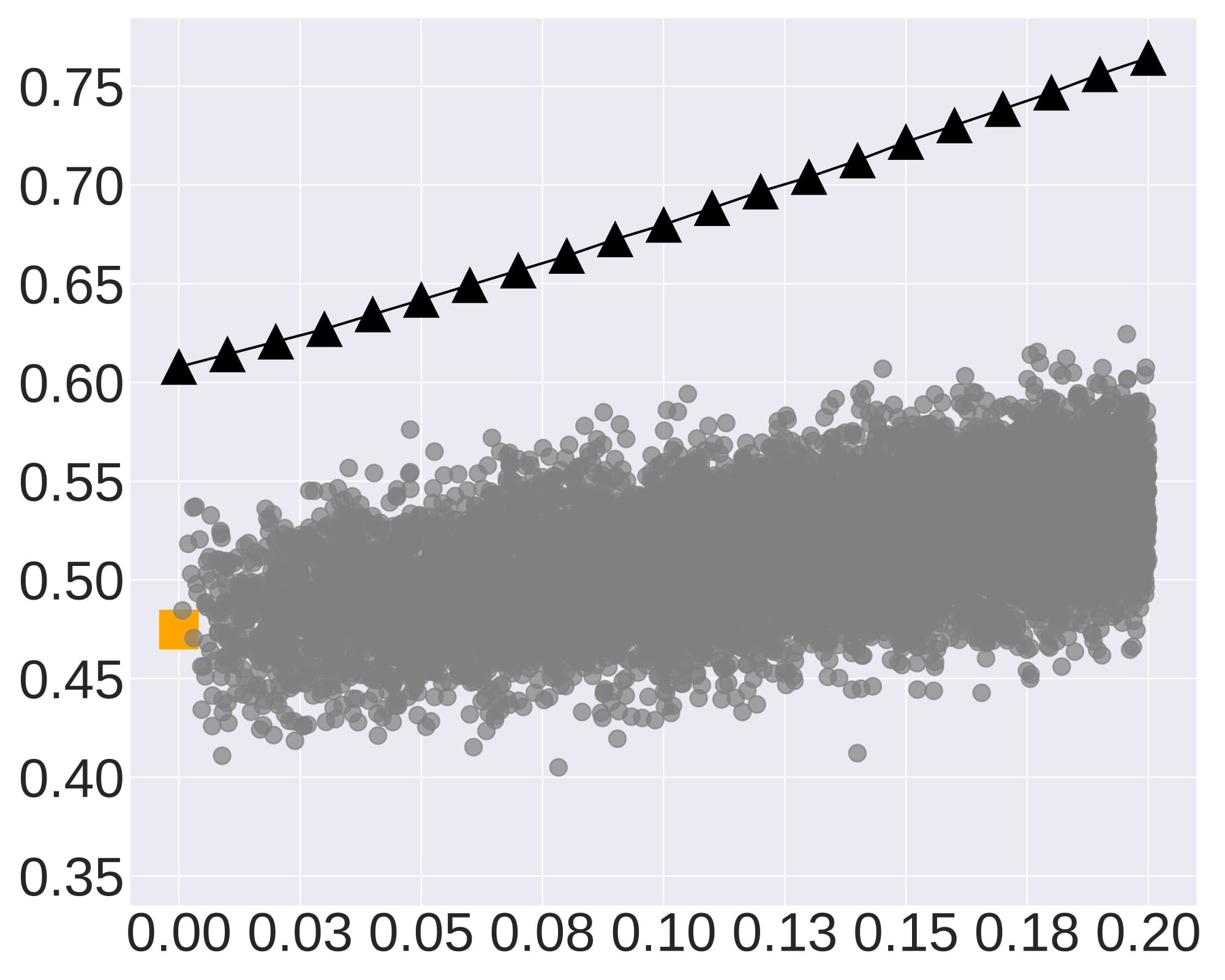}}
  \vspace{-0.5em}
 	\centerline{\footnotesize{~~~Hellinger Distance $\rho$}}
 \vspace{-0.5em}
 \end{minipage}
}
\subfigure{
\centerline{\includegraphics[width=0.75\textwidth]{figures/legend3.pdf}}}
\caption{Conformal generation risk $\hat{\alpha}_{\text{rag}}$ and empirical risks with BM25 under distribution shifts with $N_{\text{rag}}=15, \lambda_g=1, \lambda_s=1.0$.}
\label{fig:bound_and_simulation_bm25_dist_shft}
\end{figure*}

\begin{figure*}[t]
\subfigure{
    \rotatebox{90}{\hspace{-3.5em} Evaluation Risk}
    \begin{minipage}{0.24\linewidth}
    \centerline{\footnotesize{\quad AESLC}}
 	\vspace{1pt}
\centerline{\includegraphics[width=1.0\textwidth]{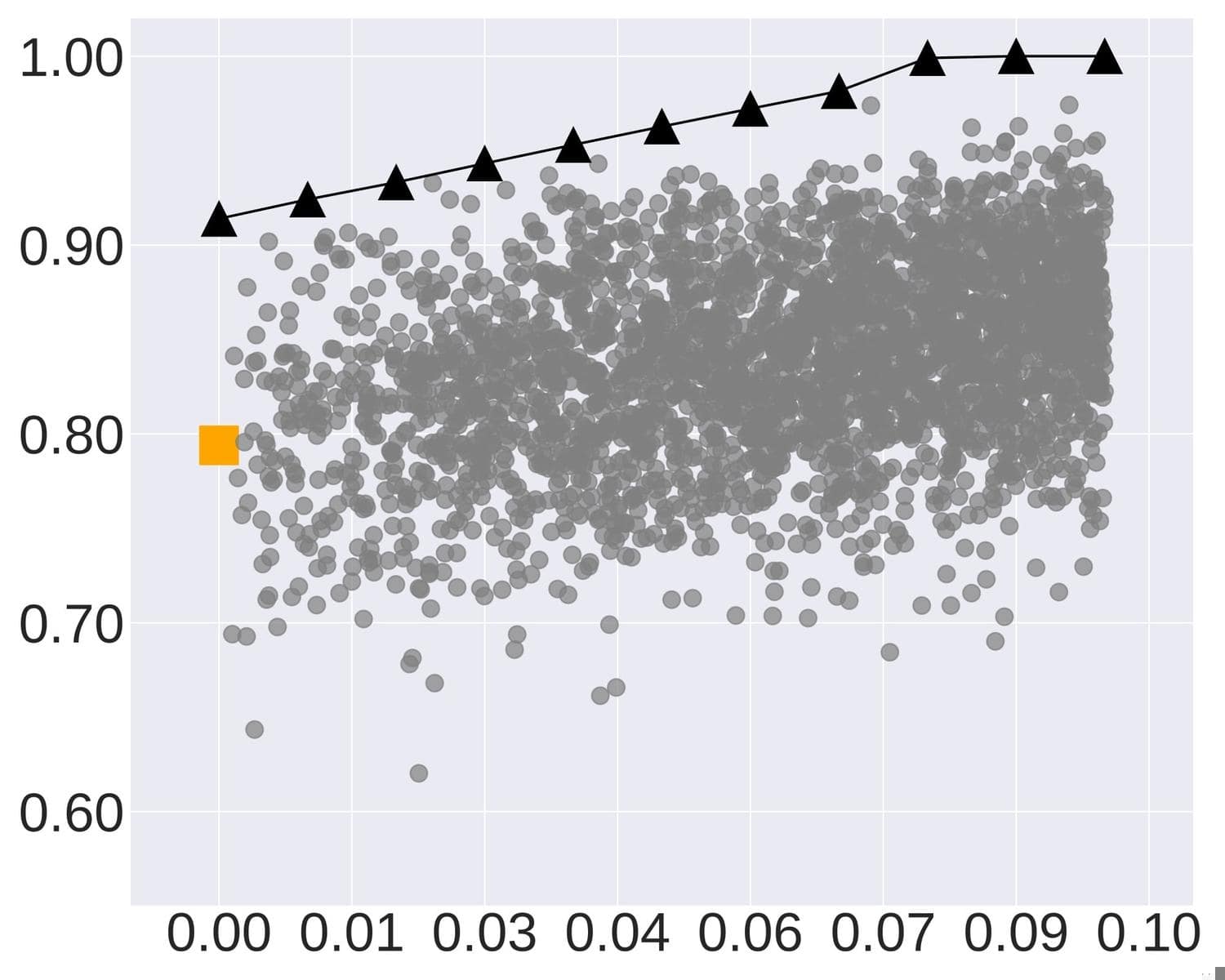}}
  \vspace{-0.5em}
 	\centerline{\footnotesize{~~~Hellinger Distance $\rho$}}
 \vspace{-0.5em}
 \end{minipage}
 \begin{minipage}{0.24\linewidth}
    \centerline{\footnotesize{\quad CommonGen}}
 	\vspace{1pt}
 	\centerline{\includegraphics[width=1.0\textwidth]{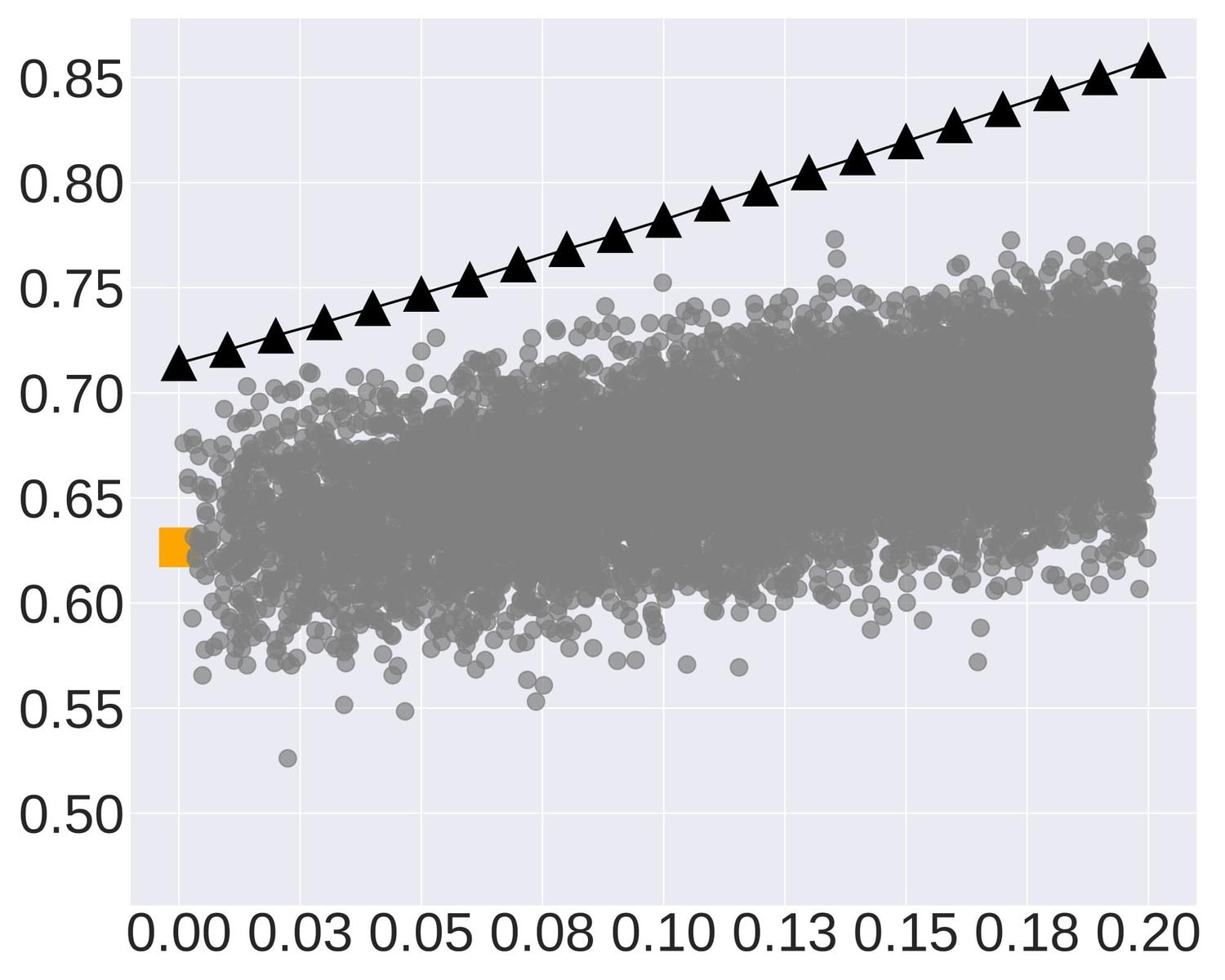}}
  \vspace{-0.5em}
 	\centerline{\footnotesize{~~~Hellinger Distance $\rho$}}
 \vspace{-0.5em}
 \end{minipage}
 \begin{minipage}{0.24\linewidth}
    \centerline{\footnotesize{\quad DART}}
 	\vspace{1pt}
 	\centerline{\includegraphics[width=1.0\textwidth]{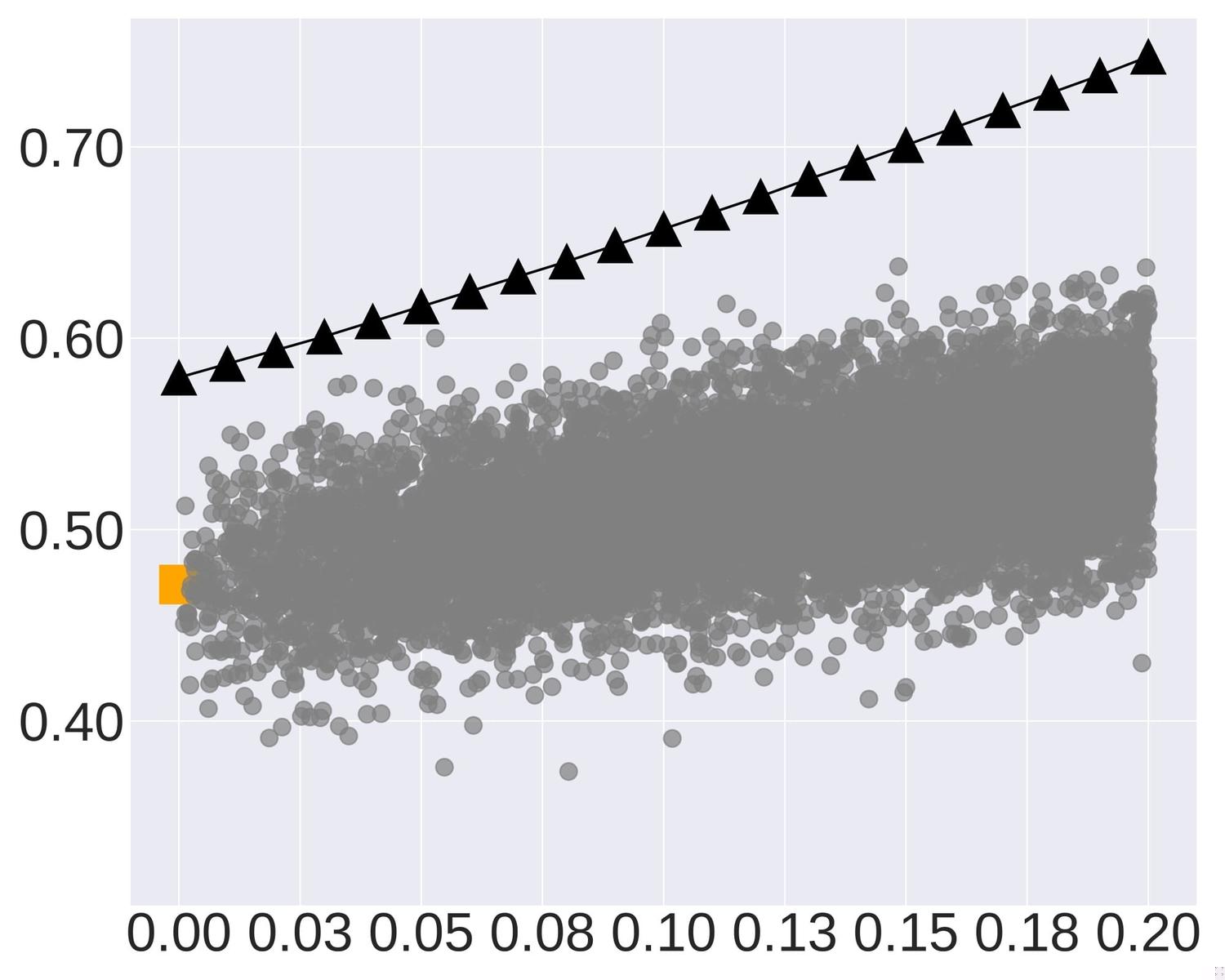}}
  \vspace{-0.5em}
 	\centerline{\footnotesize{~~~Hellinger Distance $\rho$}}
 \vspace{-0.5em}
 \end{minipage}
 \begin{minipage}{0.24\linewidth}
    \centerline{\footnotesize{\quad E2E}}
 	\vspace{1pt}
 	\centerline{\includegraphics[width=1.0\textwidth]{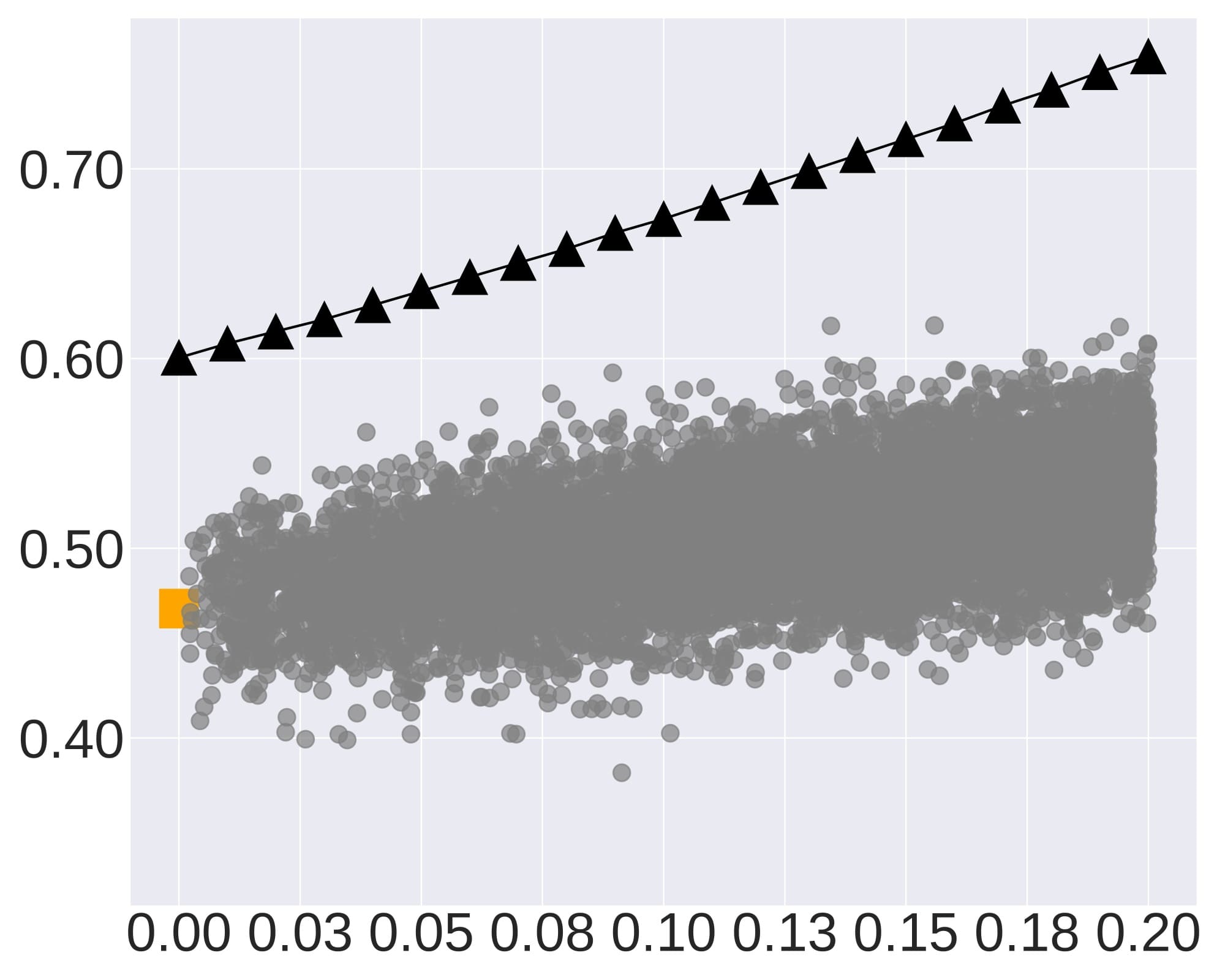}}
  \vspace{-0.5em}
 	\centerline{\footnotesize{~~~Hellinger Distance $\rho$}}
 \vspace{-0.5em}
 \end{minipage}
}
\subfigure{
\centerline{\includegraphics[width=0.75\textwidth]{figures/legend3.pdf}}}
\caption{Conformal generation risk $\hat{\alpha}_{\text{rag}}$ and empirical risks with BAAI/bge under distribution shifts with $N_{\text{rag}}=15, \lambda_g=1, \lambda_s=1.0$.}
\label{fig:bound_and_simulation_baai_dist_shft}
\end{figure*}

\begin{figure*}[t]
\subfigure{
    \rotatebox{90}{\hspace{-5.0em} Conformal Risk $\alpha_{\text{rag}}$}
    \begin{minipage}{0.24\linewidth}
    \centerline{\footnotesize{\quad AESLC}}
 	\vspace{1pt}
\centerline{\includegraphics[width=1.0\textwidth]{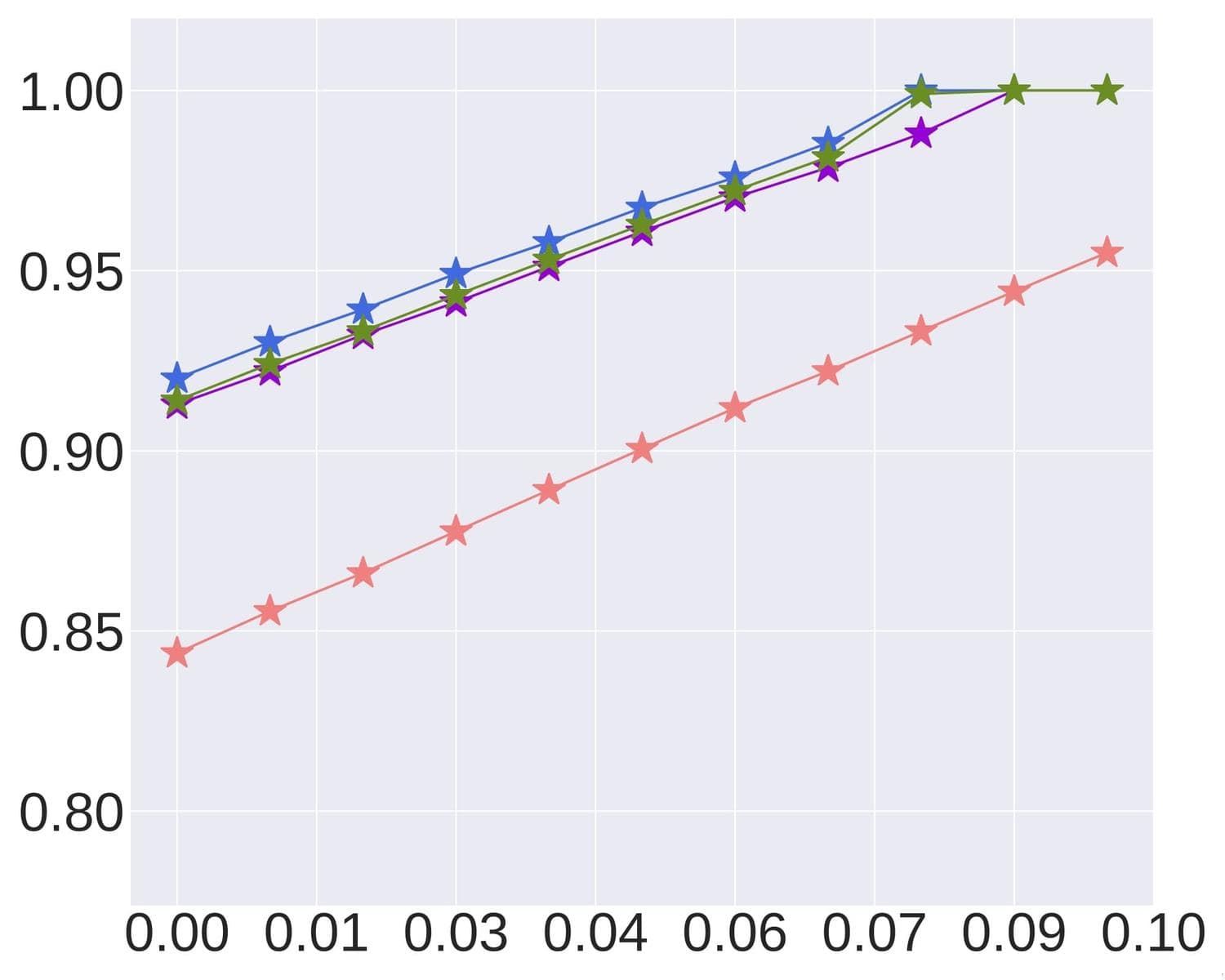}}
  \vspace{-0.5em}
 	\centerline{\footnotesize{~~~Hellinger Distance $\rho$}}
 \vspace{-0.5em}
 \end{minipage}
 \begin{minipage}{0.24\linewidth}
    \centerline{\footnotesize{\quad CommonGen}}
 	\vspace{1pt}
 	\centerline{\includegraphics[width=1.0\textwidth]{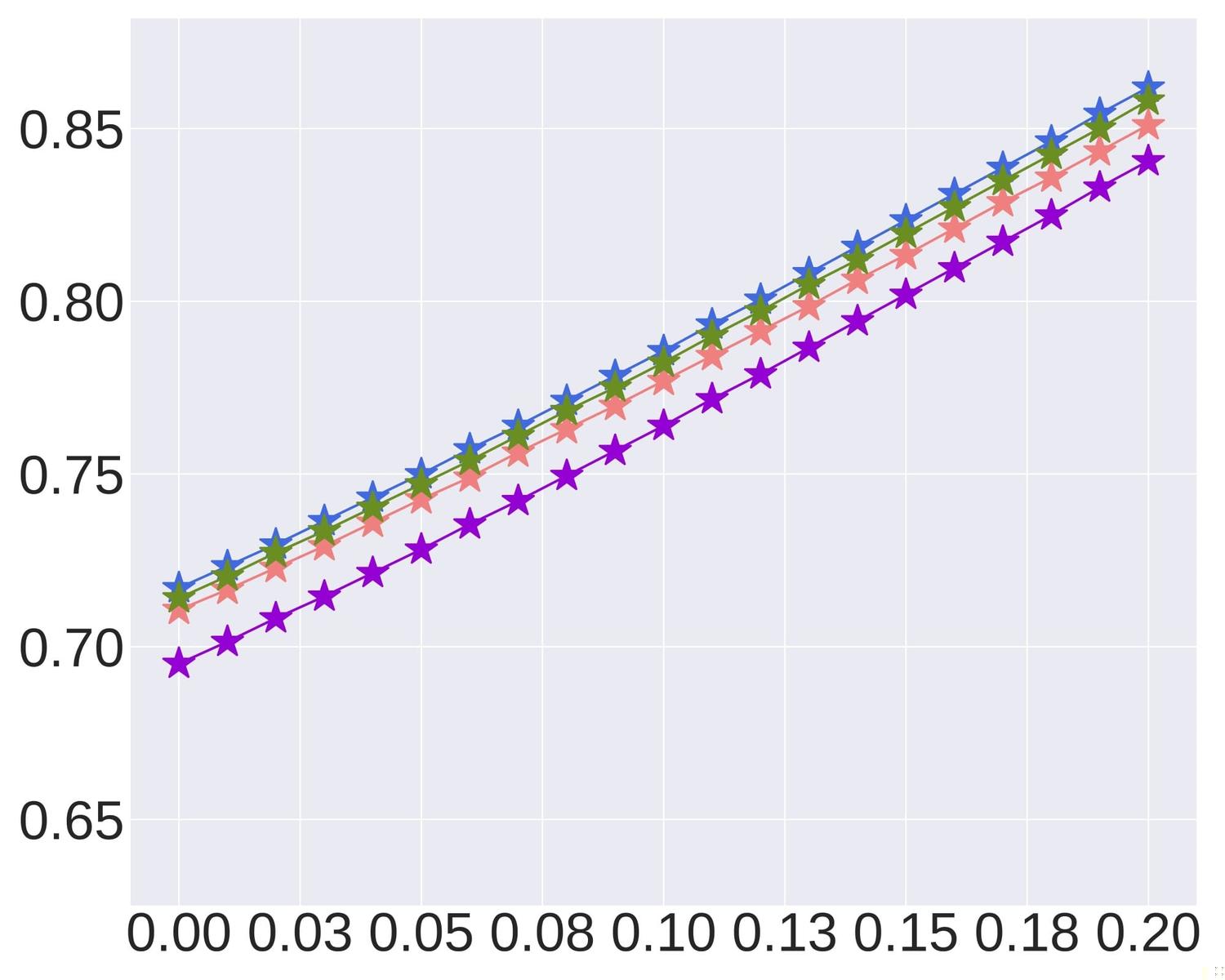}}
  \vspace{-0.5em}
 	\centerline{\footnotesize{~~~Hellinger Distance $\rho$}}
 \vspace{-0.5em}
 \end{minipage}
 \begin{minipage}{0.24\linewidth}
    \centerline{\footnotesize{\quad DART}}
 	\vspace{1pt}
 	\centerline{\includegraphics[width=1.0\textwidth]{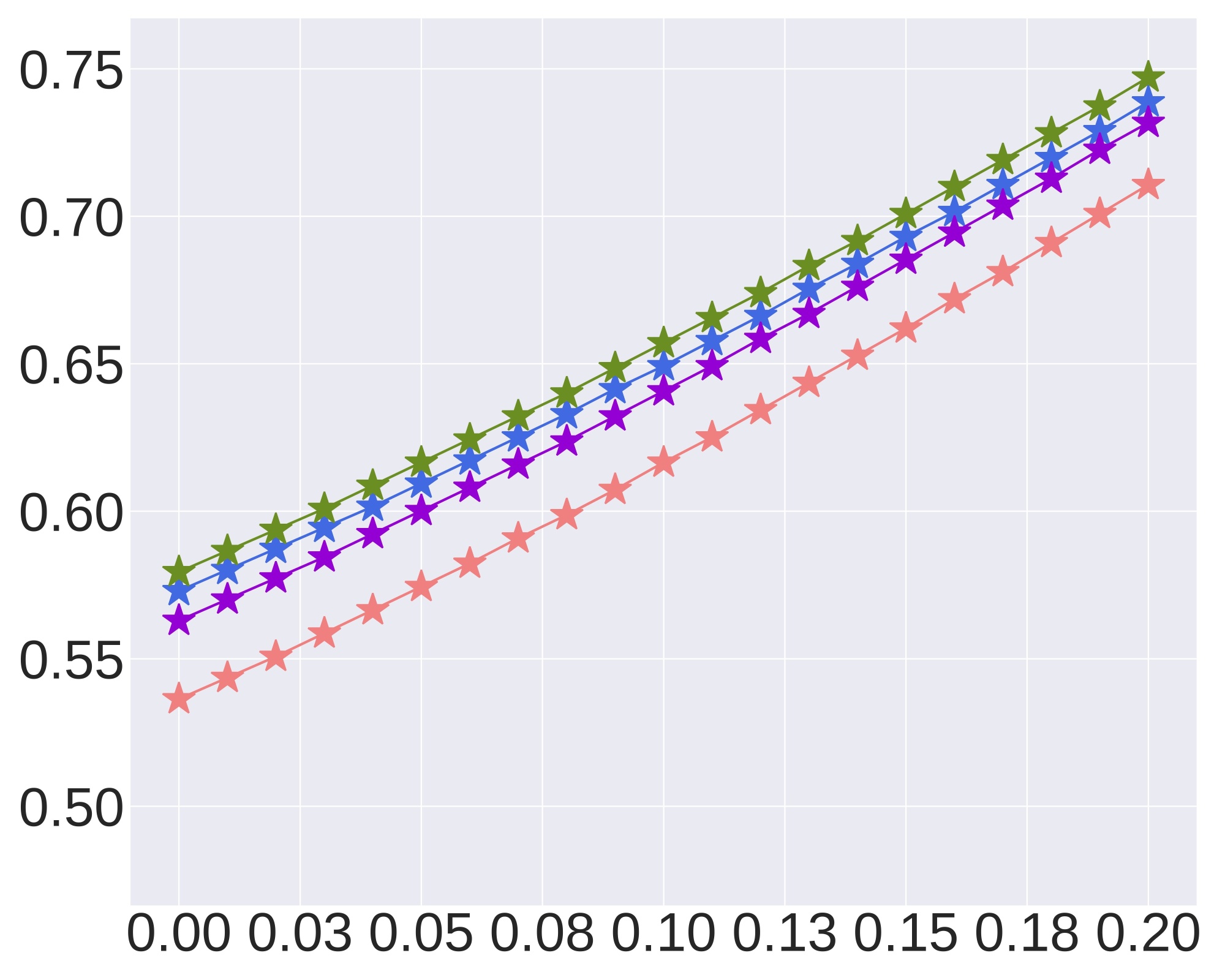}}
  \vspace{-0.5em}
 	\centerline{\footnotesize{~~~Hellinger Distance $\rho$}}
 \vspace{-0.5em}
 \end{minipage}
 \begin{minipage}{0.24\linewidth}
    \centerline{\footnotesize{\quad E2E}}
 	\vspace{1pt}
 	\centerline{\includegraphics[width=1.0\textwidth]{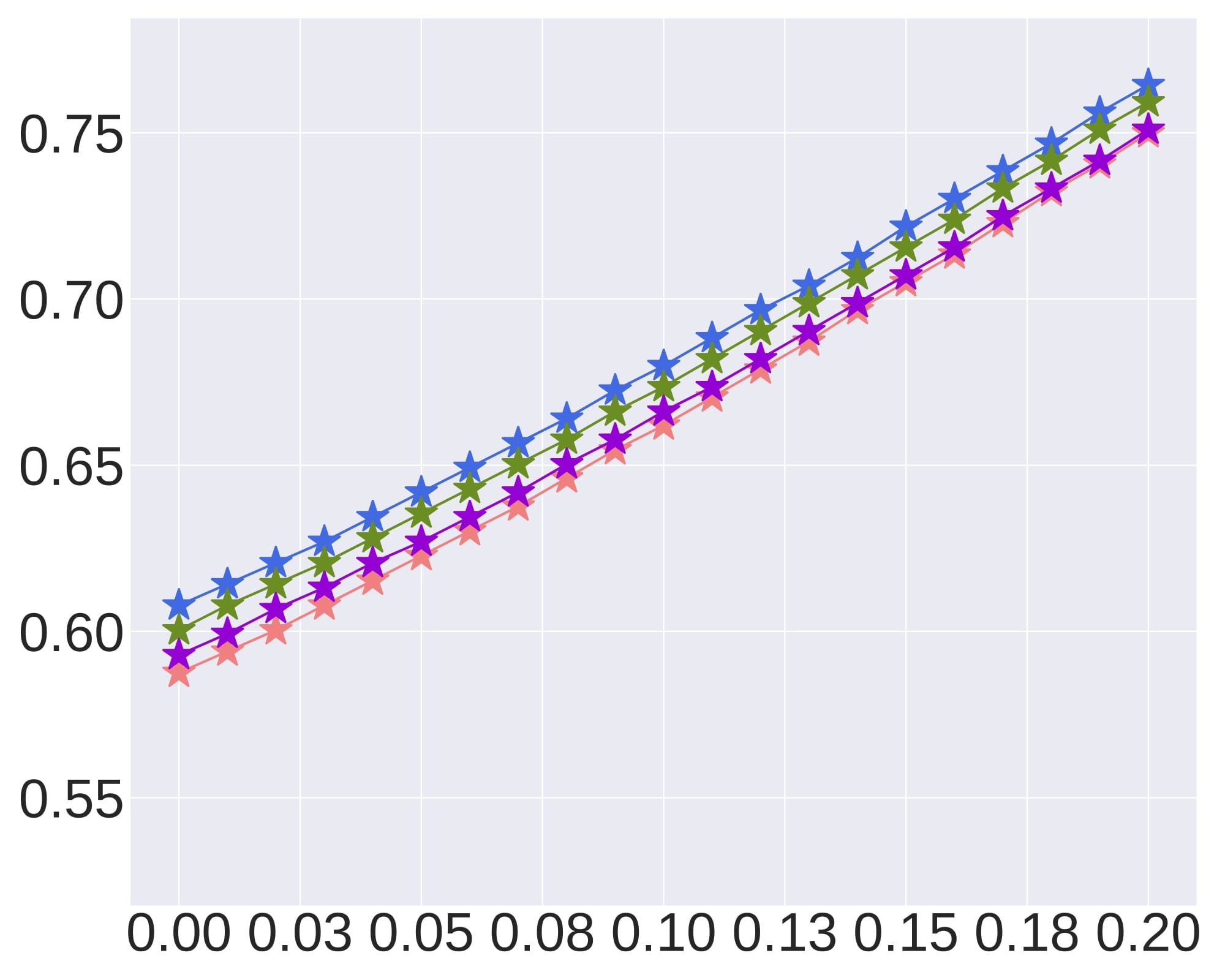}}
  \vspace{-0.5em}
 	\centerline{\footnotesize{~~~Hellinger Distance $\rho$}}
 \vspace{-0.5em}
 \end{minipage}
}
\subfigure{
\centerline{\includegraphics[width=0.53\textwidth]{figures/legend2.pdf}}}
\caption{Conformal generation risk $\alpha_{\text{rag}}$ vs. Hellinger distance $\rho$ for different retrieval models under distribution shifts with $N_{\text{rag}}=15, \lambda_g=1, \lambda_s=1.0$.}
\label{fig:comparison_retrieval_dist_shft}
\end{figure*}

\begin{figure*}[t]
\subfigure{
 \begin{minipage}{0.49\linewidth}
    \centerline{\footnotesize{\quad DART}}
 	\vspace{-0.8pt}
 \centerline{\includegraphics[width=1.0\textwidth]{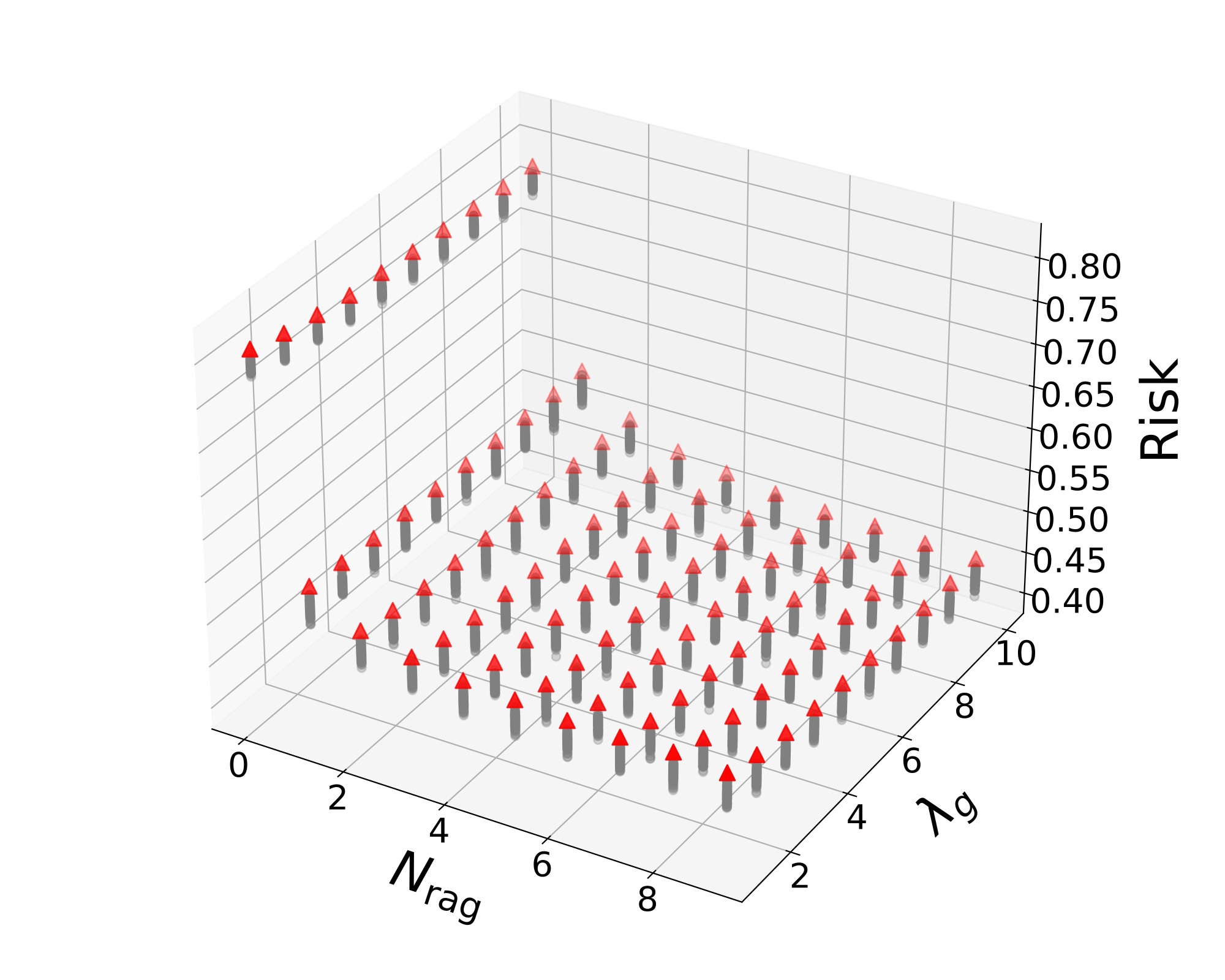}}
  \vspace{-1em}
 \end{minipage}
 \begin{minipage}{0.49\linewidth}
    \centerline{\footnotesize{\quad E2E}}
 	\vspace{-0.8pt}
 	\centerline{\includegraphics[width=1.0\textwidth]{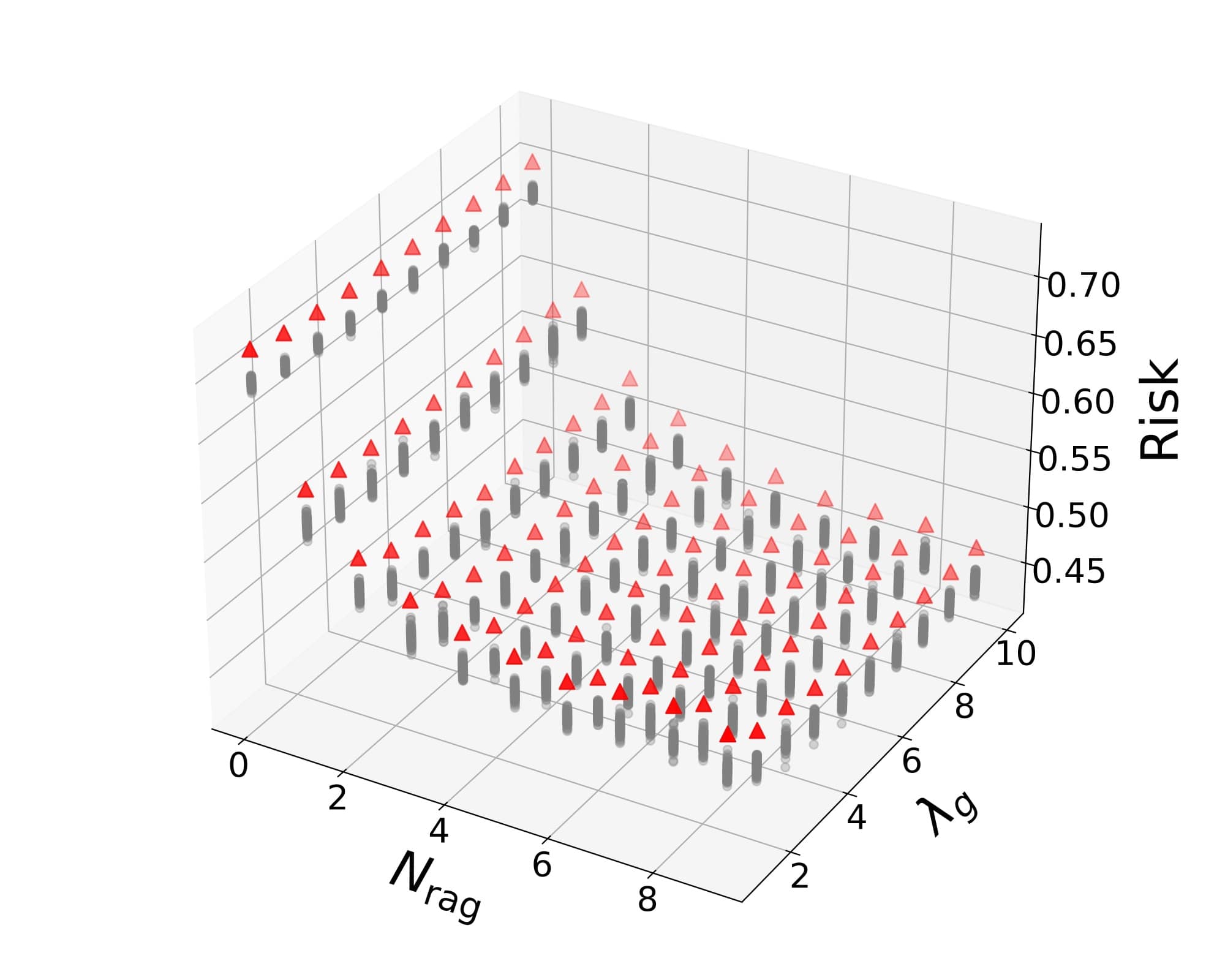}}
  \vspace{-1em}
 \end{minipage}
}
\subfigure{
\centerline{\includegraphics[width=0.5\textwidth]{figures/legend4.pdf}}}
\caption{Conformal generation risk $\hat{\alpha}_{\text{rag}}$ and simulated empirical risks with different $\lambda_g$ and $N_{\text{rag}}$ for OpenAI/ada.}
\label{fig:bound_and_simulation_openai_multi_gen_2_dart}
\vspace{-1em}
\end{figure*}

\begin{figure*}[t]
\subfigure{
    \rotatebox{90}{\hspace{-3.5em} Evaluation Risk}
    \begin{minipage}{0.24\linewidth}
    \centerline{\footnotesize{\quad AESLC}}
 	\vspace{1pt}
\centerline{\includegraphics[width=1.0\textwidth]{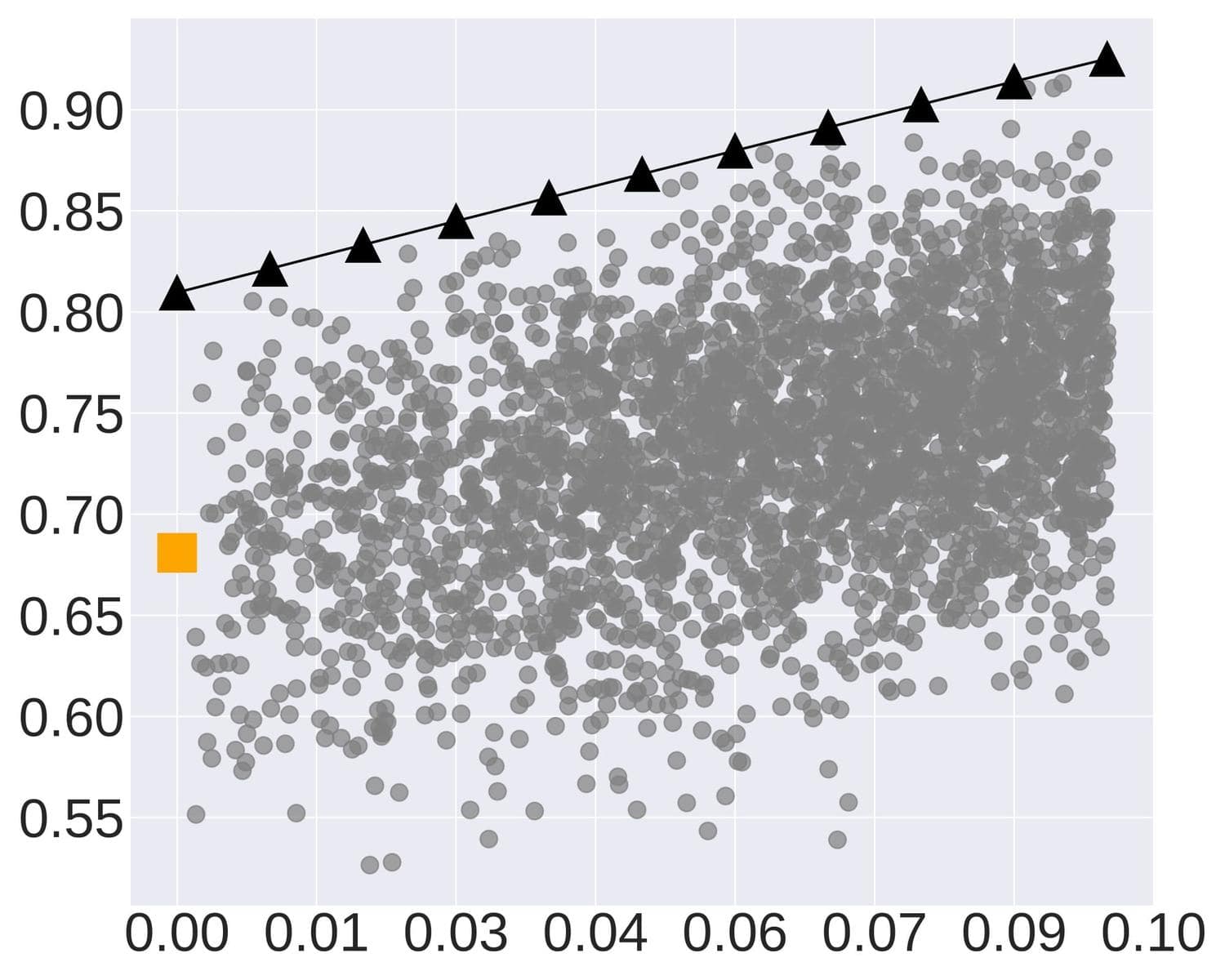}}
  \vspace{-0.5em}
 	\centerline{\footnotesize{~~~Hellinger Distance}}
 \vspace{-0.5em}
 \end{minipage}
 \begin{minipage}{0.24\linewidth}
    \centerline{\footnotesize{\quad CommonGen}}
 	\vspace{1pt}
 	\centerline{\includegraphics[width=1.0\textwidth]{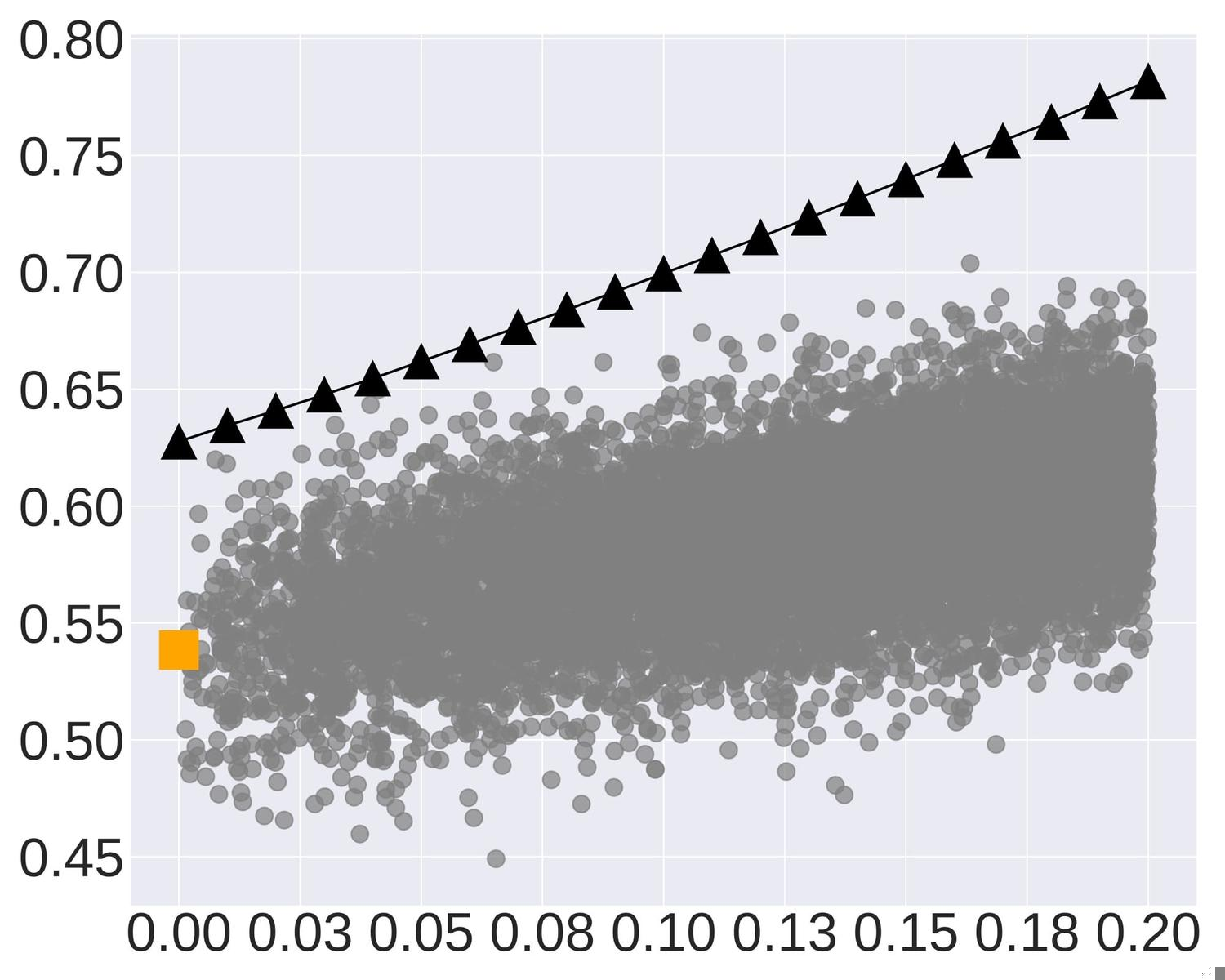}}
  \vspace{-0.5em}
 	\centerline{\footnotesize{~~~Hellinger Distance}}
 \vspace{-0.5em}
 \end{minipage}
 \begin{minipage}{0.24\linewidth}
    \centerline{\footnotesize{\quad DART}}
 	\vspace{1pt}
 	\centerline{\includegraphics[width=1.0\textwidth]{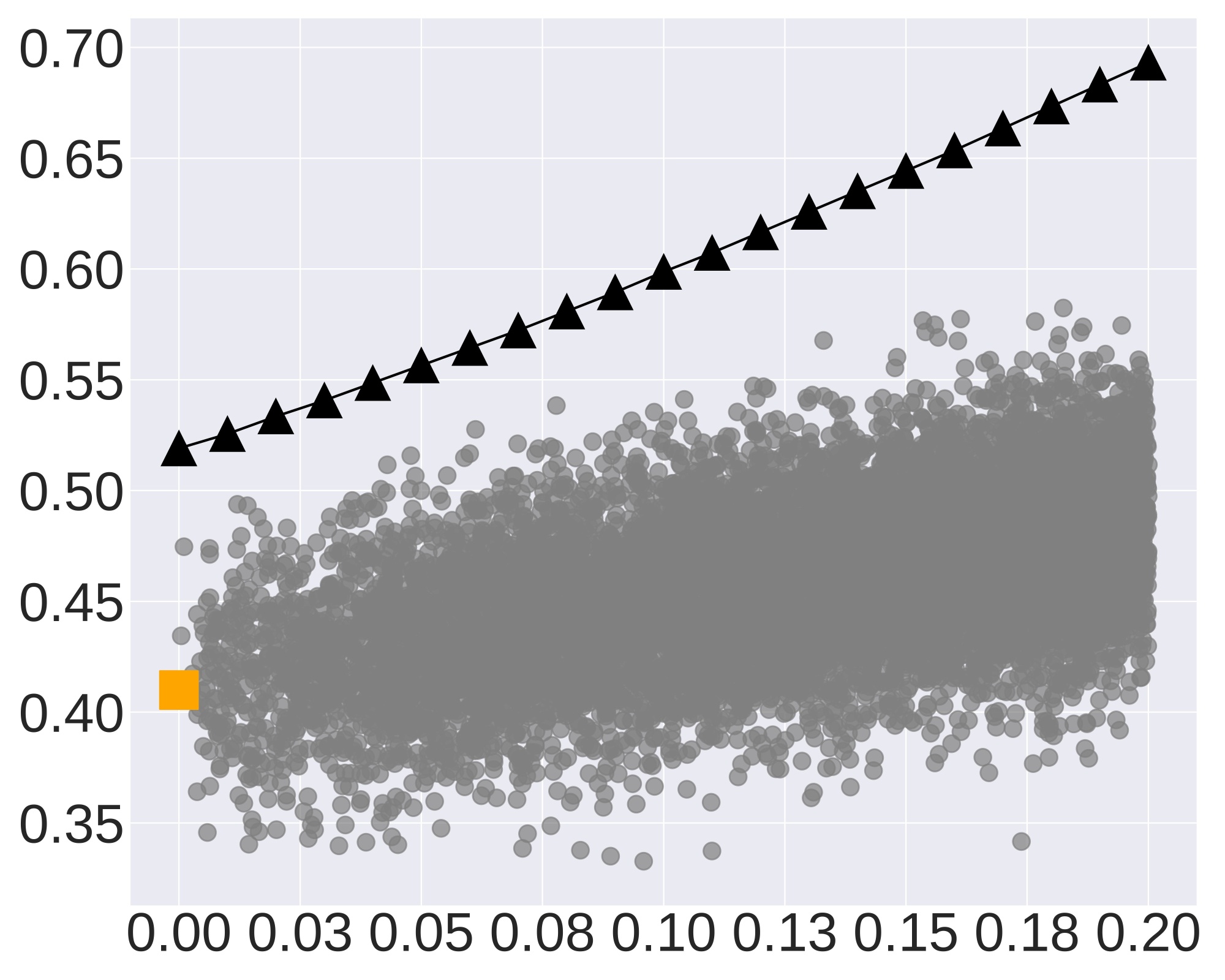}}
  \vspace{-0.5em}
 	\centerline{\footnotesize{~~~Hellinger Distance}}
 \vspace{-0.5em}
 \end{minipage}
 \begin{minipage}{0.24\linewidth}
    \centerline{\footnotesize{\quad E2E}}
 	\vspace{1pt}
 	\centerline{\includegraphics[width=1.0\textwidth]{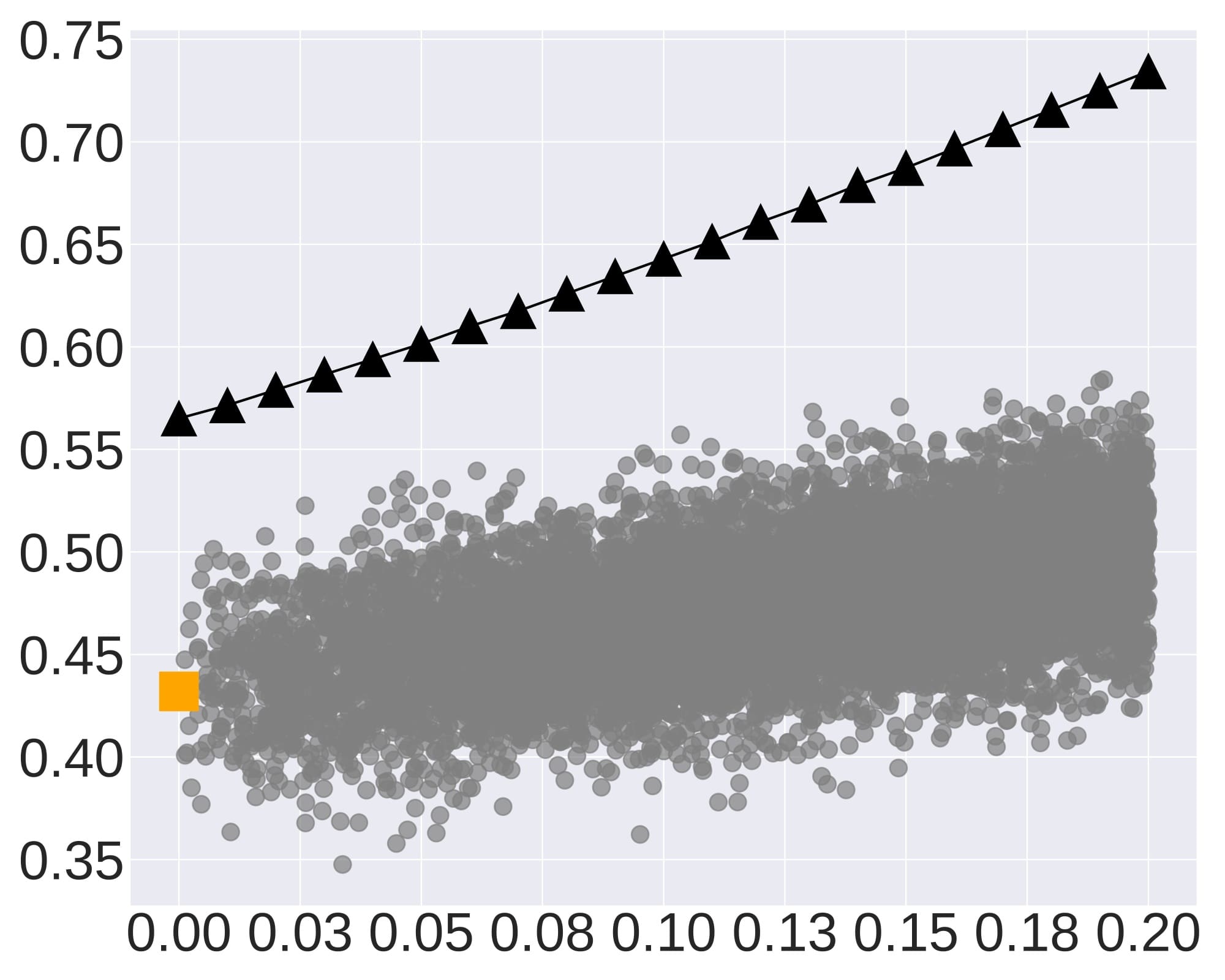}}
  \vspace{-0.5em}
 	\centerline{\footnotesize{~~~Hellinger Distance}}
 \vspace{-0.5em}
 \end{minipage}
}
\subfigure{
\centerline{\includegraphics[width=0.75\textwidth]{figures/legend3.pdf}}}
\caption{Conformal generation risk $\hat{\alpha}_{\text{rag}}$ and empirical risks with $N_{\text{rag}}=5, \lambda_g=20, \lambda_s=1.0$ under distribution shifts for Biencoder-SFT.}
\label{fig:bound_and_simulation_llmr_multi_gen_dist_shft}
\end{figure*}

\begin{figure*}[t]
\subfigure{
    \rotatebox{90}{\hspace{-3.5em} Evaluation Risk}
    \begin{minipage}{0.24\linewidth}
    \centerline{\footnotesize{\quad AESLC}}
 	\vspace{1pt}
\centerline{\includegraphics[width=1.0\textwidth]{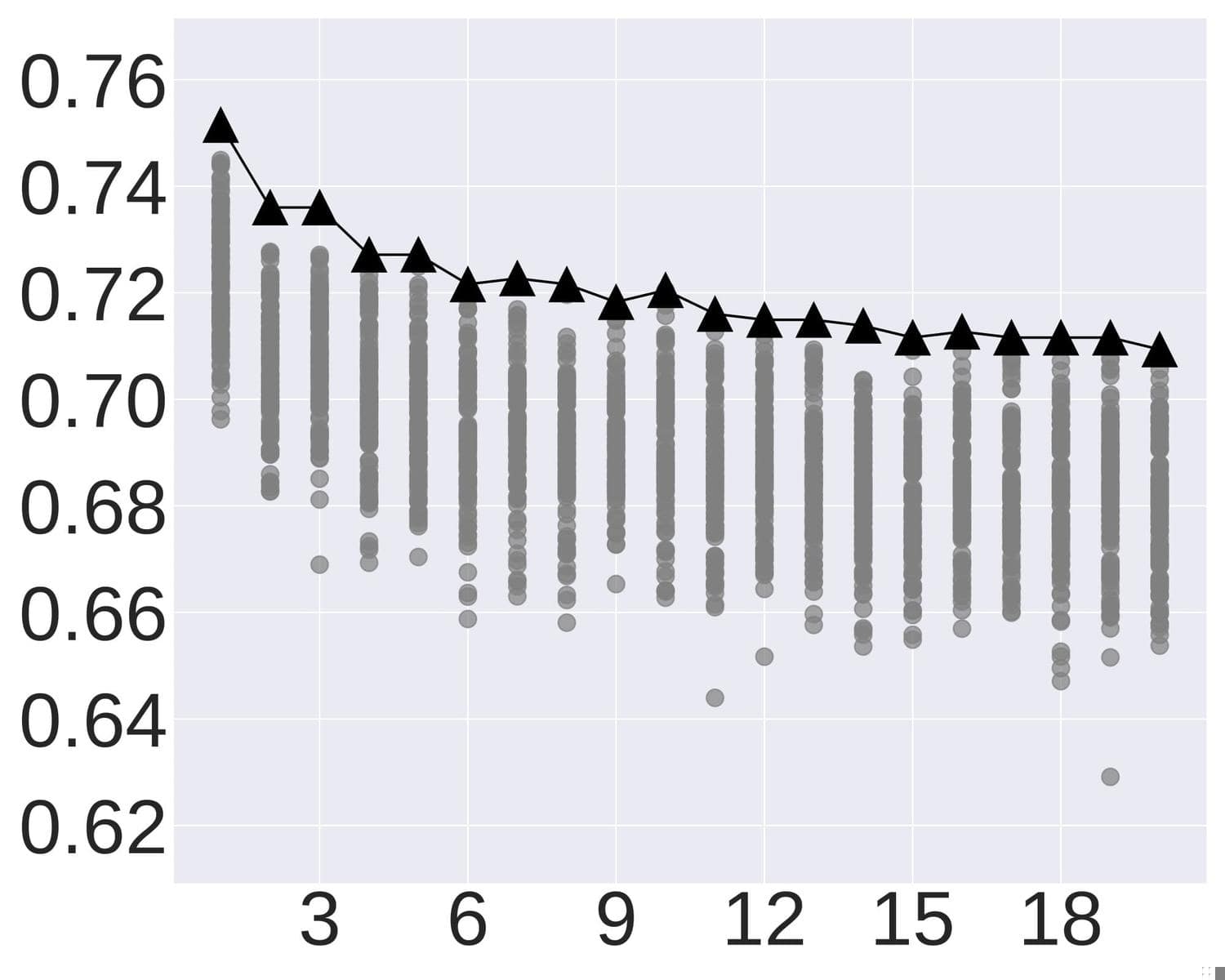}}
  \vspace{-0.5em}
 	\centerline{\footnotesize{~~~\# Generations $\lambda_g$}}
 \vspace{-0.5em}
 \end{minipage}
 \begin{minipage}{0.24\linewidth}
    \centerline{\footnotesize{\quad CommonGen}}
 	\vspace{1pt}
 	\centerline{\includegraphics[width=1.0\textwidth]{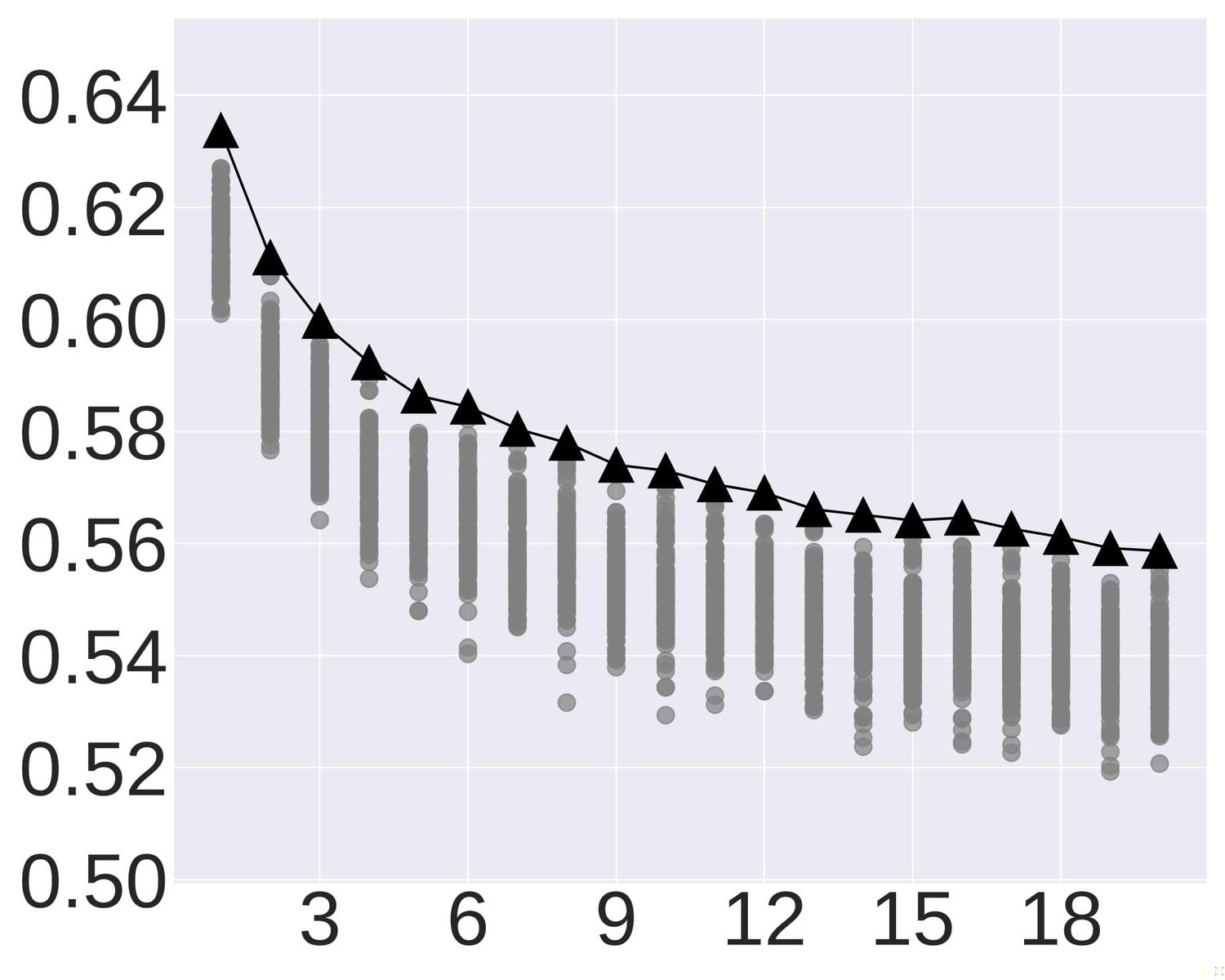}}
  \vspace{-0.5em}
 	\centerline{\footnotesize{~~~\# Generations $\lambda_g$}}
 \vspace{-0.5em}
 \end{minipage}
 \begin{minipage}{0.24\linewidth}
    \centerline{\footnotesize{\quad DART}}
 	\vspace{1pt}
 	\centerline{\includegraphics[width=1.0\textwidth]{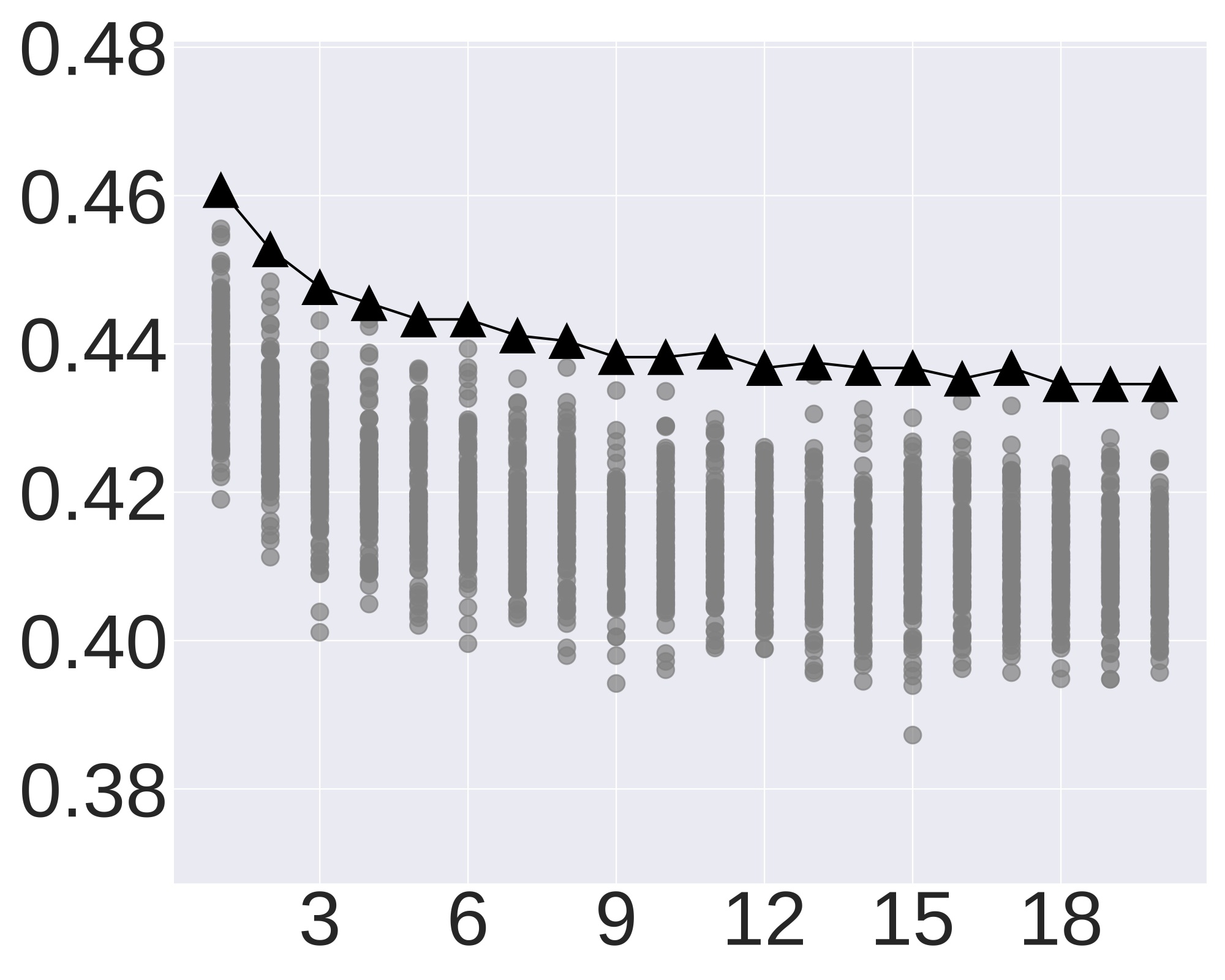}}
  \vspace{-0.5em}
 	\centerline{\footnotesize{~~~\# Generations $\lambda_g$}}
 \vspace{-0.5em}
 \end{minipage}
 \begin{minipage}{0.24\linewidth}
    \centerline{\footnotesize{\quad E2E}}
 	\vspace{1pt}
 	\centerline{\includegraphics[width=1.0\textwidth]{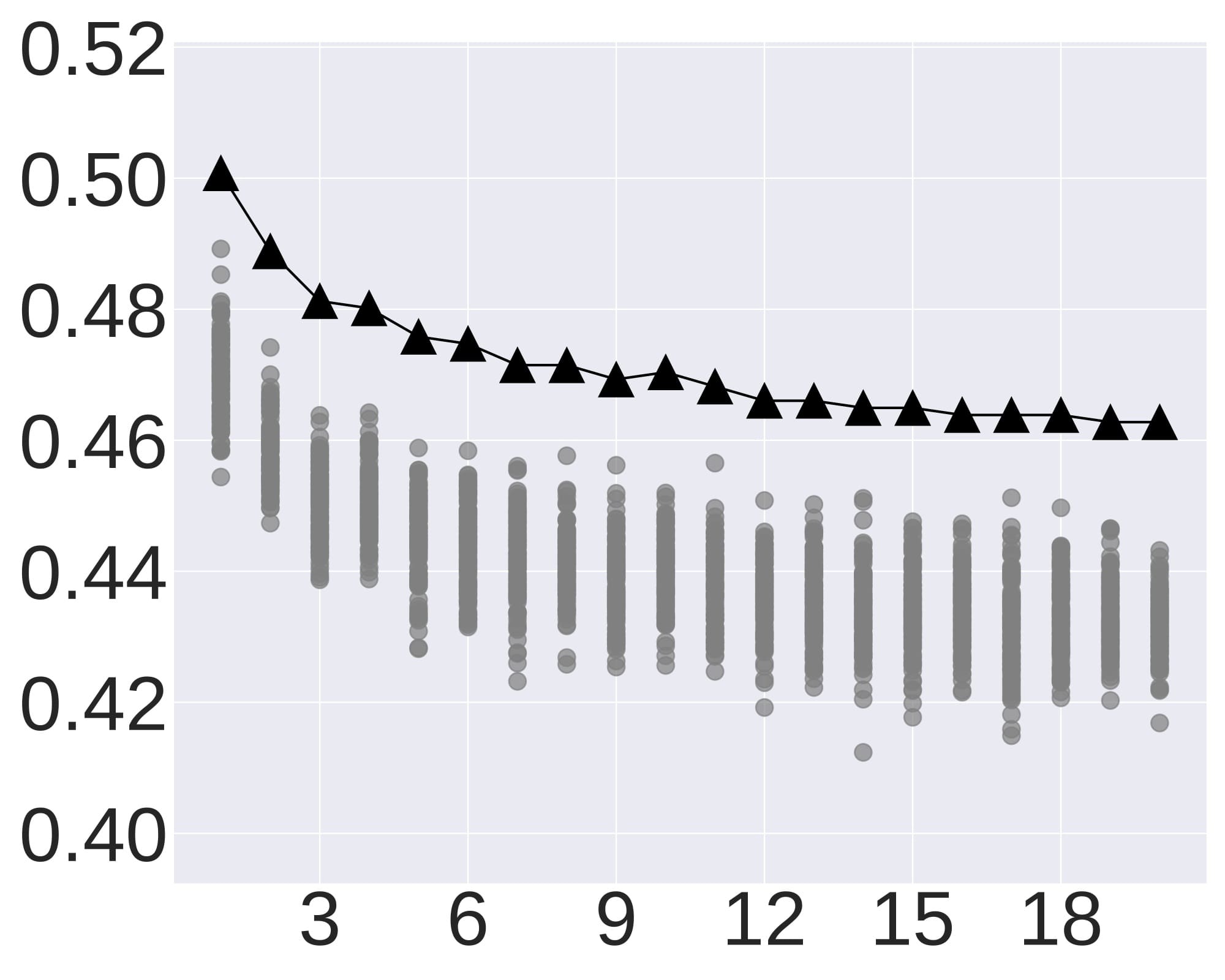}}
  \vspace{-0.5em}
 	\centerline{\footnotesize{~~~\# Generations $\lambda_g$}}
 \vspace{-0.5em}
 \end{minipage}
}
\subfigure{
\centerline{\includegraphics[width=0.4\textwidth]{figures/legend.pdf}}}
\caption{Conformal generation risk $\hat{\alpha}_{\text{rag}}$ and simulated empirical risks for different generation set size $\lambda_g$ and fixed $N_{\text{rag}}=5, \lambda_s=1.0$ for Biencoder-SFT.}
\label{fig:bound_and_simulation_llmr_multi_gen}
\end{figure*}

\begin{figure*}[ht]
\subfigure{
    \rotatebox{90}{\hspace{-3.5em} Evaluation Risk}
    \begin{minipage}{0.24\linewidth}
    \centerline{\footnotesize{\quad AESLC}}
 	\vspace{1pt}
\centerline{\includegraphics[width=1.0\textwidth]{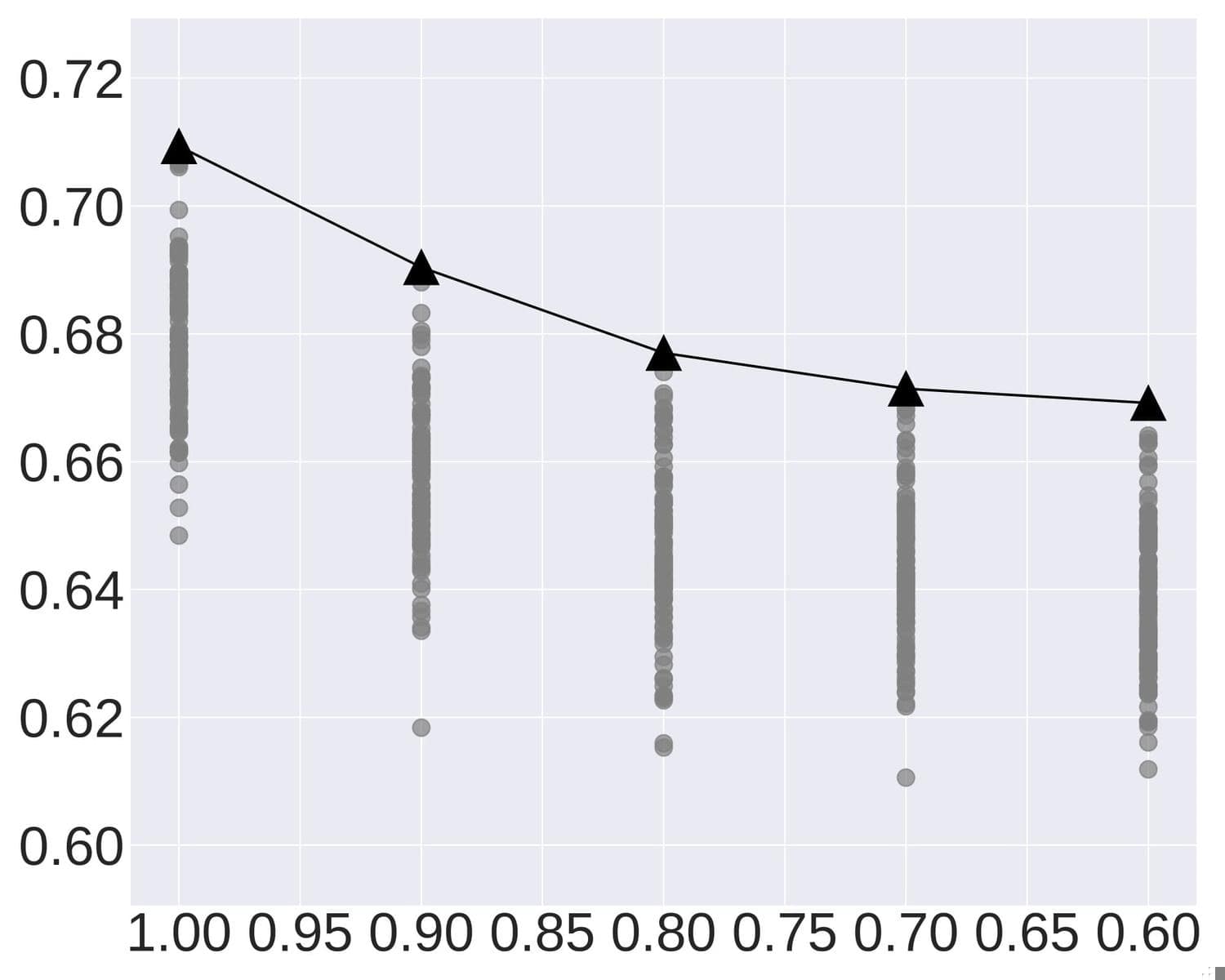}}
  \vspace{-0.5em}
 	\centerline{\footnotesize{~~~diversity threshold $\lambda_s$}}
 \vspace{-0.5em}
 \end{minipage}
 \begin{minipage}{0.24\linewidth}
    \centerline{\footnotesize{\quad CommonGen}}
 	\vspace{1pt}
 	\centerline{\includegraphics[width=1.0\textwidth]{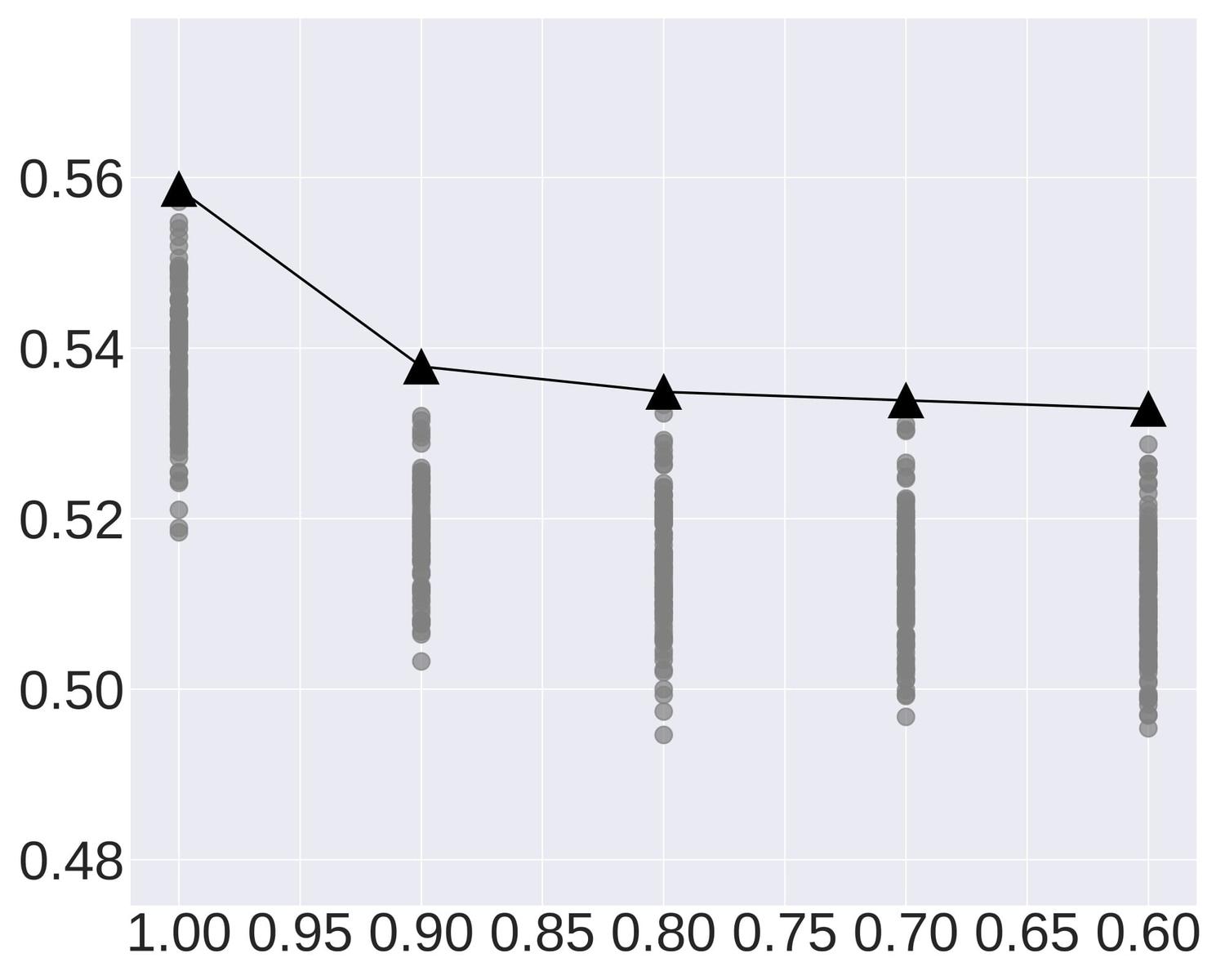}}
  \vspace{-0.5em}
 	\centerline{\footnotesize{~~~diversity threshold $\lambda_s$}}
 \vspace{-0.5em}
 \end{minipage}
 \begin{minipage}{0.24\linewidth}
    \centerline{\footnotesize{\quad DART}}
 	\vspace{1pt}
 	\centerline{\includegraphics[width=1.0\textwidth]{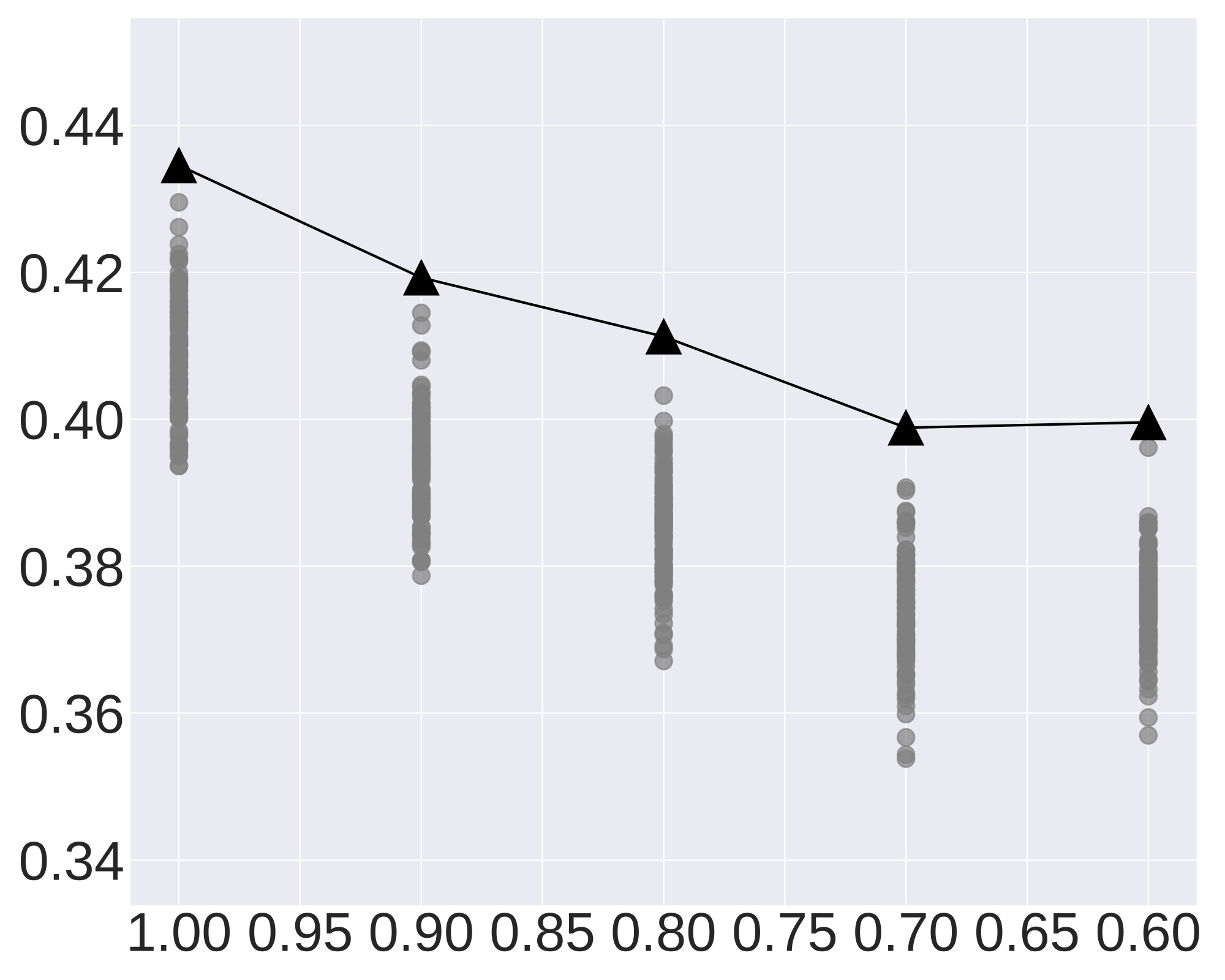}}
  \vspace{-0.5em}
 	\centerline{\footnotesize{~~~diversity threshold $\lambda_s$}}
 \vspace{-0.5em}
 \end{minipage}
 \begin{minipage}{0.24\linewidth}
    \centerline{\footnotesize{\quad E2E}}
 	\vspace{1pt}
 	\centerline{\includegraphics[width=1.0\textwidth]{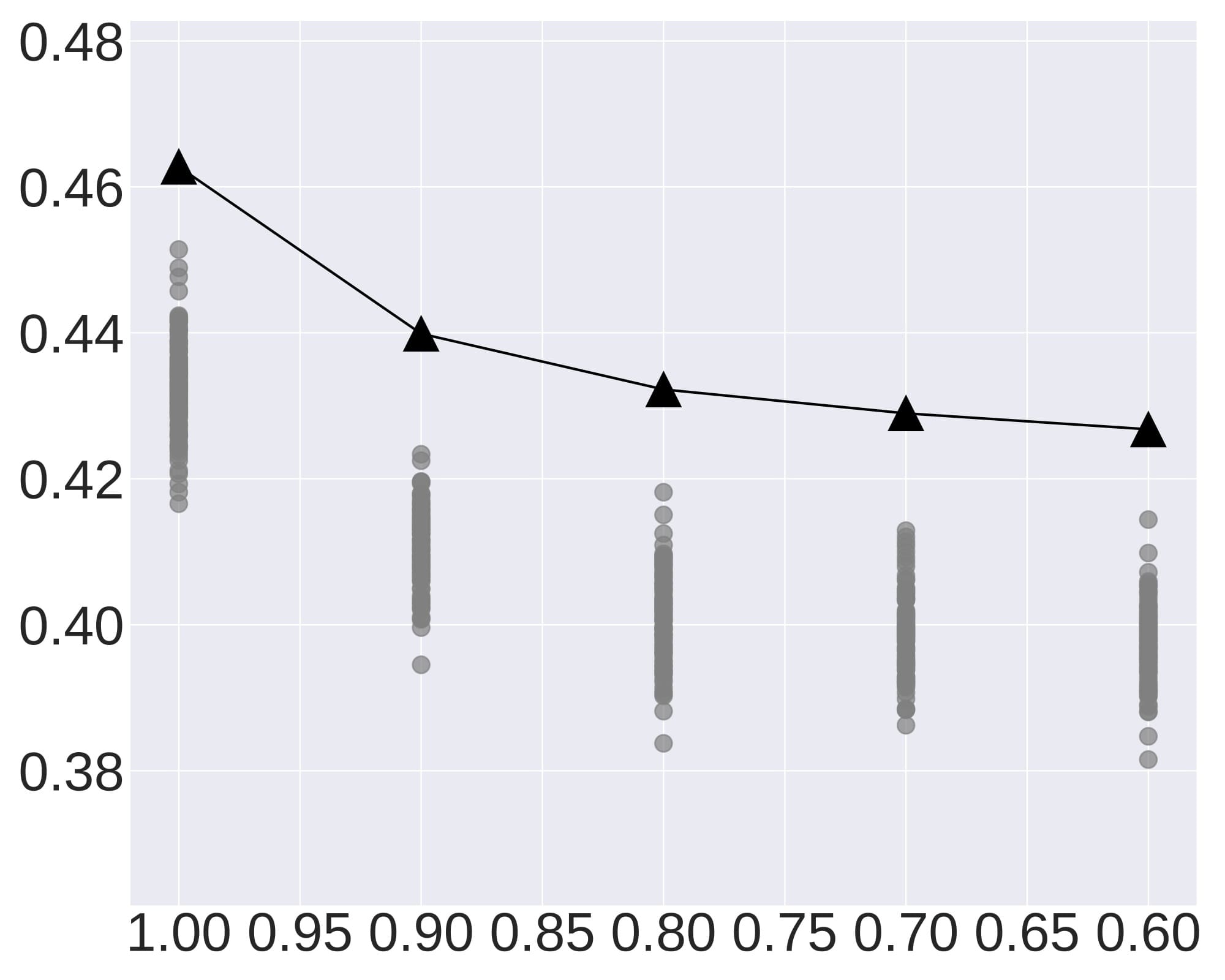}}
  \vspace{-0.5em}
 	\centerline{\footnotesize{~~~diversity threshold $\lambda_s$}}
 \vspace{-0.5em}
 \end{minipage}
}
\subfigure{
\centerline{\includegraphics[width=0.4\textwidth]{figures/legend.pdf}}}
\caption{Conformal generation risk $\hat{\alpha}_{\text{rag}}$ and simulated empirical risks for different diversity threshold $\lambda_s$ with fixed $N_{\text{rag}}=5, \lambda_g=20$ for Biencoder-SFT.}
\label{fig:bound_and_simulation_llmr_multi_gen_lambda_s}
\vspace{-1em}
\end{figure*}

\subsection{Experiment setup}
\label{app:setup}

\paragraph{Datasets} We evaluate \name on four NLP datasets with retrieval augmented generation utilizing an external knowledge base.
(1) AESLC: The Annotated Enron Subject Line Corpus (AESLC) dataset \cite{zhang2019email} contains a collection of email messages from employees of the Enron Corporation. It contains two primary features: the ``email\_body``, which is the text of the email body, and the ``subject\_line``, which is the text of the email subject. The task is to generate the email subject given the email body. 
(2) The Generative Commonsense Reasoning (CommonGen) dataset \cite{lin2019commongen} is a collection of commonsense descriptions, amounting to 79k descriptions over 35k unique concept sets, constructed via crowdsourcing and existing caption corpora. The task is to generate a sentence that uses given concepts in a coherent and commonsense way.
(3) The Data Record to Text (DART) dataset \cite{nan2020dart} is a large-scale dataset designed to facilitate the generation of text from structured data records. It comprises 82,191 examples spanning various domains, with each example being a set of Resource Description Framework (RDF) triples derived from data records in tables and tree ontologies of schemas. These are annotated with sentence descriptions that encapsulate all the facts presented in the RDF triplet.
(4) The End-to-End Generation (E2E) dataset \cite{novikova2017e2e} contains about 50k comments in the restaurant domain. The task is to generate text from meaning representations (MRs), which are structured inputs that describe various aspects of a restaurant. These MRs consist of slots and values that a model needs to convert into natural language descriptions that are coherent and fluent.

\textbf{External knowledge base}: Following \cite{wang2023learning,wei2021finetuned,cheng2023uprise}, we construct the external knowledge base as the union of a total of 30 publicly available datasets from 9 distinct categories with over 6 million documents.

\paragraph{Metrics} For generation tasks, we leverage ROUGE-L to quantify the quality of generations. ROUGE-L measures the longest common subsequence between a candidate generation and reference texts and is typically adopted for generation quality evaluations across the literature \cite{lin2004rouge,gatt2018survey}. A low ROUGE-L implies poor quality generations and accordingly a high generation risk. Therefore, we adopt the risk function as $1-\text{ROUGE-L}$ to bound the risk in $[0,1]$. Note that \name is agnostic to selection of risk functions, and thus, practitioners can specify any function that suits their specific use cases.

\paragraph{Retrieval models} We consider four types of retrieval models. (1) BM25 \cite{robertson2009probabilistic} is a sparse encoding metric used to rank the candidate documents given a query in information retrieval tasks. BM25 scores are linearly weighted combinations of token-level scores and can be analytically formulated as a function of term frequency, inverse document frequency, and length of documents. (2) BAAI general embedding (BAAI/bge) \cite{zhang2023retrieve} trains the SOTA embedding model in the MTEB benchmark \cite{muennighoff2022mteb} and supports the diverse retrieval augmentation needs of LLMs by a reward formulation based on LLMs’ feedback, the stabilization of knowledge distillation, multi-task fine-tuning with explicit instructions, and homogeneous in-batch negative sampling. (3) OpenAI ada-text-embedding-02 (OpenAI/ada) is a close source text embedding models designed to convert text into high-dimensional vectors, which capture the semantic meaning of the input text and can be used for a variety of tasks, such as text similarity, classification, and clustering. (4) Biencoder-supervised fine-tuning (Biencoder-SFT) \cite{wang2023learning} is a bi-encoder based dense retriever trained with contrastive loss and hard negative sampling strategies. It iteratively trains the retrieval model with hard negative samples identified by computing similarity scores by the current retrieval model.

\paragraph{Implementation details} We leverage Llama-7b for inference without specification. We perform conformal calibration on validation sets for different NLP datasets and fix the confidence levels $1-\delta=0.9$ across {all} evaluations. Lastly, we concatenate the retrieved in-context examples and the query sample with line break tokens.


\subsection{Conformal risk bounds evaluation}
\label{app:exp_no_dist}

\noindent\textbf{Soundness and tightness of generation risk guarantees in \name}. To achieve the generation risk guarantee in \Cref{eq:conformal_guarantee}, \name computes the upper confidence bound of generation risk by \Cref{eq:conformal_risk_}, which takes the empirical risk, calibration size and confidence level as input. We evaluate the conformal risks of RAG models $\alpha_{\text{rag}}$ with variations of the number of retrieved examples $N_{\text{rag}}$ by calibration statistics on the validation set. 
To validate the soundness and tightness of the risk bounds, we randomly sample multiple test sets from a pool of test samples and compute the empirical risks on the sampled test sets. The sampling protocol is detailed in \Cref{alg:simulation_no_shft} in \Cref{app:exp_no_dist}.
We provide the results for  BM25, BAAI/bge, Biencoder-SFT in \Cref{fig:bound_and_simulation_bm25,fig:bound_and_simulation_baai,fig:bound_and_simulation_llmr}.
The results show that \textbf{\underline{(1)}} the certified conformal risks $\alpha_{\text{rag}}$ (black up-pointing triangles) are larger than the empirical risks of sampled test sets (grey points) and some empirical risks approach the risk bounds, demonstrating the soundness and tightness of risk bounds in \name, and \textbf{\underline{(2)}} the conformal risks decrease much as the number of retrieved examples $N_{\text{rag}}$ increases, which shows the effectiveness of retrieved in-context examples in RAG model and aligns with our theoretical analysis in \Cref{thm:gene_rag}. 

\noindent\textbf{Comparisons of \name risk bounds for SOTA retrieval models}. We also compare the conformal risk bounds of \name for different retrieval models, including sparse encoding method BM25, and SOTA dense encoding models BAAI/bge, OpenAI/ada, Biencoder-SFT. The results in \Cref{fig:comparison_retrieval} show that \textbf{\underline{(1)}} RAG benefits in achieving a lower conformal risks of generations for different retrieval models, and \textbf{\underline{(2)}} Biencoder-SFT is the most performant generally since the model is trained in the same domain as the test sets, while OpenAI/ada trained with an open corpus also demonstrates impressive effectiveness and even outperforms Biencoder-SFT on CommonGen dataset.

\subsection{Conformal risk bounds evaluations under distribution shifts}
\label{app:exp_dist}

\noindent\textbf{Soundness and tightness of distribution-shifted conformal risk bounds in \name}.
The test user input text can be out of the calibration distribution in practice, which needs correction of the conformal risk bounds considering the distribution drift. We provide the first distribution-shfted conformal risk bounds for general bounded risk functions in \Cref{pro:conf_shf}. We evaluate the bounds in \Cref{pro:eq1} as a function of the distribution shifting distance measured by Hellinger distance $\rho$. To validate the bounds, we also construct various test sets with covariate shift induced by sample weight shifting. Specifically, different weights can be assigned to test samples in the shifted test sets and the Hellinger distance can also be explicitly computed by the original sample weights and shifted sample weights. We provide the detailed procedure in \Cref{alg:simulation} in \Cref{app:exp_dist}.
We provide the results of conformal risk bounds and simulated empirical risks with $N_{\text{rag}}=15$ in \Cref{fig:bound_and_simulation_bm25_dist_shft,fig:bound_and_simulation_baai_dist_shft,fig:bound_and_simulation_llmr_dist_shft} with BM25, BAAI/bge, and Biencoder-SFT.
The results demonstrate that \textbf{\underline{(1)}} the distribution-shifted conformal risk bounds in \Cref{pro:conf_shf} is sound and tight for different retrieval models, and \textbf{\underline{(2)}} the conformal risk bounds increases linearly with the Hellinger distance and are non-trivial for a Hellinger distance up to $0.2$.

\noindent\textbf{Comparisons of \name distribution-shifted risk bounds for SOTA retrieval models}.
We compare the distribution-drift conformal risk bounds for different retrieval models in \Cref{fig:comparison_retrieval_dist_shft}. The results show that \textbf{\underline{(1)}} the bounds for different retrieval models increases linearly with a same slope as the Hellinger distance $\rho$ increases, and \textbf{\underline{(2)}} Biencoder-SFT and OpenAI/ada demonstrate lower conformal risk bounds for different distances due to a lower initial risk without distribution shifts.

\subsection{Conformal risk bounds with multi-dimensional RAG configurations}
\label{app:exp_multi}

We already theoretically and empirically demonstrate the effectiveness of retrieved in-context examples quantified by $N_{\text{rag}}$. To further improve the conformal risk bounds, we can addtionally consider more RAG parameters such as the number of generations $\lambda_g$ in the generation set and a similarity threshold $\lambda_s$ to control the diversity of generations as \Cref{alg:gen_pro}. We follow the RAG generation protocol in \Cref{alg:gen_pro} and define the risk function as the minimal risks among all candidate generations. We similarly construct random test sets and provide the results on DART and E2E in \Cref{fig:bound_and_simulation_openai_multi_gen_2}. The results show that \textbf{\underline{(1)}} the conformal risk bounds for multi-dimensional RAG configurations are still sound and tight, and \textbf{\underline{(2)}} a larger $N_{\text{rag}}$ can reduce the conformal risks more sensitively compared to the number of generations $\lambda_g$, demonstrating the effectiveness of more retrieved in-context examples.
We also fix the number of retrieved examples $N_{\text{rag}}=5$ and the diversity threshold $\lambda_s=1.0$, and evaluate the certified conformal risk and empirical risk of randomly sampled test set in \Cref{fig:bound_and_simulation_llmr_multi_gen}, which demonstrates that although not effective as the number of retrieved examples, a larger generation set size also benefits a low generation risk. 
Further more, we fix $N_{\text{rag}}=5, \lambda_g=20, \lambda_s=1.0$ and evaluate the risks for different distribution drift distance in \Cref{fig:bound_and_simulation_llmr_multi_gen_dist_shft}.

\subsection{Valid configurations identification with specified risk levels}
\label{app:exp_multi_2}

In conformal analysis \textbf{\underline{(2)}} in \Cref{sec:crc1}, given a desired risk level $\alpha$, \name can certify a set of valid configurations $\hat{\Lambda}_\alpha$ such that any configuration in the set results in a conformal risk below $\alpha$. We use Bonferroni correction for family-wise error rate control and evaluate empirical risks on randomly sampled test sets for the identified valid configurations. We provide the results on DART and ECE datasets in \Cref{fig:bound_and_simulation_openai_multi_gen_2_dart}.
The results demonstrate that \textbf{\underline{(1)}} the certification is empirically sound since the empirical risks of valid configurations $\hat{\Lambda}_\alpha$ are always below the desired level $\alpha$, and \textbf{\underline{(2)}} a large number of retrieved examples $N_{\text{rag}}$ and a large generation set size $\lambda_g$ are effective in achieving a low generation risk since the valid configuration set includes the region with a large $N_{\text{rag}}$ and $\lambda_g$.
We also individually demonstrate the effectiveness of generation set size $\lambda_g$ and diversity threshold $\lambda_s$ in \Cref{fig:bound_and_simulation_llmr_multi_gen,fig:bound_and_simulation_llmr_multi_gen_lambda_s}.

\subsection{Conformal generation risks with different inference models}
\label{subsec:llm}

\begin{table}[ht]
\centering
\caption{Comparison of conformal generation risk $\hat{\alpha}_{\text{rag}}$ with different $N_{\text{rag}}$ using Llama-2-7b, Mistral-7B-Instruct-v0.2, and Llama-2-13b. The results are evaluated on the AESLC dataset with text-embedding-ada-002 from OpenAI as the retrieval model.}
\label{tab:llm}
\begin{tabular}{lcccccc}
\hline
Model & $N_{\text{rag}}=0$ & $N_{\text{rag}}=1$ & $N_{\text{rag}}=2$ & $N_{\text{rag}}=3$ & $N_{\text{rag}}=4$ & $N_{\text{rag}}=5$ \\
\hline
Llama-2-7b & 0.953 & 0.957 & 0.879 & 0.868 & 0.847 & 0.836 \\
Mistral-7B-Instruct-v0.2 & 0.897 & 0.829 & 0.813 & 0.795 & 0.792 & 0.793 \\
Llama-2-13b & 0.889 & 0.823 & 0.802 & 0.792 & 0.772 & 0.772 \\
\hline
\end{tabular}
\end{table}

We conduct additional evaluations of the conformal generation risk (upper bound of the generation risk) of \name with different types of inference models as well as models with different sizes. The results in \Cref{tab:llm} demonstrate that (1) for different inference models, the conformal generation risk decreases effectively as the number of retrieved in-context examples $N_{\text{rag}}$ increases, and (2) given that Llama-2-13b outperforms Llama-2-7b in the evaluations, models with larger sizes in the same model family achieve a lower conformal generation risk.

\subsection{Qualitative example}
\label{subsec:qualitative_example}
\vspace{-1em}

\begin{table}[h!]
\centering
\caption{A qualitative example.}
\begin{tabular}{|p{4cm}|p{5cm}|p{5cm}|}
\hline
\textbf{Prompt} & \textbf{Generation set} \\ \hline
Vanilla generation & Which team won the 2020 World Cup? & {“The 2020 World Cup was not held in 2020 due to the COVID-19 pandemic.”} \\ \hline
RAG generation protocol ($N_{\text{rag}}=1$) & The World Cup is held every four years. France won the 2018 World Cup in Russia. Argentina won the 2022 World Cup in Katar. \newline Which team won the 2020 World Cup? & {“The 2020 World Cup has not yet been held.”} \\ \hline
RAG generation protocol ($N_{\text{rag}}=1, \lambda_g=3$) & The World Cup is held every four years. France won the 2018 World Cup in Russia. Argentina won the 2022 World Cup in Katar. \newline Which team won the 2020 World Cup? & {“The 2020 World Cup has not yet been held.”, “2020 World Cup is not held yet.”, “2020 World Cup has not held yet.”} \\ \hline
RAG generation protocol ($N_{\text{rag}}=1, \lambda_g=3, \lambda_s=0.5$) & The World Cup is held every four years. France won the 2018 World Cup in Russia. Argentina won the 2022 World Cup in Katar. \newline Which team won the 2020 World Cup? & {“The 2020 World Cup has not yet been held.”, “The 2020 World Cup was not held due to the COVID-19 pandemic”, “2020 was not the year for any World Cup as the tournament”} \\ \hline
\end{tabular}
\label{tab:generation_protocols}
\end{table}

In \Cref{tab:generation_protocols}, we use the following qualitative example to illustrate the effectiveness of our RAG protocol. Consider the input prompt, "Which team won the 2020 World Cup?" This prompt is inherently misleading, as there was no World Cup event scheduled for 2020.
In the vanilla generation process, the model (LLAMA-2-7B-32K-INSTRUCT) produces a misconception stating, “The 2020 World Cup was not held in 2020 due to the COVID-19 pandemic.” By integrating one retrieved result ($N_{rag}=1$) into the prompt: "The World Cup is held every four years." the model no longer falsely attributes the absence of the event to COVID-19. However, it still fails to recognize that 2020 was not a designated year for the World Cup.
Simply increasing the generation set size to $\lambda_g=3$ yields similar results, failing to address the core issue. Nonetheless, by imposing diversity constraints on the generations $\lambda_s=0.5$, the model correctly identifies the crux of the matter: “2020 was not the year for any World Cup as the tournament”.

\end{document}